\documentclass[jair,twoside,11pt,theapa]{article}
\usepackage{jair, theapa, rawfonts}

\ShortHeadings{Optimized Measurements for Sequential Diagnosis}
{
}
\firstpageno{1}

\usepackage{natbib}

\usepackage{amsmath}
\usepackage[all]{xy}
\usepackage{cases}
\usepackage{times}
\usepackage{graphicx}
\usepackage{latexsym,footnote}
\usepackage{algorithm}
\usepackage[noend]{algpseudocode}
\algrenewcommand\algorithmicrequire{\textbf{Input:}}
\algrenewcommand\algorithmicensure{\textbf{Output:}}
\usepackage{mathnotation}
\usepackage{array}
\usepackage{dashrule}
\usepackage[T1]{fontenc} 
\usepackage[draft]{fixme}
\usepackage{subfigure}
\usepackage{multirow}
\usepackage{tabularx}
\usepackage{cuted}
\usepackage{amssymb}
\usepackage{changepage}
\usepackage{enumerate}
\usepackage{amsthm}
\usepackage{color}															 
\usepackage[table]{xcolor}
\usepackage{booktabs}
\usepackage{arydshln}
\usepackage{cancel}
\usepackage{rotating}
\usepackage{framed}
\usepackage{float}
\usepackage{ marvosym }
\usepackage[a4paper]{geometry}
\usepackage{changepage}
\usepackage{url}
\usepackage{amsfonts}
\usepackage{wasysym}
\usepackage{environ}
\usepackage{tikz}

\usepackage{xspace}

\usepackage{enumitem}

\usepackage{xcolor}
\usepackage{threeparttable}

\newcounter{examplecounter}
\newenvironment{example}{
	\refstepcounter{examplecounter}%
	
	\vspace{7pt}
	\noindent\textbf{Example \arabic{examplecounter}}%
	\quad 
}{
	
	\vspace{7pt}
}

\newcounter{remarkcounter}
\newenvironment{remark}{
	\refstepcounter{remarkcounter}%
	
	\vspace{7pt}
	\noindent\textbf{Remark \arabic{remarkcounter}}%
	\quad
}{
	
	\vspace{7pt}
}

\definecolor{light-gray1}{gray}{0.95}

\newcommand{\tax}{\mathit{\alpha}}

\newcommand{\mo}{\mathcal{K}}

\newcommand{\mb}{\mathcal{B}}
\newcommand{\md}{\mathcal{D}}

\newcommand{\mc}{\mathcal{C}}

\newcommand{\Tp}{\mathit{P}}
\newcommand{\Tn}{\mathit{N}}
\newcommand{\tp}{\mathit{p}}
\newcommand{\tn}{\mathit{n}}

\newcommand{\ot}{\mathcal{K}^{*}}

\newcommand{\mD}{{\bf{D}}}

\newcommand{\minD}{{\bf{minD}}}
\newcommand{\allD}{{\bf{allD}}}
\newcommand{\dx}[1]{{\bf D}_{#1}^+}
\newcommand{\dnx}[1]{{\bf D}_{#1}^{-}}
\newcommand{\dz}[1]{{\bf D}_{#1}^0}

\newcommand{\RQ}{{\mathit{R}}}

\newcommand{\mQ}{{\bf{Q}}}

\newcommand{\Pt}{\mathfrak{P}}

\newcommand{\LD}{{\bf{D}}}
\newcommand{\Disc}{\mathsf{Disc}}

\newcommand{\sd}{\textsc{sd}\xspace}
\newcommand{\comps}{\textsc{comps}\xspace}
\newcommand{\meas}{\textsc{meas}\xspace}
\newcommand{\obs}{\textsc{obs}\xspace}
\newcommand{\dpi}{\mathsf{DPI}\xspace}
\newcommand{\exdpi}{\mathsf{ExK}\xspace}
\newcommand{\exdpiM}{\mathsf{ExM}\xspace}
\newcommand{\exdpiMK}{\mathsf{ExM2K}\xspace}
\newcommand{\mdpi}{\mathsf{mDPI}\xspace}
\newcommand{\kdpi}{\mathsf{kDPI}\xspace}
\newcommand{\ab}{\textsc{ab}}
\newcommand{\qqm}{\mathsf{qqm}}
\newcommand{\oracle}{\mathsf{ans}}

\newcommand{\sdaa}[1]{\textsc{sd}^*[#1]}

\newcommand{\dg}{\ensuremath{\Delta}\xspace}

\newtheorem{definition}{Definition}{}
\newtheorem{theorem}{Theorem}{}
\newtheorem{prob_def}{Problem}{}
\newtheorem{proposition}{Proposition}{}
\newtheorem{lemma}{Lemma}{}
\newtheorem{corollary}{Corollary}{}

\newtheorem{conjecture}{Conjecture}{}

\begin{document}
	
\setlength{\abovedisplayskip}{7pt}
\setlength{\belowdisplayskip}{7pt}
\setlength{\abovedisplayshortskip}{0pt}
\setlength{\belowdisplayshortskip}{0pt}

\title{
A Generally Applicable, Highly Scalable Measurement Computation and Optimization Approach to Sequential Model-Based Diagnosis\thanks{Parts of this work have been accepted for publication at DX'17 -- Workshop on Principles of Diagnosis. First, the reduction of model-based diagnosis problems to knowledge base debugging problems (Sec.~\ref{sec:reduction_of_mbd_to_kbd}) is treated in \citep{rodler17dx_reduction}. Second, the query computation approach dealt with in Sec.~\ref{sec:contribution} is discussed in \citep{rodler17dx_query}. This work extends the previously published ones significantly in several respects. For instance, it comprises a much more detailed treatment of the underlying theory, all proofs, numerous illustrating examples, various additional remarks and a much more comprehensive experimental evaluation.}}

\author{\name Patrick Rodler\thanks{Corresponding author} \email patrick.rodler@aau.at \\
       \name Wolfgang Schmid \email wolfgang.schmid@aau.at \\
       \name Konstantin Schekotihin \email konstantin.schekotihin@aau.at \\
       \addr Alpen-Adria Universit\"at Klagenfurt, Universit\"atsstra{\ss}e 65-67,\\
       9020 Klagenfurt, Austria
   }


\maketitle

\begin{abstract}
\emph{Model-based diagnosis} deals with the identification of the real cause of a system's malfunc-tion based on a formal system model and observations of the system behavior. When a malfunc-tion is detected, there is usually not enough information available to pinpoint the real cause and one needs to discriminate between multiple fault hypotheses (\emph{diagnoses}). 
To this end, \emph{sequential diagnosis} approaches ask an oracle for additional system measurements.  

This work presents strategies for (optimal) measurement selection in model-based sequential diagnosis. 
In particular, assuming a set of leading diagnoses being given, we show how \emph{queries} (sets of measurements) can be computed and optimized along two dimensions: expected number of queries and cost per query. 
By means of a suitable decoupling of two optimizations and a clever search space reduction the computations are done without any inference engine calls. 
For the full search space, we give a method requiring only a polynomial number of inferences and show how query properties can be guaranteed which existing methods do not provide. 
Evaluation results using real-world problems indicate that the new method computes (virtually) optimal queries instantly independently of the size and complexity of the considered diagnosis problems and outperforms equally general methods not exploiting the proposed theory by orders of magnitude.
\end{abstract}

\section{Introduction}
\label{sec:intro}
%
\textbf{Model-based diagnosis (MBD)} is a widely applied approach to finding explanations
for unexpected behavior of observed systems such as hardware \citep{Reiter87,dressler1996consistency}, software \citep{DBLP:conf/ijcai/StumptnerW99,DBLP:conf/aadebug/MateisSWW00,steinbauer2005detecting}, knowledge bases \citep{Parsia2005,Kalyanpur2006a,Shchekotykhin2012,Rodler2015phd}, discrete event systems \citep{darwiche1996exploiting,pencole2005formal}, feature models \citep{DBLP:journals/jss/WhiteBSTDC10} and user interfaces \citep{DBLP:journals/apin/FelfernigFISTJ09}.
%
MBD assumes a formal \emph{system model} and a set of relevant possibly faulty \emph{system components} (e.g.\ lines of code, gates in a circuit). The model includes descriptions of the interrelation between the components (e.g.\ wires between gates), descriptions of the components' nominal behavior (e.g.\ relation between inputs and outputs of a gate) and other relevant knowledge (e.g.\ axioms of Boolean logic). An MBD problem arises if \emph{observations} (e.g.\ sensor readings, system outputs) of the system's behavior differ from predictions based on the system model. In this case, the set of observations is inconsistent with the system model under the assumption that all system components are exhibiting a nominal behavior. The sought solution to an MBD problem is a \emph{diagnosis} pinpointing the faulty components causing the observed system failure. 
Normally,
however, due to initially insufficient observations, this fault localization is ambiguous and multiple possible diagnoses exist.

\noindent\textbf{Sequential Diagnosis} methods \citep{dekleer1987,pietersma2005model,DBLP:journals/jair/FeldmanPG10a,Siddiqi2011,Shchekotykhin2012} address this issue. These collect additional information by generating a sequence of \emph{queries} and assume available some \emph{oracle} providing answers to these queries. Depending on the MBD application domain, queries can be, for instance, measurements (e.g.\ probes in a circuit), system tests (observations about the system's behavior upon new system inputs), questions to a domain expert (e.g.\ to a doctor when debugging a medical knowledge base) or component inspections (e.g.\ checking the battery of a car). Likewise, the instantiation of the oracle might be, for instance, an electrical engineer performing probes using a voltmeter, an IDE running software tests or a car mechanic inspecting components of a vehicle. 
If queries are chosen properly, each query's answer eliminates some diagnoses and thus reduces the diagnostic uncertainty (pruning of the space of possible diagnoses). 	
As query answering is normally costly, the goal of sequential diagnosis is to \emph{minimize the diagnostic cost} in terms of, e.g., time, manpower or equipment required to \emph{achieve a diagnostic goal}, e.g., the extraction of a diagnosis with a probability above some threshold or the isolation of a single remaining diagnosis (which then corresponds to the \emph{actual diagnosis}, i.e.\ the actual cause of the system failure). 

\begin{figure}[t]
	\centering
	\includegraphics[width=0.99\textwidth]{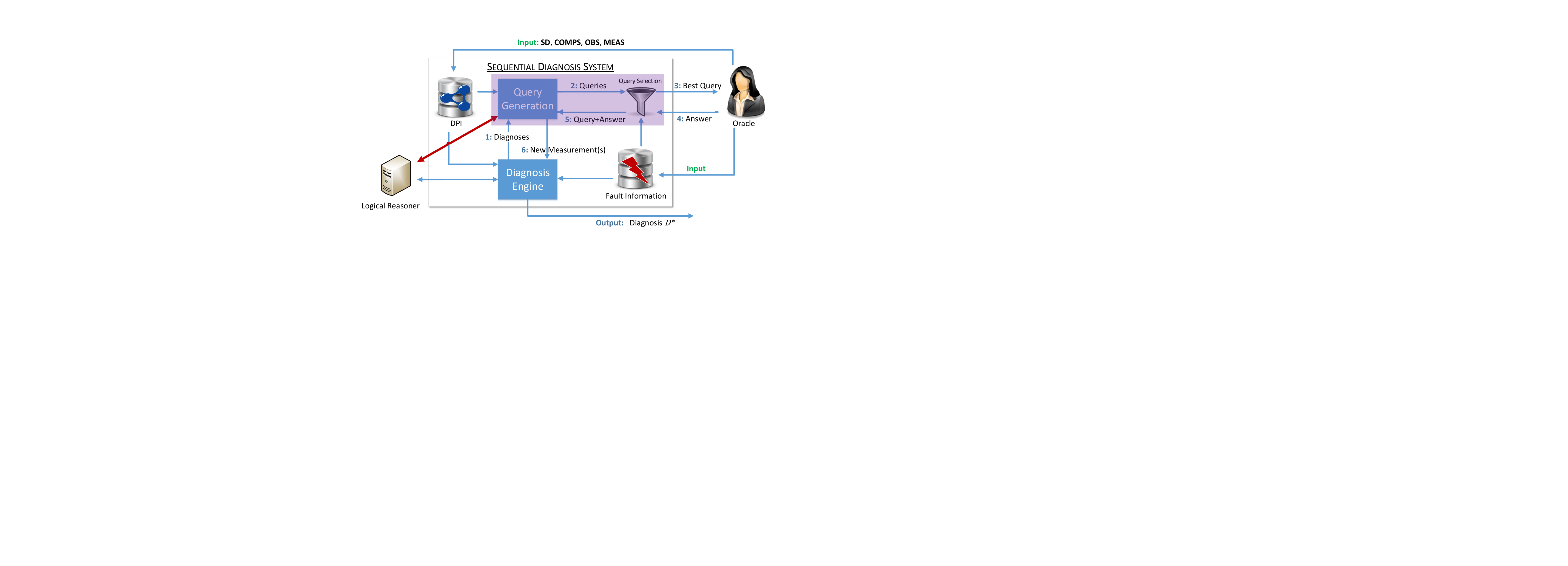}
	\caption{\small Schematic view on a generic sequential diagnosis system. The area shaded in violet shows the part of the system optimized by the approach in this work. The red arrow emphasizes that (expensive) reasoner calls have to be minimized.}
	\label{fig:seq_diag_system}
\end{figure}

\noindent\textbf{A generic sequential diagnosis system} is illustrated by Fig.~\ref{fig:seq_diag_system}.
It gets the inputs $\sd$ (system description), $\comps$ (system components), $\obs$ (initial observations), $\meas$ (additional observations / performed measurements), which altogether make up a \emph{diagnosis problem instance} (DPI), and possibly some fault information (e.g.\ in terms of failure probabilities of system components). The usual workflow (see numbers in Fig.~\ref{fig:seq_diag_system}) followed by such a system involves the
(1) computation of a (feasible) set of diagnoses by a diagnosis engine using the DPI and fault information, (2) computation of a set of query candidates by a query generation module based on the given diagnoses, (3) selection of the best query from the given candidates, (4) answering of this query by the interacting oracle, (5+6) addition of the returned query along with its answer to the DPI in terms of new measurements ($\meas$). The diagnosis engine uses these new measurements to perform various updates (e.g.\ pruning of the diagnoses space, adapting the fault information). If the diagnostic goal is not accomplished, the entire process starts anew from (1). Otherwise, the best diagnosis is output. 
The focus of this work lies on the optimization of steps (2) and (3) in terms of both efficiency and output quality (see violet shaded area in Fig.~\ref{fig:seq_diag_system}).

Note, the steps (1) and (2) draw on a logical reasoner. Since logical reasoning is one of the main sources of complexity in sequential diagnosis, the amount of reasoning should be ideally as minimal as possible, indicated by the red arrow in Fig.~\ref{fig:seq_diag_system}. Basically, there are two \textbf{different reasoning paradigms} 
sequential diagnosis systems might use, glass-box and black-box.
\emph{Glass-box approaches} directly integrate reasoning with diagnoses finding with the goal of achieving better performance. To this end the internals of the reasoner are suitably modified or, respectively, reasoners are complemented by additional services, e.g., bookkeeping in an ATMS \citep{dekleer1986atms}. One example \citep{dekleer1987} is the storing of (minimal) environments (sets of logical sentences sufficient) for entailments predicted by the system model. These are leveraged to compute so-called nogood sets \citep{dekleer1986atms}, i.e.\ environments for entailments inconsistent with observations. The latter can be directly used for diagnoses construction. 
Glass-box approaches are therefore dependent on the particular (modified) reasoner and thus on the particular logic for which the reasoner is sound and complete. 
\emph{Black-box approaches} use the reasoner as an oracle for answering consistency or entailment queries. The reasoner is used as-is without requiring any alterations to its implementation or any supplements. Consequently, these approaches are independent of the logic used for describing the system model and of the particular reasoner employed, and can benefit from latest improvements of reasoning algorithms. For instance, black-box approaches can switch to reasoners specialized in a certain sublanguage (e.g.\ polynomial-time reasoner ELK \citep{kazakov2014} for  
OWL EL \citep{kroetzsch2010efficient}) of a logic (e.g.\ OWL 2 \citep{Grau2008a} where reasoning is N2EXPTIME-complete) ``for free'' in a simple plug-in fashion if the system description is formalized in this sublanguage. 

First, while glass-box approaches in many cases offer some performance gain over black-box approaches, this gain was shown to be not that significant -- in most cases the time cost of both paradigms lay within the same order of magnitude -- in extensive evaluations carried out by \citep{Horridge2011a} using Description Logics \citep{DLHandbook} of reasoning complexity ranging from polynomial to N2EXPTIME-complete. Black-box approaches even outperformed glass-box approaches in a significant number of cases, witnessed in similar experiments conducted by \citep{Kalyanpur2006a}. When using bookkeeping methods, the information stored by these might grow exponentially with the problem size \citep{schiex1994nogood}. Moreover, switching to more efficient reasoners (e.g., for fragments of a logic, see above) is not (easily) possible for glass-box approaches.
Second, system descriptions ($\sd$) in MBD might use a wide range of different knowledge representation formalisms such as First-Order Logic fragments, Propositional Logic, Horn clauses, equations, constraints, Description Logics or OWL.
For these reasons we present a \emph{logics- and reasoner-independent black-box approach} to sequential diagnosis which is \emph{appropriate for all monotonic and decidable knowledge representation languages}. This preserves a maximal generality of our approach and makes it broadly applicable across different MBD application domains. 
%
%
%

Because the problem of \textbf{optimal query selection}\footnote{Also known as \emph{Optimal Test Sequencing Problem} \citep{pattipati1990} or \emph{Optimal Decision Tree Problem} \citep{hyafil1976}.} is NP-complete \citep{hyafil1976}, sequential diagnosis approaches have to bear on a trade-off between query optimality and computational complexity. Therefore, it is current practice to rely on myopic (usually one-step lookahead)
methods to guide diagnoses discrimination \citep{dekleer1987,DBLP:journals/jair/FeldmanPG10a,gonzalez2011spectrum,Shchekotykhin2012,Rodler2013}. Empirical \citep{dekleer1992,Shchekotykhin2012,Rodler2013} and theoretical \citep{pattipati1990} evaluations have evidenced that such heuristic methods in many cases deliver reasonable and in some scenarios even (nearly) optimal results. Moreover, query selection based on a multi-step lookahead is computationally prohibitive due to the involved expensive model-based reasoning (cf.\ Sec.~\ref{sec:related}). 
In common with the above-mentioned approaches 
we model the query selection heuristic as a \emph{query selection measure} $m$ assigning a real-value to each query based on its quality (regarding diagnoses discrimination). One popular such measure is entropy \citep{dekleer1987}, which favors queries with a maximal expected information gain or, equivalently, a maximal expected reduction of the diagnostic uncertainty. The goal of any such measure $m$ is the \emph{minimization of the number of queries} required until achieving the appointed diagnostic goal.

Whereas sequential diagnosis approaches usually incorporate the optimization of a query selection measure $m$, they often do not optimize the query (answering) cost such as the time required to perform measurements~\citep{heckerman1995decision}. We model this cost by a \emph{query cost measure} $c$, a function allocating a real-valued cost to each query. 
The approach suggested in this work is devised to compute optimized queries along the $m$ and $c$ axes at each (query selection) step in the sequential diagnosis process while minimizing the required computational resources. 
More concretely, the contributions of this work are the following: 

\noindent\textbf{Contributions.} 
We present a novel query optimization method that is generally applicable to any MBD problem in the sense of \citep{dekleer1987,Reiter87} and
\begin{enumerate}[noitemsep,topsep=5pt]
	\item defines a query as a set of First-Order Logic sentences and thus generalizes the \emph{measurement} notion of \citep{dekleer1987,Reiter87},
	\item given a set of \emph{leading} 
	\emph{diagnoses} \citep{DBLP:conf/ijcai/KleerW89}, allows the two-dimensional optimization of the next query in terms of the \emph{expected number of subsequent queries} (measure $m$) \emph{and query cost} (measure $c$), 
	\item for an aptly refined (yet exponential) query search space, finds -- \emph{without any reasoner calls} -- the \emph{globally} optimal query w.r.t.\ measure $c$ that \emph{globally} optimizes measure $m$,\footnote{The term \emph{globally optimal} has its standard meaning (cf.\ \citep[p. 184]{Luenberger:2015:LNP:2843008}) and emphasizes that the optimum over \emph{all} queries in the respective query search space is meant.} 
	\item for the full query search space, finds -- with a polynomial number of reasoner calls -- the (under reasonable assumptions) globally optimal query w.r.t.\ $m$ that includes, if possible, only ``cost-preferred'' sentences (e.g.\ those answerable using built-in sensors),
	\item guarantees the proposal of queries that discriminate between \emph{all} leading diagnoses and that \emph{unambiguously identify the actual diagnosis}.
\end{enumerate}
Furthermore, 
\begin{enumerate}[resume]
	\item we show that any MBD problem can be reduced to a Knowledge Base Debugging (KBD) problem \citep{Shchekotykhin2012,Rodler2015phd}. This result establishes a formal relationship between these two paradigms, shows the greater generality of the latter and enables the transferral of findings in the KBD domain to the MBD domain.   
\end{enumerate}
In a nutshell, \textbf{the presented query optimization method} can be subdivided into three phases, P1, P2 and P3. In the first place, P1 optimizes the next query's discrimination properties (e.g.\ the expected information gain) based on the criteria imposed by the given QSM $m$, realized by a heuristic backtracking search. Then, as a first option, P2 computes an optimal query $Q^*$ regarding the given QCM $c$ by running a uniform-cost hitting set tree search over a suitable (and explicitly given) set of partial leading diagnoses. This is done in a way $Q^*$ meets exactly the optimal discrimination properties determined in P1. P2 explores the largest possible query search space that can be handled without any reasoner calls in a complete way. The output $Q^*$ suggests the inspection of the system component(s) that is least expensive for the oracle (QCM $c$) among all those that yield the highest information (QSM $m$). As a second option and alternative to P2, P3 performs a two-step optimization consisting of a first generalization of the addressed search space and a subsequent divide-and-conquer exploration of this search space focused on cost-preferred measurements. P3 returns a cost-optimal query $Q^*$ (w.r.t.\ some QCM $c$) complying with the optimal discrimination properties fixed in P1. $Q^*$ may include measurements of arbitrary type, depending on priorly definable requirements. 

Roughly, the \textbf{efficiency of the novel approach}
is possible by the recognition that the optimizations of $m$ and $c$ can be decoupled and by using logical monotonicity as well as the inherent (already inferred) information in the ($\subseteq$-minimal) leading diagnoses. The latter is leveraged to achieve a retention of costly reasoner calls until the final query computation stage (P3), and hence to reduce them to a minimum. In particular, the method is inexpensive as it 
\begin{enumerate}[label={(\alph*)},noitemsep,topsep=5pt]
	\item avoids the generation and examination of unnecessary (non-discriminating) or duplicate query candidates, 
	\item \emph{actually} computes only the \emph{single} best query by its ability to estimate a query's quality without computing it, and 
	\item guarantees soundness and completeness w.r.t.\ an exponential query search space independently of the properties and output of a reasoner.
\end{enumerate}
Modern sequential diagnosis methods like~\citep{dekleer1987} and its derivatives \citep{DBLP:journals/jair/FeldmanPG10a,Shchekotykhin2012,Rodler2013} do not meet all properties (a) -- (c). The black-box approaches among them
extensively call a reasoner  
in order to compute a query. As we show in our evaluations, 
the presented method can \emph{save an exponential overhead compared to these approaches}.

Moreover, we emphasize that our approach can also \emph{deal with problems where the query space is implicit}, i.e.\ all possible system measurements cannot be enumerated in polynomial time in the size of the system model. E.g., in a digital circuit all measurement points (and hence the possible queries) are given explicitly by the circuit's wires which can be directly extracted from the system description (\sd). In, e.g., knowledge-based problems, by contrast, the possible measurements, i.e.\ questions to an expert, must be (expensively) inferred and are not efficiently enumerable. In fact, we show that for problems involving implicit queries, \emph{approaches not using the proposed theory might be drastically incomplete} and hence might miss optimal queries.
%
%

Finally, by the generality of our query notion, our method \emph{explores a more complex search space} than~\citep{dekleer1987,dekleer1993}, thereby guaranteeing property (5) above.

\noindent\textbf{Organization.} The rest of this work is organized as follows. Sec.~\ref{sec:basics} provides theoretical foundations needed in later sections. In particular, it gives a short introduction on Model-Based Diagnosis (MBD) in Sec.~\ref{sec:mbd}, on Knowledge Base Debugging (KBD) in Sec.~\ref{sec:kbd} and formally proves that each MBD problem can be reduced to a KBD problem 
in Sec.~\ref{sec:reduction_of_mbd_to_kbd}. Henceforth, the work focuses w.l.o.g.\ just on KBD. Basics on Sequential Diagnosis including important definitions, the formal characterization of the addressed problem, and a generic algorithm to solve this problem are treated in Sec.~\ref{sec:sequential_diagnosis}. 
The main part of the paper starts with Sec.~\ref{sec:contribution}, where we first formalize the measurement selection problem (Sec.~\ref{sec:measurement_selection}) and then discuss the proposed novel algorithm to solve this problem (Sec.~\ref{sec:presented_algo}). The presentation of our method is subdivided into a first part attempting to give the reader a prior intuition, motivation and overview of the later introduced theoretical concepts (Sec.~\ref{sec:intuition+overview}), and three further parts, one dedicated to each phase (P1, P2 and P3) of the new algorithm (Sec.~\ref{sec:P1}, \ref{sec:P2} and \ref{sec:P3}). Besides an extensively exemplified expansion of the relevant theory, each phase description includes a complexity analysis. A formal specification of the computed solution's properties for P1+P2 is given in Sec.~\ref{sec:solution_produced_by_P1+P2} and for P3 in Sec.~\ref{sec:solution_produced_by_P3}. Finally, Sec.~\ref{sec:recap_of_presented_algo} recapitulates the entire approach by means of a detailed example.
Sec.~\ref{sec:eval} includes the description of our experimental evaluations in order to complement the theoretical findings of Sec.~\ref{sec:presented_algo}. The experimental settings are explicated in Sec.~\ref{sec:experiments}, whereas the experimental results are discussed in Sec.~\ref{sec:experimental_results}. 
Subsequently, there is a section on related work (Sec.~\ref{sec:related}) before we conclude with Sec.~\ref{sec:conclusion}. Appendix~A comprises all proofs that are not given in the text. Appendix~B provides a table including all important symbols used in the text along with their meaning.

\section{Preliminaries}
\label{sec:basics}
In this section, we revise the general theory of Model-Based Diagnosis (MBD) proposed by \citep{Reiter87}, define the knowledge base debugging framework (KBD) we will use to formalize MBD problems in this work, and demonstrate that KBD is a generalization of MBD.

\subsection{Model-Based Diagnosis}
\label{sec:mbd}
We briefly review the classical model-based diagnosis (MBD) problem described by \citep{Reiter87}. 
At first, we characterize a \emph{system}, e.g.\ a digital circuit, a car or some software, which is the subject of a diagnosis task:
\begin{definition}[System]\label{def:diagnosable_system}
	A \emph{system} 
	is a tuple $(\sd, \comps)$ where $\sd$, the system description, is a set of First-Order Logic sentences, and  $\comps$, the system components, is a finite set of constants $c_1,\dots,c_n$. 	
\end{definition}
The distinguished unary ``abnormal'' predicate $\ab$ is used in $\sd$ to model the expected behavior of components $c \in \comps$.
Let us denote the First-Order Logic sentence describing this expected behavior of $c$ by $beh(c)$ and let $\sd_{beh} := \setof{\lnot\ab(c) \to beh(c)\mid c \in \comps}$. The latter subsumes a statement of the form ``if $c$ is nominal (not abnormal), then its behavior is $beh(c)$'' for each system component $c \in \comps$. Any behavior different from $beh(c)$ implies that $c$ is at fault, i.e.\ $\ab(c)$ holds. But, an abnormal component does not necessarily manifest a faulty behavior in each situation (\emph{weak fault model} \citep{Kleer1992,feldman2009solving}), e.g.\ for an or-gate $c$ stuck at 1 faulty behavior $\lnot beh(c)$ can only be observed if both inputs are 0.
Further, $\sd$ might include general axioms describing the system domain or descriptions of the interplay between the system components. Let us call the set of these general axioms $\sd_{gen}$. So, $\sd = \sd_{beh} \cup \sd_{gen}$.

The behavior of a system $(\sd, \comps)$ assuming all components working correctly is captured by the description $\sd \cup \setof{\lnot\textsc{ab}(c) \mid c \in \comps}$. Note, this description is equal to $\sd_{gen} \cup \setof{beh(c)\mid c\in\comps}$.

A \emph{diagnosis problem} arises when the observed system behavior -- represented by a finite set of First-Order Logic sentences $\obs$ -- differs from the expected system behavior. Formally, this means that $\sd \cup \setof{\lnot\textsc{ab}(c) \mid c \in \comps} \cup \obs \models \bot$. For instance, in circuit diagnosis \obs might be the observation of the system inputs and outputs.    

There are usually multiple different hypotheses (\emph{diagnoses}) that explain 
the discrepancy between observed and predicted system behavior. 
Discrimination between these hypotheses can then be accomplished by means of additional observations $\meas$ called measurements \citep{Reiter87,dekleer1987}. Each measurement $m$ in the set of measurements $\meas$ is a set of First-Order Logic sentences \citep{Reiter87} describing additional knowledge about the actual system behavior, e.g.\ whether a particular wire in a faulty circuit is high or low. Usually new measurements are conducted and added to $\meas$ until some diagnostic goal $G$ is achieved, e.g.\ the presence of just a single or one highly probable remaining hypothesis. Each added measurement $m$, if chosen properly, will invalidate some hypotheses. Throughout this paper we assume \emph{stationary health} \citep{DBLP:journals/jair/FeldmanPG10a}, 
i.e.\ that one and the same (faulty) behavior can be constantly reproduced for each $c \in \comps$ during system diagnosis.

Formalized, these notions lead to the definitions of an \emph{MBD diagnosis problem instance (MBD-DPI)} and of an \emph{MBD-diagnosis}.
\begin{definition}[MBD-DPI]\label{def:MBD-DPI}
	Let $\obs$ (system observations) be a finite set of First-Order Logic sentences, $\meas$ (measurements) be a finite set including finite sets $m_i$ of First-Order Logic sentences, and let $(\sd,\comps)$ be a system.
	Then the tuple $(\sd,\comps,\obs,\meas)$ is an \emph{MBD diagnosis problem instance (MBD-DPI)}.
\end{definition}
\begin{definition}\label{def:SD*}
Let $\dpi :=(\sd,\comps,\obs,\meas)$ be an MBD-DPI and $U_{\meas}$ denote the union of all $m \in \meas$. Then $\sdaa{\dg} := \sd \cup \setof{\textsc{ab}(c) \mid c \in \dg} \cup \setof{\lnot\ab(c)\mid c \in \comps\setminus \dg} \cup \obs \cup U_{\meas}$ for $\dg \subseteq \comps$ denotes the behavior description of the system $(\sd,\comps)$ 
\begin{itemize}[noitemsep,topsep=5pt]
	\item under the current state of knowledge given by the $\dpi$ in terms of $\obs$ and $\meas$, and
	\item under the assumption that all components in $\dg \subseteq \comps$ are faulty and all components in $\comps \setminus \dg$ are healthy.
\end{itemize}
\end{definition}
\begin{definition}[MBD-Diagnosis]\label{def:MBD-diagnosis}
	Let $\dpi :=(\sd,\comps,\obs,\meas)$ be an MBD-DPI.
	Then $\dg \subseteq \comps$ is an \emph{MBD-diagnosis for $\dpi$} iff 
	%
	$\sdaa{\dg}$ is consistent ($\dg$ explains $\obs$ and $\meas$).
	An MBD-diagnosis $\dg$ for $\dpi$ is called \emph{minimal} iff there is no MBD-diagnosis $\dg'$ for $\dpi$ such that $\dg' \subset \dg$.
\end{definition}
In many practical applications there are multiple (minimal) MBD-diagnoses for a given MBD-DPI. Without additional information about the system, one cannot conjecture a unique diagnosis. The idea is then to perform measurements in order to discriminate between competing (minimal) MBD-diagnoses until a sufficient degree of diagnostic certainty (the specified diagnostic goal $G$) is reached. This is the problem addressed by \emph{Sequential MBD} and can be stated as follows:
\begin{prob_def}[Sequential MBD]\label{prob_def:sequential_MBD} \textcolor{white}{.}

\noindent\textbf{Given:} An MBD-DPI $\dpi := (\sd,\comps$, $\obs,\meas)$ and a diagnostic goal $G$. 

\noindent\textbf{Find:} 
%
%
%
$\meas_{\mathit{new}} \supseteq \emptyset$ and $\dg$, where $\meas_{\mathit{new}}$ is a set of new measurements such that $\dg$ is a minimal MBD-diagnosis for the MBD-DPI $\dpi_{\mathit{new}} := (\sd,\comps,\obs$, $\meas\cup\meas_{\mathit{new}})$ and $\dg$ satisfies $G$. 
\end{prob_def}

\begin{remark}\label{rem:leading_diags__diagnostic_goal}
Due to the intractability of the computation of the entire set of minimal diagnoses \citep{Bylander1991}, both the measurement selection and the decision whether a diagnostic goal $G$ is satisfied for some diagnosis $\md$ is usually made by using a (computationally feasible) set of \emph{leading minimal diagnoses} $\mD$ \citep{DBLP:conf/ijcai/KleerW89}. $\mD$ acts as an approximation of all minimal diagnoses for the given DPI and usually comprises the \emph{most probable minimal} \citep{DBLP:conf/ijcai/KleerW89} or \emph{minimum-cardinality} \citep{DBLP:journals/jair/FeldmanPG10a} diagnoses for a DPI. Given a set of leading minimal diagnoses $\mD$ for $\dpi_{\mathit{new}}$, examples for the specification of $G$ are $G_1:=$ ``$\md$ is the only minimal diagnosis for $\dpi_{\mathit{new}}$'' \citep{dekleer1993}, $G_2:=$ ``$\md$ exceeds some predefined probability threshold $t$'', e.g.\ $t:=0.95$ \citep{dekleer1987,Shchekotykhin2012} or $G_3:=$ ``$\md$ has $\geq k$ times the probability of all other elements in $\mD$''. Note that the goal $G_1$ represents a maximally strict requirement on the final diagnostic result as it requires the verification of the invalidity of all but the correct minimal diagnosis (we call a diagnostic goal $G_i$ \emph{more strict} than a diagnostic goal $G_j$ if $G_j$ is satisfied earlier in any diagnostic session than $G_i$). The specification of (constants in) $G$ depends on the seriousness of misdiagnosis, e.g.\ higher probability thresholds signify higher criticality.\qed

\end{remark}
In general, the size of the search space for minimal MBD-diagnoses for $(\sd,\comps,\obs$, $\meas)$ is in $O(2^{|\comps|})$. 
A useful concept to restrict this search space 
is the one of an \emph{MBD-conflict} \citep{Reiter87,dekleer1987}, a set of components whose elements cannot all be healthy given $\obs$ and $\meas$:
\begin{definition}[MBD-Conflict]\label{def:MBD-conflict}
Let $\dpi :=(\sd,\comps,\obs,\meas)$ be an MBD-DPI. Then $C \subseteq \comps$ is an MBD-conflict for $\dpi$ iff $\sd \cup \setof{\lnot\ab(c) \mid c \in C} \cup \obs \cup U_{\meas}$ is inconsistent.
An MBD-conflict $C$ for $\dpi$ is called \emph{minimal} iff there is no MBD-conflict $C'$ for $\dpi$ such that $C' \subset C$.	
\end{definition} 
\begin{definition}[Hitting Set]\label{def:hs}
	Let $S=\setof{S_1,\dots,S_n}$ be a collection of sets. Then $H$ is called a \emph{hitting set of $S$} iff $H \subseteq U_S$ and $H \cap S_i \neq \emptyset$ for all $i=1,\dots,n$. 
	A hitting set $H$ of $S$ is \emph{minimal} iff there is no hitting set $H'$ of $S$ such that $H' \subset H$.
\end{definition}
The following result \citep{Reiter87} can be used to determine MBD-diagnoses through the computation of MBD-conflicts:
\begin{theorem}\label{theorem:reiter_diag_is_hitting_set}
	A (minimal) MBD-diagnosis for a DPI is a (minimal) hitting set of all minimal MBD-conflicts for this DPI.
\end{theorem}

\begin{figure}
	\begin{minipage}{0.65\textwidth}
		\centering
		\includegraphics[width=0.99\linewidth]{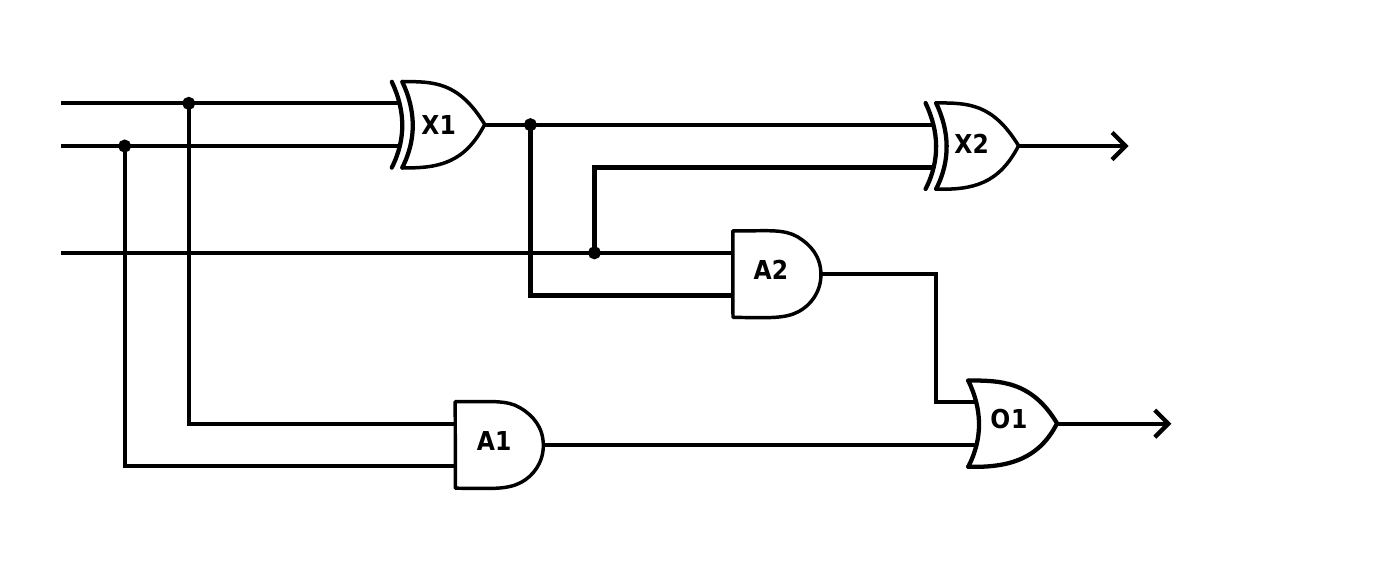}
	\end{minipage}	
	\begin{minipage}{0.3\textwidth}
		\scriptsize
		\begin{tabular}{c}
			\addlinespace[0pt]\toprule\addlinespace[0pt]
			circuit inputs (from top to bottom) \\
			1 \\
			0 \\
			1 \\
			\addlinespace[0pt]\midrule\addlinespace[0pt]
			circuit outputs (from top to bottom) \\
			1 \\
			0 \\
			\addlinespace[0pt]\bottomrule	
		\end{tabular}
	\end{minipage}
	\caption{\small MBD Example due to \citep{Reiter87} from the domain of circuit diagnosis.}
	\label{fig:circuitreiter87}
\end{figure}

\begin{example}\label{ex:circuit_MBD-DPI_diags_conflicts_measurements}
Let us revisit the circuit diagnosis example given in \citep{Reiter87} shown in Fig.~\ref{fig:circuitreiter87}. 
The first step towards diagnosing the circuit using MBD is to formulate the problem as an MBD-DPI. The result $\exdpiM := (\sd,\comps,\obs,\meas)$ is given by Tab.~\ref{tab:circuitreiter87_MBD-DPI} and explained next.

The circuit, i.e.\ the system to be diagnosed, includes five gates $X_1, X_2$ (xor-gates), $A_1, A_2$ (and-gates) and $O_1$ (or-gate), which are at the same time the system components $\comps$ of interest. 
The system description $\sd = \sd_{\mathit{beh}} \cup \sd_{\mathit{gen}}$ consists of a knowledge base $\sd_{\mathit{beh}} = \setof{\tax_{1},\dots,\tax_{5}}$ describing the behavior of each gate given it is working properly, e.g.\ for gate $X_1$, $\sd_{\mathit{beh}}$ includes the sentence $\tax_1 := (\lnot \ab(X_1) \to out(X_1) = xor(in1(X_1),in2(X_1)))$. Besides, $\sd$ includes a knowledge base $\sd_{\mathit{gen}} = \setof{\tax_{6},\dots,\tax_{12}}$ describing which gate-terminals are connected by wires, e.g.\ the wire connecting $X_1$ to $X_2$ is defined by the sentence $\tax_7 := (out(X_1) = in1(X_2))$. For simplicity we omit the explicit statement of additional general domain knowledge in $\sd_{\mathit{gen}}$ such as axioms for Boolean algebra or axioms restricting wires to only either $0$ or $1$ values. The observations $\obs = \setof{\tax_{13},\dots,\tax_{17}}$ are given by the system inputs and outputs (see the table in Fig.~\ref{fig:circuitreiter87}). Finally, since there are no already performed measurements, the set $\meas$ is empty.

Assuming all components are healthy, i.e.\ all gates function properly, 
we find out that $\sdaa{\emptyset}$ is inconsistent (cf.\ Def.~\ref{def:SD*}). 
That is, the assumption of no faulty components conflicts with the observations $\obs$ made. E.g., if $X_1$ and $X_2$ manifest nominal behavior, we can deduce that the output $out(X_2) = 0$ which contradicts the observation sentence $\tax_{16} := (out(X_2) = 1)$. Supposing either of the components $X_1$, $X_2$ to be nominal, we can no longer deduce $out(X_2) = 0$ (or any other sentence contradicting $\obs$). Therefore, $C_1 := \setof{X_1,X_2}$ is a minimal MBD-conflict (cf.\ Def.~\ref{def:MBD-conflict}). Similarly, we find that $C_2 := \setof{X_1,A_2,O_1}$ is the only other minimal MBD-conflict for $\exdpiM$. Computing minimal hitting sets of all minimal MBD-conflicts $C_1, C_2$, we obtain three minimal MBD-diagnoses $\dg_1 := \setof{X_1}$, $\dg_2 := \setof{X_2,A_2}$ and $\dg_3 := \setof{X_2,O_1}$.

Let the diagnostic goal $G$ be the achievement of complete diagnostic certainty, i.e.\ to single out the correct minimal MBD-diagnosis. The goal of the MBD-problem is then to find new measurements $m_1,\ldots,m_k$ such that there is a single minimal diagnosis $\dg$ for $(\sd,\comps,\obs$, $\meas\cup\setof{m_1,\ldots,m_k})$. Let the first measurement $m_1$ be the observation of the terminal $out(X_1)$, and let the value of it be $0$. Then, $\dg_1$ is still a minimal MBD-diagnosis for 
$\exdpiM_{\mathit{new}}:=(\sd$, $\comps,\obs,\meas\cup\setof{\setof{out(X_1)=0}})$ since the abnormality of $X_1$ explains both $\obs$ and $\meas$. Moreover, all other MBD-diagnoses for $\exdpiM_{\mathit{new}}$ must contain $X_1$ (since its faultiness is the only explanation for $\meas$) and thus be supersets of $\dg_1$. Hence, $\dg_1$ is the only minimal MBD-diagnosis for $\exdpiM_{\mathit{new}}$ and thus the actually faulty component in this scenario is $X_1$ (under the assumption that a $\subseteq$-minimal set of components is broken). This fact could be derived by conducting only one measurement.\qed
\end{example}

\renewcommand{\arraystretch}{1.2} 

\begin{table}[t]
	\begin{minipage}{0.5\textwidth} 
\scriptsize
\centering
\rowcolors[]{2}{gray!8}{gray!16} 
\begin{tabular}{ c c c c c} 
	\rowcolor{gray!40}
	\toprule\addlinespace[0pt]
	$i$ & $\tax_i$ & $\sd_{\mathit{beh}}$ & $\sd_{\mathit{gen}}$ & $\obs$ \\ \addlinespace[0pt]\midrule\addlinespace[0pt]
	1 & $\lnot\ab(X_1) \to beh(X_1)$ & $\bullet$ & & 	\\
	2 & $\lnot\ab(X_2) \to beh(X_2)$ & $\bullet$ & &  	\\
	3 & $\lnot\ab(A_1) \to beh(A_1)$ & $\bullet$ & &  	\\
	4 & $\lnot\ab(A_2) \to beh(A_2)$ & $\bullet$ & &  	\\
	5 & $\lnot\ab(O_1) \to beh(O_1)$ & $\bullet$ &  & 	\\
	6 & $out(X_1) = in2(A_2)$ & & $\bullet$  &  	\\
	7 & $out(X_1) = in1(X_2)$ & & $\bullet$  &  	\\
	8 & $out(A_2) = in1(O_1)$ & &  $\bullet$ & 	\\
	9 & $in1(A_2) = in2(X_2)$ & &  $\bullet$ & 	\\
	10 & $in1(X_1) = in1(A_1)$ & &  $\bullet$ & 	\\
	11 & $in2(X_1) = in2(A_1)$ & &  $\bullet$ & 	\\
	12 & $out(A_1) = in2(O_1)$ & &  $\bullet$ & 	\\
	13 & $in1(X_1) = 1$ & & &  $\bullet$ 	\\
	14 & $in2(X_1) = 0$ & & &  $\bullet$  	\\
	15 & $in1(A_2) = 1$ & & &  $\bullet$  	\\
	16 & $out(X_2) = 1$ & & &  $\bullet$  	\\
	17 & $out(O_1) = 0$ & & &  $\bullet$  	\\
	\addlinespace[0pt]\bottomrule 
	\rowcolor{gray!40}
	\multicolumn{5}{c}{$\comps$} \\ \addlinespace[0pt]\midrule\addlinespace[0pt]
	\multicolumn{5}{c}{$\setof{X_1,X_2,A_1,A_2,O_1}$} 	\\ \addlinespace[0pt]\toprule\addlinespace[0pt]
	\rowcolor{gray!40}
	$c$ & \multicolumn{4}{c}{$beh(c)$ for $c \in \comps$} \\ \addlinespace[0pt]\midrule\addlinespace[0pt]
	$X_1$ & \multicolumn{4}{c}{$out(X_1) = xor(in1(X_1),in2(X_1))$} 	\\
	$X_2$ & \multicolumn{4}{c}{$out(X_2) = xor(in1(X_2),in2(X_2))$} 	\\
	$A_1$ & \multicolumn{4}{c}{$out(A_1) = xor(in1(A_1),in2(A_1))$} 	\\
	$A_2$ & \multicolumn{4}{c}{$out(A_2) = xor(in1(A_2),in2(A_2))$} 	\\
	$O_1$ & \multicolumn{4}{c}{$out(O_1) = xor(in1(O_1),in2(O_1))$} 	\\ \addlinespace[0pt]\toprule\addlinespace[0pt]
	\rowcolor{gray!40}
	$i$ & \multicolumn{4}{c}{$\meas$} \\ \addlinespace[0pt]\midrule\addlinespace[0pt]
	$\times$ & \multicolumn{4}{c}{$\times$} 	\\ \addlinespace[0pt]\bottomrule
\end{tabular}
		\caption{\small MBD-DPI $\exdpiM$ obtained from circuit diagnosis problem in Fig.~\ref{fig:circuitreiter87}.}
		\label{tab:circuitreiter87_MBD-DPI}
	\end{minipage}
	\hfill
	\begin{minipage}{0.45\textwidth}
		\scriptsize
		\centering
		\rowcolors[]{2}{gray!8}{gray!16} 
		\begin{tabular}{ c c c c } 
			\rowcolor{gray!40}
			\toprule\addlinespace[0pt]
			$i$ & $\tax_i$ & $\mo$ & $\mb$  \\ \addlinespace[0pt]\midrule\addlinespace[0pt]
			1 & $out(X_1) = xor(in1(X_1),in2(X_1))$ & $\bullet$ & 	\\
			2 & $out(X_2) = xor(in1(X_2),in2(X_2))$ & $\bullet$ &  	\\
			3 & $out(A_1) = and(in1(A_1),in2(A_1))$ & $\bullet$ &  	\\
			4 & $out(A_2) = and(in1(A_2),in2(A_2))$ & $\bullet$ &  	\\
			5 & $out(O_1) = or(in1(O_1),in2(O_1))$ & $\bullet$ &  	\\
			6 & $out(X_1) = in2(A_2)$ & & $\bullet$   	\\
			7 & $out(X_1) = in1(X_2)$ & & $\bullet$   	\\
			8 & $out(A_2) = in1(O_1)$ & &  $\bullet$	\\
			9 & $in1(A_2) = in2(X_2)$ & &  $\bullet$	\\
			10 & $in1(X_1) = in1(A_1)$ & &  $\bullet$	\\
			11 & $in2(X_1) = in2(A_1)$ & &  $\bullet$	\\
			12 & $out(A_1) = in2(O_1)$ & &  $\bullet$	\\
			13 & $in1(X_1) = 1$ & &  $\bullet$	\\
			14 & $in2(X_1) = 0$ & &  $\bullet$	\\
			15 & $in1(A_2) = 1$ & &  $\bullet$	\\
			16 & $out(X_2) = 1$ & &  $\bullet$	\\
			17 & $out(O_1) = 0$ & &  $\bullet$	\\
			\addlinespace[0pt]\bottomrule 
			\rowcolor{gray!40}
			$i$ & \multicolumn{3}{c}{$\tp_i\in\Tp$} \\ \addlinespace[0pt]\midrule\addlinespace[0pt]
			$\times$ & \multicolumn{3}{c}{$\times$} 	\\ \addlinespace[0pt]\toprule\addlinespace[0pt]
			\rowcolor{gray!40}
			$i$ & \multicolumn{3}{c}{$\tn_i\in\Tn$} \\ \addlinespace[0pt]\midrule\addlinespace[0pt]
			$\times$ & \multicolumn{3}{c}{$\times$} 	\\ \addlinespace[0pt]
			\toprule\addlinespace[0pt]
			\rowcolor{gray!40}
			$i$ & \multicolumn{3}{c}{$r_i\in\RQ$} \\ \addlinespace[0pt]\midrule\addlinespace[0pt]
			$1$ & \multicolumn{3}{c}{consistency} \\ 
			\addlinespace[0pt]\bottomrule
			\toprule\addlinespace[0pt]
			\rowcolor{gray!40}
			\multicolumn{4}{c}{min KBD-conflicts} \\ \addlinespace[0pt]\midrule\addlinespace[0pt]
			\multicolumn{4}{c}{$\setof{\tax_1,\tax_2},\setof{\tax_1,\tax_4,\tax_5}$} \\
			\toprule\addlinespace[0pt]
			\rowcolor{gray!40}
			\multicolumn{4}{c}{min KBD-diagnoses} \\ \addlinespace[0pt]\midrule\addlinespace[0pt]
			\multicolumn{4}{c}{$\setof{\tax_1},\setof{\tax_2,\tax_4}, \setof{\tax_2,\tax_5}$} \\
			\addlinespace[0pt]\bottomrule
		\end{tabular}
		\caption{\small KBD-DPI $\exdpiMK$ obtained from MBD-DPI $\exdpiM$ from Tab.~\ref{tab:circuitreiter87_MBD-DPI}.}
		\label{tab:circuitreiter87_KBD-DPI}
	\end{minipage}
\end{table}

\subsection{Knowledge Base Debugging}
\label{sec:kbd}
In this section we revisit the knowledge base debugging (KBD) problem \citep{friedrich2005gdm,Shchekotykhin2012,Rodler2015phd} which we will use subsequently
as a generalized reformulation of Reiter's original MBD problem described above.
Besides offering some notational conveniences, 
KBD 
allows users to specify 
negative measurements (or test cases) \citep{DBLP:journals/ai/FelfernigFJS04}. Contrary to (positive) measurements $m \in \meas$ as characterized above, negative measurements state properties that \emph{must not hold}. In other words, any diagnosis must fulfill that -- under its assumption -- the system description together with the observations and positive measurements does not entail any negative measurement.
Additionally, it is possible in KBD to postulate stronger logical properties apart from consistency.
For example, when debugging an ontology (i.e.\ a system where \comps are ontology axioms) one might want the assumption of a diagnosis to yield a \emph{coherent} \citep{Schlobach2007,Parsia2005} system description (repaired ontology), i.e.\ one without unsatisfiable classes. In First-Order Logic terms (using logic programming notation), an \emph{unsatisfiable class} in a KB $\mo$ is an $n$-ary predicate $r$ such that $\mo \models \forall \mathbf{X} \, \lnot r(\mathbf{X})$ where $\mathbf{X} = X_1,\ldots,X_n$. 
That is, coherency means that every predicate in $\mo$ can have some instance without yielding an inconsistency. 

Another possible use case for the adoption of (logical) requirements such as coherency is the fault localization in flawed (e.g.\ inconsistent) system models used for MBD. For instance, a model (which is itself a KB) used to describe the circuit in Fig.~\ref{fig:circuitreiter87} might include an unsatisfiable class $\mathit{xor}$ (which essentially makes the model inconsistent after the creation of, e.g., the sentence $\mathit{xor}(X_1)$ declaring $X_1$ as an xor-gate). The reason for this incoherency might be that $\sd_{gen}$ includes the sentences $\mathit{xor}(G) \to \mathit{gate}(G)$ and $\mathit{gate}(G) \to \mathit{and}(G)\lor \mathit{or}(G) \lor \mathit{not}(G)$ (where the system modeler forgot to include $\mathit{xor}(G)$) as well as sentences stating that no instance can be of more than one type of gate. That is, KBD (with the coherency requirement) could be used in such scenario 
to repair the model thus enabling a sound diagnostic process.   
%

\subsubsection{The Used Notation}
\label{sec:notation} 
Let $\mathcal{L}$ denote some formal knowledge representation language. We will call $\tax_{\mathcal{L}},\tax_{1,\mathcal{L}}, \tax_{2,\mathcal{L}},\ldots \in \mathcal{L}$ \emph{logical sentences over $\mathcal{L}$} and a set of logical sentences $\mo_{\mathcal{L}} \subseteq 2^{\mathcal{L}}$ a \emph{knowledge base (KB) over $\mathcal{L}$}. Sentences in $\mo_{\mathcal{L}}$ will sometimes be referred to as \emph{axioms}.
We denote by $\models_{\mathcal{L}} \;\subseteq 2^{\mathcal{L}} \times \mathcal{L}$ the semantic entailment relation for the logic $\mathcal{L}$ and we write $\mo_{\mathcal{L}} \models_{\mathcal{L}} \tax_{\mathcal{L}}$ to state that $\tax_{\mathcal{L}}$ is a logical consequence of the KB $\mo_{\mathcal{L}}$. For brevity, we will write $\mo_{1,\mathcal{L}} \models_{\mathcal{L}} \mo_{2,\mathcal{L}}$ for two KBs $\mo_{1,\mathcal{L}}$ and $\mo_{2,\mathcal{L}}$ to denote that $\mo_{1,\mathcal{L}} \models_{\mathcal{L}} \tax_{\mathcal{L}}$ for \emph{all} $\tax_{\mathcal{L}} \in \mo_{2,\mathcal{L}}$ and $\mo_{1,\mathcal{L}} \not\models_{\mathcal{L}} \tax_{\mathcal{L}}$ to state that $\mo_{1,\mathcal{L}} \not\models_{\mathcal{L}} \tax_{\mathcal{L}}$ for \emph{some} $\tax_{\mathcal{L}} \in \mo_{2,\mathcal{L}}$. 

Given a collection of sets $X$, we use $U_X$ and $I_X$ to denote the union and intersection, respectively, of all elements in $X$. Further, Tab.~\ref{tab:abbreviations} (see Appendix~B) summarizes the meaning of other formalisms used in the paper (many of them introduced at some later point). 

\subsubsection{Assumptions}
\label{sec:assumptions}
The KBD techniques described in this work are applicable to any knowledge representation formalism $\mathcal{L}$ which is \emph{Tarskian}, i.e.\ for which 
the semantic entailment relation $\models_{\mathcal{L}}$ is monotonic, idempotent and extensive \citep{tarski1983logic,ribeiro2012belief} and for which reasoning procedures for \emph{deciding consistency} 
of a KB over $\mathcal{L}$ are available. 
\begin{definition}
The relation $\models_{\mathcal{L}}$ is called
\begin{itemize}[noitemsep, topsep=5pt]
	\item \label{logic:cond1} \emph{monotonic} iff whenever $\mo_{\mathcal{L}} \models_{\mathcal{L}} \tax_{i,\mathcal{L}}$ then $\mo_{\mathcal{L}} \cup \setof{\tax_{k,\mathcal{L}}} \models_{\mathcal{L}} \tax_{i,\mathcal{L}}$ \\(i.e.\ adding new sentences to a KB cannot invalidate any entailments of the KB)
	\item \label{logic:cond2} \emph{idempotent} iff $\mo_{\mathcal{L}} \models_{\mathcal{L}} \tax_{i,\mathcal{L}}$ and $\mo_{\mathcal{L}} \cup \setof{\tax_{i,\mathcal{L}}} \models_{\mathcal{L}} \tax_{k,\mathcal{L}}$ implies $\mo_{\mathcal{L}} \models_{\mathcal{L}} \tax_{k,\mathcal{L}}$ \\ 
	(i.e.\ adding entailed sentences to a KB does not yield new entailments of the KB)
	\item \label{logic:cond3} \emph{extensive} iff $\mo_{\mathcal{L}} \models_{\mathcal{L}} \tax_{\mathcal{L}}$ for all $\tax_{\mathcal{L}} \in \mo_{\mathcal{L}}$ \\ (i.e.\ each KB entails all sentences it comprises).
\end{itemize}
\end{definition}   
In the following, ``sentence'' will always mean ``logical sentence''. 
We will omit the index $\mathcal{L}$ for brevity when referring to sentences or KBs, tacitly assuming that any sentence or  KB we speak of is formulated over some (fixed) language $\mathcal{L}$ where $\mathcal{L}$ meets the conditions given above. 

Examples of logics that comply with these requirements include, but are not restricted to Propositional Logic, Datalog~\citep{Ceri1989a}, (decidable fragments of) First-Order Predicate Logic, The Web Ontology Language (OWL~\citep{patel2004owl}, OWL~2~\citep{Grau2008a,Motik2009a}), sublanguages thereof such as the OWL~2 EL Profile (with polynomial time reasoning complexity \citep{kazakov2014}), Boolean or linear equations and various Description Logics~\citep{DLHandbook} and constraint languages.


\subsubsection{Definitions and Properties} 
\label{sec:defs_and_props}
We next state the KBD problem and give some important definitions and properties (discussed in detail in \citep{Rodler2015phd}).

The inputs to a KB debugging problem can be characterized as follows:
Given is a \emph{KB $\mo$ to be repaired} and a KB $\mb$ (\emph{background knowledge}).  
All sentences in $\mb$ are considered correct and all sentences in $\mo$ are considered potentially faulty. $\mo \cup \mb$ does not meet postulated \emph{requirements} $\RQ$ (where consistency is a least requirement\footnote{We assume consistency a minimal requirement to a solution KB provided by a debugging system, as inconsistency makes a KB completely useless from the semantic point of view.}) 
or does not feature desired semantic properties, called test cases. 
\emph{Positive test cases} (aggregated in the set $\Tp$) correspond to necessary entailments and \emph{negative test cases} (aggregated in the set $\Tn$) represent necessary non-entailments of the correct (repaired) KB (together with the background KB $\mb$). \label{etc:test_cases_are_sets_or_conjuntions_of_formulas} 
Each test case $\tp \in \Tp$ and $\tn \in \Tn$ is \emph{a set of} sentences. The meaning of a positive test case $\tp \in \Tp$ is that the union of the repaired KB and $\mb$ must entail each sentence (or the conjunction of sentences) in $\tp$, whereas a negative test case $\tn \in \Tn$ signalizes that some sentence (or the conjunction of sentences) in $\tn$ must not be entailed by this union.

The described inputs to the KB debugging problem are captured by the notion of a \emph{KBD diagnosis problem instance (KBD-DPI)}:
\begin{definition}[KBD-DPI]\label{def:dpi}
	Let 
	\begin{itemize}[noitemsep,topsep=5pt]
		\item $\mo$ be a KB,
		\item $\Tp, \Tn$ be sets including sets of sentences,
		\item $\RQ \supseteq \setof{\text{consistency}}$ be a set of (logical) requirements,
		\item $\mb$ be a KB such that $\mo \cap \mb = \emptyset$ and $\mb$ satisfies all requirements $r \in \RQ$, and
		\item the cardinality of all sets $\mo$, $\mb$, $\Tp$, $\Tn$ be finite.
	\end{itemize}
	Then we call the tuple $\langle\mo,\mb,\Tp,\Tn\rangle_\RQ$ a \emph{KBD diagnosis problem instance (KBD-DPI)}.
\end{definition}

\begin{table}[t]
	\begin{minipage}{0.45\textwidth} 
		\scriptsize
		\centering
		\rowcolors[]{2}{gray!8}{gray!16} 
		\begin{tabular}{ c c c c } 
			\rowcolor{gray!40}
			\toprule\addlinespace[0pt]
			$i$ & $\tax_i$ & $\mo$ & $\mb$  \\ \addlinespace[0pt]\midrule\addlinespace[0pt]
			1 & $\lnot H \lor \lnot G$ & $\bullet$ & 	\\
			2 & $X \lor F \to H$ & $\bullet$ &  	\\
			3 & $E \to \lnot M \land X$ & $\bullet$ &  	\\
			4 & $A \to \lnot F$ & $\bullet$ &  	\\
			5 & $K \to E$ & $\bullet$ &  	\\
			6 & $C \to B$ & $\bullet$ &  	\\
			7 & $M \to C \land Z$ & $\bullet$ &  	\\
			8 & $H \to A$ &  &  $\bullet$	\\
			9 & $\lnot B \lor K$ &  &  $\bullet$	\\
			\rowcolor{gray!40}
			\toprule\addlinespace[0pt]
			$i$ & \multicolumn{3}{c}{$\tp_i\in\Tp$} \\ \addlinespace[0pt]\midrule\addlinespace[0pt]
			$1$ & \multicolumn{3}{c}{$\setof{\lnot X \to \lnot Z}$} 	\\ \addlinespace[0pt]\toprule\addlinespace[0pt]
			\rowcolor{gray!40}
			$i$ & \multicolumn{3}{c}{$\tn_i\in\Tn$} \\ \addlinespace[0pt]\midrule\addlinespace[0pt]
			$1$ & \multicolumn{3}{c}{$\setof{M \to A}$} 	\\ \addlinespace[0pt]
			$2$ & \multicolumn{3}{c}{$\setof{E \to \lnot G}$} 	\\ \addlinespace[0pt]
			$3$ & \multicolumn{3}{c}{$\setof{F \to L}$} 	\\ \addlinespace[0pt]
			\toprule\addlinespace[0pt]
			\rowcolor{gray!40}
			$i$ & \multicolumn{3}{c}{$r_i\in\RQ$} \\ \addlinespace[0pt]\midrule\addlinespace[0pt]
			$1$ & \multicolumn{3}{c}{consistency} \\ \addlinespace[0pt]\bottomrule
		\end{tabular}
	\caption{\small Running example KBD-DPI $\exdpi$ over Propositional Logic.}
	\label{tab:example_dpi_0}
	\end{minipage}
	\hfill
	\begin{minipage}{0.55\textwidth}
		\scriptsize
		\centering
		\rowcolors[]{2}{gray!8}{gray!16} 
		\begin{tabular}{ c c c} 
			\rowcolor{gray!40}
			\toprule\addlinespace[0pt]
			min KBD-conflict $X$ & $\setof{i \, |\, \tax_i \in X} $ & explanation \\ \addlinespace[0pt]\midrule\addlinespace[0pt]
			$\mc_1$ & $\setof{1,2,3}$ & $\models n_2$ \\
			$\mc_2$ & $\setof{2,4}$ & $\cup \setof{8} \models \lnot F \;(\models n_3)$ \\
			$\mc_3$ & $\setof{2,7}$ & $\cup \setof{p_1,8} \models n_1$\\
			$\mc_4$ & $\setof{3,5,6,7}$ & $\cup \setof{9} \models \lnot M \; (\models n_1)$ \\
			\rowcolor{gray!40}
			\toprule\addlinespace[0pt]
			min KBD-diagnosis $X$ & $\setof{i \, |\, \tax_i \in X} $ & explanation \\ \addlinespace[0pt]\midrule\addlinespace[0pt]
			$\md_1$ & $\setof{2,3}$ & Theorem\ \ref{theorem:mindiag_mincs} \\
			$\md_2$ & $\setof{2,5}$ & Theorem\ \ref{theorem:mindiag_mincs}\\
			$\md_3$ & $\setof{2,6}$ & Theorem\ \ref{theorem:mindiag_mincs}\\
			$\md_4$ & $\setof{2,7}$ & Theorem\ \ref{theorem:mindiag_mincs}\\
			$\md_5$ & $\setof{1,4,7}$ & Theorem\ \ref{theorem:mindiag_mincs}\\
			$\md_6$ & $\setof{3,4,7}$ & Theorem\ \ref{theorem:mindiag_mincs}\\
			\addlinespace[0pt]\bottomrule
		\end{tabular}
	\caption{\small Minimal KBD-conflicts and KBD-diagnoses for the KBD-DPI $\exdpi$ in Tab.~\ref{tab:example_dpi_0}.}
	\label{tab:min_diags+conflicts_example_DPI_0}
	\end{minipage}
\end{table}

\begin{example}\label{ex:dpi}
	An example $\exdpi$ of a Propositional Logic KBD-DPI is depicted by Tab.~\ref{tab:example_dpi_0}. $\exdpi$ will serve as a running example throughout this paper. It includes a KB $\mo$ with seven axioms $\tax_1,\dots,\tax_7$, a background KB $\mb$ with two axioms $\tax_8,\tax_9$, one singleton positive test case $\tp_1$ and three singleton negative test cases $\tn_1,\tn_2,\tn_3$. There is one requirement $r_1 = \emph{consistency}$ in $\RQ$ imposed on the correct (repaired) KB. It is easy to verify that the standalone KB $\mb = \setof{\tax_8,\tax_9}$ is consistent, i.e.\ satisfies all $r \in \RQ$, and that $\mo \cap \mb = \emptyset$. Hence, $\exdpi$ indeed constitutes a KBD-DPI as per Def.~\ref{def:dpi}. 
	\qed
\end{example}

A solution (KB) for a DPI is characterized as follows:
\begin{definition}[Solution KB]\label{def:solution_KB} Let $\dpi := \langle\mo,\mb,\Tp,\Tn\rangle_\RQ$ be a KBD-DPI. Then a KB $\ot$ is called \emph{solution KB w.r.t.\ $\dpi$}
iff all the following conditions hold:
	\begin{eqnarray}
	\forall \, r  \in \RQ&:& \;\ot \cup \mb \,\text{ fulfills }\, r  \label{e:1} \\ 
	\forall \,\tp \in \Tp&:& \;\ot \cup \mb \,\models\, \tp					\label{e:2} \\ 
	\forall \,\tn \in \Tn&:& \;\ot \cup \mb \,\not\models\, \tn .		\label{e:3}  
	\end{eqnarray}
	A solution KB $\ot$ w.r.t.\ $\dpi$ is called \emph{maximal}
	iff there is no solution KB $\mo'$ w.r.t.\ $\dpi$ such that $\mo' \cap \mo \supset \ot\cap\mo$ (i.e.\ $\ot$ has a set-maximal intersection with $\mo$ among all solution KBs). 
\end{definition}
Usually, observing the \emph{Principle of Parsimony} \citep{Reiter87}, maximal solution KBs $\ot$ will be preferred to non-maximal ones since they result from the input KB $\mo$ through the modification of a minimal set of axioms. 
\begin{example}\label{ex:solution_KB}
	For the KBD-DPI $\exdpi$ given by Tab.~\ref{tab:example_dpi_0}, $\mo = \setof{\tax_1,\dots,\tax_7}$ is not a solution KB w.r.t.\ $\langle\mo,\mb,\Tp,\Tn\rangle_\RQ$ since, e.g.\, clearly $\mo \cup \mb = \setof{\tax_1,\dots,\tax_9} \not\models \tp_1$ which is a positive test case and therefore has to be entailed. 
	Another reason why $\mo = \setof{\tax_1,\dots,\tax_7}$ is not a solution KB w.r.t.\ $\exdpi$ is that $\mo \cup \mb \supset \setof{\tax_1,\tax_2,\tax_3} \models \tn_2$, which is a negative test case and hence must not be an entailment. This is straightforward since $\setof{\tax_1,\tax_2,\tax_3}$ imply $E \to X$, $X \to H$ and $H \to \lnot G$ and thus clearly $\tn_2 = \setof{E \to \lnot G}$.
	
	On the other hand, $\mo_a^* := \setof{} \cup \setof{Z \to X}$ is clearly a solution KB w.r.t.\ $\exdpi$ as $\setof{Z \to X} \cup \mb$ is obviously consistent (satisfies all $r\in\RQ$), does entail $\tp_1 \in \Tp$ and does not entail any $\tn_i \in \Tn, (i \in \setof{1,2,3})$. However, $\mo_a^*$ is not a maximal solution KB since, e.g.\, $\tax_5 = (K \to E) \in \mo$ can be added to $\mo_a^*$ without resulting in the violation of any of the Equations \eqref{e:1} -- \eqref{e:3}.
	Note that also e.g.\ $\{\lnot X \to \lnot Z, A_1 \to A_2, A_2 \to A_3, \dots, A_{k-1}\to A_k\}$ for arbitrary finite $k \geq 0$ is a solution KB, albeit not a maximal one, although it has no axioms in common with $\mo$ and includes an arbitrary number of axioms not occurring in $\mo$. 
	However, to maintain a maximum amount of the knowledge specified in the KB $\mo$ of interest, one will usually prefer minimally invasive modifications (i.e.\ maximal solution KBs) while repairing faults in $\mo$.  
	
	Maximal solution KBs w.r.t.\ the given DPI are, e.g.\, $\mo_b^* := \setof{\tax_1,\tax_4,\tax_5,\tax_6,\tax_7,\tp_1}$ (resulting from the deletion of $\setof{\tax_2,\tax_3}$ from $\mo$ and the addition of $\tp_1$) or $\mo_c^* := \setof{\tax_1,\tax_2,\tax_5,\tax_6,\tp_1}$ (resulting from the deletion of $\setof{\tax_1,\tax_4,\tax_7}$ from $\mo$ and the addition of $\tp_1$). That these KBs constitute solution KBs can be verified by checking the three conditions named by Def.~\ref{def:solution_KB}. Indeed, adding an additional axiom in $\mo$ to any of the two KBs leads to the entailment of a negative test case $\tn\in\Tn$. That is, no solution KB can contain a proper superset of the axioms from $\mo$ that are contained in any of the two solution KBs $\mo_b^*$ and $\mo_c^*$. Hence, both are maximal.\qed
\end{example}
\begin{remark}\label{rem:infinitely_many_solutionKBs}
There are generally infinitely many (maximal) solution KBs resulting from the deletion of one and the same set of axioms $\md$ from the original KB $\mo$. 
This stems from the fact that there are infinitely many (semantically equivalent) syntactical variants of any set of suitable sentences that can be added to $\mo \setminus \md$ in order for Eq.~\eqref{e:2} to be satisfied. One reason for this is that there are infinitely many tautologies that might be included in these sentences, another reason is that sentences can be equivalently rewritten, e.g.\ $A \to B \equiv A \to B \lor \lnot A \equiv A \to B \lor \lnot A \lor \lnot A \equiv \dots$.\qed
\end{remark}
In terms of our running example, this circumstance can be illustrated as follows:
\begin{example}
Consider again $\exdpi$ in Tab.~\ref{tab:example_dpi_0} and assume $\md = \setof{\tax_2,\tax_3}$ is deleted from $\mo$. Then one solution KB constructible from $\mo \setminus \md$ is $\mo_b^*$ given in the last example. To determine the maximal solution KB $\mo_b^*$ from $\mo \setminus \md$, the most straightforward way of adding just all sentences occurring in positive test cases in $\Tp$ has been chosen in this case. Other maximal solution KBs obtainable from adding sentences to $\mo \setminus \md$ are, e.g.\, $\mo_{b1}^* := \setof{\tax_1,\tax_4,\tax_5,\tax_6,\tax_7,Z \to X}$ (which differs syntactically, but not semantically from $\mo_b^*$) and $\mo_{b2}^* := \setof{\tax_1,\tax_4,\tax_5,\tax_6,\tax_7,Z \to X \land W}$ (which differs both syntactically and semantically from $\mo_b^*$ yielding the entailment $Z \to W$ which is not implied by $\mo_b^*$).\qed
\end{example} 
%
Despite generally multiple semantically different solution KBs, the diagnostic evidence of a DPI in terms of positive test cases $\Tp$ does not justify the inclusion of sentences (semantically) different from $U_\Tp$ (cf.\ \citep{friedrich2005gdm,Shchekotykhin2012}). Since we are moreover interested in only \emph{one instance} of a solution KB resulting from $\mo \setminus \md$ for each $\md$, we define $\mo \setminus \md \cup U_{\Tp}$ as the \emph{canonical solution KB for $\md$ w.r.t.\ $\dpi$} iff $\mo \setminus \md \cup U_{\Tp}$ is a solution KB w.r.t.\ $\dpi$.
%

A \emph{KBD-diagnosis} is defined in terms of the axioms $\md$ that must be deleted from the KB $\mo$ of a DPI in order to construct a solution KB w.r.t.\ this DPI. In particular, the deletion of $\md$ from $\mo$ targets the fulfillment of Equations \eqref{e:1} and \eqref{e:3} such that $U_{\Tp}$ can be added to the resulting modified KB $\mo \setminus \md$ without introducing any new violations of \eqref{e:1} or \eqref{e:3}.
\begin{definition}[KBD-Diagnosis]\label{def:diagnosis}
	Let $\dpi := \langle\mo,\mb,\Tp,\Tn\rangle_\RQ$ be a KBD-DPI. A set of sentences $\md \subseteq \mo$ is called a \emph{KBD-diagnosis w.r.t.\ $\dpi$}
	iff $(\mo\setminus\md)\cup U_\Tp$ is a solution KB w.r.t.\ $\dpi$ (i.e.\ $\ot := (\mo\setminus\md)\cup U_\Tp$ meets Equations~\eqref{e:1} -- \eqref{e:3}).
	A KBD-diagnosis $\md$ w.r.t.\ $\dpi$ is 
	\begin{itemize}[noitemsep,topsep=5pt]
		\item \emph{minimal}
		iff there is no $\md' \subset \md$ such that 
		$\md'$ is a KBD-diagnosis w.r.t.\ $\dpi$
		\item a \emph{minimum cardinality KBD-diagnosis w.r.t.\ $\dpi$} iff there is no KBD-diagnosis $\md'$ w.r.t.\ $\dpi$ such that $|\md'| < |\md|$.
	\end{itemize}
	We will write $\md \in \allD_{\dpi}$ to state that $\md$ is a KBD-diagnosis w.r.t.\ $\dpi$ and $\md \in \minD_{\dpi}$ to state that $\md$ is a minimal KBD-diagnosis w.r.t.\ $\dpi$.
\end{definition}
\begin{remark}\label{rem:ad_KBD-diagnosis_def}
Since $(\mo\setminus\md)\cup U_\Tp$ trivially satisfies \eqref{e:2} due to the inclusion of $U_\Tp$, $\md$ is a KBD-diagnosis w.r.t.\ $\dpi$ iff $\ot := (\mo\setminus\md)\cup U_\Tp$ satisfies \eqref{e:1} and \eqref{e:3}.\qed
\end{remark}
The next theorem captures the relationship between maximal canonical solution KBs and minimal KBD-diagnoses w.r.t.\ a DPI. In fact, it tells us that we can concentrate only on the computation of minimal KBD-diagnoses in order to find all maximal canonical solution KBs.
\begin{theorem}\label{theorem:relation_between_max-sol-KB_and_min-diagnosis} 
Let $\dpi := \langle\mo,\mb,\Tp,\Tn\rangle_\RQ$ be a KBD-DPI. Then the set of all maximal canonical solution KBs w.r.t.\ $\dpi$ is given by $\setof{(\mo \setminus \md)\cup U_\Tp \mid \md \text{ is a minimal KBD-diagnosis w.r.t.\ }\dpi}$. 
%
\end{theorem} 
In a completely analogous way as MBD-conflicts 
provide an effective mechanism for focusing the search for MBD-diagnoses,
we can exploit KBD-conflicts for KBD-diagnoses calculation. Simply put, a (minimal) KBD-conflict is a (minimal) per se faulty subset of the original KB $\mo$, i.e.\ 
one source causing the faultiness of $\mo$ in the context of $\mb \cup U_\Tp$. For a KBD-conflict there is no extension that yields a solution KB. Instead, such an extension is only possible after deleting appropriate axioms from the KBD-conflict.
\begin{definition}[KBD-Conflict]\label{def:cs} Let $\dpi := \langle\mo,\mb,\Tp,\Tn\rangle_\RQ$ be a KBD-DPI. A set of formulas $\mc \subseteq \mo$ is called a \emph{KBD-conflict w.r.t.\ $\dpi$}
iff $\mc \cup U_{\Tp}$ is not a solution KB w.r.t.\ $\dpi$ (i.e.\ $\ot := \mc \cup U_{\Tp}$ violates at least one of the Equations~\eqref{e:1} -- \eqref{e:3}). A KBD-conflict $\mc$ w.r.t.\ $\dpi$ is \emph{minimal}
iff there is no $\mc' \subset \mc$ such that $\mc'$ is a KBD-conflict w.r.t.\ $\dpi$.
\end{definition}
\begin{theorem}\citep[Prop.~2]{friedrich2005gdm}\label{theorem:mindiag_mincs}
Let $\dpi$ be a KBD-DPI. Then a (minimal) KBD-diagnosis w.r.t.\ $\dpi$ is a (minimal) hitting set of all minimal conflicts w.r.t.\ $\dpi$.
\end{theorem}
\begin{proposition}\citep[Prop.~3.4]{Rodler2015phd}\label{prop:diagn_exists_iff}
Let $\dpi := \langle\mo,\mb,\Tp,\Tn\rangle_\RQ$ be a KBD-DPI. Then a KBD-diagnosis w.r.t.\ $\dpi$ exists iff $\mb \cup U_\Tp$ satisfies all $r \in \RQ$ and $\mb \cup U_\Tp \not\models \tn$ for all $\tn \in \Tn$.
\end{proposition}
\begin{example}\label{ex:min_conflict_sets}
%
%
Tab.~\ref{tab:min_diags+conflicts_example_DPI_0} gives a list of all minimal KBD-conflicts w.r.t.\ our running example $\exdpi$.
Let us briefly reflect why these are KBD-conflicts (cf.\ third col.\ of Tab.~\ref{tab:min_diags+conflicts_example_DPI_0}). 
Recall Ex.~\ref{ex:solution_KB}, where we 
explained
why $\mc_1$ is a KBD-conflict (violation of $\tn_2 \in \Tn$).
$\mc_1$ is minimal since, first, it is consistent, i.e.\ satisfies all $r \in \RQ$, and does not entail any of the negative test cases $\tn_1, \tn_3$. So, by logical monotonicity no proper subset of $\mc_1$ can violate $r$, $\tn_1$ or $\tn_3$. Second, the elimination of any axiom $\tax_i (i\in \setof{1,2,3})$ from $\mc_1$ breaks the entailment of the negative test case $\tn_2$. 
 
Regarding $\mc_2 := \setof{\tax_2,\tax_4}$, we have that (any superset of) $\mc_2$ is a KBD-conflict due to (the monotonicity of Propositional Logic and) the fact that $\tax_2 \equiv \setof{X \to H, F \to H}$ together with $\tax_8 (\in \mb) = H \to A$ and $\tax_4 = A \to \lnot F$ clearly yields $F \to \lnot F \equiv \lnot F$ which, in particular, implies $n_3 = \setof{F \to L} \equiv \setof{\lnot F \lor L}$. 

$\mc_3$ is a minimal KBD-conflict since it is a $\subseteq$-minimal subset of the KB $\mo$ which, along with $\mb$ and $U_{\Tp}$ (in particular with $\tax_8 \in \mb$ and $\tp_1 \in \Tp$), implies that $\tn_1 \in \Tn$ must be true. To see this, realize that $\tax_7 \models M \to Z$, $\tp_1 = Z \to X$, $\tax_2 \models X \to H$ and $\tax_8 = H \to A$, from which $\tn_1 = \setof{M \to A}$ follows in a straightforward way.

Finally, $\mc_4$ is a KBD-conflict since $\tax_7 \models M\to C$, $\tax_6 = C \to B$, $\tax_9 \equiv B \to K$, $\tax_5 = K \to E$ and $\tax_3 \models E \to \lnot M$. Again, it is now obvious that this chain yields the entailment $\lnot M$ which in turn entails $\setof{\lnot M \lor A} \equiv \setof{M \to A} = \tn_1$. Clearly, the removal of any axiom from this chain breaks the entailment $\lnot M$. As this chain is neither inconsistent nor implies any negative test cases other than $\tn_1$, the conflict $\mc_4$ is also minimal. 
It is not very hard to verify that there are no other minimal KBD-conflicts w.r.t.\ $\exdpi$ apart from $\mc_1,\dots,\mc_4$.
\qed 
\end{example}

\begin{example}\label{ex:diag_is_hitting_set_of_conflict}
The set $\minD_{\exdpi}$ of all minimal KBD-diagnoses w.r.t.\ $\exdpi$ (Tab.~\ref{tab:example_dpi_0}) is shown in Tab.~\ref{tab:min_diags+conflicts_example_DPI_0}.
Theorem~\ref{theorem:mindiag_mincs} and the illustration (given in Ex.~\ref{ex:min_conflict_sets}) of why $\mc_1,\dots,\mc_4$ constitute a complete set of minimal KBD-conflicts w.r.t.\ $\exdpi$ provide the explanation for $\minD_{\exdpi}$.
%
%
%
For instance, $\md_1 = \setof{\tax_2,\tax_3}$ ``hits'' the element $\tax_2$ of $\mc_i (i \in \setof{1,2,3})$ and the element $\tax_3$ of $\mc_4$. Note also that it hits two elements of $\mc_1$ which, however, is not necessarily an indication of the non-minimality of the hitting set. Indeed, if $\tax_2$ is deleted from $\md_1$, it has an empty intersection with $\mc_2$ and $\mc_3$ and, otherwise, if $\tax_3$ is deleted from it, it becomes disjoint with $\mc_4$. Hence $\md_1$ is actually a minimal hitting set of all minimal KBD-conflicts.
%
\qed
\end{example}
The relationship between the notions \emph{KBD-diagnosis}, \emph{solution KB} and \emph{KBD-conflict} is as follows (cf.\ \citep[Cor.~3.3]{Rodler2015phd}):
\begin{proposition}\label{prop:notions_equiv} 
	Let $\md \subseteq \mo$. Then the following statements are equivalent:
	\begin{enumerate}[noitemsep,topsep=5pt]
		\item $\md$ is a KBD-diagnosis w.r.t.\ $\langle\mo,\mb,\Tp,\Tn\rangle_\RQ$.
		\item $(\mo \setminus \md) \cup U_\Tp$ is a solution KB w.r.t.\ $\langle\mo,\mb,\Tp,\Tn\rangle_\RQ$.
		\item $(\mo \setminus \md)$ is not a KBD-conflict w.r.t.\ $\langle\mo,\mb,\Tp,\Tn\rangle_\RQ$.
	\end{enumerate}
\end{proposition}
\begin{example}\label{ex:notions_equiv}
	Since, e.g., $\mo\setminus\md := \setof{\tax_1,\tax_2}$ is not a KBD-conflict w.r.t.\ $\exdpi$ (Tab.~\ref{tab:example_dpi_0}), we obtain that $\md = \mo \setminus (\mo\setminus\md) = \setof{\tax_1,\ldots,\tax_7} \setminus \setof{\tax_1,\tax_2} = \setof{\tax_3,\dots,\tax_7}$ is a KBD-diagnosis w.r.t.\ $\exdpi$, albeit not a minimal one ($\tax_5$ and $\tax_6$ can be deleted from it while preserving its KBD-diagnosis property). 
	Further on, $(\mo \setminus \md) \cup U_\Tp = \setof{\tax_1,\tax_2,\tp_1}$ must be a solution KB w.r.t.\ $\exdpi$.
	\qed
\end{example}

\subsection{Reducing Reiter's MBD Problem to KB Debugging}
\label{sec:reduction_of_mbd_to_kbd}
We next demonstrate that the classical MBD problem described in Sec.~\ref{sec:mbd} can be reduced to the KBD problem explicated in Sec.~\ref{sec:kbd} \citep{rodler17dx_reduction}. That is, any MBD-DPI can be modeled as a KBD-DPI, and the solutions of the latter directly yield the solutions of the former.
%

\begin{theorem}[Reduction of MBD to KBD]\label{theorem:reduction_of_MBD_to_KBD}
	Let $\mdpi :=(\sd,\comps,\obs,\meas)$ be an MBD-DPI where $\comps = \setof{c_1,\dots,c_n}$. Then, $\mdpi$ can be formulated as a KBD-DPI $\kdpi$ such that there is a bijective correspondence between KBD-diagnoses for $\kdpi$ and MBD-diagnoses for $\mdpi$. Moreover, all MBD-diagnoses for $\mdpi$ can be computed from the KBD-diagnoses for $\kdpi$. 
	%
\end{theorem}
\begin{proof}
	We first show how $\mdpi$ can be formulated as a KBD-DPI $\kdpi$. To this end, we specify how $\kdpi = \langle\mo,\mb,\Tp,\Tn\rangle_{\RQ}$ can be written in terms of the components of $\mdpi = (\sd_{\mathit{beh}}\cup\sd_{\mathit{gen}},\comps,\obs,\meas)$:
	\begin{align}
	\mo &= \setof{\tax_{i}\mid \tax_{i} := beh(c_i), c_i \in \comps}  \label{eq:MBDtoKBD_K} \\
	\mb &=   \obs\cup\sd_{\mathit{gen}} \label{eq:MBDtoKBD_B}\\ 
	\Tp &=   \meas    \label{eq:MBDtoKBD_P}\\ 
	\Tn &= \emptyset  \label{eq:MBDtoKBD_N} \\
	\RQ &= \setof{\mathit{consistency}} \label{eq:MBDtoKBD_R}
	\end{align}
	That is, $\mo$ captures $\sd_{\mathit{beh}} \cup \setof{\lnot\ab(c_i) \mid c_i \in \comps}$, i.e.\ the nominal behavioral descriptions of all system components.
	By Def.~\ref{def:diagnosis} and Remark~\ref{rem:ad_KBD-diagnosis_def}, $\md \subseteq \mo$ is a KBD-diagnosis for $\kdpi$ iff 
	\begin{align}
	(\mo\setminus\md)\cup\mb\cup U_\Tp &\text{ satisfies all } r \in \RQ \quad\text{(i.e.\ is consistent)} 
	\label{eq:MBDtoKBD_diag_cond_1} 
	\end{align}
	and
	\begin{align}
	(\mo\setminus\md)\cup\mb\cup U_\Tp &\not\models \tn \text{ for all } \tn \in \Tn   \label{eq:MBDtoKBD_diag_cond_2}
	\end{align}
	Let now $\md$ be an arbitrary KBD-diagnosis for $\kdpi$ such that $\md = \setof{\tax_i \mid i \in I}$ for the index set $I \subseteq \setof{1,\dots,n}$. 
	
	Using \eqref{eq:MBDtoKBD_K} -- \eqref{eq:MBDtoKBD_R} above, condition \eqref{eq:MBDtoKBD_diag_cond_1} for $\md$ is equivalent to the consistency of $\sd_{\mathit{beh}} \cup \{\textsc{ab}(c_i) \mid i \in I\} \cup \setof{\lnot\textsc{ab}(c_i) \mid i \in \setof{1,\dots,n} \setminus I} \cup \obs \cup \sd_{\mathit{gen}} \cup U_{\meas}$ which in turn yields that 
	\begin{align}
	\sd \cup \setof{\textsc{ab}(c_i) \mid c_i \in \dg} \cup \setof{\lnot\textsc{ab}(c_i) \mid c_i \in \comps \setminus \dg} \cup \obs \cup U_{\meas} \text{ is consistent}   \label{eq:MBDtoKBD_MBDdiag_cond}
	\end{align}
	for $\dg := \setof{c_i \mid c_i \in \comps, i \in I}$. But, \eqref{eq:MBDtoKBD_MBDdiag_cond} is exactly the condition defining an MBD-diagnosis (see Def.~\ref{def:MBD-diagnosis}). 
	Note, since $\Tn = \emptyset$ by \eqref{eq:MBDtoKBD_N}, condition \eqref{eq:MBDtoKBD_diag_cond_2} is satisfied for any $\md$ satisfying \eqref{eq:MBDtoKBD_diag_cond_1} and can thus be neglected. Hence, $\md = \setof{\tax_i \mid i \in I} \subseteq \mo$ is a KBD-diagnosis w.r.t. $\kdpi$ iff $\dg = \setof{c_i \mid c_i \in \comps, i \in I} \subseteq \comps$ is an MBD-diagnosis for $\mdpi$.
\end{proof}
Also, there is a bijective correspondence between KBD-conflicts and MBD-conflicts:
\begin{proposition}\label{prop:kbd-conflict_1-to-1_mbd-conflict}
Let $\mdpi = (\sd,\comps,\obs,\meas)$ be an MBD-DPI and $\kdpi = \langle\mo,\mb,\Tp$, $\Tn\rangle_{\RQ}$ a KBD-DPI modeling $\mdpi$ as per \eqref{eq:MBDtoKBD_K} -- \eqref{eq:MBDtoKBD_R}. Further, let $\comps = \setof{c_1,\ldots,c_n}$ and $I \subseteq \setof{1,\ldots,n}$. Then, $C = \setof{c_i \mid c_i \in \comps, i \in I} \subseteq \comps$ is an MBD-conflict for $\mdpi$ iff $\mc = \setof{\tax_i \mid i \in I} \subseteq \mo$ is a KBD-conflict w.r.t.\ $\kdpi$.
\end{proposition}
\begin{proof}
$\mc$ is a KBD-conflict w.r.t.\ $\kdpi$ iff $\mo \setminus \mc = \setof{\tax_i \mid i \in \setof{1,\dots,n}\setminus I}$ is not a KBD-diagnosis w.r.t.\ $\kdpi$ (Prop.~\ref{prop:notions_equiv}) iff $\setof{c_i \mid c_i \in \comps, i \in \setof{1,\dots,n}\setminus I}$ is not an MBD-diagnosis for $\mdpi$ (Theorem~\ref{theorem:reduction_of_MBD_to_KBD}) iff $\setof{c_i \mid c_i \in \comps, i \in I} = C$ is an MBD-conflict for $\mdpi$ (\citep[Prop.~4.2]{Reiter87}).	
%
\end{proof}
Let us exemplify these theoretical results:
\begin{example}\label{ex:reduction_MBD_to_KBD}
Reconsider the circuit diagnosis example (Fig.~\ref{fig:circuitreiter87}). The formalization of the circuit problem as an MBD-DPI $\exdpiM$ was discussed in Ex.~\ref{ex:circuit_MBD-DPI_diags_conflicts_measurements}. The formulation of this MBD-DPI as a KBD-DPI $\exdpiMK$ as per Eq.~\eqref{eq:MBDtoKBD_K} -- \eqref{eq:MBDtoKBD_R} is depicted by Tab.~\ref{tab:circuitreiter87_KBD-DPI}. All minimal KBD-conflicts and their minimal hitting sets, i.e.\ the minimal KBD-diagnoses (Theorem~\ref{theorem:mindiag_mincs}), are given in the lower part of Tab.~\ref{tab:circuitreiter87_KBD-DPI}.	For instance, $\mc = \setof{\tax_1,\tax_4,\tax_5}$ is a KBD-conflict w.r.t.\ $\exdpiMK$ since $\mc \cup \mb \cup U_{\Tp} \models \bot$. We briefly sketch why this holds. $\tax_{13}(\in \mb) = (in1(X_1)=1)$, 
$\tax_{14}(\in \mb) = (in2(X_1)=0)$ and 
$\tax_{1} = (out(X_1)=\mathit{xor}(in1(X_1),in2(X_1)))$
imply that $out(X_1)=\mathit{xor}(1,0)=1$, which, along with
$\tax_{6}(\in \mb) = (out(X_1) = in2(A_2))$, entails $in2(A_2)=1$, which in turn, together with
$\tax_{15}(\in \mb) = (in1(A_2)=1)$ and  
$\tax_{4} = (out(A_2)=\mathit{and}(in1(A_2),in2(A_2)))$, lets us deduce that $out(A_2) = \mathit{and}(1,1)=1$.
Because of $\tax_{8}(\in \mb) = (out(A_2) = in1(O_1))$ we have that $in1(O_1) = 1$ which yields $out(O_1)=\mathit{or}(1,in2(O_1)) = 1$ 
due to 
$\tax_{5} = (out(O_1)=\mathit{or}(in1(O_1),in2(O_1)))$. However, $\tax_{17} \in \mb$ states that $\mathit{out}(O_1)=0$, a contradiction.

$\mc$ is minimal since all elements of $\mc$ were necessary to derive the outlined contradiction. In fact, no proper subset of $\mc$ can be used to deduce any negative test case (trivially, as the set $\Tn$ is empty) or any contradiction (possibly different from the one given above). Intuitively, the latter holds since any $\mc' \subset \mc$ includes too few behavioral descriptions of components so that there is no ``open'' path for constraint propagation from inputs to outputs of the circuit. $\mc$, on the other hand, enables to propagate information from all three inputs via gates $X_1$, $A_2$ and $O_1$ towards the second output. What becomes nicely evident at this point is the principle of transformation between MBD and KBD. Whereas in MBD behavioral descriptions of components are ``disabled'' via abnormality assumptions about components, in KBD it is exactly these descriptions that make up the KB, and they are ``inactivated'' by just deleting them from the KB. 

The justification for the minimal KBD-conflict $\setof{\tax_1,\tax_2}$ follows essentially the same argumentation as was given in Ex.~\ref{ex:circuit_MBD-DPI_diags_conflicts_measurements} to explain $C_1$.\qed
\end{example} 

To sum up, we can find all diagnoses for any MBD-DPI by representing it as a KBD-DPI and solving the KBD-DPI (Theorem~\ref{theorem:reduction_of_MBD_to_KBD}). Thus, 
KBD methods \citep{Shchekotykhin2012,Rodler2015phd} are suitable for MBD as well.
Moreover, 
computing all minimal diagnoses for KBD-DPIs 
leads us to all maximal (canonical) solution KBs w.r.t.\ the DPI in a trivial way (Theorem~\ref{theorem:relation_between_max-sol-KB_and_min-diagnosis}). 

Due to these results we can henceforth w.l.o.g.\ restrict our focus to KBD problems and the computation of minimal diagnoses w.r.t.\ these problems. However, we bear in mind that the presented methods apply to MBD problems as well and the obtained solutions can be easily reformulated as solutions for knowledge-based system debugging.
 
Hence, whenever we will write \emph{DPI}, \emph{diagnosis} and \emph{conflict} in the rest of this work, we will refer to \emph{KBD-DPI}, \emph{KBD-diagnosis} and \emph{KBD-conflict}, respectively. 
The problem of \emph{Sequential Diagnosis}, which will 
generalize
the \emph{Sequential MBD-Problem} as per Prob.~\ref{prob_def:sequential_MBD}, will be discussed in detail in the next section.

\subsection{Sequential Diagnosis}
\label{sec:sequential_diagnosis}
Given multiple diagnoses for a DPI, sequential diagnosis techniques \citep{dekleer1987,DBLP:conf/ijcai/BrodieRMO03,DBLP:journals/jair/FeldmanPG10a,Siddiqi2011,Rodler2013,Shchekotykhin2014,Rodler2015phd} target the acquisition of additional information to minimize the diagnostic uncertainty, i.e.\ to reach a predefined diagnostic goal $G$ (cf.\ Remark~\ref{rem:leading_diags__diagnostic_goal} for some examples). Depending on the sequential diagnosis framework, different types of information might be incorporated. For example, the framework used by \citep{DBLP:conf/ijcai/BrodieRMO03,shchekotykhin2016efficient} tests (sets of) components 
\emph{directly} and takes the information about their normal/abnormal state into account. On the other hand, the approaches of \citep{dekleer1987,Siddiqi2011} \emph{indirectly} measure values of variables influenced by the normal/abnormal behavior of components. As opposed to these \emph{probing} techniques, \emph{testing} approaches \citep{pattipati1990,DBLP:journals/tsmc/ShakeriRPP00,DBLP:journals/jair/FeldmanPG10a} observe particular system outputs after varying particular system inputs, i.e.\ the gathered information in this paradigm corresponds to input-output vectors. 

Our approach uses a way of information representation that, in principle, allows to model all aforementioned paradigms (see Ex.~\ref{ex:query_representation}). Namely, we 
define a proposed measurement generally as a set of sentences (over some logic complying with the criteria given in Sec.~\ref{sec:assumptions}), according to \citep{Reiter87}. 
We call a proposed measurement a \emph{query} \citep{settles2012} 
if the additional information it gives eliminates in any case at least one (known) diagnosis \citep{Shchekotykhin2012,Rodler2015phd}. 
Further on, we assume an entity, called \emph{oracle}, capable of performing the required measurements. That is, an oracle answers queries by assessing the correctness of the sentences in the query.
When diagnosing physical systems \citep{dekleer1987,Reiter87,heckerman1995decision}, the oracle might be constituted by a human operator or automatic sensors 
making observations. For instance, when diagnosing a car, a car mechanic might act as an oracle. During the diagnosis of knowledge-based systems \citep{Rodler2015phd} such as configuration systems \citep{DBLP:journals/ai/FelfernigFJS04} or ontologies \citep{Shchekotykhin2012}, the oracle could be a domain expert or some automatic information extraction system providing domain-specific knowledge. 
%

Given a query $Q = \setof{\tax_1,\ldots,\tax_k}$ containing the sentences\footnote{We could also w.l.o.g.\ define a query 
to be a \emph{single} logical sentence because it is interpreted as the conjunction of the sentences it contains, which is simply a ``bigger'' sentence. For technical reasons, we stick to the representation as a set of sentences (cf.\ \citep{Reiter87}), since we will present query minimization approaches for reducing the number of sentences in the query. This would correspond to reducing the length or complexity of the sentence in the single sentence interpretation of a query.} $\tax_1,\ldots,\tax_k$, posing $Q$ to the oracle means asking whether $\bigwedge_{i=1}^{k}\tax_i$ must be $\true$, or equivalently, whether each single sentence $\tax_i \in Q$ must be $\true$. Hence, a query is answered by $\true$ ($t$) if the performed measurements confirm all sentences in $Q$, and by $\false$ if the measurements disprove some sentence(s) in $Q$. Depending on the concrete diagnosis task at hand, queries are answered w.r.t.\ different \emph{reference points}. For instance, in the KB debugging domain, the desired model of the domain of interest, i.e.\ the correct KB, is the relevant reference point. That is, measurements in this case might correspond to cognitive activity (of a domain expert thinking about the truth of the sentences in $Q$) or the process of information extraction (of e.g.\ some system browsing some knowledge source relevant to $Q$). On the other hand, when diagnosing some physical device, the reference point is constituted by the actual behavior of the device. In this case a measurement is the observation of some system aspect(s) relevant to $Q$.
So, given a reference point $\mathit{Ref}$, a positive answer to the query $Q$ means that $\mathit{Ref} \models Q$, a negative one that $\mathit{Ref} \not\models Q$. 

In the sequential diagnosis process, the information provided by answered queries is incorporated into the current DPI, yielding a new (updated) DPI. In particular, a positively answered query $Q$ is added as a positive test case to the current DPI $\tuple{\mo,\mb,\Tp,\Tn}_\RQ$ resulting in the new DPI $\tuple{\mo,\mb,\Tp\cup\setof{Q},\Tn}_\RQ$. Likewise, a negatively answered query $Q$ is added as a negative test case to the current DPI $\tuple{\mo,\mb,\Tp,\Tn}_\RQ$ resulting in the new DPI $\tuple{\mo,\mb,\Tp,\Tn\cup\setof{Q}}_\RQ$. In this vein, the successive addition of new 
answered queries to the test cases
gradually reduces the diagnostic uncertainty by restricting the set of diagnoses.
Note, if an oracle is able to provide any additional information, sentence(s) $Y$, beyond the mere query answer, e.g., an explanation or justification for a negative query answer, the presented approach enables to integrate and exploit this information for the invalidation of further diagnoses. To this end, $Y$ is simply added to the set $\Tp$ as a positive test case.

\subsubsection{Definitions and Properties}
\label{sec:query:definition_and_properties}
We now present the concept of a query in more formal terms. 
In the following, given a DPI $\dpi:=\langle\mo,\mb,\Tp,\Tn\rangle_\RQ$ and some minimal diagnosis $\md_i$ w.r.t.\ $\dpi$, we will use the following abbreviation for the canonical solution KB obtained by deletion of $\md_i$ along with the given background knowledge $\mb$:
\begin{align} 
\mo^{*}_i \; := \; (\mo \setminus \md_i) \cup \mb \cup U_\Tp \label{eq:sol_ont_candidate} 
\end{align}

\begin{proposition}\label{prop:partition}
Let $\dpi:=\langle\mo,\mb,\Tp,\Tn\rangle_\RQ$ be a DPI, $X$ be a set of sentences and $\mD \subseteq \minD_{\dpi}$. Then $X$ induces a partition $\Pt_\mD(X) := \tuple{\dx{}(X),\dnx{}(X),\dz{}(X)}$ on $\mD$ where\footnote{We will often say ``$X$ violates $\RQ$ or $\Tn$'' to state that $(\exists \tn \in \Tn: X \models \tn) \lor (\exists r \in \RQ: X \text{ violates } r)$.} 
\begin{align*}
\dx{}(X) &:= \{\md_i \in \mD \mid \mo^{*}_i \models X\} \\
\dnx{}(X) &:= \{\md_i \in \mD \mid (\exists \tn \in \Tn: \mo^{*}_i \cup X \models \tn) \lor (\exists r \in \RQ: \mo^{*}_i \cup X \text{ violates } r)\} \\
\dz{}(X) &:= \mD \setminus (\dx{}(X) \cup \dnx{}(X))
\end{align*}
\end{proposition}
Since the computation of all (minimal) diagnoses is computationally prohibitive in general, we exploit a subset $\mD$ of all minimal diagnoses w.r.t.\ a DPI for measurement selection. $\mD$ is referred to as the \emph{leading diagnoses} (cf.\ Rem.~\ref{rem:leading_diags__diagnostic_goal}). 
From a query we postulate two properties. It must for any outcome \textbf{(1)}~invalidate at least one (leading) diagnosis (\emph{search space restriction}) and \textbf{(2)}~preserve the validity of at least one (leading) diagnosis (\emph{solution preservation}). 
In fact, the sets $\dx{}(X)$ and $\dnx{}(X)$ are the key in deciding whether a set of sentences $X$ is a query or not. 
Based on Prop.~\ref{prop:partition}, we define:
\begin{definition}[Query, q-Partition]\label{def:query_q-partition}
Let $\dpi:=\langle\mo,\mb,\Tp,\Tn\rangle_\RQ$ be a DPI, $\mD \subseteq \minD_{\langle\mo,\mb,\Tp,\Tn\rangle_\RQ}$ be the leading diagnoses and $Q$ be a set of sentences with $\Pt_\mD(Q) = \tuple{\dx{}(Q),\dnx{}(Q),\dz{}(Q)}$. Then $Q$ is a \emph{query w.r.t.\ $\mD$} 
iff $Q\neq \emptyset$, $\dx{}(Q) \neq \emptyset$ and $\dnx{}(Q) \neq \emptyset$. We denote the set of all queries w.r.t.\ $\mD$ by $\mQ_{\mD}$. Further, we refer to the set of those $Q \in \mQ_\mD$ with $\dz{}(Q) = \emptyset$ by $\mQ_{\mD}^{\bcancel{0}}$.

$\Pt_\mD(Q)$ is called \emph{the q-partition of $Q$} (or: \emph{a q-partition}) iff $Q$ is a query. Inversely, $Q$ is called \emph{a query with (or: for) the q-partition} $\Pt_\mD(Q)$.
Given a q-partition $\Pt$, we sometimes denote its three entries in turn by $\dx{}(\Pt)$, $\dnx{}(\Pt)$ and $\dz{}(\Pt)$.\footnote{In existing literature, e.g.\ \citep{Shchekotykhin2012,Rodler2013,ksgf2010}, a \emph{q-partition} is often simply referred to as \emph{partition}. We call it q-partition to emphasize that not each partition of $\mD$ into three sets is necessarily a q-partition.} 
\end{definition}
Given the formal definition of a query, the \emph{oracle} is formally defined as function $\oracle: \mQ_\mD \to \setof{t,f}$ which outputs an answer $\oracle(Q)$ for $Q \in \mQ_\mD$.

$\dx{}(Q)$ and $\dnx{}(Q)$ denote those diagnoses in $\mD$ consistent only with $Q$'s positive and negative outcome, respectively, and $\dz{}(Q)$ those consistent with both outcomes. In other words, given the prior DPI $\dpi:=\langle\mo,\mb,\Tp,\Tn\rangle_\RQ$ and leading diagnoses $\mD \subseteq \minD_{\dpi}$, then the posterior still valid diagnoses from $\mD$ w.r.t.\ $\tuple{\mo,\mb,\Tp\cup\setof{Q},\Tn}_\RQ$ (i.e.\ after adding $Q$ to the positive test cases) are those in $\dx{}(Q) \cup \dz{}(Q)$. The posterior still valid subset of $\mD$ w.r.t.\ $\tuple{\mo,\mb,\Tp,\Tn\cup\setof{Q}}_\RQ$ (i.e.\ after adding $Q$ to the negative test cases) is $\dnx{}(Q) \cup \dz{}(Q)$. We also say that diagnoses in $\dx{}(Q)$ \emph{predict $Q$'s positive answer}, those in $\dnx{}(Q)$ \emph{predict $Q$'s negative answer}, and those in $\dz{}(Q)$ \emph{do not predict any of $Q$'s answers}.
Since $Q \in \mQ_{\mD}$ (Def.~\ref{def:query_q-partition}) implies that both $\dx{}(Q)$ and $\dnx{}(Q)$ are non-empty, clearly $Q$'s outcomes both dismiss and preserve at least one diagnosis, as postulated. 
 
Of course, in many cases a query will also invalidate some (unknown) non-leading diagnoses in $\allD_{\dpi} \setminus \mD$. In fact, each query $Q \in \mQ_\mD$ is necessarily also a query w.r.t.\ \emph{all minimal} diagnoses, i.e.\ $Q \in \mQ_{\minD_{\dpi}}$, and a query w.r.t.\ \emph{all} diagnoses, i.e.\ $Q \in \mQ_{\allD_{\dpi}}$.
However, there might be sets of sentences $X \notin \mQ_{\mD}$ which are in fact queries w.r.t.\ a different (e.g.\ a larger) set of leading diagnoses $\mD' \neq \mD$, i.e.\ $X \in \mQ_{\mD'}$ .
Still, we point out that, facing the general intractability of the computation of multiple diagnoses, the best we can do is using the currently given evidence in terms of the leading diagnoses $\mD$ to differentiate between queries $Q \in \mQ_\mD$ (\emph{definitely discriminating} among all diagnoses) and non-queries $Q' \notin \mQ_\mD$ (\emph{potentially non-discriminating} among all diagnoses).

\label{etc:discussion_why_D0=emptyset}As the set $\dz{}(Q)$ comprises those diagnoses that cannot be eliminated given any of $Q$'s outcomes, queries with non-empty $\dz{}(Q)$ have a weaker discrimination power than others. The reason is that they discriminate only among the (leading) diagnoses $\mD \setminus \dz{}(Q)$. Therefore, one will usually try to focus on queries with a set $\dz{}(Q)$ as minimal in size as possible.\footnote{In fact, one \emph{can} construct examples where a query $Q$ with $\dz{}(Q) \neq \emptyset$ is better 
(w.r.t.\ to some query goodness measure $m$) than another one, $Q'$, with $\dz{}(Q')=\emptyset$. E.g., let $m$ be the entropy measure \citep{dekleer1987} and $p$ be a probability measure, then one such example is $Q$ with $\tuple{p(\dx{}(Q)),p(\dnx{}(Q)),p(\dz{}(Q))}=\tuple{0.49,0.49,0.02}$ and $Q'$ with $\tuple{p(\dx{}(Q')),p(\dnx{}(Q')),p(\dz{}(Q'))}=\tuple{0.99,0.01,0}$. Nevertheless, first, in practice, given that $Q$ is in $\mQ_\mD$, the existence of a query $Q'' \in \mQ_\mD$ with small $|p(\dx{}(Q''))-p(\dx{}(Q))|$ and small $|p(\dnx{}(Q''))-p(\dnx{}(Q))|$ as well as $p(\dz{}(Q'')) = 0$ is likely (e.g.\ by making $Q$ logically stronger, cf.\ \citep[Chap.~8]{Rodler2015phd}). Second, we need to compare the \emph{best} query with empty $\dz{}$ with the \emph{best} query with non-empty $\dz{}$ (as we will present a search that finds the best query among those with empty $\dz{}$). Except for very small search spaces (where brute force methods considering all queries are anyway practical), a query like $Q'$ will most probably not be the best query with empty $\dz{}$. Third,  
the query space $|\mQ_\mD|$ is normally large enough to ensure the existence of (even multiple) very close-to-optimal queries as per some measure $m$ even though those with non-empty $\dz{}$ are neglected (cf.~Sec.~\ref{sec:eval}).} In fact, \citep[p.~107 ff.]{Rodler2015phd} shows that it is always possible to enforce queries with empty $\dz{}(Q)$ by making the comprised sentences sufficiently strong (in logical terms). Our new method presented in Sec.~\ref{sec:contribution} guarantees the computation of only $Q$'s with $\dz{}(Q) = \emptyset$. On the one hand, this involves a focus on the promising query candidates (in that better discrimination among leading diagnoses lets us expect better discrimination among all diagnoses). \label{etc:QP_search_space_size}On the other hand, it reduces the query (or more precisely: the q-partition) search space from $O(3^{|\mD|})$ (all 3-partitions of $|\mD|$ diagnoses) to $O(2^{|\mD|})$ (all 3-partitions of $\mD$ with one of the 3 sets, i.e.\ $\dz{}$, being empty). For example, the methods of   \citep{dekleer1987,Shchekotykhin2012,Rodler2013} do not 
ensure these properties.
%
%
%
\begin{example}\label{ex:query_representation}
	Consider the electronic circuit in Fig.~\ref{fig:circuitreiter87}. We exemplify how queries consisting of (e.g.\ First-Order Logic) sentences can be used to model \emph{direct (component) probing}, \emph{indirect probing} and \emph{testing}. 
	
	\emph{Direct Probing:} A query representing a direct test of a component, say $X_1$, would be represented as $Q = \setof{beh(X_1)} = \{out(X_1) = \mathit{xor}(in1(X_1),in2(X_1))\}$ (cf.\ Tab.~\ref{tab:circuitreiter87_MBD-DPI} and Tab.~\ref{tab:circuitreiter87_KBD-DPI}), i.e.\ ``Does component $X_1$ work as expected?''. Such a direct test of a component might, depending on the application, involve visible, tangible or audible inspection, component examination using specialized tools, a check of operation logs for the component, etc.   
	For instance, given a car that does not start, a direct component probe could involve testing whether the battery is working or dead using a battery test device. 
	
	\emph{Indirect Probing:} An indirect test of gates $X_1$ and $A_1$ could be formulated as a query $Q = \setof{out(A_2)=1}$. The reason why these two gates are implicitly tested by answering $Q$ is that these are the only gates influencing whether the tested wire between $A_2$ and $O_1$ is high or low (cf.\ Fig.~\ref{fig:circuitreiter87}). 
	Note, if $Q$ is answered negatively, this tells us that at least one component among $\setof{X_1,A_2}$ is defective. But if $Q$ is positively answered, this gives us no definite information about the state of any of the two components (\emph{weak fault model}). That is, both could be nominal, any single one could be flawed, e.g.\ stuck-at-1, or both could be abnormal, e.g.\ $X_1$ stuck-at-0 and $A_2$ stuck-at-1.
	In the car example, an indirect test of the battery charge (and possibly some other components) could involve a test of the car's ignition.
	
	\emph{Testing:} Let us say we want to acquire diagnostic information by experimenting with various inputs and observing the resulting outputs of the circuit. A query testing whether the desired output $(0,0)$ results from the given input $(0,0,0)$, would be of the form $Q = \{(in1(X_1)=0 \land in2(X_1) = 0 \land in1(A_2) = 0) \to (out(X_2)=0 \land out(O_1)=0)\}$.\qed    
\end{example}

\begin{example}\label{ex:queries_wrt_leading_diags}
Let us consider our running example DPI $\exdpi = \langle\mo,\mb,\Tp,\Tn\rangle_\RQ$ (Tab.~\ref{tab:example_dpi_0}) let us further suppose that some diagnosis computation algorithm provides the set of leading diagnoses $\mD = \{\md_1,\md_2,\md_3$, $\md_4\}$ (see Tab.~\ref{tab:min_diags+conflicts_example_DPI_0}). Then $Q_1 := \setof{M \to B}$ is a query w.r.t.\ $\mD$, i.e.\ $Q_1 \in \mQ_\mD$. 
To verify this, we use Def.~\ref{def:query_q-partition} directly and show that both $\dx{}(Q_1) \neq \emptyset$ and $\dnx{}(Q_1) \neq \emptyset$. 

We first consider the leading diagnosis $\md_4 = \setof{\tax_2,\tax_7} \in \mD$. As per Eq.~\eqref{eq:sol_ont_candidate}, we build the solution KB (with background knowledge) resulting from the application of $\md_4$ as
$\mo^*_4 := (\mo \setminus \md_4) \cup \mb \cup U_{\Tp}$. Since this KB does not entail $Q_1$ (as can be easily verified using Tab.~\ref{tab:example_dpi_0} and Tab.~\ref{tab:min_diags+conflicts_example_DPI_0}), $\md_4 \notin \dx{}(Q_1)$ (cf.\ $\dx{}(X)$ in Prop.~\ref{prop:partition}). So, we check whether $\md_4$ is an element of $\dnx{}(Q_1)$. To this end, we first build  
$\mo^*_4 \cup Q_1 := (\mo \setminus \md_4) \cup \mb \cup U_{\Tp} \cup Q_1 = \setof{\tax_1,\tax_3,\tax_4,\tax_5,\tax_6,\tax_8,\tax_9,\tp_1,M\to B}$. As $\tax_9 \equiv B \to K$, $\tax_5 = K \to E$ and $\tax_3 \models E \to \lnot M$, it is clear that $\setof{M \to A} = \tn_1 \in \Tn$ is an entailment of $\mo^*_4 \cup Q_1$. Hence, $\md_4 \in \dnx{}(Q_1)$ (cf.\ $\dnx{}(X)$ in Prop.~\ref{prop:partition}) which is why $\dnx{}(Q_1) \neq \emptyset$. 

Now, we have a look at $\md_2 = \setof{\tax_2,\tax_5}$. We have that $\mo^*_2 := (\mo \setminus \md_2) \cup \mb \cup U_{\Tp} = \setof{\tax_1,\tax_3,\tax_4,\tax_6,\tax_7,\tax_8,\tax_9,\tp_1} \models \setof{M \to B} = Q_1$ due to $\tax_7 \models M\to C$ and $\tax_6 = C \to B$. Therefore, $\md_2 \in \dx{}(Q_1)$ which is why $\dx{}(Q_1) \neq \emptyset$. All in all, we have proven that $Q_1 \in \mQ_\mD$.

If we complete the assignment to sets in the q-partition for all $\md_i \in \mD$, we obtain the q-partition $\Pt_{\mD}(Q_1) = \tuple{\setof{\md_1,\md_2},\setof{\md_3,\md_4},\emptyset}$. 
The justification for the assignment of $\md_1$ and $\md_3$ is a follows. As $\mo_1^*$ includes $\tax_6$, $\tax_7$ we can conclude in an analogous way as above that $\mo_1^* \models Q_1$. In the case of $\md_3$, we observe that $\mo_3^* \cup Q_1$ includes $\tax_3$, $\tax_5$, $M\to B$ and $\tax_9$. As explicated above, these axioms entail $\tn_1 \in \Tn$.
The q-partition $\Pt_{\mD}(Q_1)$ indicates that $\md_1,\md_2$ are invalidated given the oracle answers $Q_1$ negatively, i.e.\ $\md_1,\md_2 \notin \minD_{\langle\mo,\mb,\Tp,\Tn\cup\setof{Q_1}\rangle_\RQ}$. Conversely, $\md_3,\md_4$ are ruled out given $Q_1$'s positive answer, i.e.\ $\md_3,\md_4 \notin \minD_{\langle\mo,\mb,\Tp\cup\setof{Q_1},\Tn\rangle_\RQ}$.\qed 
\end{example}

Some important properties of q-partitions and their relationship to queries are summarized by the next proposition:
\begin{proposition}\label{prop:properties_of_q-partitions}\citep[Sec.\ 7.3 -- 7.6]{Rodler2015phd} Let $\dpi := \langle\mo,\mb,\Tp,\Tn\rangle_\RQ$ be a DPI and $\mD \subseteq \minD_{\dpi}$. Further, let $Q\in\mQ_\mD$. Then:
	~\begin{enumerate} 
		\item \label{prop:properties_of_q-partitions:enum:q-partition_is_partition} $\langle \dx{}(Q)$, $\dnx{}(Q)$, $\dz{}(Q) \rangle$ is a partition of $\mD$.
		\item \label{prop:properties_of_q-partitions:enum:dx_dnx_dz_contain_exactly_those_diags_that_are...} $\dnx{}(Q)$ contains exactly those diagnoses $\md_i\in\mD$ where $\md_i$ is not a diagnosis w.r.t.\ $\tuple{\mo,\mb,\Tp\cup\setof{Q},\Tn}_\RQ$. $\dx{}(Q)$ contains exactly those diagnoses $\md_i\in\mD$ where $\md_i$ is not a diagnosis w.r.t.\ $\tuple{\mo,\mb,\Tp,\Tn\cup\setof{Q}}_\RQ$. $\dz{}(Q)$ contains exactly those diagnoses $\md_i\in\mD$ where $\md_i$ is a diagnosis w.r.t.\ both $\tuple{\mo,\mb,\Tp\cup\setof{Q},\Tn}_\RQ$ and $\tuple{\mo,\mb,\Tp,\Tn\cup\setof{Q}}_\RQ$. 
		\item \label{prop:properties_of_q-partitions:enum:q-partition_is_unique_for_query} For $Q$ there is one and only one q-partition $\langle\dx{}(Q),$ $\dnx{}(Q), \dz{}(Q)\rangle$.
		\item \label{prop:properties_of_q-partitions:enum:query_is_set_of_common_ent} $Q$ is a set of common entailments of all KBs in  $\setof{\mo^{*}_i\,|\,\md_i\in\dx{}(Q)}$. That is, letting $\mathit{DC}(X)$ denote the deductive closure of a set of sentences $X$, $Q$ is a subset of the intersection of all $\mathit{DC}(\mo^{*}_i)$ where $\md_i$ used to construct $\mo^{*}_i$ is an element of $\dx{}(Q)$. 
		\item \label{prop:properties_of_q-partitions:enum:set_of_logical_formulas_is_query_iff...} A set of sentences $X\neq\emptyset$ is a query w.r.t. $\mD$ iff $\dx{}(X) \neq \emptyset$ and $\dnx{}(X)~\neq~\emptyset$.
		\item  \label{prop:properties_of_q-partitions:enum:for_each_q-partition_dx_is_empty_and_dnx_is_empty} For each q-partition $\Pt_\mD(Q) = \langle \dx{}(Q), \dnx{}(Q), \dz{}(Q)\rangle$ it holds that $\dx{}(Q) \neq \emptyset$ and $\dnx{}(Q)~\neq~\emptyset$.
		\item \label{prop:properties_of_q-partitions:enum:D+=d_i_is_q-partition_and_lower_bound_of_queries} If $|\mD|\geq 2$, then 
		\begin{enumerate}
			\item $Q:=U_\mD\setminus\md_i$ is a query w.r.t.\ $\mD$ for all $\md_i\in\mD$,
			\item $\langle \setof{\md_i}, \mD \setminus \setof{\md_i}, \emptyset\rangle$ is the q-partition associated with $Q$, and
			\item a lower bound for the number of queries w.r.t.\ $\mD$ is $|\mD|$.
		\end{enumerate}
	\end{enumerate}
\end{proposition}

\subsubsection{The Sequential Diagnosis Problem}
\label{sec:sequential_diagnosis_problem}

The Sequential Diagnosis Problem we consider next is similar to the Sequential MBD-Problem (Prob.~\ref{prob_def:sequential_MBD}). The difference is that the former generalizes the latter by assuming an oracle that is allowed to not only specify positive test cases (cf.\ $\meas$ in Sec.~\ref{sec:defs_and_props}) but also negative ones in order to narrow down the set of possible diagnoses. 
\begin{prob_def}[Sequential Diagnosis]\label{prob_def:sequential_KBD}\textcolor{white}{.}
		
\noindent\textbf{Given:} A DPI $\dpi := \tuple{\mo,\mb,\Tp,\Tn}_\RQ$ and a diagnostic goal $G$. 

\noindent\textbf{Find:} $\Tp_{\mathit{new}}, \Tn_{\mathit{new}} \supseteq \emptyset$ and $\md$, where $\Tp_{\mathit{new}}, \Tn_{\mathit{new}}$ are sets of positive and negative test cases, respectively, such that $\md$ is a minimal diagnosis w.r.t.\ $\dpi_{\mathit{new}} := \tuple{\mo,\mb,\Tp\cup\Tp_{\mathit{new}},\Tn\cup\Tn_{\mathit{new}}}_\RQ$ and $\md$ satisfies $G$. 
\end{prob_def}

A generic algorithm solving this problem is given next.

\subsubsection{A Generic Sequential Diagnosis Algorithm}
\label{sec:sequential_diagnosis_algo}
The overall sequential diagnosis algorithm
we take as a basis is described by Alg.~\ref{algo:sequential_diagnosis}. Similar algorithms are used e.g.\ in \citep{dekleer1993,Shchekotykhin2012,Rodler2015phd}. Next, we briefly comment on the inputs, the output and the various steps of the algorithm (referred to by their line number in Alg.~\ref{algo:sequential_diagnosis}).

\emph{(Inputs):} The algorithm gets a DPI $\dpi$ and a diagnostic goal $G$ as inputs (cf.\ Prob.~\ref{prob_def:sequential_KBD}). Further on, we assume some probability measure $p$ that can be used to compute fault probabilities of sentences $\tax_i \in \mo$ and of diagnoses $\md \subseteq \mo$. That is, we regard $p$ as (i)~a function $p: \mo \to [0,1]$ assigning to each axiom in $\mo$ (or: component in $\comps$) a fault probability and (ii)~a function $p: \allD_{\dpi} \to [0,1]$ mapping each diagnosis $\md$ w.r.t.\ $\dpi$ to its probability $p(\md)$. The latter is interpreted as the probability that 
all axioms in $\md$ are faulty and all axioms in $\mo \setminus \md$ are correct.  

In the circuit example of Fig.~\ref{fig:circuitreiter87} and other physical devices, $p$ might result from known or estimated fault probabilities of components (e.g.\ obtained from the component manufacturer or by observation) and other heuristic or experiential information \citep{de2004fundamentals}. In a knowledge-based application, $p$ might result from (an integration of) information about e.g.\ common logical fault patterns \citep{Roussey2009a}, 
logs of previous faults recorded by KB editors such as Prot\'eg\'e \citep{Noy2000} or WebProt\'eg\'e \citep{Tudorache2011}, 
user fault probabilities regarding syntactical \citep{Shchekotykhin2012} or ontological \citep{Rodler2015phd} elements in the KB, 
possibly coupled with provenance information about the KB's sentences \citep{Kalyanpur2006}. \citep[Sec.~4.6]{Rodler2015phd} provides a detailed discussion of applicable information sources for $p$ and a derivation of diagnoses fault probabilities from axiom (or: component) fault probabilities or other sources. Hence, one can w.l.o.g.\ provide Alg.~\ref{algo:sequential_diagnosis} only with axiom fault probabilities, i.e.\ the function in (i) above. Because the function in (ii) can be derived from the one in (i).

Lastly, the algorithm is provided with a query quality measure $\qqm$ which enables the comparison of queries in $\mQ_\mD$ and thus the determination of a favorable next query in each iteration.

\emph{(Line~\ref{algoline:inter_onto_debug:leading_diags}):} As a first step, the function \textsc{computeLeadingDiagnoses} computes a set of leading diagnoses $\mD \subseteq \minD_{\dpi}$ where $|\mD| \geq 2$ if $|\minD_{\dpi}| \geq 2$. $\mD$ usually comprises a number of most probable (by exploiting $p$) or minimum-cardinality diagnoses. To this end, several algorithms might be employed such as \textsc{HS-Tree} \citep{Reiter87} (possibly coupled with a minimal-conflict searcher, e.g.\ \textsc{QuickXplain} \citep{junker04}, \textsc{MergeXplain} \citep{shchekotykhin2015mergexplain} or \textsc{Progression} \citep{marques2013minimal}), \textsc{staticHS} or \textsc{dynamicHS} \citep{Rodler2015phd}, HS-DAG \citep{greiner1989correction}, \textsc{Inv-HS-Tree} \citep{Shchekotykhin2014} or Boolean algorithms \citep{jiang2003computation,pill2012optimizations}.
The computed number $|\mD|$ of leading diagnoses might be, e.g., predefined to some constant $k$ \citep{dekleer1993,Shchekotykhin2012}, dependent on desired minimal and maximal bounds or computation time \citep{Rodler2015phd}, or a function of the probability measure $p$ \citep{DBLP:conf/ijcai/KleerW89}. 

\emph{(Line~\ref{algoline:inter_onto_debug:if_goal_reached}):} Then, the algorithm tests (usually by means of $p$, cf.\ Remark~\ref{rem:leading_diags__diagnostic_goal}) whether the diagnostic goal $G$ is satisfied for $\mD$. If so, the must probable leading diagnosis is returned \emph{(line~\ref{algoline:inter_onto_debug:return_most_prob_diag})} and the algorithm stops.

\emph{(Line~\ref{algoline:inter_onto_debug:calc_query}):} Otherwise, the function \textsc{calcQuery} determines a query $Q$ by means of $\mD$, $\dpi$, $p$ and the $\qqm$. Roughly, $Q$ should discriminate optimally among $\minD_{\dpi}$, or more precisely, among the leading diagnoses $\mD$, which is the currently available evidence regarding $\minD_{\dpi}$. The meaning of ``discriminating optimally'' is determined by $\qqm$ which possibly relies upon $p$. That is, a query with best (or sufficiently good \citep{dekleer1987}) value as per $\qqm$ is sought among the queries in $\mQ_\mD$. One example of a $\qqm$ is the information-entropy-based $\$(.)$ function suggested by \citep{dekleer1987}.   

\emph{(Line~\ref{algoline:inter_onto_debug:answer}):}
The calculated query is presented to the oracle and an answer $\oracle(Q)$ is returned. This is the (only) point in the algorithm where an oracle inquiry takes place. For technical reasons, the oracle function $\oracle$ is regarded as a total function, i.e.\ the oracle is assumed to be able to answer \emph{any} posed query. We emphasize however that this is not a necessary requirement for our presented method which can handle \emph{do not know} answers (e.g.\ if some measurement points in a physical device are not accessible to a technician) as well by simply offering the oracle the next-best query. Hence, one can imagine a (not shown) loop between lines~\ref{algoline:inter_onto_debug:calc_query} and \ref{algoline:inter_onto_debug:answer} of the algorithm.

\emph{(Line~\ref{algoline:inter_onto_debug:update_DPI}):}
Given the query-answer pair $(Q,\oracle(Q))$, the current DPI is updated by a respective addition of $Q$ to the positive test cases $\Tp$ (i.e.\ $\Tp \gets \Tp \cup \setof{Q}$) if $\oracle(Q) = t$ and to $\Tn$ (i.e.\ $\Tn \gets \Tn \cup \setof{Q}$) in case $\oracle(Q) = f$. Moreover, the function \textsc{updateDPI} involves an adaptation of the diagnoses probability measure based on $\oracle(Q)$ in terms of a Bayesian probability update according to \citep{dekleer1987,Shchekotykhin2012}. That is, for $Q$'s answer $a_Q \in \setof{t,f}$ the new probability of any $\md \in \allD_{\dpi}$ is computed as
\begin{align*} 
p(\md \mid \oracle(Q) = a_Q) = \frac{p( \oracle(Q) = a_Q \mid \md)\;\,p(\md)}{p(\oracle(Q) = a_Q)} 
\end{align*}
The probabilities required to evaluate the above-noted formula are established as follows. Given the current leading diagnoses $\mD$, we define 
$p(X) := \sum_{\md \in X} p(\md)$
for $X \subseteq \mD$ and assume $p$ to be normalized over $\mD$ such that that $p(\mD)=1$.
Since $\mD$ includes only still possible diagnoses, $p(\md) > 0$ must hold for all $\md \in \mD$. 
Further, assuming that each non-predicting diagnosis $\md \in \dz{}(Q)$ predicts each answer with a probability of $\frac{1}{2}$, we define
\begin{align}\label{eq:prob_of_pos_query_answer}
p(\oracle(Q)=t) := p(\dx{}(Q))+\frac{p(\dz{}(Q))}{2}
\end{align}
(i.e.\ the probability of the leading diagnoses predicting $Q$'s positive answer plus half the probability of the non-predicting leading diagnoses)
and
$p(\oracle(Q)=f) = 1- p(\oracle(Q)=t)$
(i.e.\ the probability of the leading diagnoses predicting $Q$'s negative answer plus half the probability of the non-predicting leading diagnoses).
Finally, 
\begin{equation*}\label{eq:cond_query_prob}
p(\oracle(Q)=t\mid \md) := 
\begin{cases}
1, 						& \mbox{if } \md \in \dx{}(Q) \\   
0, 						& \mbox{if } \md \in \dnx{}(Q) \\
\frac{1}{2}, 			& \mbox{if } \md \in \dz{}(Q)	  
\end{cases} 
\end{equation*}
and $p(\oracle(Q)=f\mid \md) = 1-p(\oracle(Q)=t\mid \md)$.
\begin{remark}
When Alg.~\ref{algo:sequential_diagnosis} computes a new set of leading diagnoses $\mD_{\mathit{new}}$ (line~\ref{algoline:inter_onto_debug:leading_diags}) after executing the DPI update in line~\ref{algoline:inter_onto_debug:update_DPI}, it computes $\mD_{\mathit{new}}$ as minimal diagnoses \emph{w.r.t.\ the new DPI}. 
That is, $\mD$ will usually (but not necessarily always, cf.\ \citep[Rem.~12.6]{Rodler2015phd}) comprise the remaining diagnoses 
from the leading diagnoses $\mD$ used in the previous iteration and some new ones computed in line~\ref{algoline:inter_onto_debug:leading_diags}. The remaining diagnoses from $\mD$ given $\oracle(Q) = a_Q$ are $\dx{}(Q) \cup \dz{}(Q)$ for $a_Q = t$ and $\dnx{}(Q) \cup \dz{}(Q)$ for $a_Q = f$ (cf.\ Prop.~\ref{prop:properties_of_q-partitions}.\ref{prop:properties_of_q-partitions:enum:dx_dnx_dz_contain_exactly_those_diags_that_are...}). \qed
\end{remark}
\emph{(Outputs):} The algorithm executes the while-loop until the given diagnostic goal $G$ is fulfilled. Let $\dpi^*$ be the current DPI and $\mD$ be the current leading diagnoses in the iteration where this holds. Then Alg.~\ref{algo:sequential_diagnosis} returns the most probable minimal diagnosis $\md^* \in \mD \subseteq\minD_{\dpi^*}$.

%
%
%

\begin{algorithm}[t]
	\small
	\caption{Sequential Diagnosis}\label{algo:sequential_diagnosis}
	\begin{algorithmic}[1]
		\Require DPI $\dpi := \langle\mo,\mb,\Tp,\Tn\rangle_\RQ$, diagnostic goal $G$, 
		%
		probability measure $p$ (used to compute sentence and diagnoses fault probabilities),
		query quality measure $\qqm$
		\Ensure The most probable minimal diagnosis $\md^* \in \mD \subseteq\minD_{\dpi^*}$ where $\dpi^*$ is some DPI $\langle\mo,\mb,\Tp',\Tn'\rangle_\RQ$ where $\Tp' \supseteq \Tp$ and $\Tn' \supseteq \Tn$ such that the diagnostic goal $G$ is met for $\mD$  
		\While{\true}  \label{algoline:inter_onto_debug:while}
		\State $\mD \gets \Call{computeLeadingDiagnoses}{\dpi,p}$
		\label{algoline:inter_onto_debug:leading_diags}
		\If{$\Call{goalReached}{G,\mD,p}$} \label{algoline:inter_onto_debug:if_goal_reached}
		\State \Return $\Call{getMostProbableDiagnosis}{\mD,p}$
		\label{algoline:inter_onto_debug:return_most_prob_diag}
		\EndIf
		\State $Q \gets \Call{calcQuery}{\mD,\dpi, p, \qqm}$   \label{algoline:inter_onto_debug:calc_query}  \Comment{see Algorithm~\ref{algo:query_comp}}
		\State $answer \gets \oracle(Q)$ \label{algoline:inter_onto_debug:answer}				\Comment{oracle inquiry}
		\State $\dpi \gets \Call{updateDPI}{\dpi, p, Q, answer}$
		\label{algoline:inter_onto_debug:update_DPI}
		\EndWhile
	\end{algorithmic}
	\normalsize
\end{algorithm}

\subsubsection{Algorithm Correctness}
\label{sec:sequential_diagnosis_algo_correctness}
To show that Alg.~\ref{algo:sequential_diagnosis} solves the Sequential Diagnosis problem (Prob.~\ref{prob_def:sequential_KBD}), we use the fact that for any non-singleton set of leading diagnoses $\mD$ a query -- and hence the opportunity to discriminate among elements of $\mD$ -- exists \citep[Sec.~7.6]{Rodler2015phd}: 
\begin{proposition}\label{prop:query_exists}
Let $\dpi$ be a DPI and $\mD \subseteq \minD_{\dpi}$ such that $|\mD| \geq 2$. Then $\mQ_{\mD} \neq \emptyset$.
\end{proposition}
\begin{theorem}\label{theorem:correctness_of_sequential_diagnosis_algorithm}
Let \textsc{computeLeadingDiagnoses} be a sound and complete procedure for the computation of minimal diagnoses w.r.t.\ a DPI, \textsc{calcQuery} be a sound method for query computation that returns at least one query for any $\mD$ where $\mQ_\mD \neq \emptyset$, and $G$ be an arbitrary diagnostic goal that is at most as strict as $G_1$ in Rem.~\ref{rem:leading_diags__diagnostic_goal}. Then, for arbitrary inputs $\dpi$ (DPI), $p$ (diagnoses probability measure) and $\qqm$ (query quality measure), Alg.~\ref{algo:sequential_diagnosis} solves Prob.~\ref{prob_def:sequential_KBD}.
\end{theorem}
\begin{proof} (Sketch)
Let us assume that $G := G_1$ (``there is only a single minimal diagnosis w.r.t.\ the current DPI'') from Rem.~\ref{rem:leading_diags__diagnostic_goal} and let us refer to the DPI used by Alg.~\ref{algo:sequential_diagnosis} in iteration $i$ (of the while-loop) by $\dpi_i$, i.e.\ the input DPI $\dpi$ is denoted by $\dpi_1$. If the first call of \textsc{computeLeadingDiagnoses} returns a singleton $\mD = \setof{\md}$, then $|\minD_{\dpi}| = 1$, thus $G_1$ is met and Alg.~\ref{algo:sequential_diagnosis} returns $\md$ in line~\ref{algoline:inter_onto_debug:return_most_prob_diag} which obviously meets Prob.~\ref{prob_def:sequential_KBD} ($\Tp_{\mathit{new}} = \Tn_{\mathit{new}} = \emptyset$).

Otherwise, the first call of \textsc{computeLeadingDiagnoses} returns some $\mD$ where $|\mD| \geq 2$. If $G$ is not satisfied in line~\ref{algoline:inter_onto_debug:if_goal_reached}, then \textsc{calcQuery} will return a query $Q$ by Prop.~\ref{prop:query_exists}. Due to 
\citep[Cor.~12.4]{Rodler2015phd},
each answer $\oracle(Q)$ implies that $\allD_{\dpi_1} \supset \allD_{\dpi_2}$ where $\dpi_2$ is the result of applying \textsc{updateDPI} in line~\ref{algoline:inter_onto_debug:update_DPI} to $\dpi_1$.

We can now adopt the same argumentation for the second and any further call of \textsc{computeLeadingDiagnoses} (iterations $2,3,\dots$). However, $|\allD_{\dpi_1}|$ must be finite as each diagnosis is a subset of $\mo$ and $\mo$ must be finite due to Def.~\ref{def:dpi}. Hence, there must be some finite $k$ such that $G$ is met in line~\ref{algoline:inter_onto_debug:if_goal_reached} of iteration $k$. Therefore, Alg.~\ref{algo:sequential_diagnosis} outputs the single minimal diagnosis $\md$ w.r.t.\ $\dpi_k$ which meets Prob.~\ref{prob_def:sequential_KBD} ($\Tp_{\mathit{new}}$ includes all the $k_1$ positively answered queries in these $k$ iterations, and $\Tn_{\mathit{new}}$ all the $k-k_1$ negatively answered ones).

Since any diagnostic measure $G$ Alg.~\ref{algo:sequential_diagnosis} might be used with is 
at most as strict (cf.\ Rem.~\ref{rem:leading_diags__diagnostic_goal}) as $G_1$, which we used for our argumentations, we obtain that such a finite $k$ must always exist.
\end{proof}

\subsubsection{Applicability and Diagnostic Accuracy}
\label{sec:applicability_diagnostic_accuracy}
Prop.~\ref{prop:query_exists} and Theorem~\ref{theorem:correctness_of_sequential_diagnosis_algorithm} have two further implications: First, a precomputation of only two minimal diagnoses is required in each iteration to generate a query and proceed with sequential diagnosis. Despite its NP-hardness, the generation of two (or more) 
minimal diagnoses is practical in many real-world settings \citep{DBLP:conf/aaai/Kleer91,Shchekotykhin2014}, making query-based sequential diagnosis commonly applicable. 
Second, the query-based approach guarantees perfect diagnostic accuracy, i.e.\ the unambiguous identification of the actual diagnosis (e.g.\ by using the diagnostic goal $G := G_1$ in Alg.~\ref{algo:sequential_diagnosis}).

\section{Efficient Optimized Query Selection for Sequential Model-Based Diagnosis}
\label{sec:contribution}
In this section we present the main contribution of this work, which is a novel implementation of the \textsc{calcQuery} function in Alg.~\ref{algo:sequential_diagnosis}. But first, we have a look at the measurement selection problem in sequential diagnosis.

\subsection{Measurement Selection for Sequential Diagnosis}
\label{sec:measurement_selection}
As argued, the (q-)partition $\Pt_\LD(Q)$ enables both the \emph{verification} whether a candidate $Q$ is indeed a query and an \emph{estimation of the impact} $Q$'s outcomes have in terms of diagnoses invalidation. And, given axiom (or: component) fault probabilities, it enables to \emph{gauge the probability of observing a positive or negative query outcome}. 
%
%
%
%
Active learning \emph{query selection measures (QSMs)} $m: Q \mapsto m(Q) \in \mathbb{R}$ \citep{settles2012} use exactly these query properties characterized by the q-partition to assess how favorable a query is. They aim at selecting queries such that the expected number of queries until obtaining a deterministic diagnostic result is minimized, i.e.\ 
\begin{equation}
\sum_{\md \subseteq \mo} p(\md) q_{\#}(\md) \;\rightarrow\; \min \label{eq:qsm_goal}
\end{equation} 
where 
$q_{\#}(\md)$ is the number of queries required, given the initial DPI, to derive that $\md$ must be the actual diagnosis. 
Solving this problem
is known to be NP-complete as it amounts to optimal binary decision tree construction \citep{hyafil1976}. Hence, as it is common practice in sequential diagnosis \citep{dekleer1987,DBLP:conf/ijcai/BrodieRMO03,pietersma2005model,Shchekotykhin2012,Rodler2013}, we restrict our algorithm to the usage of QSMs that make a locally optimal query selection through a \emph{one-step lookahead}. This has been shown to be optimal in many cases and nearly optimal in most cases \citep{dekleer1992}.  
Several different QSMs $m$ such as \emph{split-in-half}, \emph{entropy}, or \emph{risk-optimization} have been proposed, well studied and compared against each other~\citep{dekleer1987,Shchekotykhin2012,Rodler2013,rodler17dx_activelearning}. For instance, using entropy as QSM, $m$ would be exactly the scoring function $\$()$ derived in \citep{dekleer1987}. Note, we assume w.l.o.g.\ that the optimal query w.r.t.\ any $m$ is the one with \emph{minimal} $m(Q)$.

Besides minimizing the number of queries in a diagnostic session (Eq.~\eqref{eq:qsm_goal}), a further goal can be the minimization of the query (answering) cost arising for the oracle.
%
For instance, assume a malfunctioning physical system such as a car or a turbine. Then there might be parts of the system which are easier accessible, cheaper (in terms of required tools, time or manpower), less dangerous, etc.\ for measurements than others. In a car, it is clearly much easier and faster to check some, say cable, that is directly accessible after opening the engine cover than some internals of the engine. Apart from that, systems might include built-in sensors able to provide information about certain parts of the system quasi for free, whereas other parts must be manually measured. On the other hand, e.g., in knowledge-based systems, there might be sentences about the intended domain that are easier to evaluate than others. For example, sentences comprising complex logical notation are certainly more difficult to read, understand and thus to answer than, e.g., facts or simple implications (cf.\ \citep{Ceraso71,Horridge2011b}). Moreover, aside from the syntax of sentences, the comprehension of their content in terms of 
the expressed topic, can be a smaller or larger hurdle, depending on the oracle's expertise in the topic. E.g., given a faulty biomedical KB \citep{noy2009bioportal}, a biologist acting as an oracle might be much less confident in answering a medicine-specific query than a biology-specific one.

To this end, we allow the user of our tool to specify a \emph{query cost measure (QCM)}\label{etc:QCM} $c: Q \mapsto c(Q) \in \mathbb{R}^+$ that assigns to each query a real-valued cost. Examples of QCMs are 
\begin{itemize}[noitemsep]
\item $c_{\Sigma}(Q) := \sum_{i=1}^{k} c_i \qquad\qquad\quad\;\;\;$ (minimize overall query cost)
\item $c_{\max}(Q) := \max_{i \in \setof{1,\dots,k}} c_i \qquad$ (minimize maximal cost of single measurements)
\item $c_{|\cdot|}(Q) := |Q| \qquad\qquad\qquad\quad\;\;\,$ (minimize number of measurements)
\end{itemize}
where $Q = \setof{q_1,\dots,q_k}$ and $c_i$ is the cost of evaluating the truth of the sentence $q_i$.
For example, the QCM $c_{\Sigma}$ could be used to minimize query answering time given that the $c_i$'s represent measurement time. Alternatively, $c_{\max}$ could be adopted, e.g., when the $c_i$'s account for human cognitive load, in order to keep the necessary cognitive skill to answer the query minimal. In scenarios where all potential measurements are known or assumed to be (approximately) equally costly to answer, one might pursue queries with a minimal number of sentences, which would be reflected by using the QCM $c_{|\cdot|}$. An example would be a digital circuit where all components and wires, respectively, are equally well accessible for measurements. Note that the QCM $c_{|\cdot|}$ is a special case of $c_{\Sigma}$, i.e.\ $c_{\Sigma}$ simulates $c_{|\cdot|}$ if $c_i = c_j$ for all $i,j$. In the following, we assume w.l.o.g.\ that the optimal query w.r.t.\ any $c$ is the one with \emph{minimal} $c(Q)$.

While the costs $c_i$ (e.g.\ accessibility of systems parts) relevant for physical systems might be more or less directly derivable from the structure or other properties of the system, the costs $c_i$ (e.g.\ comprehensibility of sentences) relevant for knowledge-based systems might be plausibly derivable from the available fault information \citep[Sec.~4.6.1]{Rodler2015phd}. For instance, the cognitive complexity of understanding a queried sentence correctly in terms of syntax and topic can be expected to be closely related to the fault probability of the sentence computed from fault information about logical elements (e.g.\ $\forall$, $\lnot$, $\land$) and non-logical elements (e.g.\ concepts such as $\mathsf{toxoplasmosis}$ or $\mathsf{odontalgia}$ in a medical KB). Intuitively, the more a sentence captures the expertise of the oracle (i.e.\ the lower its fault probability for the oracle is), the easier it is to understand and answer for the oracle.
%

\subsubsection{The Addressed Measurement Selection Problem}
\label{sec:measurement_selection_problem}
Now, the problem tackled by the new algorithm introduced in this work is:
\begin{prob_def}[Optimal Query Selection]\label{prob:query_optimization} \textcolor{white}{.}
	
\noindent\textbf{Given:} A DPI $\dpi$,
 $\mD \subseteq \minD_{\dpi}$ where $|\mD|\geq 2$, a QSM $m$, 
a QCM $c$, 
a query search space $\mathbf{S}\subseteq\mQ_{\mD}$. 

\noindent\textbf{Find:} 
A query $Q^*$ with minimal cost w.r.t.\ $c$ among all queries in $\mathbf{S}$ that are optimal w.r.t.\ $m$. Formally: $Q^* = \argmin_{Q \in \mathbf{OptQ}(m,\mathbf{S})} c(Q)$ where $\mathbf{OptQ}(m,\mathbf{S}) := \setof{Q' \mid Q' = \arg\min_{Q\in\mathbf{S}} m(Q)}$. 
%
%
%
\end{prob_def}
Note, there can be multiple equally good solutions $Q^* \in \mQ_{\LD}$ to Prob.~\ref{prob:query_optimization}.

\subsection{The Suggested Algorithm}
\label{sec:presented_algo}
In this section we propose a novel algorithm (Alg.~\ref{algo:query_comp}) for two-way optimized query computation in sequential diagnosis which solves Prob.~\ref{prob:query_optimization} (for different settings of the query search space $\mathbf{S}$).\footnote{A Prot\'{e}g\'{e} plugin for KB debugging implementing i.a.\ the presented algorithm can be found on http://isbi.aau.at/ontodebug/. Prot\'{e}g\'{e} \citep{Noy2000} is the most widely used open-source KB (ontology) editor in the world and available at https://protege.stanford.edu/.} 
The described query computation procedure can be divided into three phases: Phase P1 
(line~\ref{algoline:query_comp:optimizeQP}), (the default) Phase P2 
(line~\ref{algoline:query_comp:optimizeQforQP}) and (the optional) Phase P3 
(lines~\ref{algoline:query_comp:enrichQ}-\ref{algoline:query_comp:optiminimizeQ}).
After giving the reader an intuition and overview of its functioning, we explain all three phases of it. Further implementation details can be found in the extended version \citep{DBLP:journals/corr/Rodler16a} of the paper.


\begin{algorithm}[t]
	\small
	\caption{Optimized Query Computation}\label{algo:query_comp}
	\begin{algorithmic}[1]
\Require DPI $\dpi$,
$\mD \subseteq \minD_{\dpi}$ where $|\mD|\geq 2$, a QSM $m$, a QCM $c$ (including information about sentence costs $c_i$, cf.\ page~\pageref{etc:QCM}), probability measure $p$ (to compute axiom fault probabilities), 
threshold $t_m$ 
(i.e.\ $|m(Q) - m_{opt}| \leq t_m \Rightarrow Q$ regarded as optimal; $m_{opt} :=$ optimal value of $m$),
inference engine $\mathit{Inf}$, set $\mathit{ET}$ of entailment types, $\mathit{pref}$ (preference information used for query optimization), a Boolean $\mathsf{enhance}$ (if $\true$, optional query enhancement is run)
\Ensure an optimized query $Q^* \in \mQ_{\mD}$ w.r.t.\ $m$, $t_m$ and $c$	(cf.\ Theorems~\ref{theorem:P1+P2_solve_query_optimization_problem} and \ref{theorem:P3_solves_problem_1})		
\State $\Pt \gets \Call{optimizeQPartition}{\mD,p,m,t_m}$ \label{algoline:query_comp:optimizeQP}    \Comment{P1}
\If{$\mathsf{enhance} = \true$} \label{algoline:query_comp:if_enhance=true}
\State $Q' \gets \Call{expandQueryForQPartition}{\dpi, \Pt, ET, \mathit{Inf}}$    \Comment{(optional) P3} \label{algoline:query_comp:enrichQ}
\State $Q^* \gets \Call{optiMinimizeQueryForQPartition}{\dpi,\Pt,Q',\mathit{pref}, \mathit{Inf}}$	 \Comment{(optional) P3}	\label{algoline:query_comp:optiminimizeQ}
\Else
\State $Q^* \gets \Call{optimizeQueryForQPartition}{\Pt,c}$ \label{algoline:query_comp:optimizeQforQP}   \Comment{(default) P2}
\EndIf
\State \Return $Q^*$
	\end{algorithmic}
	\normalsize
\end{algorithm}

\subsubsection{Intuition and Overview}
\label{sec:intuition+overview}

The main idea to achieve an inexpensive high-quality query generation is the exploitation of the information inherent in the $\subseteq$-minimal leading diagnoses and a decoupling of the optimizations of a QSM (first), i.e.\ the minimization of the expected number of queries until diagnostic certainty is given, and a QCM (second), i.e.\ the minimization of the query (answering) costs.

\noindent\textbf{Phase P1}, given a QSM $m$, determines an optimal QP w.r.t.\ $m$ (which implies that all queries having this QP are optimal regarding $m$) by completely avoiding the use of reasoning services. For this purpose, we present a polynomial-space (heuristic) search technique that explores a generally exponential space of q-partitions in a sound and complete way. That is, without expensive reasoner calls, all non-QPs are automatically neglected (i.e.\ no unnecessary time is spent for exploring non-QPs) and the proven optimal QP in the explored space is found. As the key to the realization of this search, we introduce the notions of \emph{canonical queries} and \emph{canonical QPs}. Additionally, the use of these two concepts in phase P1 automatically disregards a broad class of suboptimal QPs, which leads to a generally significant refinement of the relevant search space.
After an optimal QP $\Pt$ as per $m$ has been determined and fixed, there are two options (parameter $\mathsf{enhance}$): The execution of either
\vspace{-6pt}
\begin{itemize}[noitemsep]
	\item phase P2 (restricted search space, no reasoner calls, instantaneous output) or
	\item phase P3 (full search space, polynomial number of reasoner calls, reasonably fast output).
\end{itemize}

\noindent\textbf{Phase P2} ($\mathsf{enhance} = \false$) realizes the finding of an optimal query w.r.t.\ any of the QCMs discussed on page~\pageref{etc:QCM} for the optimal QP $\Pt$ from P1. In order to make this query computation very efficient, P2 focuses on a restricted class of queries wherein the calculation of a (globally) optimal query is possible without reasoning aid. The query output by P2 represents the system component(s) whose inspection is least expensive (QCM $c$) among all those that yield the highest information (QSM $m$). In this vein, P2 computes an optimized query for scenarios such as the ones discussed in \citep{DBLP:conf/ijcai/BrodieRMO03}.

\noindent\textbf{Phase P3} ($\mathsf{enhance} = \true$) performs a query enhancement, 
subdivided into two consecutive steps. The first step ``enriches'' the canonical query of the optimal QP $\Pt$ from P1 by additional sentences of preferred types (\emph{augmentation to the full search space regarding preferred queries}). The second step finds an optimized $\subseteq$-minimal contraction of the enriched canonical query (\emph{searching for an optimal solution in full search space}).  
Overall, phase P3 computes a query that globally optimizes the QCM $c_{\max}$ (see page~\pageref{etc:QCM}) among all queries that -- under the reasonable assumption given by Conjecture~\ref{conj:CQPs=QPs} (see later) -- globally optimize the QSM $m$ over the full query search space. To this end, P3 requires only a polynomial number of reasoner calls. Moreover, given any predefined set of preferred query elements (e.g.\ measurements with small costs), P3 ensures that the returned optimized query includes only such preferred elements, if such a query exists.

\subsubsection{Phase P1: Optimizing the Q-Partition}
\label{sec:P1}
In this section we describe the functioning of \textsc{optimizeQPartition} in Alg.~\ref{algo:query_comp}.
At this first stage P1, we optimize the given QSM $m$ -- for now without regard to the QCM $c$, which is optimized later in phase P2. This \emph{decoupling of optimization steps}\label{etc:decoupling_of_optimization_steps} is possible since the QSM value $m(Q)$ of a query $Q$ is only affected by the (unique) q-partition of $Q$ and not by $Q$ itself. On the contrary, the QCM value $c(Q)$ is a function of (the sentences in) $Q$ only and not of $Q$'s q-partition.
Therefore, the search performed in P1 will consider only q-partitions and target the determination of a (close-to) optimal q-partition $\Pt$. This q-partition remains fixed throughout all further phases of the algorithm, where an optimal query with exactly this q-partition $\Pt$ is sought. 
\begin{remark}\label{rem:optimization_c_before_m_not_reasonable}
A decoupling of optimizations in the reverse order, i.e.\ QCM $c$ before QSM $m$, is not reasonably possible. First, the q-partition is a necessary prerequisite for the verification whether a set of sentences is in fact a query (see Def.~\ref{def:query_q-partition}). Because without the associated \text{(q-)partition} one has no guiding principle of which $Q$ is allowed (i.e.\ a query) and which is not (i.e.\ a set of sentences that is not a query) when trying to find some \emph{query} $Q$ with minimal $c(Q)$. Second, once the query is fixed, so is its q-partition (see Prop.~\ref{prop:properties_of_q-partitions}.\ref{prop:properties_of_q-partitions:enum:q-partition_is_unique_for_query}). So, there is no chance to optimize $m$ for some already fixed query $Q$.\qed
\end{remark}

\paragraph{Canonical Queries and Q-Partitions.}
As the determination of an optimal q-partition w.r.t.\ the given QSM $m$ should be as efficient 
as possible, we want to neglect any partition which is not a q-partition in the search. That is, we do not even want to generate any non-q-partitions. However, to verify whether a given $3$-partition $\Pt$ of a set of leading diagnoses $\mD$ is a q-partition, we need a query $Q \neq \emptyset$ with a q-partition $\Pt_\mD(Q) = \Pt$. $Q$ 
can be seen as a concrete witness proving that $\Pt$ is not solely a partition of $\mD$, but indeed a q-partition (cf.\ Def.~\ref{def:query_q-partition}). 
But:
\begin{proposition}\label{prop:query_vs_QP}
Let $\Pt$ be a (fixed) q-partition. Then $|\{Q \mid Q \text{ query},\Pt_\mD(Q) = \Pt\}| > 1$, i.e.\ there are multiple queries for $\Pt$.
\end{proposition}
\begin{proof}
Since a query is a non-empty set of sentences, it must include at least one sentence $\tax$. This sentence can be equivalently rewritten in different ways, e.g.\ $\tax \equiv \tax \land \tau$ for an arbitrary tautology $\tau$.
\end{proof}	
Whereas the proof draws on a semantically equivalent rewriting (that is possible for any sentence in infinitely many ways) to show Prop.~\ref{prop:query_vs_QP}, we point out that there are, in most cases, queries with equal q-partitions that are semantically non-equivalent -- and not even rewritings of one another. This holds true also for $\subseteq$-minimal queries. i.e.\ queries where the removal of any sentence in them leads to a change of their associated q-partition. We will provide some examples later. 

The idea is now to appoint one \emph{well-defined} representative query for each q-partition, i.e.\ we seek the definition of a unique query for each q-partition such that the former allows us to verify the latter. And, we want such a query to be \emph{easily computable}. 
Furthermore, in order to devise a time and space saving q-partition search method, the potential size of the explored search space should to be minimized. To achieve this, a key idea is to omit those q-partitions in the search that are proven suboptimal. One such class of suboptimal q-partitions are those $\Pt$ with non-empty $\dz{}(\Pt)$ because they do not discriminate among all (leading) diagnoses (cf.\ the discussion in Sec.~\ref{sec:query:definition_and_properties}). Hence, a second postulation to the representative queries for q-partitions is that the focus on such queries implies the \emph{exclusion of the mentioned suboptimal q-partitions}. In other words, each suboptimal q-partition should have no such representative query.

%
To realize these postulations, we introduce the notion of a \emph{canonical query (CQ)}.
The requirement of easy computability means that we would like to be able to determine a CQ without performing any expensive or (generally) intractable operations. Since by Prop.~\ref{prop:partition}.\ref{prop:properties_of_q-partitions:enum:query_is_set_of_common_ent} any query for a q-partition $\tuple{\dx{},\dnx{},\dz{}}$ is a subset of the common entailments of all KBs in the set $\setof{\mo_i^* \mid \md_i \in \dx{}}$, 
the operations of interest in query computation are entailment calculations. Once calls to a reasoning engine are involved, the complexity of one such call is already NP-complete for Propositional Logic. However, a straightforward way of entailment calculation without involving reasoning aid or other expensive operations is the restriction to the computation of explicit entailments. 
An entailment $\alpha$ of a KB $X$ is called \emph{explicit} iff $\alpha \in X$, \emph{implicit} otherwise \citep[Def.~8.1]{Rodler2015phd}. But, as indicated by \citep[Prop.~8.3]{Rodler2015phd}, just these explicit entailments are also the key to achieve a disregard of suboptimal q-partitions. Therefore, CQs should be characterized as queries including only explicit entailments.
Henceforth, we call a query $Q \in \mQ_\mD$ \label{etc:EE_query_def}\emph{explicit-entailments query} iff  $Q \subseteq \mo$.

Indeed, a restriction to the consideration of only explicit-entailments queries leads to a focus on non-suboptimal q-partitions: 
\begin{proposition}\label{prop:explicit-entailments_queries_have_empty_dz}
	Let $\dpi := \langle\mo,\mb,\Tp,\Tn\rangle_\RQ$ be a DPI, $\mD \subseteq \minD_{\dpi}$ and $Q$ a query in $\mQ_\mD$ such that $Q \subseteq \mo$. Then $\dz{}(Q) = \emptyset$.
\end{proposition}
\begin{proof}
	We have to show that for an arbitrary diagnosis $\md_i \in \mD$ either $\md_i\in\dx{}(Q)$ or $\md_i\in\dnx{}(Q)$. Therefore two cases must be considered: (a)~$\mo\setminus\md_i \supseteq Q$ and (b)~$\mo\setminus\md_i \not\supseteq Q$. In case (a), by the fact that the entailment relation is extensive for $\mathcal{L}$, $\mo\setminus\md_i \models Q$ and thus, by monotonicity of $\mathcal{L}$, $\mo^{*}_i = (\mo\setminus\md_i) \cup \mb \cup U_P \models Q$. So, $\md_i \in \dx{}(Q)$. In case (b) there exists some axiom $\tax \in Q \subseteq \mo$ such that $\tax\notin \mo\setminus \md_i$, which means that $(\mo \setminus \md_i) \cup Q \supset (\mo\setminus\md_i)$. From this we can derive that $\mo^{*}_i \cup Q$ must violate $\RQ$ or $\Tn$ by the $\subseteq$-minimality property of each diagnosis in $\mD$, in particular of $\md_i$. Hence, $\md_i \in \dnx{}(Q)$.
\end{proof}
The proof of Prop.~\ref{prop:explicit-entailments_queries_have_empty_dz} exhibits a decisive advantage of using explicit-entailments queries. In fact, it reveals that the task of verifying whether a set of explicit entailments is a query in $\mQ_\mD$ is very easy in that it can be reduced to set comparisons. That is, subset checks are traded for reasoning. To build the q-partition $\Pt(Q)$ associated with some explicit-entailments query $Q$ it must solely be tested for each $\md_i\in\mD$ whether $\mo\setminus\md_i \supseteq Q$ or not. More specifically:
\begin{proposition}\label{prop:ee-query_q-partition_construction_by_means_of_set_comparison}
Let $\dpi := \langle\mo,\mb,\Tp,\Tn\rangle_\RQ$ be a DPI, $\mD \subseteq \minD_{\dpi}$ and $Q$ be a query in $\mQ_\mD$ such that $Q \subseteq \mo$. Then $\md \in \dx{}(Q)$ iff $\mo\setminus\md \supseteq Q$ and $\md \in \dnx{}(Q)$ iff $\mo\setminus\md \not\supseteq Q$.
\end{proposition}
The next proposition describes the shape any explicit-entailments query must have:
\begin{proposition}\label{prop:expl_ent_query_must_neednot_mustnot_include_ax}
Let $\dpi$ be any DPI, $\mD \subseteq \minD_{\dpi}$ and $Q$ be a query in $\mQ_\mD$ such that $Q \subseteq \mo$. Then $Q$ must include some axiom(s) in $U_\mD$ (i.e.\ $Q\cap U_\mD \neq \emptyset$), can but need not include any axioms in $\mo\setminus U_\mD$, 
and must not include any axioms in $I_\mD$ (i.e.\ $Q \cap I_\mD = \emptyset$).
Further on, elimination of axioms in $\mo\setminus U_\mD$ from $Q$ does not affect the q-partition $\Pt(Q)$ of $Q$.
\end{proposition}
As suggested by Prop.~\ref{prop:expl_ent_query_must_neednot_mustnot_include_ax}, the sentences in the KB $\mo$ of a DPI that are crucial for the definition of an explicit-entailments query -- and hence for the characterization of a CQ -- are given by $U_\mD$ as well as $I_\mD$. In fact, as the proofs of Lemmata~\ref{lem:EE-query_no_intersection_with_I_D} and \ref{lem:EE-query_must_have_intersection_with_U_D} in Appendix~A show, the non-inclusion of sentences from $I_\mD$ is necessary to allow for non-empty $\dx{}(Q)$ and the inclusion of some sentence(s) from $U_\mD$ is required to allow for non-empty $\dnx{}(Q)$. Without these properties a set of sentences $Q$ does not discriminate among the leading diagnoses $\mD$. 
For this reason, we give these essential sentences $U_\mD \setminus I_\mD$ a distinct name:
\begin{definition}\label{def:discax}
We call $\Disc_\mD := U_\mD \setminus I_\mD$ the \emph{discrimination sentences wrt.\ $\mD$}. 
\end{definition}

\begin{example}\label{ex:disc_ax}
Let us consider the set of leading diagnoses 
\[\mD = \setof{\md_1,\md_2,\md_3} = \setof{\{\tax_2,\tax_3\},\{\tax_2,\tax_5\},\{\tax_2,\tax_6\}} \quad \mbox{(cf.\ Tab.~\ref{tab:min_diags+conflicts_example_DPI_0})} \] 
w.r.t.\ our running example DPI $\exdpi$ (Tab.~\ref{tab:example_dpi_0}). Then, $U_\mD = \setof{\tax_2,\tax_3,\tax_5,\tax_6}$ and $I_\mD = \setof{\tax_2}$. Now, all $\subseteq$-minimal explicit-entailments query candidates we might build according to Prop.~\ref{prop:expl_ent_query_must_neednot_mustnot_include_ax} (which provides necessary criteria to explicit-entailments queries) are 
\begin{align}\label{eq:candidate_ee-queries}
\setof{\setof{\tax_3,\tax_5,\tax_6},\setof{\tax_3,\tax_5},\setof{\tax_3,\tax_6},\setof{\tax_5,\tax_6},\setof{\tax_3},\setof{\tax_5},\setof{\tax_6}}
\end{align}
That is, all these seven candidates include at least one element out of $U_\mD$ and no elements out of $\mo \setminus U_\mD = \setof{\tax_1,\tax_4,\tax_7}$ or $I_\mD$. 
Clearly, there are exactly six different q-partition candidates with empty $\dz{}$ w.r.t.\ the three diagnoses in $\mD$, i.e.\ there are three possibilities to select one, and three possibilities to select two diagnoses to constitute the set $\dx{}$ with $\emptyset\subset\dx{}\subset\mD$ (the set $\dnx{}$ is already set after $\dx{}$ is chosen since $\tuple{\dx{},\dnx{},\dz{}}$ is a partition of $\mD$). Hence, $\mQ_\mD$ might comprise at most six explicit-entailments queries with different q-partitions because each one of them must feature an empty $\dz{}$-set in its q-partition, as stated by Prop.~\ref{prop:explicit-entailments_queries_have_empty_dz}.
By the \emph{Pigeonhole Principle}, either at least two candidates in Eq.~\eqref{eq:candidate_ee-queries} have the same q-partition or at least one candidate is not a query at all. As we require that there must be exactly one CQ per q-partition and seek a method that computes only queries (and no candidates that turn out to be no queries), we see that Prop.~\ref{prop:expl_ent_query_must_neednot_mustnot_include_ax} is not yet restrictive enough to constitute also a sufficient criterion for CQs.
	
So let us now find out where the black sheep among the candidates above is. The key to finding it is the fact that each query $Q$ is a common entailment of all $\mo_i^* := (\mo \setminus \md_i) \cup \mb \cup U_\Tp$ 
where $\md_i$ is in the $\dx{}(Q)$-set of $Q$'s q-partition (cf.\ Prop.~\ref{prop:properties_of_q-partitions}.\ref{prop:properties_of_q-partitions:enum:query_is_set_of_common_ent}). 
Since the candidates for CQs in Eq.~\eqref{eq:candidate_ee-queries} are all constituted of just explicit entailments $\tax_i \in \mo$, 
we immediately see that we must postulate that each CQ $Q$ is a set of common elements of all $\mo \setminus \md_i$ where $\md_i$ is in the $\dx{}(Q)$-set of $Q$'s q-partition. The $\mo \setminus \md_i$ sets for diagnoses $\md_i \in \mD$ are given below. Starting to verify this for the first candidate $\setof{\tax_3,\tax_5,\tax_6}$ above, we quickly find out that there is no possible $\dx{}(Q)$-set of a q-partition such that $Q \subseteq \bigcap_{\md_i \in \dx{}(Q)} \mo \setminus \md_i$ because none of these intersected sets includes all elements out of $\setof{\tax_3,\tax_5,\tax_6}$. Thus, the first candidate is no query at all. 
\begin{align*}
\mo \setminus \md_1 &= \setof{\tax_1,\tax_4,\tax_5,\tax_6,\tax_7} \\
\mo \setminus \md_2 &= \setof{\tax_1,\tax_3,\tax_4,\tax_6,\tax_7} \\
\mo \setminus \md_3 &= \setof{\tax_1,\tax_3,\tax_4,\tax_5,\tax_7} 
\end{align*}
	
Performing an analogue verification for the other candidates, we recognize that all of them are indeed queries and no two of them exhibit the same q-partition. Concretely, the q-partitions associated with the queries in the set above (minus the first set $\setof{\tax_3,\tax_5,\tax_6}$) are as follows: 
\begin{align}
\begin{split} \label{eq:canQ+canQP_for_diags1,2,3}
\Pt(\setof{\tax_3,\tax_5}) &= \tuple{\setof{\md_3},\setof{\md_1,\md_2},\emptyset} \\
\Pt(\setof{\tax_3,\tax_6}) &= \tuple{\setof{\md_2},\setof{\md_1,\md_3},\emptyset} \\
\Pt(\setof{\tax_5,\tax_6}) &= \tuple{\setof{\md_1},\setof{\md_2,\md_3},\emptyset} \\
\Pt(\setof{\tax_3}) &= \tuple{\setof{\md_2,\md_3},\setof{\md_1},\emptyset} \\
\Pt(\setof{\tax_5}) &= \tuple{\setof{\md_1,\md_3},\setof{\md_2},\emptyset} \\
\Pt(\setof{\tax_6}) &= \tuple{\setof{\md_1,\md_2},\setof{\md_3},\emptyset} \qed 
\end{split}
\end{align}
\end{example}

Based on these thoughts, we now define a CQ as follows:
\begin{definition}[Canonical Query]\label{def:canonical_query}
Let $\dpi$ be a DPI, $\mD \subseteq \minD_{\dpi}$ and $\emptyset\subset\dx{}\subset\mD$. Then $Q_{\mathsf{can}}(\dx{}) := (\mo \setminus U_{\dx{}}) \cap \Disc_\mD$ is \emph{the canonical query (CQ) w.r.t.\ the seed $\dx{}$} if $Q_{\mathsf{can}}(\dx{}) \neq \emptyset$. Else, $Q_{\mathsf{can}}(\dx{})$ is undefined.\footnote{We will often not mention the seed of a CQ if it is not relevant in a particular discussion.}
\end{definition}
To interpret this definition, note that $\mo \setminus U_{\dx{}}$ are exactly the common explicit entailments of $\setof{\mo_i^* \mid \md \in \dx{}}$ (cf.\ Prop.~\ref{prop:properties_of_q-partitions}.\ref{prop:properties_of_q-partitions:enum:query_is_set_of_common_ent}). 
Intuitively, the CQ extracts all discrimination sentences 
$\Disc_\mD$ from these entailments, thereby removing all elements that do not affect the q-partition (cf.\ Prop.~\ref{prop:expl_ent_query_must_neednot_mustnot_include_ax}). Recall that we requested a well-defined representative for q-partitions; hence, we specify this representative in a way it includes no obviously immaterial elements. 
\begin{remark}\label{rem:multiple_seed_lead_to_same_CQ}
There might be multiple seeds that lead to the same canonical query. That is, $Q_{\mathsf{can}}(\dx{i}) = Q_{\mathsf{can}}(\dx{j})$ might hold for seeds $\dx{i} \neq \dx{j}$ (because in spite of this difference $U_{\dx{i}} = U_{\dx{j}}$ might be true). 
For instance, let $\mD := \setof{\md_1,\md_2,\md_4,\md_5} = \{\{\tax_2,\tax_3\}, \{\tax_2,\tax_5\}$, $\{\tax_2,\tax_7\},  \{\tax_1,\tax_4.\tax_7\}\}$ (cf.\ Tab.\ref{tab:min_diags+conflicts_example_DPI_0}) be a set of leading diagnoses w.r.t.\ our running example DPI $\exdpi$ (Tab.~\ref{tab:example_dpi_0}). Then the seeds $\dx{1}:=\{\md_1,\md_5\}$ as well as $\dx{2}:=\{\md_1,\md_4,\md_5\}$ give rise to the same CQ $Q = \{\tax_5\}$.\qed  
\end{remark} 

It is trivial to see from Def.~\ref{def:canonical_query} that:
\begin{proposition}\label{prop:canonical_query_is_ee-query}
Any canonical query is an explicit-entailments query. 
\end{proposition}
\begin{proposition}\label{prop:canonical_query_unique_for_seed}
Let $\emptyset\subset\dx{}\subset\mD$. Then, if existent, the canonical query w.r.t.\ $\dx{}$ is unique. 
\end{proposition}
We now show that a CQ is indeed a query in the sense of Def.~\ref{def:query_q-partition}.
\begin{proposition}\label{prop:canonical_query_is_query}
If $Q$ is a canonical query, then $Q$ is a query.
\end{proposition} 
We define a canonical q-partition as a q-partition for which there is a canonical query with exactly this q-partition:
\begin{definition}[Canonical Q-Partition]\label{def:canonical_q-partition}
Let $\dpi$ be a DPI and $\mD \subseteq \minD_{\dpi}$ where $|\mD| \geq 2$. Let further $\Pt' = \langle\dx{},\dnx{}, \emptyset\rangle$ be a partition of $\mD$. Then we call $\Pt'$ a \emph{canonical q-partition (CQP)} iff $\Pt' = \Pt_\mD(Q_{\mathsf{can}}(\dx{}))$, i.e.\ $\langle\dx{},\dnx{}, \emptyset\rangle = \langle\dx{}(Q_{\mathsf{can}}(\dx{})),\dnx{}(Q_{\mathsf{can}}(\dx{})), \emptyset\rangle$. In other words, given the partition $\langle\dx{},\dnx{}, \emptyset\rangle$, the canonical query w.r.t.\ the seed $\dx{}$ must have exactly the q-partition $\langle\dx{},\dnx{}, \emptyset\rangle$.
\end{definition}
\begin{example}
Eq.~\eqref{eq:canQ+canQP_for_diags1,2,3} shows exactly all CQs and CQPs w.r.t.\ the set $\mD$ given in Ex.~\ref{ex:disc_ax}.\qed
\end{example}
\begin{remark}\label{rem:ad_CQP}
In general, the expression $\dx{}(Q_{\mathsf{can}}(\dx{}))$ is not necessarily equal to $\dx{}$. The $\dx{}$ within parentheses is the seed used to construct the CQ $Q_{\mathsf{can}}(\dx{})$ (cf.~Def.~\ref{def:canonical_query}), whereas the function $\dx{}(X)$, given a set of sentences $X$, maps to the set of all leading diagnoses $\md_i \in \mD$ for which $\mo_i^*$ entails $X$. As explained in Rem.~\ref{rem:multiple_seed_lead_to_same_CQ}, there might be different seeds that imply the same CQ. For instance, recalling the example given in Rem.~\ref{rem:multiple_seed_lead_to_same_CQ}, we have that $\dx{}(Q_{\mathsf{can}}(\dx{1})) = \dx{}(Q_{\mathsf{can}}(\dx{2})) = \{\md_1,\md_4,\md_5\} = \dx{2}$. We could thence say that a CQP is exactly a q-partition $\Pt$ whose $\dx{}(\Pt)$ set is \emph{stable} under the application of the functions $Q_{\mathsf{can}}()$ and $\dx{}()$, i.e.\ $\dx{}(Q_{\mathsf{can}}(\dx{}(\Pt))) = \dx{}(\Pt)$.\qed
\end{remark}
As a direct consequence of Def.~\ref{def:canonical_q-partition} and Prop.~\ref{prop:canonical_query_is_query} we obtain that each CQP is a q-partition:
\begin{corollary}\label{cor:canonical_q-partition_is_q-partition}
Each canonical q-partition is a q-partition.
\end{corollary}
Moreover, there is a one-to-one relationship between CQPs and CQs:
\begin{proposition}\label{prop:1-to-1_relation_between_CQs_and_CQPs}
Given a canonical q-partition, there is exactly one canonical query associated with it and vice versa. In particular, the unique canonical query associated with the canonical q-partition $\Pt$ is the set of sentences $Q_{\mathsf{can}}(\dx{}(\Pt))$.
\end{proposition}
No CQP is a suboptimal q-partition in the sense of our discussion in Sec.~\ref{sec:defs_and_props} in that each CQ discriminates among \emph{all} leading diagnoses: 
\begin{proposition}\label{prop:canonical_q-partition_has_empty_dz}
Any canonical q-partition $\langle\dx{},\dnx{}, \dz{}\rangle$ satisfies $\dz{} = \emptyset$.
\end{proposition}
\begin{proof}
Let $\Pt' = \langle\dx{},\dnx{}, \dz{}\rangle$ be a canonical q-partition. Then, by Def.~\ref{def:canonical_q-partition}, 
there is a canonical query $Q$ for which $\Pt'$ is equal to the q-partition $\Pt_\mD(Q) = \langle\dx{}(Q),\dnx{}(Q), \dz{}(Q)\rangle$ of $Q$. By Def.~\ref{def:canonical_query}, $\emptyset \subset Q \subseteq \mo$. Hence, by Prop.~\ref{prop:explicit-entailments_queries_have_empty_dz}, $\dz{}(Q) = \emptyset$ and thus $\dz{} = \emptyset$ must hold.
\end{proof}
Let us at this point illustrate the introduced notions by the following example:
\begin{table}[t]
	\scriptsize
	\centering
	\rowcolors[]{2}{gray!8}{gray!16} 
	\begin{tabular}{ c c c } 
		\rowcolor{gray!40}
		\toprule\addlinespace[0pt]
		Seed $\mathbf{S}$ & $\setof{i\,|\, \tax_i \in Q_{\mathsf{can}}(\mathbf{S})}$ & canonical q-partition \\ 
		\addlinespace[0pt]\midrule\addlinespace[0pt]
		$\setof{\md_5,\md_6}$ & $\setof{2,5,6} \cap \setof{1,2,3,4,7} = \setof{2}$ & $\tuple{\setof{\md_5,\md_6},\setof{\md_1},\emptyset}$ \\
		$\setof{\md_1,\md_6}$ & $\setof{1,5,6} \cap \setof{1,2,3,4,7} = \setof{1}$ & $\tuple{\setof{\md_1,\md_6},\setof{\md_5},\emptyset}$ \\
		$\setof{\md_1,\md_5}$ & $\setof{5,6} \cap \setof{1,2,3,4,7} = \emptyset$ & $\times$ \\
		$\setof{\md_1}$ & $\setof{1,4,7} \cap \setof{1,2,3,4,7} = \setof{1,4,7}$ & $\tuple{\setof{\md_1},\setof{\md_5,\md_6},\emptyset}$ \\
		$\setof{\md_5}$ & $\setof{2,3} \cap \setof{1,2,3,4,7} = \setof{2,3}$ & $\tuple{\setof{\md_5},\setof{\md_1,\md_6},\emptyset}$ \\
		$\setof{\md_6}$ & $\setof{1,2} \cap \setof{1,2,3,4,7} = \setof{1,2}$ & $\tuple{\setof{\md_6},\setof{\md_1,\md_5},\emptyset}$ \\
		\addlinespace[0pt]\bottomrule 
	\end{tabular}
	\caption{\small All CQs and associated CQPs w.r.t.\ $\mD = \setof{\md_1,\md_5,\md_6}$ (cf.\ Tab.~\ref{tab:min_diags+conflicts_example_DPI_0}) and the example DPI $\exdpi$ given by Tab.~\ref{tab:example_dpi_0}.}
	\label{tab:Qcan+canQPart_for_example_dpi_0}
\end{table}
\begin{example}\label{ex:canonical_queries_q-partitions}
Consider the leading diagnoses 
\begin{align} \label{eq:ex_canonical_queries_q-partitions:leading_diags}
\mD = \setof{\md_1,\md_5,\md_6} = \setof{\setof{\tax_2,\tax_3},\setof{\tax_1,\tax_4,\tax_7},\setof{\tax_3,\tax_4,\tax_7}} \quad \mbox{(cf.\ Tab.~\ref{tab:min_diags+conflicts_example_DPI_0})}
\end{align}
w.r.t.\ our example DPI $\exdpi$ (Tab.~\ref{tab:example_dpi_0}).
The potential solution KBs given this set of leading diagnoses $\mD$ are $\setof{\mo_1^{*},\mo_2^{*},\mo_3^{*}}$ (cf.\ Eq.~\eqref{eq:sol_ont_candidate}) where 
\begin{align*}
\mo_1^{*} &= \setof{\tax_1,\tax_4,\tax_5,\tax_6,\tax_7,\tax_8,\tax_9,\tp_1} \\
\mo_2^{*} &= \setof{\tax_2,\tax_3,\tax_5,\tax_6,\tax_8,\tax_9,\tp_1} \\
\mo_3^{*} &= \setof{\tax_1,\tax_2,\tax_5,\tax_6,\tax_8,\tax_9,\tp_1} 
\end{align*}
The discrimination sentences $\Disc_\mD$ are $U_\mD \setminus I_\mD = \setof{\tax_1,\tax_2,\tax_3,\tax_4,\tax_7}$. Tab.~\ref{tab:Qcan+canQPart_for_example_dpi_0} lists all possible seeds $\mathbf{S}$ (i.e.\ proper non-empty subsets of $\mD$) and, if existent, the respective (unique) CQ $Q_{\mathsf{can}}(\mathbf{S})$ as well as the associated (unique) CQP. Note that the CQ for the seed $\mathbf{S} = \setof{\md_1,\md_5}$ is undefined which is why there is no CQP with a \mbox{$\dx{}$-set} $\setof{\md_1,\md_5}$. 
This holds since $\mo \setminus U_{\mathbf{S}} = \setof{\tax_1,\dots,\tax_7}\setminus (\setof{\tax_2,\tax_3} \cup \setof{\tax_1,\tax_4,\tax_7}) = \setof{\tax_5,\tax_6}$ has an empty intersection with $\Disc_\mD$. So, by Def.~\ref{def:canonical_query}, $Q_{\mathsf{can}}(\mathbf{S}) = \emptyset$.

Additionally, we point out that there is no query $Q$ (and hence no q-partition) -- and therefore not just no \emph{canonical} query -- for which $\dx{}(Q)$ corresponds to $\setof{\md_1,\md_5}$. Because for such $Q$ to exist, $\md_6\in\dnx{}(Q)$ must hold. Under this assumption, there must be a set of common entailments $Ents$ of $\mo_1^{*}$ and $\mo_5^{*}$ (cf.\ Prop.~\ref{prop:properties_of_q-partitions}.\ref{prop:properties_of_q-partitions:enum:query_is_set_of_common_ent}) which, along with $\mo_6^{*}$, violates $\RQ$ or $\Tn$ (cf.\ Prop.~\ref{prop:partition}). 
As all sentences in $Ents$ are entailed by $\mo \setminus (\md_1 \cup \md_5) \cup \mb \cup U_\Tp$ as well, 
and due to the observation that $\mo \setminus (\md_1 \cup \md_5) \subset \mo \setminus \md_6$, by monotonicity of Propositional Logic, every common entailment of $\mo_1^{*}$ and $\mo_5^{*}$ is also an entailment of $\mo_6^{*}$. Due to the definition of a solution KB (cf.\ Def.~\ref{def:solution_KB}), which implies that $\mo_6^*$ neither violates $\RQ$ nor $\Tn$, this means that there cannot be a query $Q$ satisfying $\dx{}(Q)=\setof{\md_1,\md_5}$. So, obviously, in this example every q-partition (with empty $\dz{}$) is also a CQP.\qed
\end{example}
\paragraph{Advantages of Using Canonical Queries and Q-Partitions.}
\label{etc:advantages_of_CQs+CQPs}The restriction to the consideration of only CQs during phase P1 has some nice implications: 
\begin{enumerate}
\item \label{enum:advantages_of_CQs:no_reasoner} CQs and CQPs can be generated by cheap set operations. \emph{No inference engine calls} are required.   (Prop.~\ref{prop:ee-query_q-partition_construction_by_means_of_set_comparison}) \\
\textbf{Remark:} The main causes why this is possible are the $\subseteq$-minimality of leading diagnoses and the monotonicity of the logic underlying the DPI. Intuitively, the concepts of CQs and CQPs just leverage the already available information (obtained by logical reasoning during diagnosis computation) inherent in the leading diagnoses in a clever way to avoid any further reasoning during the q-partition search.
\item \label{enum:advantages_of_CQs:no_unnecessary_candidates} Each CQ is a query in $\mQ_\mD$ for sure, no verification of its q-partition (as per Def.~\ref{def:query_q-partition}) is required, thence \emph{no unnecessary candidates} (which turn out to be no queries) \emph{are generated}. (Prop.~\ref{prop:canonical_query_is_query})
\item \label{enum:advantages_of_CQs:full_discrimination_power} Automatic computation of \emph{only queries $Q$ with full discrimination power regarding $\mD$} 
(i.e.\ of those $Q$ with empty $\dz{}(Q)$). (Prop.~\ref{prop:canonical_q-partition_has_empty_dz})
\item \label{enum:advantages_of_CQs:no_duplicates} \emph{No duplicate queries or q-partitions are generated} as there is a one-to-one relationship between CQs and CQPs. (Prop.~\ref{prop:1-to-1_relation_between_CQs_and_CQPs})
%
\item \label{enum:advantages_of_CQs:independence_of_entailments_computed_by_reasoner} The explored search space for q-partitions is \emph{not dependent on} the particular (entailments output by an) \emph{inference engine}, as CQs are explicit-entailments queries. (Prop.~\ref{prop:canonical_query_is_ee-query})
\end{enumerate}

We emphasize that all these properties do \textbf{not} hold for normal (i.e.\ non-canonical) queries and q-partitions. The overwhelming impact on the query computation time of this strategy of first narrowing down the search space to only such computationally ``benign'' queries to find an optimal q-partition, and the later reintroduction of the full query search space when searching for an optimal query for this fixed optimal q-partition will be demonstrated in Sec.~\ref{sec:eval}.

\paragraph{The Q-Partition Search Procedure.}
Now, having at hand the notion of a CQP, we describe the sound and complete (heuristic) CQP search procedure performed in P1.

A (heuristic) search problem \citep{russellnorvig2010} is defined by the \emph{initial state}, a \emph{successor function} enumerating all direct neighbor states of a state, the \emph{step costs} from a state to a successor state, the \emph{goal test} to determine if a given state is a goal state or not, and (possibly) some \emph{heuristic function} to estimate the remaining effort from each state towards a goal state. 

We define the initial state, i.e.\ the partition $\langle\dx{},\dnx{},\dz{}\rangle$ to start with, as $\langle\emptyset,\mD,\emptyset\rangle$ where $\mD$ is the given set of leading diagnoses.
The idea is to transfer diagnoses step-by-step from $\dnx{}$ to $\dx{}$ to construct all CQPs \emph{systematically}. The step costs are irrelevant in our application, as only the found q-partition \emph{as such} counts. In other words, the required solution is a q-partition and not the path to reach it from the initial state.
Heuristics derived from the given QSM $m$ can be (optionally) integrated into the search to enable faster convergence to a goal state. A q-partition $\Pt$ is 
a goal state if it optimizes $m$ up to a given threshold $t_m$ (cf.\ \citep{dekleer1987}, see Alg.~\ref{algo:query_comp}). To make this precise, let us call a query $Q$ \emph{optimal w.r.t.\ $m$ and $t_m$} iff $|m(Q) - m_{opt}| \leq t_m$ where $m_{opt}$ is the optimal theoretically achievable value of $m$. Then $\Pt$ is a \emph{goal}\label{etc:goal_state} iff an arbitrary query $Q$ (e.g.\ the CQ) for $\Pt$ is optimal w.r.t.\ $m$ and $t_m$. Recall that each query $Q$ for $\Pt$ yields the same $m(Q)$ as $m$ is only dependent on the q-partition of a query (cf.\ the discussion on page~\pageref{etc:decoupling_of_optimization_steps}).\footnote{Therefore, we will sometimes write $m(\Pt)$ to denote $m(Q)$ for arbitrary $Q$ for which $\Pt_\mD(Q) = \Pt$.}

The search strategy adopted by the CQP search in phase P1 can be characterized as depth-first, local best-first 
backtracking strategy. We now explicate informally how the search tree is evolved by means of this strategy. 
Starting from the initial partition $\tuple{\emptyset,\mD,\emptyset}$ 
(or:\ root node), the search will proceed \emph{downwards} until (a)~a goal q-partition has been found, (b)~all successors of the currently analyzed q-partition (or:\ node) have been pruned (based on the QSM $m$) or (c)~there are no successors of the currently analyzed q-partition (or:\ node). This behavior is implied by the \emph{depth-first} strategy. 

At each current q-partition (or:\ node), the focus moves on to the best \emph{successor} q-partition (or: \emph{child} node), possibly according to some given heuristic function (based on the QSM $m$). This behavior is implied by the \emph{local best-first} strategy. 

The search procedure is ready to backtrack in case all successors (or:\ child nodes) of a q-partition (or:\ node) have been explored or pruned and no goal q-partition has been found yet. In this case, the next-best unexplored sibling of the node will be analyzed next according to the used local best-first depth-first strategy. This behavior is implied by the \emph{backtracking} strategy. 

We emphasize that this local best-first depth-first backtracking strategy involves a \emph{linear space complexity}, as opposed to a (global) best-first strategy.

What we have not formally specified yet is the used successor function. We dedicate the rest of the q-partition search procedure description to the derivation and definition of the successor function. The soundness and completeness of this function with regard to the computation of all and only CQPs for $\mD$ will guarantee the soundness and completeness w.r.t.\ CQPs of the overall search procedure.  

In order to characterize a suitable successor function, we define a direct neighbor of a q-partition as follows:
\begin{definition}\label{def:minimal_transformation}
Let $\dpi$ be a DPI, $\mD \subseteq \minD_{\dpi}$ and
$\Pt_i := \langle \dx{i},\dnx{i},\emptyset\rangle$,
$\Pt_j := \langle \dx{j},\dnx{j},\emptyset\rangle$ be partitions of $\mD$. 
Then, $\Pt_i \mapsto \Pt_j$ is a \emph{minimal $\dx{}$-transformation} from $\Pt_i$ to $\Pt_j$ iff $\Pt_j$ is a CQP, $\dx{i} \subset \dx{j}$ and there is no CQP $\langle \dx{k},\dnx{k},\emptyset\rangle$ with $\dx{i} \subset \dx{k} \subset \dx{j}$. 
A CQP $\Pt'$ is called a \emph{successor of a partition $\Pt$} iff $\Pt'$ results from $\Pt$ by a minimal $\dx{}$-transformation.
\end{definition}
Intuitively, a successor $\Pt'$ of a partition $\Pt$ results from the transfer of a $\subseteq$-minimal set of diagnoses from $\dnx{}(\Pt)$ (comprising all leading diagnoses $\mD$ if $\Pt$ is the initial state) to $\dx{}(\Pt)$ (empty for initial state $\Pt$) such that the resulting partition $\Pt'$ is a CQP.

The successor function $S_{\mathsf{all}}$ then maps a given partition $\Pt$ of $\mD$ to the set of all its possible successors, i.e.\ to the set including all CQPs that result from $\Pt$ by a minimal $\dx{}$-transformation.
The reliance upon a minimal $\dx{}$-transformation guarantees that the search is complete w.r.t.\ CQPs because one cannot skip over any CQPs when transforming a state into a direct successor state. As the initial state is not a q-partition (cf.\ Def.~\ref{def:query_q-partition}), the definition of the successor function $S_{\mathsf{all}}$ involves specifying 
\begin{itemize}[noitemsep,topsep=5pt]
	\item a function $S_{\mathsf{init}}$ that maps \emph{the initial state} to the set of all CQPs that can be reached by it by a single minimal $\dx{}$-transformation, and
	\item a function $S_{\mathsf{next}}$ that maps \emph{any CQP} to the set of all CQPs that can be reached by it by a single minimal $\dx{}$-transformation.
\end{itemize}
$S_{\mathsf{init}}$ can be easily specified by means of the following proposition which is a consequence of Prop.~\ref{prop:properties_of_q-partitions}.\ref{prop:properties_of_q-partitions:enum:D+=d_i_is_q-partition_and_lower_bound_of_queries}.
\begin{proposition}\label{prop:--di,MD-di,0--_is_canonical_q-partition_for_all_di_in_mD}
Let $\dpi$ be a DPI and $\mD \subseteq \minD_{\dpi}$ where $|\mD|\geq 2$. Then, $\langle\{\md_i\},\mD \setminus \{\md_i\},\emptyset\rangle$ is a canonical q-partition for all $\md_i \in \mD$.
\end{proposition}
Since only one diagnosis is transferred from the initial $\dnx{}$ set $\mD$ to the $\dx{}$-set of the partition to obtain any CQP as per Prop.~\ref{prop:--di,MD-di,0--_is_canonical_q-partition_for_all_di_in_mD} from the initial state, it is clear that all these CQPs indeed result from the application of a minimal $\dx{}$-transformation from the initial state (\emph{soundness}). As for all other CQPs the $\dx{}$ set differs by more than one diagnosis from the initial $\dx{}$ set $\emptyset$ (and thus includes a proper superset of $\setof{\md}$ for some $\md \in \mD$), 
it is obvious that all other CQPs do not result from the initial state by some minimal $\dx{}$-transformation (\emph{completeness}). Hence:
\begin{proposition}\label{prop:S_init_sound+complete}
	Given the initial state $\Pt_0 := \tuple{\emptyset,\mD,\emptyset}$, the function 
	\[S_{\mathsf{init}}: \tuple{\emptyset,\mD,\emptyset} \mapsto \setof{\tuple{\setof{\md},\mD\setminus\setof{\md},\emptyset}\,|\,\md\in\mD}\] 
	is sound and complete, i.e.\ it produces all and only (canonical) q-partitions resulting from $\Pt_0$ by minimal $\dx{}$-transformations.
\end{proposition}
Note that in fact all q-partitions with empty $\dz{}$, not only all CQPs, that a reachable from $\Pt_0$ by a minimal $\dx{}$-transformation, are generated by $S_{\mathsf{init}}$. The reason for this is that there are no other possibilities to build (q-)partitions with a singleton $\dx{}$ set and an empty $\dz{}$ set.

In order to define $S_{\mathsf{next}}$, we utilize Prop.~\ref{prop:suff+nec_criteria_when_partition_is_q-partition} which 
provides sufficient and necessary criteria when a partition of $\mD$ is a CQP.
\begin{proposition}\label{prop:suff+nec_criteria_when_partition_is_q-partition}
Let $\dpi := \tuple{\mo,\mb,\Tp,\Tn}_\RQ$ be a DPI, $\mD \subseteq \minD_{\dpi}$ and $\Pt = \langle \dx{},\dnx{},\emptyset\rangle$ be a partition of $\mD$ with $\dx{}\neq \emptyset$ and $\dnx{}\neq \emptyset$. Then, $\Pt$ is a canonical q-partition iff 
\begin{enumerate}[noitemsep, topsep=5pt]
	\item $U_{\dx{}} \subset U_{\mD}$ and
	\item there is no $\md_j \in \dnx{}$ such that $\md_j \subseteq U_{\dx{}}$.
\end{enumerate}
\end{proposition}
The following example uses Prop.~\ref{prop:suff+nec_criteria_when_partition_is_q-partition} to check the CQP property for two candidate partitions:
\begin{example}\label{ex:Pt_is_q-partition_iff}
Let 
(by referring to $\tax_i$ by $i$ for clarity)
\begin{align}
\begin{split} \label{eq:ex:Pt_is_q-partition_iff:Pt_1}
\Pt_1 :=& \tuple{\setof{\md_1,\md_2,\md_3},\setof{\md_4,\md_5,\md_6},\emptyset} \\
=& \tuple{\setof{\setof{2,3},\setof{2,5},\setof{2,6}},\setof{\setof{2,7},\setof{1,4,7},\setof{3,4,7}},\emptyset} 
\end{split} \\
\begin{split} \label{eq:ex:Pt_is_q-partition_iff:Pt_2}
\Pt_2 :=& \tuple{\setof{\md_1,\md_2,\md_5},\setof{\md_3,\md_4,\md_6},\emptyset} \\
=& \tuple{\setof{\setof{2,3},\setof{2,5},\setof{1,4,7}},\setof{\setof{2,6},\setof{2,7},\setof{3,4,7}},\emptyset}  
\end{split}
\end{align}
be two partitions of the leading diagnoses $\mD = \minD_{\exdpi}$ for our example DPI $\exdpi$ in Tab.~\ref{tab:example_dpi_0}.
Then $U_{\dx{}(\Pt_1)} = \setof{2,3,5,6}$, $U_{\dx{}(\Pt_2)} = \setof{1,2,3,4,5,7}$ and $U_\mD = \setof{1,\dots,7}$. Since $U_{\dx{}(\Pt_1)} \subset U_\mD$ as well as $U_{\dx{}(\Pt_2)} \subset U_\mD$, the first condition of Prop.~\ref{prop:suff+nec_criteria_when_partition_is_q-partition} is satisfied for both partitions of $\mD$. As to the second condition, given that $7 \in \md_4,\md_5,\md_6$, but $7 \not\in U_{\dx{}(\Pt_1)}$, it holds that $\md_j \not\subseteq U_{\dx{}(\Pt_1)}$ for all $j \in \setof{4,5,6}$. Therefore, $\Pt_1$ is a CQP. $\Pt_2$, on the other hand, is not a CQP because, e.g., $\md_4 = \setof{2,7} \subset \setof{1,2,3,4,5,7} = U_{\dx{}(\Pt_2)}$ (second condition of Prop.~\ref{prop:suff+nec_criteria_when_partition_is_q-partition} violated).
We point out that one can verify that there is in fact no query with q-partition $\Pt_2$. That is, the partition $\Pt_2$ of $\mD$ is no CQP \emph{and} no q-partition.
\qed 
\end{example}
Let in the following for a DPI $\langle\mo,\mb,\Tp,\Tn\rangle_\RQ$ and a partition $\Pt_k = \langle \dx{k}, \dnx{k}, \emptyset\rangle$ of $\mD$ and all $\md_i \in \mD \subseteq \minD_{\langle\mo,\mb,\Tp,\Tn\rangle_\RQ}$ 
\begin{align}
\md^{(k)}_i := \md_i \setminus U_{\dx{k}} \label{eq:md_i^(k)}
\end{align}
The next corollary establishes the relationship between Eq.~\eqref{eq:md_i^(k)} and CQPs based on Prop.~\ref{prop:suff+nec_criteria_when_partition_is_q-partition}:
\begin{corollary}\label{cor:not_q-partition_iff_md_i^(k)=emptyset_for_md_i_in_dnx_k}
	Let $\mD\subseteq\minD_{\langle\mo,\mb,\Tp,\Tn\rangle_\RQ}$, $\Pt_k = \langle \dx{k}, \dnx{k}, \emptyset\rangle$ a partition of $\mD$ with $\dx{k},\dnx{k} \neq \emptyset$ and $U_{\dx{k}} \subset U_\mD$. Then $\Pt := \tuple{\dx{k},\dnx{k},\emptyset}$ is a canonical q-partition iff 
	\begin{enumerate}
		\item $\md_i^{(k)} = \emptyset$ for all $\md_i \in \dx{k}$, and
		\item $\md_i^{(k)} \neq \emptyset$ for all $\md_i \in \dnx{k}$.
	\end{enumerate}
\end{corollary}
\begin{example}\label{ex:md_i^(k)}
For the purpose of illustration, let us examine both partitions discussed in Ex.~\ref{ex:Pt_is_q-partition_iff} again by means of Cor.~\ref{cor:not_q-partition_iff_md_i^(k)=emptyset_for_md_i_in_dnx_k}. To this end, we write the partitions $\Pt_k = \langle \dx{k},\dnx{k},\emptyset\rangle$ (for $k = 1,2$) in the form
 \[\left\langle\setof{\md_i^{(k)}\,|\,\md_i \in \dx{k}},\setof{\md_i^{(k)}\,|\,\md_i \in \dnx{k}},\emptyset\right\rangle\] 
Moreover, natural numbers $j$ in the sets again refer to the respective sentences $\tax_j$ (as in Ex.~\ref{ex:Pt_is_q-partition_iff}).
For the partition $\Pt_1$ in Eq.~\eqref{eq:ex:Pt_is_q-partition_iff:Pt_1} (i.e.\ for $k = 1$) we get $\tuple{\setof{\emptyset,\emptyset,\emptyset},\setof{\setof{7},\setof{1,4,7},\setof{4,7}},\emptyset}$.
%
Since all $\md_i^{(1)}$ in $\dnx{1}$ are non-empty (and all $\md_i^{(1)}$ in $\dx{1}$ are empty, which is always trivially fulfilled), Cor.~\ref{cor:not_q-partition_iff_md_i^(k)=emptyset_for_md_i_in_dnx_k} confirms the result we obtained in Ex.~\ref{ex:Pt_is_q-partition_iff}, namely that $\Pt_1$ is a CQP.

On the contrary, $\Pt_2$ (i.e.\ $k=2$) given by Eq.~\eqref{eq:ex:Pt_is_q-partition_iff:Pt_2} is not a CQP according to Cor.~\ref{cor:not_q-partition_iff_md_i^(k)=emptyset_for_md_i_in_dnx_k} since, represented in the same form as $\Pt_1$ above, $\Pt_2$ evaluates to $\tuple{\setof{\emptyset,\emptyset,\emptyset},\setof{\setof{6},\emptyset,\emptyset},\emptyset}$. We see that $\md_4^{(2)}$ and $\md_6^{(2)}$ are both equal to the empty set, but $\md_4,\md_6 \in \dnx{2}$ which must not be the case if $\Pt_2$ is a CQP due to Cor.~\ref{cor:not_q-partition_iff_md_i^(k)=emptyset_for_md_i_in_dnx_k}.
Again, the result we got in Ex.~\ref{ex:Pt_is_q-partition_iff} is successfully verified. 
In fact, $\Pt_2$ can be transformed into a CQP by transferring all diagnoses $\md_i \in \dnx{2}$ where $\md_i^{(2)} = \emptyset$, i.e.\ $\md_4$ and $\md_6$, to $\dx{2}$. The resulting partition, in this case $\tuple{\setof{\md_1,\md_2,\md_4,\md_5,\md_6},\setof{\md_3},\emptyset}$, is then a CQP according to Cor.~\ref{cor:not_q-partition_iff_md_i^(k)=emptyset_for_md_i_in_dnx_k}. This necessary shift of diagnoses from $\dnx{k}$ to $\dx{k}$ is also the main idea exploited in the specification of the function $S_{\mathsf{next}}$. The next example picks up on this issue in more detail.\qed 
\end{example}
The next example analyzes situations that might occur when transferring a single diagnosis from the $\dnx{k}$ set of a CQP to its $\dx{k}$ set in order to generate a successor CQP of it. In particular, it makes evident that (i)~not every diagnosis in $\dnx{k}$ might be eligible to be shifted to $\dx{k}$ in terms of a minimal $\dx{}$-transformation and (ii)~the transfer of some (eligible) diagnosis $\md$ might necessitate the transfer of other diagnoses, which we informally call \emph{necessary followers of $\md$} in the following. That is, minimal $\dx{}$-transformations might involve the simultaneous relocation of multiple diagnoses.   
\begin{example}\label{ex:nec_follower} 
We continue discussing our running example DPI $\exdpi$ (see Tab.~\ref{tab:example_dpi_0}). Assume as in Ex.~\ref{ex:md_i^(k)} that $\mD = \minD_{\exdpi}$. Let us consider the CQP $\Pt_k := \tuple{\setof{\md_1,\md_2},\setof{\md_3,\md_4,\md_5,\md_6},\emptyset}$. Written in the form 
\begin{align}\label{eq:standard_representation_of_can_q-partitions}
\tuple{U_{\dx{k}},\setof{\md_i^{(k)}\,|\,\md_i \in \dnx{k}}} \qquad \text{(\emph{standard representation of CQPs})}
\end{align}
$\Pt_k$ is given by 
\begin{align} \label{eq:ex_necessary_follower:U_D+_and_traits}
\tuple{\setof{2,3,5},\setof{\setof{6},\setof{7},\setof{1,4,7},\setof{4,7}}}
\end{align}
Note, by the definition of $\md_i^{(k)}$ (see Eq.~\eqref{eq:md_i^(k)}) the sets that are elements of the right-hand set of this tuple are exactly the diagnoses in $\dnx{k}$ reduced by the elements that occur in the left-hand set of this tuple. For instance, $\setof{6}$ results from $\md_3 \setminus U_{\dx{k}} = \setof{2,6} \setminus \setof{2,3,5}$. We now analyze $\Pt_k$ w.r.t.\ necessary followers of the diagnoses in $\dnx{k}$. The set of necessary followers of $\md_3$ is empty. The same holds for $\md_4$. However, $\md_5$ has two necessary followers, namely $\setof{\md_4,\md_6}$, whereas $\md_6$ has one, given by $\md_4$. The intuition is that transferring $\md_3$ with  $\md_3^{(k)} = \setof{6}$ to $\dx{k}$ yields the new set $\dx{k^*} := \setof{\md_1,\md_2,\md_3}$ with $U_{\dx{k^*}} = \setof{2,3,5,6}$. This new set however causes no set $\md_i^{(k^*)}$ for $\md_i$ in $\dnx{k^*}$ to become the empty set. Hence the transfer of $\md_3$ necessitates no relocations of any other elements in $\dnx{k^*}$. 

For $\md_6$, the situation is different. Here the new set $U_{\dx{k^*}} = \setof{2,3,4,5,7}$ which implicates \[\setof{\md_3^{(k^*)},\md_4^{(k^*)},\md_5^{(k^*)}} = \setof{\setof{6},\emptyset,\setof{1}}\] 
for $\md_i$ in $\dnx{k^*}$.
Application of Cor.~\ref{cor:not_q-partition_iff_md_i^(k)=emptyset_for_md_i_in_dnx_k} now yields that $\Pt_{k^*}$ is not a CQP due to the empty set $\md_4^{(k^*)}$. As explained in Ex.~\ref{ex:md_i^(k)}, turning $\Pt_{k^*}$ into a CQP requires the transfer of all diagnoses $\md_i$ associated with empty sets $\md_i^{(k^*)}$ to $\dx{k^*}$.

Importantly, notice that the CQP $\Pt_{s^*} := \tuple{\setof{\md_1,\md_2,\md_4,\md_6},\setof{\md_3,\md_5},\emptyset}$ resulting from this cannot be reached from $\Pt_k$ by means of a minimal $\dx{}$-transformation. The reason is that $\Pt_{s} := \tuple{\setof{\md_1,\md_2,\md_4},\setof{\md_3,\md_5,\md_6},\emptyset}$ is a CQP as well and results from $\Pt_k$ by fewer changes to $\dx{k}$ than $\Pt_{s^*}$. In fact, only CQPs resulting from the transfer of diagnoses $\md_i \in \dnx{k}$ with $\subseteq$-minimal $\md_i^{(k)}$ to $\dx{k}$ are reachable from $\Pt_k$ by a minimal $\dx{}$-transformation. As becomes evident from Eq.~\eqref{eq:ex_necessary_follower:U_D+_and_traits}, only the CQPs created from $\Pt_k$ by means of a shift of $\md_3$ (with $\md_3^{(k)} = \setof{6}$) or $\md_4$ ($\setof{7}$) to $\dx{k}$ are successors of $\Pt_k$ compliant with the definition of a minimal $\dx{}$-transformation (Def.~\ref{def:minimal_transformation}).  

Finally, let us inspect the CQP $\Pt_r = \setof{\setof{\md_2,\md_3,\md_4,\md_5},\setof{\md_1,\md_6},\emptyset}$ with the standard representation $\tuple{\setof{1,2,4,5,6,7},\setof{\setof{3},\setof{3}}}$. 
We find that there are no CQPs reachable by a minimal $\dx{}$-transformation from $\Pt_r$. The reason behind this is that transferring either of the diagnoses $\md_1,\md_6$ in $\dnx{r}$ to $\dx{r}$ requires the transfer of the other, since both are necessary followers of each other. An empty $\dnx{}$-set -- and hence no (canonical) q-partition -- would be the result. Generally, two diagnoses $\md_i,\md_j \in \dnx{r}$ bear a necessary follower relation to one another (w.r.t.\ a CQP $\Pt_r$) iff $\md_i^{(r)} = \md_i^{(r)}$. \qed
\end{example}
The previous examples indicate that the sets $\md_i^{(k)}$ (see Eq.~\eqref{eq:md_i^(k)}) for $\md_i \in \dnx{k}$ play a central role when it comes to specifying the successors of the CQP $\Pt_k$ in terms if minimal $\dx{}$-transformations. For this reason we dedicate a special name to them:
\begin{definition}\label{def:trait}
Let $\Pt_k = \langle \dx{k}, \dnx{k}, \emptyset\rangle$ be a CQP and $\md_i \in \dnx{k}$. Then $\md_i^{(k)}$ is called \emph{the trait of $\md_i$ (w.r.t.\ $\Pt_k$)}. The relation associating two diagnoses in $\dnx{k}$ iff their trait is equal is denoted by $\sim_k$.	
\end{definition} 
Clearly: 
\begin{proposition}
$\sim_k$ is a equivalence relation (over $\dnx{k}$).
\end{proposition}
Let us denote the equivalence classes w.r.t.\ $\sim_k$ by $[\md_i]^{\sim_k}$ where $\md_i \in \dnx{k}$.
\begin{example}\label{ex:equiv_rel+traits}
Consider the CQP $\Pt_k = \setof{\setof{\md_4,\md_5},\setof{\md_1,\md_2,\md_3,\md_6},\emptyset}$ related to our running example DPI $\exdpi$ (Tab.~\ref{tab:example_dpi_0}). Using the standard representation of CQPs (introduced by Eq.~\eqref{eq:standard_representation_of_can_q-partitions}) this q-partition amounts to $\tuple{\setof{1,2,4,7},\setof{\setof{3},\setof{5},\setof{6},\setof{3}}}$. Now, \[
\sim_k \;= \setof{\tuple{\md_1,\md_1},\tuple{\md_2,\md_2},\tuple{\md_3,\md_3},\tuple{\md_6,\md_6},\tuple{\md_1,\md_6},\tuple{\md_6,\md_1}}
\] 
and the equivalence classes w.r.t.\ $\sim_k$ are 
\[
\setof{[\md_1]^{\sim_k},[\md_2]^{\sim_k},[\md_3]^{\sim_k}} = \setof{\setof{\md_1,\md_6},\setof{\md_2},\setof{\md_3}}
\]
Note that $[\md_1]^{\sim_k} = [\md_6]^{\sim_k}$ holds.
The number of the equivalence classes gives an upper bound of the number of successors resulting from a minimal $\dx{}$-transformation from $\Pt_k$. 
The traits of these equivalence classes are given by 
\[
\setof{\setof{3},\setof{5},\setof{6}}
\]
These can be just read from the standard representation above. Since all traits are $\subseteq$-minimal, there are exactly three successors of $\Pt_k$ as per Def.~\ref{def:minimal_transformation}.\qed
\end{example}
The concept of a trait and the relation $\sim_k$ enable the formal characterization $S_{\mathsf{next}}$ as follows:
\begin{corollary}\label{cor:S_next_sound+complete}
Let $\Pt_k := \tuple{\dx{k},\dnx{k},\emptyset}$ be a canonical q-partition, $\mathbf{EQ}^{\sim_k}$ be the set of all equivalence classes w.r.t.\ $\sim_k$ and
\begin{align*}
\mathbf{EQ}^{\sim_k}_{\subseteq} &:= \setof{[\md_i]^{\sim_k} \mid\, \not\exists j: \md_j^{(k)} \subset \md_i^{(k)} } 
\end{align*}
i.e.\ $\mathbf{EQ}^{\sim_k}_{\subseteq}$ includes all equivalence classes w.r.t.\ $\sim_k$ which have a $\subseteq$-minimal trait.
Then, the function 
\begin{align*}
S_{\mathsf{next}}: \tuple{\dx{k},\dnx{k},\emptyset} \mapsto \begin{cases}
\setof{\tuple{\dx{k}\cup E,\dnx{k}\setminus E,\emptyset}\,|\,E \in \mathbf{EQ}^{\sim_k}_{\subseteq} } & \text{if $|\mathbf{EQ}^{\sim_k}| \geq 2$}\\
\emptyset & \text{otherwise}
\end{cases} 
\end{align*}
is sound and complete, i.e.\ it produces 
all and only canonical q-partitions resulting from $\Pt_k$ by minimal $\dx{}$-transformations.
\end{corollary}
Prop.~\ref{prop:S_init_sound+complete} and Cor.~\ref{cor:S_next_sound+complete} immediately entail that the successor function $S_{\mathsf{all}}$, defined as $S_{\mathsf{init}}$ if the input is the initial state $\tuple{\emptyset,\mD,\emptyset}$ and as $S_{\mathsf{next}}$ otherwise, is sound and complete as regards successor CQP generation in terms of Def.~\ref{def:minimal_transformation}. Let the backtracking algorithm implemented by phase P1 return the best found canonical q-partition, given that all possible states have been explored and no goal has been found. Then:
\begin{theorem}\label{theorem:P1_sound_complete}
The backtracking algorithm performed by phase P1 using successor function $S_{\mathsf{all}}$ is sound and complete. That is:
\begin{itemize}[noitemsep,topsep=5pt]
	\item \emph{(Completeness)} If there is a canonical q-partition which is a goal (as defined on page~\pageref{etc:goal_state}), then P1 returns a canonical q-partition which is a goal.
	\item \emph{(Soundness)} If P1 returns a q-partition $\Pt$, then $\Pt$ is canonical. Further, $\Pt$ is a goal, if a goal exists. Otherwise, $\Pt$ is the best existing canonical q-partition w.r.t.\ $m$ and $t_m$.
\end{itemize}
\end{theorem}
\begin{proof}
The theorem follows directly from the soundness and completeness of the successor function $S_{\mathsf{all}}$ (Prop.~\ref{prop:S_init_sound+complete}, Cor.~\ref{cor:S_next_sound+complete}) and the fact that backtracking algorithms over finite search spaces using a sound and complete successor function are sound and complete (cf.\ \citep[Chap.~4]{CP_Handbook}).
\end{proof}

\paragraph{Size of the Explored Search Space.}
Through Cor.~\ref{cor:not_q-partition_iff_md_i^(k)=emptyset_for_md_i_in_dnx_k} 
we realize that $U_{\dx{}}$ already defines a CQP uniquely.
The next corollary exploits this fact to compute the number of CQPs w.r.t.\ a set of leading minimal diagnoses $\mD$.
Note that the number of CQPs w.r.t.\ $\mD$ is equal to the number of CQs w.r.t.\ $\mD$ (Prop.~\ref{prop:1-to-1_relation_between_CQs_and_CQPs}) which in turn constitutes a lower bound of the number of \emph{all} (semantically different) \emph{queries} w.r.t.\ $\mD$. Furthermore, since $S_{\mathsf{all}}$ is a sound and complete successor function 
(and thus prohibits the non-consideration of any CQP), 
the number of CQPs w.r.t.\ $\mD$ is exactly the size of the search space explored by phase P1 in the worst case. The worst case will occur if there are no CQPs that are goal states w.r.t.\ the QSM $m$ and the given optimality threshold $t_m$, or if there is a single goal state which happens to be explored only after all other states have been explored. Note that the full search space will rarely be completely explored in practice even if the worst case occurs. This is due to pruning techniques (based on $m$) \cite[p.~98ff.\ and Alg.~6]{DBLP:journals/corr/Rodler16a} that can be incorporated into the search. In our evaluations (Sec.~\ref{sec:eval}) we even observed the exploration of only a negligible fraction of the search space in most cases. 
\begin{corollary}\label{cor:upper_lower_bound_for_canonical_q-partitions}
	Let $\mD\subseteq\minD_{\langle\mo,\mb,\Tp,\Tn\rangle_\RQ}$ with $|\mD| \geq 2$. Then, for the number $c$ of canonical q-partitions w.r.t.\ $\mD$ the following holds:
	\begin{align}
	c = \left|\setof{U_{\dx{}}\,|\, \emptyset \subset \dx{} \subset \mD} \setminus \setof{U_\mD}\right| \geq |\mD|
	\label{eq:number_of_canonical_q-partitions}
	\end{align} 
\end{corollary} 
\begin{example}\label{ex:number_of_CQPs}
To concretize Cor.~\ref{cor:upper_lower_bound_for_canonical_q-partitions}, let us apply it to our example DPI $\exdpi$ (Tab.~\ref{tab:example_dpi_0}) using the set of leading diagnoses $\mD := \minD_{\exdpi} = \setof{\setof{2,3},\setof{2,5},\setof{2,6},\setof{2,7},\setof{1,4,7},\setof{3,4,7}}$ (cf.\ Tab.~\ref{tab:min_diags+conflicts_example_DPI_0}). Building all possible unions of sets in $\mD$, i.e.\ all possible $U_{\dx{}}$ sets, such that each union is not equal to (i.e.\ a proper subset of) $U_\mD = \setof{1,\dots,7}$, yields $29$ \emph{different} $U_{\dx{}}$ sets. Note, there might be many more different $\dx{}$ sets for $\emptyset \subset \dx{} \subset \mD$ than $U_{\dx{}}$ sets. These $U_{\dx{}}$ sets directly correspond to the CQPs w.r.t.\ $\mD$, i.e.\ the CQP associated with $U_{\dx{}}$ is $\tuple{\dx{},\mD\setminus\dx{},\emptyset}$. There are no other CQPs.
%
Since there is one and only one CQ per CQP (cf.\ Prop.~\ref{prop:canonical_query_unique_for_seed}), we can immediately infer from this result that there are exactly $29$ (semantically) different CQs w.r.t.\ $\mD$.
\qed
\end{example}

\paragraph{Canonical Q-Partitions versus All Q-Partitions.}
Whether q-partitions $\langle\dx{},\dnx{},\emptyset\rangle$ exist which are no CQPs is not yet clarified, but both
theoretical and empirical evidence indicate the negative.

First, \citep[Sec.~3.4.2]{DBLP:journals/corr/Rodler16a} provides a thorough theoretical analysis of the relation between canonical and non-canonical q-partitions implying that a q-partition must fulfill sophisticated requirements if it is non-canonical. When we did not succeed in deriving the conjectured contradiction resulting from the theoretical requirements to a non-canonical q-partition which would rule out such cases theoretically, we tried hard to devise, at least in theory, an instance of a non-canonical q-partition. But we were not able to come up with one.

Second, \citep[Sec.~3.4.2]{DBLP:journals/corr/Rodler16a} applies the results of the conducted theoretical analysis to a comprehensive study on hundreds of real-world KBs with sizes of several thousands of sentences \citep{Horridge2012b}. The findings are that, if possible at all, the probability of the existence of non-canonical q-partitions is very low.

Third, an analysis of $\approx 900\,000$ q-partitions we ran for different leading diagnoses sets $\mD$ of different cardinalities for different DPIs (see Sec.~\ref{sec:eval}, Tab.~\ref{tab:experiment_data_set}) showed that \emph{all q-partitions were indeed CQPs}. Concretely, we were performing for each ($\mD$,DPI) combination a brute force search for q-partitions relying on a reasoning engine using the algorithm given in \citep[Alg.~2]{Shchekotykhin2012}, and another one exploiting the notions of CQs and CQPs. None of these searches returned a q-partition which is not canonical. This motivates the following conjecture:
%
\begin{conjecture}\label{conj:CQPs=QPs}
Let $\mD\subseteq\minD_{\langle\mo,\mb,\Tp,\Tn\rangle_\RQ}$ and $\mathbf{QP}_{\mD}^{\bcancel{0}}$ denote the set of all q-partitions w.r.t.\ $\mD$ with empty $\dz{}$, and $\mathbf{CQP}_\mD$ the set of canonical q-partitions w.r.t.\ $\mD$.  
Then $\mathbf{CQP}_\mD = \mathbf{QP}_{\mD}^{\bcancel{0}}$.
\end{conjecture}
\label{etc:conjecture_discussion}
Note, this conjecture is by no means necessary for the proper functioning of our presented algorithms. 
In case Conjecture~\ref{conj:CQPs=QPs} turned out to be wrong, the consequence would be just the invalidity of \emph{perfect} completeness w.r.t.\ all q-partitions achieved by the restriction to only CQPs. Still, we could cope well with that since CQs and CQPs bring along nice computational properties (cf.\ \ref{enum:advantages_of_CQs:no_reasoner} -- \ref{enum:advantages_of_CQs:independence_of_entailments_computed_by_reasoner} on page~\pageref{enum:advantages_of_CQs:no_reasoner}) and prove extremely efficient by the total avoidance of reasoning
(see Sec.~\ref{sec:eval}). Moreover, 
methods not incorporating the canonical notions prove to be strongly incomplete regarding query and QP computation due to their dependence on (the entailments computed by) the used inference engine (cf.\ Advantage~\ref{enum:advantages_of_CQs:independence_of_entailments_computed_by_reasoner} of CQs on page~\pageref{enum:advantages_of_CQs:independence_of_entailments_computed_by_reasoner}). Although executing a brute force search, they sometimes compute only $1\%$ and on average less than $40\%$ 
of the QPs our proposed novel approach is able to find. 

Also, our conducted experiments (see Sec.~\ref{sec:eval}) manifested the successful finding of optimal q-partitions in all evaluated cases for all discussed QSMs $m$ that are also adopted e.g.\ in the works of \citep{dekleer1987,Shchekotykhin2012,Rodler2013}. For, given practical numbers of leading diagnoses per iteration, e.g.\ any number $\geq 5$, cf.\ \citep{Shchekotykhin2012,Rodler2013}, the CQP search space size considered by our strategy proves by far large enough to guarantee the inclusion of (often multiple) goal q-partitions (also for negligibly small thresholds).
Theoretical support for this is given by Cor.~\ref{cor:upper_lower_bound_for_canonical_q-partitions}, empirical support by Figures~\ref{fig:summary_P1}, \ref{fig:scalability} and \ref{fig:naive_vs_new} and their discussions.

\begin{figure}[tb]
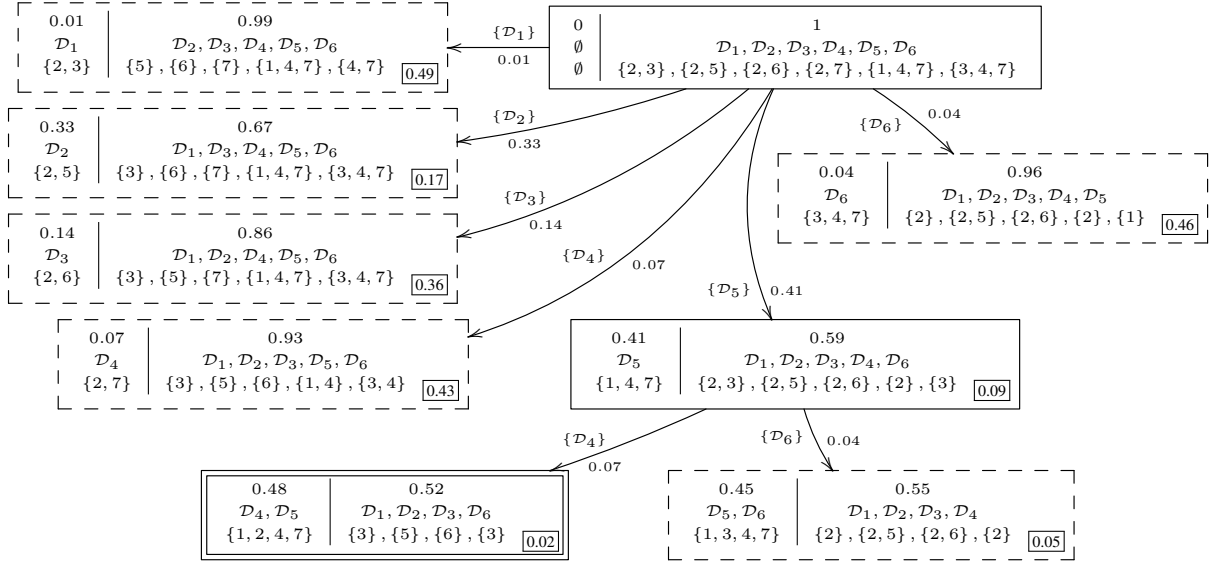

	\setlength{\fboxsep}{1.5pt}
	\tiny
	\xygraph{
		!{<0cm,0cm>;<2cm,0cm>:<0cm,2cm>::}
		!{(3.7,5)}*+[F]{\begin{tabular}{c|c} 
				$0$ & $1$ \\
				$\emptyset$&$\md_1,\md_2,\md_3,\md_4,\md_5,\md_6$ \\
				$\emptyset$&$\setof{2,3},\setof{2,5},\setof{2,6},\setof{2,7},\setof{1,4,7},\setof{3,4,7}$
		\end{tabular} }="init"
		!{(0,5)}*+[F--]{ {\begin{tabular}{c|c}
					$0.01$ & $0.99$ \\
					$\md_1$&$\md_2,\md_3,\md_4,\md_5,\md_6$ \\
					$\setof{2,3}$&$\setof{5},\setof{6},\setof{7},\setof{1,4,7},\setof{4,7}$
			\end{tabular}}_{\framebox[1.3\width]{0.49}} }="lev1_D1"
		!{(0,4.3)}*+[F--]{ {\begin{tabular}{c|c}
					$0.33$ & $0.67$ \\
					$\md_2$&$\md_1,\md_3,\md_4,\md_5,\md_6$ \\
					$\setof{2,5}$&$\setof{3},\setof{6},\setof{7},\setof{1,4,7},\setof{3,4,7}$
			\end{tabular}}_{\framebox[1.3\width]{0.17}} }="lev1_D2"
		!{(0,3.6)}*+[F--]{ {\begin{tabular}{c|c}
					$0.14$ & $0.86$ \\
					$\md_3$&$\md_1,\md_2,\md_4,\md_5,\md_6$ \\
					$\setof{2,6}$&$\setof{3},\setof{5},\setof{7},\setof{1,4,7},\setof{3,4,7}$
			\end{tabular}}_{\framebox[1.3\width]{0.36}} }="lev1_D3"
		!{(0.2,2.9)}*+[F--]{ {\begin{tabular}{c|c}
					$0.07$ & $0.93$ \\
					$\md_4$&$\md_1,\md_2,\md_3,\md_5,\md_6$ \\
					$\setof{2,7}$&$\setof{3},\setof{5},\setof{6},\setof{1,4},\setof{3,4}$
			\end{tabular}}_{\framebox[1.3\width]{0.43}} }="lev1_D4"
		!{(3.7,2.9)}*+[F]{ {\begin{tabular}{c|c}
					$0.41$ & $0.59$ \\
					$\md_5$&$\md_1,\md_2,\md_3,\md_4,\md_6$ \\
					$\setof{1,4,7}$&$\setof{2,3},\setof{2,5},\setof{2,6},\setof{2},\setof{3}$
			\end{tabular}}_{\framebox[1.3\width]{0.09}} }="lev1_D5_best"
		!{(5,4)}*+[F--]{ {\begin{tabular}{c|c}
					$0.04$ & $0.96$ \\
					$\md_6$&$\md_1,\md_2,\md_3,\md_4,\md_5$ \\
					$\setof{3,4,7}$&$\setof{2},\setof{2,5},\setof{2,6},\setof{2},\setof{1}$
			\end{tabular}}_{\framebox[1.3\width]{0.46}} }="lev1_D6"
		!{(1,1.9)}*+[F=]{ {\begin{tabular}{c|c}
					$0.48$ & $0.52$ \\
					$\md_4,\md_5$&$\md_1,\md_2,\md_3,\md_6$ \\
					$\setof{1,2,4,7}$&$\setof{3},\setof{5},\setof{6},\setof{3}$
			\end{tabular} }_{\framebox[1.3\width]{0.02}}}="lev2_D5,D4"
		!{(4.2,1.9)}*+[F--]{ {\begin{tabular}{c|c}
					$0.45$ & $0.55$ \\
					$\md_5,\md_6$&$\md_1,\md_2,\md_3,\md_4$ \\
					$\setof{1,3,4,7}$&$\setof{2},\setof{2,5},\setof{2,6},\setof{2}$
			\end{tabular} }_{\framebox[1.3\width]{0.05}}}="lev2_D5,D6"
		"init":^{0.01}_{\setof{\md_1}}"lev1_D1"
		"init":@/^2em/^{0.33}_{\setof{\md_2}}"lev1_D2"
		"init":@/^4em/^{0.14}_{\setof{\md_3}}"lev1_D3"
		"init":@/^6em/^(0.4){0.07}_(0.45){\setof{\md_4}}"lev1_D4"
		"init":@/_3em/^(0.78){0.41}_(0.75){\setof{\md_5}}"lev1_D5_best"
		"init":@/^1em/^(0.6){0.04}_{\setof{\md_6}}"lev1_D6"
		"lev1_D5_best":@/^1em/^{0.07}_{\setof{\md_4}}"lev2_D5,D4"
		"lev1_D5_best":@/_1em/^(0.5){0.04}_(0.4){\setof{\md_6}}"lev2_D5,D6"
	}
	\caption{\small Search for optimal CQP in phase P1 for example DPI $\exdpi$ (Tab.~\ref{tab:example_dpi_0}) w.r.t.\ $m := \mathsf{ENT}$ and threshold $t_{m} := 0.05$.}
	\label{fig:ex:can_q-part_search_ENT}
\end{figure}
The following example showcases one entire execution of the CQP search performed by phase P1 applied to our running example:
%
%
%
\begin{example}\label{ex:canonical_q-partition_search_ENT} Consider the example DPI $\exdpi$ (Tab.~\ref{tab:example_dpi_0}) and leading diagnoses $\mD = \minD_{\exdpi}$ (see Tab.~\ref{tab:min_diags+conflicts_example_DPI_0}). Let the diagnoses probabilities $\langle p(\md_1),p(\md_2),p(\md_3),p(\md_4),p(\md_5)$, $p(\md_6)\rangle$ for $\md_i \in \mD$ be given by $\tuple{0.01, 0.33, 0.14, 0.07, 0.41, 0.04}$. The search tree for a goal QP w.r.t.\ $m := \mathsf{ENT}$ and $t_{m} := 0.01$ produced by phase P1 is shown in Fig.~\ref{fig:ex:can_q-part_search_ENT}. At this, $\mathsf{ENT}$ denotes the entropy QSM \citep{dekleer1987,Shchekotykhin2012}. Roughly, it evaluates a query the better, the closer the query answer probabilities approach a discrete uniform distribution. Let us therefore assume a very simple heuristic function $h$ which assigns $h(\Pt) = |p(\dx{})-0.5|$ to a QP $\Pt:=\tuple{\dx{}, \dnx{},\dz{}}$ (recall that the probability of positive and negative answers amount to $p(\dx{})$ and $p(\dnx{})$, respectively, for empty $\dz{}$, cf.\ Eq.~\eqref{eq:prob_of_pos_query_answer}) where smaller $h$ values imply more promising QPs w.r.t.\ $\mathsf{ENT}$. 
Further, let us use a pruning function that stops the generation of successors at any QP $\Pt$ whose $p(\dx{})$ probability exceeds $0.5$ (as no descendant node of $\Pt$ can have a better $\mathsf{ENT}$ value than $\Pt$ itself).
For $\mathsf{ENT}$, the optimal QSM-value $m_{opt} = 0$ and thus a QP $\Pt$ is a goal state iff $|m(\Pt)|\leq t_m = 0.01$ (cf.\ Alg.~\ref{algo:query_comp}).

In Fig.~\ref{fig:ex:can_q-part_search_ENT}, a node in the search tree representing the CQP $\Pt_k = \tuple{\dx{k},\dnx{k},\dz{k}}$ is denoted by a frame including a table with three rows where (1)~the topmost row shows $p(\dx{k})\mid p(\dnx{k})$ (relevant for computing QSM $m$ and heuristic function $h$, and for making pruning decision), 
(2)~the middle row depicts $\dx{k} \mid \dnx{k}$ and (3)~the bottommost row gives the standard representation of the CQP (cf.\ Eq.~\eqref{eq:standard_representation_of_can_q-partitions}). The framed value at the bottom right corner of the large frame quotes the heuristic value $h(\Pt_k)$ computed for the CQP $\Pt_k$. No such value for the root node is given since it is not a QP and hence does not qualify as a solution. Furthermore, the (for CQPs) always empty $\dz{k}$ set is omitted. 
A frame is dashed / continuous / double if the associated node is generated (but not expanded) / expanded / a returned goal CQP. 
Arrows represent minimal $\dx{}$-transformations, i.e.\ 
an arrow's destination QP is a result of a minimal $\dx{}$-transformation applied to its source (q-)partition. Arrow labels give the set of diagnoses and the probability mass (i.e.\ the sum of the single diagnoses probabilities) moved from the $\dnx{}$-set of the source (q-)partition to the $\dx{}$-set of the destination QP. 

Starting from the root node (initial state) representing the partition $\tuple{\emptyset,\mD,\emptyset}$, the successor function $S_{\mathsf{init}}$ generates all possible CQPs resulting from the transfer of a single diagnosis from the $\dnx{}$-set of the initial state to its $\dx{}$-set. Since there are six diagnoses in $\mD$,  the initial state has exactly six successors (Prop.~\ref{prop:S_init_sound+complete}). 
From all these generated neighbor nodes of the initial state, the best one according to the heuristic function $h$ is selected for expansion. In this case, it is the CQP $\Pt_1 := \tuple{\setof{\md_5},\setof{\md_1,\md_2,\md_3,\md_4,\md_6},\emptyset}$ as it exhibits the best (lowest) heuristic value $0.09$ among all the six open nodes.
	
For $\Pt_1$, $S_{\mathsf{next}}$ generates exactly two CQPs that result from it by a minimal $\dx{}$-transformation (Cor.~\ref{cor:S_next_sound+complete}). This can be seen by considering the traits of the diagnoses in $\dnx{}(\Pt_1)$ shown in the right column of the third row in the table representing $\Pt_1$. Among the five traits there are only two $\subseteq$-minimal ones, i.e.\ $\md_4^{(1)} := \setof{2}$ and $\md_6^{(1)} := \setof{3}$. All the other traits are proper supersets of either of these. This means that all successors of $\Pt_1$ can be constructed by shifting either $\md_4$ or $\md_6$ from $\dnx{}(\Pt_1)$ to $\dx{}(\Pt_1)$ yielding $\Pt_{21} := \tuple{\setof{\md_4,\md_5},\setof{\md_1,\md_2,\md_3,\md_6},\emptyset}$ and $\Pt_{22} := \tuple{\setof{\md_5,\md_6},\setof{\md_1,\md_2,\md_3,\md_4},\emptyset}$, respectively.  
	
At this stage, the best QP among the two successors $\Pt_{21}$ and $\Pt_{22}$ of $\Pt_1$ (depth-first, \emph{local} best-first) is determined for expansion by means of $h$. As $p(\dx{}(\Pt_{21}))$ differs by less ($0.02$) from $0.5$ than $p(\dx{}(\Pt_{22}))$ ($0.05$), $\Pt_{21}$ is chosen. 
However, as $t_{m}$ has been set to $0.01$ and $m(\Pt_{21}) \approx 0.001 \leq 0.01$, $\Pt_{21}$ is a goal and returned as the solution of phase P1 of Alg.~\ref{algo:query_comp}.
%
%
Note, there were no backtrackings or tree prunings necessary as the used heuristic function guided the search directly towards a goal state. This behavior could also be frequently observed in our experiments (see Sec.~\ref{sec:eval}).
\qed
\end{example}

\paragraph{
	Complexity of P1.}
Concerning time, the worst-case scenario occurs if the search in P1 explores the entire CQP search space, e.g.\ because no goal CQP exists. As becomes directly evident through Cor.~\ref{cor:upper_lower_bound_for_canonical_q-partitions}, the worst-case time complexity of P1, assuming one time unit for the processing of a CQP, is as follows:
\begin{proposition}\label{prop:P1_time_complexity}
P1 terminates in $O(\left|\setof{U_{\dx{}}\,|\, \emptyset \subset \dx{} \subset \mD} \setminus \setof{U_\mD}\right|) \subseteq O(2^{|\mD|})$ time.
\end{proposition}  
\begin{proof}
The left $O(.)$ expression follows from Cor.~\ref{cor:upper_lower_bound_for_canonical_q-partitions}. The inclusion between the $O(.)$ expressions holds due to fact that the worst case for the left $O(.)$ expression arises exactly when all sets in $\mD$ are mutually disjoint. Because in this case each set $U_{\dx{i}}$ differs from all other sets $U_{\dx{j}}$ for $\dx{i} \neq \dx{j}$. Hence, excluding $U_{\emptyset}$ and $U_\mD$, there are exactly $2^{|\mD|}-2$ such sets. 
\end{proof}
Note, the time complexity is equal to the size of the full CQP search tree. The worst-case space complexity of P1, however, is (much) lower in general. As explained above, this is due to the depth-first, local best-first backtracking strategy pursued by the CQP search which implies a linear space complexity $O(b*d)$ where $b$ is the branching factor, i.e.\ the maximal number of generated successors for any node, and $d$ the maximal depth of the search tree. In fact, given that $\mD$ is the considered leading diagnoses set, no node (partition) occurring in the search can have more than $|\mD|$ successors (in terms of Def.~\ref{def:minimal_transformation}). The reason is that $S_{\mathsf{init}}$ generates exactly $|\mD|$ successors (Prop.~\ref{prop:S_init_sound+complete}) and $S_{\mathsf{next}}$ generates at most $|\mD|-1$ successors (Cor.~\ref{cor:S_next_sound+complete}). The latter holds, first, as $S_{\mathsf{next}}$ is only applied to a QP with non-empty $\dx{}$ which is why $\dnx{}$ can include only at most $|\mD|-1$ diagnoses that might be transferred to $\dx{}$ in the course of a minimal \text{$\dx{}$-transformation}, and, second, since a set of  $|\mD|-1$ elements cannot be partitioned into more than the same number of equivalence classes, i.e.\ $|\mathbf{EQ}^{\sim_k}_{\subseteq}| \leq |\mD|-1$ (cf.\ Cor.~\ref{cor:S_next_sound+complete}). Hence, $b = |\mD|$. The search tree depth $d = |\mD|-1$ because, starting from the initial state $\langle \dx{}, \dnx{}, \emptyset\rangle = \langle \emptyset, \mD, \emptyset\rangle$, at least one element from $\dnx{}$ is transferred to $\dx{}$ on any edge along any downward search tree branch (such that the result is still a QP, i.e.\ $\dnx{} \neq \emptyset$, cf.\ Def.~\ref{def:minimal_transformation}). All in all we have:
\begin{proposition}
P1 consumes $O(|\mD|^2)$ space.	
\end{proposition}
As we will demonstrate in Sec.~\ref{sec:eval} (see Fig.~\ref{fig:scalability}), these benign complexity results enable the search to explore search spaces with a theoretical complexity of up to $2^{500}$ within reasonable time yielding a result whose QSM-value differs negligibly from the theoretical QSM-optimum.

\subsubsection{Phase 2: Selecting an Optimal Query for the Optimal Q-Partition}
\label{sec:P2}

In this section, we describe the functioning of \textsc{optimizeQueryForQPartition} in Alg.~\ref{algo:query_comp}.
By now, we have demonstrated in Section~\ref{sec:P1} how an optimal q-partition $\Pt$ w.r.t.\ a given QSM $m$ 
can be computed by our algorithm's phase P1. What we also have at hand so far is one particular well-defined query for the identified q-partition $\Pt$, namely the CQ of $\Pt$. If we impose no further constraints on the query than an optimal discrimination (as per the used QSM) among (leading) diagnoses, which is solely determined by the query's q-partition, 
then we are already done and can simply output the CQ and ask the oracle to answer it. However, in many practical scenarios we can expect that a least quality criterion apart from efficient diagnoses discrimination, i.e.\ the minimization of the number of queries, is the $\subseteq$-minimality of queries.
This means that we will usually want to minimize the number of sentences appearing in a query and hence the effort given for the oracle to walk through them and answer the query -- \emph{while guaranteeing that the optimal discrimination properties (described by $\Pt$) are not affected by this minimization}. In other words, we consider the consultation of the oracle a very expensive ``operation'', as it might involve human interaction and time, the use of high-cost instruments or the call of potentially expensive procedures. 
%

However, exactly these latter factors might motivate a targeted selection of a query which is not solely based on $\subseteq$-minimality, but on additional query (answering) cost considerations. To enable the fulfillment of such advanced query quality criteria, we assume some QCM that our method allows the user to specify (see page~\pageref{etc:QCM}).
Hence, the problems we tackle now are 
\begin{enumerate}[noitemsep]
	\item \label{enum:P2:problem_subset-minimal_query} how to obtain a $\subseteq$-minimal query for the given optimal q-partition $\Pt$ returned by phase P1, and
	\item \label{enum:P2:problem_optimal-QCM_query} how to calculate a query that optimizes the QCM among all $\subseteq$-minimal queries existent for $\Pt$.
\end{enumerate}
%
Let us first devote our attention to problem \ref{enum:P2:problem_subset-minimal_query}.

\paragraph{Computation of $\subseteq$-Minimal Queries for a Fixed Q-Partition.}
At first sight, we might simply choose to employ the same approach that has been exploited in \citep{Shchekotykhin2012,Rodler2013,Rodler2015phd}. This approach involves the usage of a modification of the \textsc{QuickXplain} algorithm due to \citep{junker04} which implements a divide-and-conquer strategy with regular calls to a reasoning service.\footnote{A formal proof of \textsc{QuickXplain}'s correctness, a detailed description of its use in diagnosis tasks and related examples can be found in \citep[Sec.~4.4.1]{Rodler2015phd}.}
The output is 
one $\subseteq$-minimal 
subset $Q'$ 
of a given query $Q$ (in our case the CQ of $\Pt$)
such that the minimization preserves the q-partition, i.e.\ 
such that $\Pt_\mD(Q') = \Pt_\mD(Q)$. One concrete implementation of a suitable \textsc{QuickXplain}-modification for this purpose is the \textsc{minQ} algorithm presented and profoundly analyzed in \citep[Sec.~8.3 ff.]{Rodler2015phd}. Although the number of calls to a reasoner required by \textsc{minQ} 
is polynomial in $O(|Q'| \log_2 \frac{|Q|}{|Q'|})$, we will learn in this section that we can in fact do without any calls to a reasoner. This is due to the task constituting a search for a $\subseteq$-minimal \emph{explicit-entailments} query $Q' \subseteq Q$ (because the CQ $Q$ is an explicit-entailments query).

Let $\mQ^{(\mathsf{EE},\Pt)}_\mD$ denote the subset of $\mQ_\mD$ containing all explicit-entailments queries associated with a the q-partition $\Pt$.
Subsequently, we analyze the shape of $\mQ^{(\mathsf{EE},\Pt)}_\mD$. The results will be exploited to solve problems \ref{enum:P2:problem_subset-minimal_query} and \ref{enum:P2:problem_optimal-QCM_query} stated above. 
More precisely, let us consider the lattice $(2^{\Disc_{\mD}},\subseteq)$ 
consisting of all subsets of $\Disc_\mD$ (cf.\ Def.~\ref{def:discax}) which are partially ordered by $\subseteq$. Recall that by Prop.~\ref{prop:expl_ent_query_must_neednot_mustnot_include_ax} we can w.l.o.g.\ restrict the focus on queries that are subsets of $\Disc_\mD$ (because the inclusion of sentences from $\mo \setminus \Disc_\mD$ does not affect the q-partition of an explicit-entailments query). We are now interested in the elements in the lattice $(2^{\Disc_{\mD}},\subseteq)$ that constitute upper and, more importantly, lower bounds of the partially ordered set $(\mQ^{(\mathsf{EE},\Pt)}_\mD,\subseteq_{\Pt})$ where $\subseteq_{\Pt}$ denotes the query-subset relation under preservation of the q-partition $\Pt$. That is, $Q_i \subseteq_{\Pt} Q_j$ holds for $Q_i,Q_j \in \mQ^{(\mathsf{EE},\Pt)}_\mD$ iff $Q_i \subseteq Q_j$ and $\Pt_\mD(Q_i) = \Pt_\mD(Q_j)$. Given such upper and lower bounds, we have a complete description of $\mQ^{(\mathsf{EE},\Pt)}_\mD$. 
The next proposition illuminates this aspect. Let for this purpose $\mathsf{MHS}(X)$ denote the set of all minimal hitting sets of some collection of sets $X$ (cf.\ Def.~\ref{def:hs}).
\begin{proposition}\label{prop:explicit-ents_query_lower+upper_bound}
Let $\mD \subseteq \minD_{\tuple{\mo,\mb,\Tp,\Tn}_\RQ}$ and $\Pt = \langle \dx{}, \dnx{}, \emptyset\rangle$ be a q-partition w.r.t.\ $\mD$. Then $Q \subseteq \Disc_\mD$ is a query with q-partition $\Pt$ iff there is some $H \in \mathsf{MHS}(\dnx{})$ such that $H \subseteq Q \subseteq Q_{\mathsf{can}}(\dx{})$. 
\end{proposition} 
That is, the construction of a $\subseteq$-minimal explicit-entailments query for a fixed q-partition $\tuple{\dx{},\dnx{},\emptyset}$ requires finding a minimal hitting set of all diagnoses in $\dnx{}$. As regards upper and lower bounds of $\mQ^{(\mathsf{EE},\Pt)}_\mD$, the conclusion is that there are generally multiple lower bounds (given by all these minimal hitting sets) and a unique upper bound (given exactly by the CQ for $\Pt$). Therefore, $(\mQ^{(\mathsf{EE},\Pt)}_\mD,\subseteq_{\Pt})$ is \emph{a join-semilattice} (every subset of $\mQ^{(\mathsf{EE},\Pt)}_\mD$ has a least upper bound or supremum), but\emph{ not (necessarily) a meet-semilattice} (not every subset of $\mQ^{(\mathsf{EE},\Pt)}_\mD$ needs to have a greatest lower bound or infimum). 
As a consequence of this, there is generally an exponential number (in $|\dnx{}|$) of $\subseteq$-minimal queries for a fixed $\Pt$ -- despite the restriction to just explicit-entailments queries.
And, the CQ for $\Pt$ is the explicit-entailments query of maximal size (and thus the one containing the most information) for $\Pt$. We will leverage this fact to produce the most yielding query enhancement for $\Pt$ using the CQ of it in phase P3.
%
%

Hence, by Prop.~\ref{prop:explicit-ents_query_lower+upper_bound}, the computation of all $\subseteq$-minimal reductions of the CQ for the 
q-partition $\Pt = \tuple{\dx{},\dnx{},\emptyset}$ under preservation of $\Pt$ is possible by using e.g.\ the classical $\textsc{HS-Tree}$ \citep{Reiter87} (or some other hitting set algorithm mentioned in Sec.~\ref{sec:sequential_diagnosis_algo}). Let the complete hitting set tree produced by $\textsc{HS-Tree}$ for $\dnx{}$ be denoted by $T$. Then the set of all $\subseteq$-minimal queries with associated q-partition $\Pt$ is given by 
\begin{align*}
\setof{H(\mathsf{n})\,|\,\mathsf{n} \mbox{ is a node of $T$ labeled by $valid$ ($\checkmark$)}}
\end{align*}
where $H(\mathsf{n})$ denotes the set of edge labels on the path from the root node to the node $\mathsf{n}$ in $T$. 
We want to make the reader explicitly aware of the fact that a critical source of complexity when constructing a hitting set tree is 
the computation of the node labels which might be very expensive (cf.\ \citep{Reiter87,chandrasekaran2011algorithms}), e.g.\ when calls to a reasoning service are required. For instance, if a propositional satisfiability checker is employed, each consistency check it performs is already NP-complete \citep{cook1971}. 
%
In our situation, however, all the sets used to label the nodes of the tree are already \emph{explicitly given}. Hence the construction of the hitting set tree will usually be very efficient in the light of the fact that the number of diagnoses in $\dnx{}$ is bounded above by the predefined fixed parameter $|\mD|$ (which is normally relatively small, e.g.\ $\approx 10$, cf.\ \citep{Shchekotykhin2012,Rodler2013}). Apart from that, we are usually satisfied with a single $\subseteq$-minimal query which implies that we could stop the tree construction immediately after having found the first node labeled by $valid$.

Although this search is already very efficient, it can be even further accelerated. The key observation to this end is that each explicit-entailments query w.r.t.\ $\Pt_k = \tuple{\dx{k},\dnx{k},\emptyset}$ 
must not include any axioms in $U_{\dx{k}}$, which follows from 
Prop.~\ref{prop:explicit-ents_query_lower+upper_bound} and Lem.~\ref{lem:CQ_equal_to_U_D_setminus_U_D+}. This brings us back to
the concept of the trait $\md_i^{(k)} := \md_i \setminus U_{\dx{k}}$ (cf.\ Def.~\ref{def:trait} and Eq.~\eqref{eq:md_i^(k)}) 
of a diagnosis $\md_i \in \dnx{k}$ 
given $\Pt_k$. Let in the following $\mathsf{Tr}(\Pt_k)$ denote the set of all traits of diagnoses in $\dnx{k}$ w.r.t.\ the q-partition $\Pt_k$.
As a consequence of Prop.~\ref{prop:explicit-ents_query_lower+upper_bound}, we can now state the following:
\begin{corollary}\label{cor:min_exp-ents_queries_are_minHS_of_all_traits_of_diags_in_Dnx}
	Let $\mD \subseteq \minD_{\tuple{\mo,\mb,\Tp,\Tn}_\RQ}$ and $\Pt_k = \langle \dx{k}, \dnx{k}, \emptyset\rangle$ be a q-partition w.r.t.\ $\mD$. Then $Q \subseteq \Disc_\mD$ is a $\subseteq$-minimal query with q-partition $\Pt_k$ iff $Q = H$ for some $H \in \mathsf{MHS}(\mathsf{Tr}(\Pt_k))$.
\end{corollary}
Contrary to minimal diagnoses, traits of minimal diagnoses might be equal to or proper subsets of one another (cf.\ Ex.~\ref{ex:nec_follower} and \ref{ex:equiv_rel+traits}). By \citep{Reiter87} (where a proof is given by \citep[Prop.~12.6]{Rodler2015phd}), we have  
\begin{quote} If $F$ is a collection of sets, and if $S \in F$ and $S' \in F$ such that $S \subset S'$, then $F_{sub} := F \setminus \setof{S'}$ has the same minimal hitting sets as $F$. 
\end{quote}
Thus we can replace $\mathsf{Tr}(\Pt_k)$ by $\mathsf{Tr}_{\min}(\Pt_k)$ in Cor.~\ref{cor:min_exp-ents_queries_are_minHS_of_all_traits_of_diags_in_Dnx} where $\mathsf{Tr}_{\min}(\Pt_k)$ terms the set of all \emph{$\subseteq$-minimal} traits of diagnoses in $\dnx{k}$ w.r.t.\ $\Pt_k$, i.e.\ all traits $t$ in $\mathsf{Tr}(\Pt_k)$ for which there is no trait $t'$ in $\mathsf{Tr}(\Pt_k)$ such that $t' \subset t$. The possibility to solve problem \ref{enum:P2:problem_subset-minimal_query} by means of hitting set computation brings us directly to the solution of problem \ref{enum:P2:problem_optimal-QCM_query}.  

\paragraph{Computation of Optimal Queries for a Fixed Q-Partition.}
The insights gained in this section enable us to \emph{construct} a $\subseteq$-minimal query w.r.t.\ a given q-partition \emph{systematically}. 
The idea is to use a \emph{uniform-cost} variant of e.g.\ Reiter's $\textsc{HS-Tree}$ which enables the detection of minimized queries with particular properties (first). One instance of such an algorithm is the proven sound and complete \textsc{HS} algorithm proposed by \citep[Alg.~2]{Rodler2015phd}. The desired query properties are specified in form of the QCM $c$. 
Whenever the function in a uniform-cost search that assigns costs to nodes $\mathsf{nd}$ in the search tree is a monotonic set function (with regard to the set of edge labels along the branch to $\mathsf{nd}$), the search finds the goals in lowest-cost-first order (cf.\ \citep{russellnorvig2010}).
Given a set $X$, a function $f:2^X\to\mathbb{R}$ is a \emph{monotonic set function} iff $f(Y) \leq f(Z)$ whenever $Y \subseteq Z$ for $Y,Z \subseteq X$. 
Hence, as a direct consequence of Cor.~\ref{cor:min_exp-ents_queries_are_minHS_of_all_traits_of_diags_in_Dnx}: 
\begin{proposition}\label{prop:phase_P2_returns_EE-queries_in_best-first_order}
A uniform-cost hitting set computation over the collection of sets $\mathsf{Tr}_{\min}(\Pt_k)$ returns explicit-entailments queries $Q$ in best-first order w.r.t.\ their QCM value $c(Q)$ given that $c$ is a monotonic set function.
\end{proposition}
Note that it is quite natural for a (query) cost measure to be a monotonic set function, as it is hard to imagine situations where the inclusion of additional measurements makes a query less costly than before. 
In fact, all QCMs $c_{\Sigma}$, $c_{\max}$ and $c_{|\cdot|}$ discussed above (see page~\pageref{etc:QCM}) satisfy this monotonicity property. So, the usage of any of these guarantees the retrieval of the 
optimal explicit-entailments query for a fixed q-partition $\Pt$, e.g.\ the one with minimum cardinality (using $c_{|\cdot|}$) or minimal cost (using $c_{\Sigma}$).
%
%
When relying on \textsc{QuickXplain} (or \textsc{minQ}, respectively) to minimize a query in a manner its q-partition is preserved (cf.\ \citep{Rodler2015phd,Rodler2013,Shchekotykhin2012}), one has less influence on the properties of the returned query. 
This issue will be of interest in phase P3 of Alg.~\ref{algo:query_comp} (see Sec.~\ref{sec:P3}) where we will discuss the minimization of arbitrary queries and state guarantees \textsc{minQ} can give in general under suitable modifications of its input. 

Let us now exemplify the functioning of 
phase P2 of Alg.~\ref{algo:query_comp}:
\begin{example}\label{ex:phase_2}
Let the considered DPI be again $\exdpi$ (Tab.~\ref{tab:example_dpi_0}) and let the QP $\Pt_{21}$ from Ex.~\ref{ex:canonical_q-partition_search_ENT} be the output of phase P1 (function \textsc{optimizeQPartition}) and the input to phase P2 (function \textsc{optimizeQueryForQPartition}) of Alg.~\ref{algo:query_comp}, along with the QCM $c := c_{|.|}$ (see page~\pageref{etc:QCM}). 
That is, the aim is to obtain the query $Q^*$ with minimal cost where the cost amounts to the 
number
of sentences in $Q^*$.
%
Now, the set $\mathsf{Tr}_{\min}(\Pt_{21})$ of all $\subseteq$-minimal traits for $\Pt_{21}$ is $\{ \{3\}, \{5\}, \{6\} \}$ (cf.\ right column of last row of the double frame in Fig.~\ref{fig:ex:can_q-part_search_ENT}). Since all traits are singletons, they produce only one hitting set. Thence, there is a single (optimal) explicit-entailments query $\setof{3,5,6}$ for $\Pt_{21}$ which (in this case) coincides with the CQ for $\Pt_{21}$. \qed
%
\end{example}

\paragraph{Complexity of P2.} 
The problem of finding a minimum-cardinality hitting set is known to be NP-hard \citep{karp1972}. This can be interpreted as the computation of a minimum-cost hitting set using the cost function that assigns to each hitting set its cardinality. This cost function, in particular, is a monotonic set function. Therefore, this problem can be reduced to the problem of finding a minimum-cost hitting set for costs assigned by any monotonic set function. As a consequence, the latter problem is NP-hard as well. Hence, P2 addresses an NP-hard problem.

This theoretical result is discouraging at first sight. However, one can view the problem at a more fine granular level in terms of parameterized complexity \citep{downey2013fundamentals}. In fact, the problem depends on two parameters $d$ and $b$ where $d := |\mD|$ is the number of leading diagnoses and $b := \max\{|\md| \mid \md \in \mD\}$ their maximal size. The former can be predefined or at least bounded above by allowing an arbitrarily small upper bound $d \geq 2$ (without harming the proper functioning of our approach, cf.\ Prop.~\ref{prop:properties_of_q-partitions}.\ref{prop:properties_of_q-partitions:enum:D+=d_i_is_q-partition_and_lower_bound_of_queries}). In fact, the number of sets that a set produced by P2 must hit is at most $d-1$ as P2 is only called for a QP and QPs have non-empty $\dx{}$ which is why $|\dnx{}| \leq |\mD|-1 = d-1$. The latter parameter $b$ is generally bounded by the number of minimal conflicts (i.e.\ independent sources of fault) for a given DPI (since $\mD$ comprises only \emph{minimal} diagnoses). For real-world DPIs, $b$ is often relatively small and, in general, is not a function of the size of the DPI, i.e.\ the size of the KB (diagnosed system) $\mo$ in particular, cf.\ Tables 8 and 12 in \citep{Shchekotykhin2012}. For instance, the chance that a large number of components fail simultaneously is usually small in physical systems \citep{dekleer1987,DBLP:journals/tsmc/ShakeriRPP00}, as is the chance of a large number of independent faults in KBs (assuming regular validation steps are performed \citep{Shchekotykhin2014}).

Due to these arguments, it makes sense to analyze the problem addressed by P2 for the case where the parameters $b$ and $d$ are bounded. To this end, we define: A \emph{parameterized version of a decision problem} $P$ is given by a tuple $\tuple{x,k}$ where $x$ is an instance of $P$ and $k$ is a (set of) parameter(s) associated with the instance $x$. A parameterized decision problem is called \emph{fixed parameter tractable} (or: \emph{in $\mathcal{FPT}$}) iff there is an algorithm $A$ and a computable function $f$ such that, for all $x,k$, $A$ decides $\tuple{x,k}$ correctly and runs in time at most $f(k) |x|^{O(1)}$ \citep{downey2013fundamentals}. Roughly, fixed parameter tractability means that a problem becomes tractable given that its parameter(s) are bounded above by an arbitrary number.

The size $|x|$ of the hitting set problem instance $x$ for phase P2 is in $O(bd)$ (the size of the description of $S$, see Def.~\ref{def:hs}). The parameters are $k = \tuple{b,d}$. Assuming a uniform-cost \textsc{HS-Tree} construction in P2, the computation time is in $O(b^d\,|x|)$ since $b$ is the maximal branching factor, $d-1$ an upper bound of the maximal tree depth, and $|x|$ the cost of verifying whether the labels along a path already constitute a hitting set of all sets in $S$. Hence, $f(k) = f(\tuple{b,d}) = b^d$.
%
Overall, we have shown the following:
\begin{proposition}\label{prop:complexity_P2} \textcolor{white}{!}
\begin{enumerate}
	\item Phase P2 solves an NP-hard problem.
	\item Assuming $b \leq b'$ and $d \leq d'$ for fixed $b',d' \in \mathbb{N}$, the problem solved by phase P2 is in $\mathcal{FPT}$, i.e.\ fixed parameter tractable.\footnote{The hitting set problem is already in $\mathcal{FPT}$ if only $b$ is assumed fixed \citep{abu-khzam2010}. In this case the problem is called \emph{$b$-Hitting Set}.}
	\item Phase P2 runs in $O(b^d\,bd) = O(db^{d+1})$ time and requires $O(b^d+bd)$ space where $d$ is specifiable by the user and can be set to an arbitrary natural number larger than 1.  
\end{enumerate}
\end{proposition}
%
%
%
%

\subsubsection{Solution Produced by Phases 1 and 2}
\label{sec:solution_produced_by_P1+P2}
In order to establish the merit of phases P1 and P2 regarding the formulated optimal measurement selection problem (Prob.~\ref{prob:query_optimization}), the next result shows that any explicit-entailments query necessarily has a CQP as its q-partition. That is, when searching for the an optimal CQP (phase P1) and, after such a CQP $\Pt$ is found, for an optimal explicit-entailments query for $\Pt$ (phase P2), this amounts to exploring the \emph{entire} space of explicit-entailments queries.
\begin{proposition}\label{prop:each_expl_ents_query_has_CQP_as_q-partition}
	Let $\dpi = \tuple{\mo,\mb,\Tp,\Tn}_\RQ$ be a DPI and $\mD \subseteq \minD_{\dpi}$ and $Q \in \mQ_\mD$ where $Q \subseteq \mo$. Then the q-partition $\Pt_\mD(Q)$ of $Q$ is a canonical q-partition. 
\end{proposition}
Hence, the execution of phases P1 and P2 yields a solution to optimal measurement selection as per Prob.~\ref{prob:query_optimization} with restricted search space $\mathbf{S}$ \emph{without a single inference engine call}. 
\begin{theorem}\label{theorem:P1+P2_solve_query_optimization_problem}
Phases P1 and P2 (using the threshold $t_m := 0$) compute a solution $Q^*$ to Prob.~\ref{prob:query_optimization} with the search space $\mathbf{S} := \setof{X \mid X \in \mQ_\mD, X \subseteq \mo}$.
\end{theorem}
So, the query $Q^*$ output by phase P2 is optimized along two dimensions (number of queries as per the QSM $m$ and cost per query as per the QCM $c$) over the restricted search space $\mathbf{S}$. There are two ways to proceed after phase P2:
\begin{enumerate}[label=(\alph*),noitemsep]
	\item $Q^*$ can be directly proposed as the next query or
	\item an optimized query over an extended search space can be computed in phase P3.
\end{enumerate}
Considering case (a), an explicit-entailments query like $Q^*$ would correspond to a direct examination of one or more system components in a physical system (cf.\ \emph{Direct Probing} in Ex.~\ref{ex:query_representation}). Examples include the pinging of servers in a distributed system \citep{DBLP:conf/ijcai/BrodieRMO03}, the test of gates using a voltmeter in circuits \citep{dekleer1987} or the inspection of potentially faulty components of a car \citep{heckerman1995decision}. On the other hand, in knowledge-based system debugging $Q^*$ would mean e.g.\ to ask the stakeholders of a software, configuration or KB system \citep{DBLP:journals/ai/Wotawa02_1,DBLP:journals/ai/FelfernigFJS04,friedrich2005gdm} whether specified code lines, constraints or logical sentences, respectively, are correct.
In these examples, query costs can be motivated e.g.\ by the difficulty of inspecting a physical component or by the complexity of software code lines or logical sentences. 
We concentrate on case (b) in the next section where we deal with phase P3.  

\subsubsection{Phase 3: Query Expansion and Optimized Contraction}
\label{sec:P3}
Phase P3 consists of two steps. The first one involves an expansion of the CQ obtained from phase P1, thereby extending the search space $\mathbf{S}$ in terms of Prob.~\ref{prob:query_optimization}. The second one encompasses a minimization of the expanded query such that the resulting query is $\subseteq$-minimal and to comprise only ``cost-preferred'' elements, if such a query exists. We discuss both steps in turn next.

\paragraph{Step 1: Query Expansion.} 
We next describe the functioning of \textsc{expandQueryForQPartition} in Alg.~\ref{algo:query_comp}.
Here, the already optimal CQP $\Pt$ returned by P1 is regarded as an intermediate result to building a solution query to Prob.~\ref{prob:query_optimization} 
with full search space $\mathbf{S} = \mQ_{\mD}^{\bcancel{0}}$ (of queries discriminating among all elements of $\mD$, cf.\ Def.~\ref{def:query_q-partition}).
To this end, using the CQ $Q$ of $\Pt$, a (finite) set $Q_{\mathsf{exp}}$ of sentences of preferred entailment types $\mathit{ET}$ 
is computed.\footnote{$\mathit{ET}$ might be specified so as to restrict the computed entailments to, e.g., simple atoms, implication sentences of type $A \to B$, or sentences formulated only over a selected sub-vocabulary of the KB (e.g.\, given a problematic medical KB, a dermatologist might only be able to answer queries including dermatological terms, but none related to other medical disciplines).} 
Intuitively, the goal is to add $Q_{\mathsf{exp}}$ to $Q$ and achieve a larger pool of sentences from which an optimal minimized subset can be generated in the second step of phase P3. Of course, we want the (optimal) QP $\Pt$ to be unaffected by the query extension, i.e.\ it should be the same for both $Q$ and $Q \cup Q_{\mathsf{exp}}$. The usage of the CQ $Q$ as a basis for the expansion is well motivated since the CQ constitutes the most informative of all explicit-entailments queries for $\Pt$, as we have shown in Sec.~\ref{sec:P2}.
From $Q$'s extension $Q_{\mathsf{exp}} = \setof{\tax_1,\dots,\tax_r}$ we postulate that
\begin{enumerate}
	\item \label{enum:EQ1} $\alpha_1,\dots,\alpha_r \notin \mo \cup \mb \cup U_{\Tp}$ \\
	(each element of $Q_{\mathsf{exp}}$ must be ``new'', i.e.\ not an explicit entailment occurring in the (background) KB or the positive test cases)
	\item \label{enum:EQ2} $S \models \setof{\alpha_1,\dots,\alpha_r}$ where $S$ is some solution KB $S$ w.r.t.\ the given DPI $\tuple{\mo,\mb,\Tp,\Tn}_\RQ$ satisfying $Q \subseteq S \subseteq \mo \cup \mb \cup U_{\Tp}$ \\ 
	($Q_{\mathsf{exp}}$ must be ``sound'', i.e. be entailed by a fault-free KB $S$ that subsumes the CQ $Q$)
	\item \label{enum:EQ3} no $\alpha_i$ for $i\in\setof{1,\dots,r}$ is an entailment of $S \setminus Q$ \\
	(each element of $Q_{\mathsf{exp}}$ must ``depend on'' $Q$, i.e.\ $Q_{\mathsf{exp}}$ must not comprise any ``unnecessary'' entailments)
	\item \label{enum:EQ4} the syntactic type of each $\alpha_i$ for $i\in\setof{1,\dots,r}$ is some type listed in $ET$ \\
	(each element of $Q_{\mathsf{exp}}$ must be of some ``preferred'' type)
	\item \label{enum:EQ5} $\Pt = \Pt_{\mD}(Q) = \Pt_{\mD}(Q \cup Q_{\mathsf{exp}})$ \\
	(the extension must be QP-preserving)
\end{enumerate}
To realize the computation of $Q_{\mathsf{exp}}$ we assume a logical consequence operator $Ent_{\mathit{ET}} : 2^{\mathcal{L}} \to 2^{\mathcal{L}}$ which assigns a set of sentences $Ent_{\mathit{ET}}(X)$ over $\mathcal{L}$ to a set of sentences $X$ over $\mathcal{L}$ such that 
\begin{enumerate}[label=(\alph*),noitemsep]
	\item \label{enum:Ent_ET:logical_soundness} $X \models Ent_{\mathit{ET}}(X)$ (logical soundness),
	\item \label{enum:Ent_ET:type_soundness} $Ent_{\mathit{ET}}(X)$ contains only sentences of types listed in $ET$ (type soundness),
	\item \label{enum:Ent_ET:monotonicity} $Ent_{\mathit{ET}}(X') \subseteq Ent_{\mathit{ET}}(X'')$ whenever $X' \subset X''$ (monotonicity), and
	\item \label{enum:Ent_ET:ents_generated_for_set_are_genereated_for_all_supersets} for $Y \subseteq X' \subset X''$ and all $e_Y$ where $Y \models e_Y$ it holds that $e_Y \in Ent_{\mathit{ET}}(X')$ iff $e_Y \in Ent_{\mathit{ET}}(X'')$ (entailments generated for some set are generated for all its supersets).
\end{enumerate}
One possibility to realize such a service is to employ a reasoner for the (decidable) logic $\mathcal{L}$ and use it to extract (the finite set of) all entailments of a predefined type it can compute (cf.\ \citep[Remark 2.3]{Rodler2015phd}). For Propositional Horn Logic, e.g.\ one might extract only all literals that are entailments of a KB $X$. For general Propositional Logic, e.g.\ one might calculate all formulas of the form $A \odot B$ for propositional variables $A,B$ and logical operators $\odot \in \setof{\rightarrow, \leftrightarrow}$, and for Description Logics \citep{DLHandbook}, e.g.\ only all subsumption and/or class assertion formulas that are entailments could be computed. An example of entailment types that might be extracted for (decidable fragments of) First-Order Logic can be found in \citep[Example~8.1]{Rodler2015phd}. For all these examples, DL \citep{DLHandbook} and OWL \citep{Grau2008a} reasoners, respectively,
such as Pellet \citep{sirin2007}, HermiT \citep{Shearer2008}, FaCT++ \citep{Tsarkov06} or KAON2\footnote{See http://kaon2.semanticweb.org}
could be used with their
\emph{classification} and \emph{realization} reasoning services (cf.\ \citep[Sec.~9.2.2]{DLHandbook}).

Note, the operator $Ent_{\mathit{ET}}$ does not need to be complete, i.e.\ $\setof{e \mid X \models e, type(e) \in ET} \subseteq Ent_{\mathit{ET}}(X)$ does not necessarily hold. If it is complete, the last property \ref{enum:Ent_ET:ents_generated_for_set_are_genereated_for_all_supersets} of $Ent_{\mathit{ET}}$ given above is obsolete.
The next proposition shows how an operator satisfying the said properties \ref{enum:Ent_ET:logical_soundness} -- \ref{enum:Ent_ET:ents_generated_for_set_are_genereated_for_all_supersets} can be leveraged to implement the postulated query expansion.
\begin{proposition}\label{prop:enrichment_function}
Let $\mD \subseteq \minD_{\tuple{\mo,\mb,\Tp,\Tn}_\RQ}$, $Q \in \mQ_\mD$ with q-partition $\Pt_{\mD}(Q) = \tuple{\dx{},\dnx{},\emptyset}$ such that $Q\subseteq \Disc_\mD$ (in particular, the CQ for the seed $\dx{}$ is such a query). Further, let $Ent_{\mathit{ET}}(X)$ be a logical consequence operator as described by \ref{enum:Ent_ET:logical_soundness} -- \ref{enum:Ent_ET:ents_generated_for_set_are_genereated_for_all_supersets} above. The Postulations \ref{enum:EQ1}. -- \ref{enum:EQ4}. above are satisfied if 
\begin{align} \label{eq:Q_exp}
Q_{\mathsf{exp}} = \Big[ Ent_{\mathit{ET}}\big( (\mo\setminus U_{\mD}) \cup Q \cup \mb \cup U_{\Tp}\big) \;\setminus \; Ent_{\mathit{ET}}\big( (\mo\setminus U_{\mD}) \cup \mb \cup U_{\Tp}\big) \Big] \;\setminus\; Q
\end{align} 
\end{proposition}

The result of expanding the CQ $Q$ according to Prop.~\ref{prop:enrichment_function} is 
\begin{align}\label{eq:Q'_enriched_query}
Q' := Q \cup Q_{\mathsf{exp}}
\end{align}
As we show next, the expanded query $Q'$ has the same QP as $Q$. 
\begin{proposition}\label{prop:entailment_extraction_is_q-partition_preserving}
	Let $\mD \subseteq \minD_{\tuple{\mo,\mb,\Tp,\Tn}_\RQ}$ and $Q\in\mQ_\mD$ such that $Q \subseteq \Disc_\mD$. 
	Further, let $Q'$ be defined as in Eq.~\eqref{eq:Q'_enriched_query}. Then Postulation \ref{enum:EQ5}. holds, i.e.\ 
	$\Pt_\mD(Q') = \Pt_\mD(Q)$. 
\end{proposition}
\label{etc:implementation_of_Ent_ET}The function $Ent_{\mathit{ET}}$ might e.g.\ be realized by a Description Logic reasoner (computing, for instance, subsumption and realization entailments) for many decidable fragments of First-Order Logic (cf.\ \citep{DLHandbook}), by a forward chaining algorithm \citep{russellnorvig2010} for Horn Logic or by a constraint propagator for CSPs \citep{dechter2003,dekleer1987foundations,dekleer1987}.

More generally, given a sound and complete consistency checker $\mathit{CC}$ (e.g.\ some resolution-based procedure \citep{chang1973}) over (the decidable) knowledge representation formalism $\mathcal{L}$ underlying the given DPI, one can use the following implementation of the $Ent_{\mathit{ET}}$ calls in Eq.~\eqref{eq:Q_exp} to obtain a query expansion $Q_{\mathsf{exp}}$. 
Let $r$ be the desired number of entailments in the query expansion, $s$ a desired maximal and $t$ the absolute maximal number of consistency checks to be used. 
Let us refer to the left and right $Ent_{\mathit{ET}}$ calls in Eq.~\eqref{eq:Q_exp} by $Ent_1$ and $Ent_2$, respectively.
%
Now, $Ent_1$ can be realized as follows:
\begin{enumerate}[noitemsep]
	\item $i = 1$ (iteration counter), $A = \emptyset$ (already tested sentences), $E$ (computed entailments to be tested by $Ent_2$).
	\item Generate (randomly) a potentially entailed sentence $\tax_i \notin A$ of one of the postulated entailment types in $\mathit{ET}$ which is not an element of $(\mo \setminus U_\mD) \cup Q \cup \mb \cup U_\Tp$.
	\item Run $\mathit{CC}$ to prove $(\mo \setminus U_\mD) \cup Q \cup \mb \cup U_\Tp \cup \setof{\lnot\tax_i}$ inconsistent in a way that, whenever possible, sentences of $Q$ are involved in the proof (if e.g.\ $\mathit{CC}$ implements linear resolution,\footnote{Note that linear resolution is complete for full First-Order Logic \citep{chang1973}.} a way to realize this is to test sentences of $Q$ always first for applicability as a side clause for the next resolution step). If inconsistent is returned and a proof involving at least one sentence of $Q$ was found, then add $\tax_i$ to $E$. 
	\item If 
	
	\begin{tabular}{lp{9cm}}
	$|E| \geq r$ & ($r$ potential elements of $Q_{\mathsf{exp}}$ have been generated) \quad or \\
	$i+|E| \geq t$ & (the computed number of required consistency checks exceeds $t$) \quad or \\
	$i + |E| \geq s \; \land \; |E| \geq 1$ & (the computed number of required consistency checks exceeds $s$ and at least one potential element of $Q_{\mathsf{exp}}$ has been generated)
	\end{tabular}

	then terminate and pass $E$ on to $Ent_2$.
	\item Add $\tax_i$ to $A$. $i=i+1$.
\end{enumerate} 
$Ent_2$, given the output $E$ of $Ent_1$, can be realized as follows:
Run $\mathit{CC}$ to test the consistency of $(\mo \setminus U_\mD) \cup \mb \cup U_\Tp \cup \setof{\lnot\tax_i}$ for each $\tax_i \in E$. Add all $\tax_i \in E$ for which consistent is returned to $Q_{\mathsf{exp}}$ and discard all others. Return $Q_{\mathsf{exp}}$.

We note that the implementation of the $Ent_{\mathit{ET}}$ calls in Eq.~\eqref{eq:Q_exp} as per and $Ent_1$ and $Ent_2$ is compliant with all the Postulations 1. -- 5. given above. Moreover, with any implementation of $Ent_{\mathit{ET}}$, it might be the case that $Q_{\mathsf{exp}} = \emptyset$ (e.g.\ because there are in fact no implicit entailments of $(\mo \setminus U_\mD) \cup Q \cup \mb \cup U_\Tp$ that are not entailed by $(\mo \setminus U_\mD) \cup \mb \cup U_\Tp$). In this case, as the expanded query $Q'$ is just equal to the CQ $Q$ (an explicit-entailments query), phase P2 is run instead of Step 2 of phase P3 (for simplicity, this fact is not shown in Alg.~\ref{algo:query_comp}). All the theoretical results remain valid in case of empty $Q_{\mathsf{exp}}$. 
\begin{example}\label{ex:query_expansion}
As we have seen in Ex.~\ref{ex:phase_2}, the optimized query $Q^*$ returned by phase P2 is given by $\{3,5,6\} = \{ E \to \lnot M \land X, K\to E, C \to B \}$ (cf.\ Tab.~\ref{tab:example_dpi_0}). However, suppose the user is a domain expert with only rudimentary logical skills and wants to get presented a query containing only literals or simple implication sentences of the form $X \to Y$ for literals $X,Y$.
In this case, they would set $\mathsf{enhance} := \true$ (see line~\ref{algoline:query_comp:if_enhance=true} of Alg.~\ref{algo:query_comp}),
causing the algorithm to perform (the optional) phase P3. In Step 1 of this phase, 
applying an, e.g., Description Logic reasoner such as Pellet or HermiT \citep{sirin2007,Shearer2008} capable of handling the logic $\mathcal{L}$ (in this case Propositional Logic) used in $\exdpi$, the result $Q'$ of the query expansion is $\{E \to \lnot M \land X, K\to E, C \to B, C \to \lnot M, E \to X, K \to \lnot M, E \to \lnot M, B \to \lnot M\}$. 
Note, $Q'$ has the same QP as $Q^*$ by Prop.~\ref{prop:entailment_extraction_is_q-partition_preserving}.
Given $Q'$,
the idea is now to find an irreducible subset of it which has the same QP, namely $\Pt_{21}$, as $Q'$ 
(and $Q^*$) and includes only sentences (of the types) preferred by the user. This is accomplished by Step 2, which we describe next.\qed
\end{example}

\paragraph{Step 2: Optimized Contraction of the Expanded Query.}
We next describe the functioning of \textsc{optiMinimizeQueryForQPartition} in Alg.~\ref{algo:query_comp}. The objective of this function is 
the QP-preserving minimization of the expanded query $Q'$ returned by Step 1 to obtain a $\subseteq$-minimal subset $Q^*$ of $Q'$. 
In general, there are multiple possible minimizations of $Q'$. Given some information about (a user's) preferences $\mathit{pref}$ regarding elements in $Q'$, the aim is to find some preferred query among all minimal ones. 
For instance, there might be a scenario where some elements of $Q'$, the subset $Q'_{+}$, are favorable, and all other elements, i.e.\ $Q'_{-} := Q'\setminus Q'_{+}$, are unfavorable. In other words, the desideratum is that $Q^* \subseteq Q'_{+}$. For example, when diagnosing a physical system, those measurements executable automatically by available built-in sensors could be assigned to $Q'_{+}$ and the more costly manual measurements to $Q'_{-}$. Another strategy is to use the preferred entailments $Q_{\mathsf{exp}}$ computed in Step 1 as $Q'_{+}$. 
 
More generally, a user might be able to specify a strict (partial) preference order $\prec$ over $Q'$, e.g.\ by exploiting the (evaluation) costs $c_i$ of sentences $q_i \in Q'$ (cf.\ Sec.~\ref{sec:measurement_selection}) in a way that $q_i \prec q_j$ iff $c_i < c_j$.\footnote{We denote by $x \prec y$ that $x$ is preferred to $y$.} Note, the partitioning of $Q'$ into favorable ($Q'_{+}$) and non-favorable elements ($Q'_{-}$) stated before corresponds to the special case where $x \prec y$ iff $x \in Q'_{+}$ and $y \in Q'_{-}$ (i.e.\ the DAG corresponding to this order is bipartite).

Given a strict partial order $\prec$ over $Q'$, one might want to obtain the best query according to this order $\prec$. The authors of \citep{junker04} (and originally \citep{brewka1989}) define what the term ``best'' in such a context might refer to. They suggest to use an \text{(anti-)}lexicographic preference order over sets of interest (in our case: minimal subsets of $Q'$) based on some linearization\footnote{A partial order $\prec'$ over a set $X$ is a \emph{linear extension} (or: \emph{linearization}) of a partial order $\prec$ over $X$ iff $\prec'$ is a total order and $x \prec' y$ whenever $x \prec y$.  A linearization for a partial order $\prec$ over $X$ can be found in linear time in $O(|X|+n_\prec)$ where $n_\prec$ denotes the number of tuples $x \prec y$ in the partial order $\prec$ \citep[Sec.~2.2.3.]{knuth1997v1}.} of $\prec$. Adhering to this suggested notion of ``best'' we
define according to Def.~6 and 7 in \citep{junker04}:
\begin{definition}\label{def:antilex_extension}
Given a strict total order $<$ on $Q'$ (e.g.\ a linearization of $\prec$), let the elements of $Q'$ be enumerated in increasing order $q_1,\dots,q_k$ (i.e.\ $q_i < q_j$ implies $i < j$). Further, let $X, Y \subseteq Q'$. Then $X <_{\mathsf{antilex}} Y$ (in words: $X$ is antilexicographically preferred to $Y$) iff
there is some $r$ such that $q_r \in Y\setminus X$ and $X \cap \setof{q_{r+1},\dots,q_k} = Y \cap \setof{q_{r+1},\dots,q_k}$.
\end{definition}
Intuitively, $X <_{\mathsf{antilex}} Y$ iff, when visiting the elements of $Q'$ starting from the most dispreferred ones $q_k, q_{k-1}, \dots$, one first encounters only elements that are in both or none of $X,Y$, but the first element ($q_r$) that is in exactly one of the sets $X,Y$ is in $Y$. Hence, excluding those most dispreferred elements on which $X,Y$ are equal, $Y$ contains the most dispreferred element.
\begin{example}\label{ex:antilex}
Suppose $Q' = \{a,b,\dots,z\}$ and $<$ to be the standard lexicographic order on $Q'$, i.e.\ $a<b,b<c,\dots$. Then, e.g., $X := \setof{c,g,h,m,r,u,w} <_{\mathsf{antilex}} \setof{c,g,h,n,r,u,w} =: Y$ because, deleting all most dispreferred elements $\setof{r,u,w} \in X \cap Y$,
$Y$ comprises the most dispreferred element $n$.\qed 
\end{example}
\begin{definition}\label{def:preferred_min_query}
Let $Q' \in \mQ_\mD$ be a query with q-partition $\Pt$ and $\prec$ be a strict partial (preference) order over $Q'$ and let $<$ be a linearization of $\prec$. A subset $Q \subseteq Q'$ is a \emph{preferred query} iff $\Pt = \Pt_\mD(Q)$ and there is no $\bar{Q} \subseteq Q'$ such that $\Pt = \Pt_\mD(\bar{Q})$ and $\bar{Q} <_{\mathsf{antilex}} Q$.
%
%
%
\end{definition}
For the purpose of finding a preferred query given $Q'$ with QP $\Pt$ and $\prec$, one can use a variant of the divide-and-conquer method \textsc{QuickXplain} proposed in \citep{junker04}. One appropriate such variant is the \textsc{minQ} procedure given in \citep[p.~111 ff.]{Rodler2015phd}. Roughly, it works as \textsc{QuickXplain}, but calls a function \textsc{isQPartConst} (see \citep[Alg.~4]{Rodler2015phd}), which returns $\true$ iff the QP of its input is equal to $\Pt$, instead of the $\lnot \textsc{isConsistent}$ test in line 4 of \citep[Alg. $\textsc{QuickXplain}$]{junker04}. That is, the verification whether a KB is still inconsistent in \textsc{QuickXplain} is traded for a test whether the QP is still the same in \textsc{minQ}. 
%
The following proposition about \textsc{minQ} was proven in \citep[Prop.~8.7]{Rodler2015phd}:
\begin{proposition}\label{prop:minQ_returns_subseteq-minimal_query}
Let $\mD \subseteq \minD_{\dpi}$ and $Q' \in \mQ_\mD$ a query with q-partition $\Pt$. Then, \textsc{minQ}, given $Q'$, $\Pt$ and $\dpi$ as inputs, returns a $\subseteq$-minimal query $Q^* \subseteq Q'$ such that $\Pt_\mD(Q^*) = \Pt$.	
\end{proposition}
Given a sorted input, the output of \textsc{minQ} is characterized as the next proposition states. It is a direct consequence of \citep[Theorem 1]{junker04} and Prop.~\ref{prop:minQ_returns_subseteq-minimal_query}.
\begin{proposition}\label{prop:minQ_returns_preferred_query}
Let $\mD \subseteq \minD_{\dpi}$, $Q' \in \mQ_\mD$ a query with q-partition $\Pt$, $\prec$ a strict (partial) order over $Q'$ and $Q'_{\mathsf{sort}}:=[q_1,\dots,q_k]$ an ascending sorting of $Q'$ based on any linearization $<$ of $\prec$. Then, \textsc{minQ}, given $Q'_{\mathsf{sort}}$, $\Pt$ and $\dpi$ as inputs, returns a preferred query $Q^*$.
\end{proposition}
Let us explicate the implications of Prop.~\ref{prop:minQ_returns_preferred_query}. Let $|Q^*| = s$. Studying Def.~\ref{def:antilex_extension} carefully, we see that $Q^*$, 
among all $\subseteq$-minimal subsets of $Q'$ with the same QP $\Pt$ as $Q'$, is the one with the leftmost rightmost element w.r.t.\ the sorting $Q'_{\mathsf{sort}}$. Moreover, among the (possibly multiple) subsets of $Q'$ with this property, $Q^*$ is the one with the leftmost $2$nd-rightmost element w.r.t.\ the sorting $Q'_{\mathsf{sort}}$, and so on for all $l$th-rightmost elements for $l \in \setof{1,\dots,s}$. This insight leads to some interesting corollaries, which we state after a small illustrating example:
\begin{example}
Let $Q'_{\mathsf{sort}} = [x_1,q_b,q_{d1},x_2,q_a,x_3,q_{a,b},q_c,x_4,q_{a,b,c},q_{d2},x_5]$ where the $\subseteq$-minimal QP-preserving subsets of $Q'$ are $Q_a := \{q_a,q_{a,b},q_{a,b,c}\}$, $Q_b :=\{q_b,q_{a,b},q_{a,b,c}\}$, $Q_c :=\{q_c,q_{a,b,c}\}$ and $Q_d :=\{q_{d1},q_{d2}\}$ and the elements $x_i$ are irrelevant in that they do not occur in any of the subsets $Q_i$ ($i \in \setof{a,\dots,d}$) of interest. Then \textsc{minQ}, given $Q'_{\mathsf{sort}}$, returns $Q_b$. 
The explanation is as follows:
The leftmost rightmost element (regarding the used sorting of $Q'$) of any of the sets of interest is $q_{a,b,c}$ (note, $q_{d2}$ is another rightmost element, namely of $Q_d$, but lies more to the right), i.e.\ $Q_d$ is definitely not the returned set. Fixing $q_{a,b,c}$ (rightmost element of the returned set), the leftmost $2$nd-rightmost element contained in any of the remaining possible sets ($Q_a,Q_b,Q_c$) is $q_{a,b}$, i.e.\ $Q_c$ is definitely not the result. Finally, the leftmost $3$rd-rightmost 
element comprised by any of the remaining sets ($Q_a,Q_b$) with common intersection $\{q_{a,b},q_{a,b,c}\}$ of largest elements (indices $7$ to $12$) in $Q'_{\mathsf{sort}}$ is $q_b$, which is why $Q_b <_{\mathsf{antilex}} Q_a$.\qed   
\end{example}
The first corollary testifies that the QCM $c_{\max}$ (see page~\pageref{etc:QCM})
is optimized by \textsc{minQ} if costs of sentences in the query $Q'$ are available (cf.\ Sec.~\ref{sec:measurement_selection}).
\begin{corollary}\label{cor:P3_optimizes_c_max}
Let $c_i$ be the cost of $q_i \in Q'$ and $Q'_{\mathsf{sort}} = [q_1,\dots,q_k]$ be sorted in ascending order by cost, i.e.\ $c_i < c_j$ implies that $q_i$ occurs before $q_j$ in $Q'_{\mathsf{sort}}$. 
Then, \textsc{minQ}, given the other inputs as stated in Prop.~\ref{prop:minQ_returns_preferred_query}, returns a query that is optimal regarding the QCM $c_{\max}$. 
%
%
\end{corollary}
Given a partitioning of $Q'$ into favored and unfavored elements, \textsc{minQ} computes a query consisting of only favored elements, if such a query exists.
\begin{corollary}\label{cor:P3_computes_query_containing_only_preferred_elements_if_existent}
Let $Q'_{+} \subseteq Q'$ be the preferred and $Q'_{-} = Q' \setminus Q'_{+}$ be the dispreferred elements of $Q'$. Further, define $Q'_{\mathsf{sort}}$ as the list resulting from the concatenation $Q'_{+} \| Q'_{-}$. Then, if such a query $Q^* \subseteq Q'$ exists, \textsc{minQ}, given the other inputs as stated in Prop.~\ref{prop:minQ_returns_preferred_query}, returns a query $Q^* \subseteq Q'_{+}$.
\end{corollary}
If, for example, preferences are only given over elements of $Q'_{-}$ (e.g.\ because one puts all preferred entailments $Q_{\mathsf{exp}}$ (see Eq.~\eqref{eq:Q_exp}) computed in Step 1 into $Q'_{+}$ and is indifferent between them), then we can combine Cor.~\ref{cor:P3_optimizes_c_max} and \ref{cor:P3_computes_query_containing_only_preferred_elements_if_existent} to:
\begin{corollary}\label{}
Let $Q'_{+},Q'_{-}$ be as in Cor.~\ref{cor:P3_computes_query_containing_only_preferred_elements_if_existent}, $\prec$ be a (partial) strict order over $Q'_{-}$ and $Q'_{-,\mathsf{sort}}$ be a sorting of $Q'_{-}$ based on some linearization of $\prec$. Further, define $Q'_{\mathsf{sort}}$ as the list resulting from the concatenation $Q'_{+} \| Q'_{-,\mathsf{sort}}$. Then, if such a query $Q^* \subseteq Q'$ exists, \textsc{minQ}, given the other inputs as stated in Prop.~\ref{prop:minQ_returns_preferred_query}, returns a query $Q^* \subseteq Q'_{+}$. Otherwise, it returns a query that is optimal regarding the QCM $c_{\max}$.
\end{corollary}

Note that (under the assumption of P $\neq$ NP) it is not achievable by a single \textsc{minQ}-call to find a minimum-cardinality query with the same QP as $Q'$. One way to see this is by the fact that \textsc{minQ} runs in polynomial time (modulo the time required for calls to \textsc{isQPartConst}), as observed in \citep[Prop.~8.8]{Rodler2015phd}, and by Prop.~\ref{prop:complexity_P2} which states that finding cardinality-minimal QP-preserving queries 
is already an NP-hard problem for explicit-entailments queries $Q'$, 
a subclass of all queries in $\mQ_\mD$.
Hence, assuming \textsc{minQ} to be a polynomial procedure for finding QP-preserving subqueries of minimal size for arbitrary queries $Q' \in \mQ_\mD$ would imply in particular the existence of a polynomial procedure for the problem over the said subclass of queries because \textsc{isQPartConst} can be reduced to set comparisons and therefore runs in polynomial time for explicit-entailments queries (see Prop.~\ref{prop:ee-query_q-partition_construction_by_means_of_set_comparison}). Thus, we could in this case derive the equality P $=$ NP.

Furthermore, we point out that there might be multiple preferred queries since Def.~\ref{def:preferred_min_query} assumes an arbitrary linearization of the partial order $\prec$. However, if there are multiple preferred queries, then all of them are incomparable w.r.t.\ $\prec$. That is, they stand in no \emph{necessary} preference relationship with each other in the sense of Def.\ref{def:antilex_extension}. 
\begin{example}\label{ex:multiple_preferred_queries}
Let $Q'_{\mathsf{sort}} = [a_1,b_1,b_2,a_2,c_1,c_2]$ where $\setof{x_1,x_2}$ constitute all possible preferred queries and $x_1 \prec x_2$ are the only preferences given, for $x \in \setof{a,b,c}$. In this case, \textsc{minQ} will return $\setof{b_1,b_2}$ since its rightmost element ($b_2$, index 3 in $Q'_{\mathsf{sort}}$) is the leftmost of all rightmost elements $x_2$ for $x \in \setof{a,b,c}$ (indices of other rightmost elements are 4 for $a_2$ and 6 for $c_2$). But, $\prec$ admits e.g.\ also the sorting $Q'_{\mathsf{sort}} = [c_1,c_2,a_1,b_1,b_2,a_2]$ which would involve the output of $\setof{c_1,c_2}$. Since no $c_i$ is comparable with any $x_i$ for $x \in \setof{a,b}$ and $i \in \setof{1,2}$, $\prec$ is indifferent between the sets $\setof{b_1,b_2}$ (result in the first case) and $\setof{c_1,c_2}$ (result in the second case). Obviously, there is also a sorting which favors $\setof{a_1,a_2}$. If however $\prec$ includes additionally $a_2 \prec c_1$, then $\prec$ \emph{always} prefers $\setof{a_1,a_2}$ to $\setof{c_1,c_2}$ and the only possible outcomes, depending on the used linearization of $\prec$, are $\setof{a_1,a_2}$ or $\setof{b_1,b_2}$.\qed
\end{example}

Let us now reconsider our running example to see the results of applying Step 2 to the expanded query computed in Ex.~\ref{ex:query_expansion}. 
\begin{example}\label{ex:query_contraction}
Recall from Ex.~\ref{ex:phase_2} that the final computed query, if possible, should contain only literals or simple implication sentences of the form $X \to Y$ for literals $X,Y$. Therefore, the expanded query $Q'$ (see Ex.~\ref{ex:phase_2}) is to be partitioned into the preferred sentences $Q'_{+} = \{ K\to E, C \to B, C \to \lnot M, E \to X, K \to \lnot M, E \to \lnot M, B \to \lnot M \}$ and dispreferred ones $Q'_{-} = \{ E \to \lnot M \land X \}$. Given the (ordered) list $Q'_{\mathsf{sort}} := Q'_{+} \| Q'_{-}$ (see equations below, cf.\ Cor.~\ref{cor:P3_computes_query_containing_only_preferred_elements_if_existent}) 
as well as $\Pt := \Pt_{21}$ (QP of $Q'$, see Ex.~\ref{ex:canonical_q-partition_search_ENT}) and the DPI $\dpi := \exdpi$ (see Tab.~\ref{tab:example_dpi_0}) as inputs, we roughly illustrate the functioning of \textsc{minQ}. At this, we assume that \textsc{minQ} splits the still relevant sublist of $Q'_{\mathsf{sort}}$ in half in each iteration (this yields the lowest worst case complexity, cf.\ \citep{junker04}). 
Further, the single underlined sublist denotes the current input to the function \textsc{isQPartConst}, the double underlined elements are those that are already fixed elements of the returned solution $Q^*$, and the grayed out elements those that are definitely not in the returned solution $Q^*$. Finally, $\times$ and $\checkmark$ signify that the QP of the underlined subquery is different from or equal to $\Pt$, respectively. Next, we show the algorithm's actions on $Q'_{\mathsf{sort}}$:
\begin{align*}
[\underline{K\to E, C \to B, C \to \lnot M, E \to X}, K \to \lnot M, E \to \lnot M, B \to \lnot M,  E \to \lnot M \land X] & \;\;\;\checkmark \\
[\underline{K\to E, C \to B}, C \to \lnot M, E \to X, {\color{gray}K \to \lnot M, E \to \lnot M, B \to \lnot M,  E \to \lnot M \land X}] & \;\;\;\times \\
[\underline{K\to E, C \to B, C \to \lnot M}, E \to X, {\color{gray}K \to \lnot M, E \to \lnot M, B \to \lnot M,  E \to \lnot M \land X}] & \;\;\;\times \\
[\underline{K\to E, C \to B}, C \to \lnot M, \underline{\underline{E \to X}}, {\color{gray}K \to \lnot M, E \to \lnot M, B \to \lnot M,  E \to \lnot M \land X}] & \;\;\;\times \\
[K\to E, C \to B, \underline{\underline{C \to \lnot M, E \to X}}, {\color{gray}K \to \lnot M, E \to \lnot M, B \to \lnot M,  E \to \lnot M \land X}] & \;\;\;\checkmark
\end{align*}
For instance, in the first line, \textsc{isQPartConst} returns $\true$ (see $\checkmark$) since the (underlined) left half $Q'_{\mathsf{sort}}[1..4]$ is still a query with the same QP (i.e.\ $\Pt_{21}$) as $Q'$. Thence, the right half of elements can be dismissed (one solution is guaranteed to be in the left half). The latter is again split in half and the left part $Q'_{\mathsf{sort}}[1..2]:=[K\to E, C \to B]$ is tested by \textsc{isQPartConst}, which returns negatively (line 2). The reason is that 
$\Pt_{\mD}(Q'_{\mathsf{sort}}[1..2]) = \tuple{\setof{\md_1,\md_4,\md_5,\md_6},\setof{\md_2,\md_3}, \emptyset} \neq \tuple{\setof{\md_4,\md_5},\setof{\md_1,\md_2,\md_3,\md_6}, \emptyset} = \Pt_{21}$. Thus, a half of the right part is added to the left part, yielding $Q'_{\mathsf{sort}}[1..3]$ (see underlined elements in line 3), and again tested. Once more it is found that $\Pt_{\mD}(Q'_{\mathsf{sort}}[1..3]) = \tuple{\setof{\md_4,\md_5},\setof{\md_1,\md_2,\md_3}, \setof{\md_6}} \neq \Pt_{21}$. Note, due to the positive \textsc{isQPartConst} check in line 1, it is now clear that $E \to X$ (see double underline in line 4) must be in the solution $Q^*$. From now on, $E \to X$ is part of any input to \textsc{isQPartConst} argument. Since $Q'_{\mathsf{sort}}[1..2]$ (along with $E \to X$) has the QP $\tuple{\setof{\md_4,\md_5},\setof{\md_2,\md_3,\md_6}, \setof{\md_1}}$ (see line 4), it is now a fact that $C \to \lnot M$ must as well be an element of $Q^*$. Eventually, the (positive) \textsc{isQPartConst} test for $Q'_{\mathsf{sort}}[3..4]$ (see line 5) proves that the latter must be a $\subseteq$-minimal subquery $Q^*$ of $Q'$ with QP $\Pt_{21}$. Lastly, $Q^*$ is returned.

We observe that, $Q^* \subseteq Q'_{+}$ holds, as required. That is, the returned query contains only preferred sentences.
\qed
\end{example}

\paragraph{Complexity of P3.} 
Step 1 requires exactly $2$ calls of the function $Ent_{\mathit{ET}}$, if there is a reasoner implementing such a function directly (see page~\pageref{etc:implementation_of_Ent_ET}). Alternatively, using a consistency checker $\mathit{CC}$ to realize $Ent_{\mathit{ET}}$ (see page~\pageref{etc:implementation_of_Ent_ET}),
the latter requires maximally a constant number $t$ of consistency checks. 
The complexity of $Ent_{\mathit{ET}}$ (and of consistency checks) depends on the expressivity of the underlying knowledge representation formalism. 
Hence:
\begin{proposition}\label{prop:P3_step1_complexity}
Step 1 of P3 runs in $O(1)$ time (modulo reasoning time).
\end{proposition}
Step 2, as was shown by \citep{Rodler2015phd} (for \textsc{minQ}) and originally by \citep{junker04} (for \textsc{QuickXplain}), requires a polynomial number of calls to \textsc{isQPartConst}:
\begin{proposition}\label{prop:minQ_complexity}
\textsc{minQ} runs in $O(|Q^*|\log_2 \frac{|Q'|}{|Q^*|})$ time (modulo the time required for \textsc{isQPartConst}).
\end{proposition} 
In order to further refine this result, let $Q'_{\mathsf{sub}} \subseteq Q'$ be any subquery of $Q'$. Then \textsc{isQPartConst}, given the argument $Q'_{\mathsf{sub}}$, verifies the preserved membership of each $\md \in \mD$ in its respective part of the QP, i.e.\ whether $\md \in \mD^Z(Q') \implies \md \in \mD^Z(Q'_{\mathsf{sub}})$ for $Z \in \setof{+,-,0}$.
We note that each $\md \in \dx{}(Q')$ is also an element of $\dx{}(Q'_{\mathsf{sub}})$ by Prop.~\ref{prop:partition}. 
Hence, the membership verification is only necessary for the diagnoses in $\dnx{}(Q') \cup \dz{}(Q')$.
Moreover, by logical monotonicity and Prop.~\ref{prop:partition}, each diagnosis in $\dnx{}(Q')$, if not in $\dnx{}(Q'_{\mathsf{sub}})$, can be in 
$\dz{}(Q'_{\mathsf{sub}})$ or $\dx{}(Q'_{\mathsf{sub}})$, but each diagnosis in $\dz{}(Q')$, if not in $\dz{}(Q'_{\mathsf{sub}})$, can only 
be in
$\dx{}(Q'_{\mathsf{sub}})$. Therefore, as witnessed by \citep[Lem.~8.1]{Rodler2015phd}, for each $\md_r \in \dnx{}(Q')$, one needs to verify that $\md_r \in \dnx{}(Q'_{\mathsf{sub}})$, i.e.\ that some $x \in \RQ\cup \Tn$ is violated by $\mo^*_r \cup Q'_{\mathsf{sub}}$. The latter operation requires a maximum of $|\RQ|+|\Tn|$ logical consistency checks (where $|\RQ|$ is predefined and constant, i.e.\ in $O(1)$, whereas $|\Tn|$ might grow during a diagnostic session, cf.\ Def.~\ref{def:dpi}). Also due to \citep[Lem.~8.1]{Rodler2015phd}, for each $\md_r \in \dz{}(Q')$, one needs to verify that $\md_r \notin \dx{}(Q'_{\mathsf{sub}})$. This operation can involve at most $|Q'_{\mathsf{sub}}| \leq |Q'|$ logical consistency checks (to verify whether each sentence in $Q'_{\mathsf{sub}}$ is entailed by $\mo^*_r$). Importantly, if one of all these verification steps fails, i.e.\ if any diagnosis $\md_r \in \dnx{}(Q') \cup \dz{}(Q')$ has a different position in the QP of $Q'_{\mathsf{sub}}$ than in the QP of $Q'$, then \textsc{isQPartConst} immediately terminates (negatively), cf.\ \citep[Alg.~4]{Rodler2015phd}. 
Overall, we found that \textsc{isQPartConst} runs in $O(m\,|\mD|)$ time (modulo consistency checking) where $m := \max\{|\Tn|,|Q'|\}$.
The complexity of a single consistency check depends on the expressivity of the underlying knowledge representation formalism $\mathcal{L}$.
%
%

Hence, we recognize that Step 2 of P3 runs in polynomial time (disregarding the complexity of consistency checking):
\begin{proposition}\label{prop:P3_step2_complexity}
	Let $q:= \max\{m,|\mD|,|Q'|\}$. Then Step 2 of P3 runs in $O(q^4)$ time (modulo the time required for consistency checking).
\end{proposition}
\begin{proof}
By Prop.~\ref{prop:minQ_complexity} and the argumentation given, Step 2 of P3 runs in $O(m\,|\mD|\,|Q^*|\log_2 \frac{|Q'|}{|Q^*|})$ time (modulo the time required for consistency checking). Further, $|Q^*|\log_2 \frac{|Q'|}{|Q^*|} \in O(q^2)$ since $|Q^*| \leq |Q'| \leq q$ and $\log_2 \frac{|Q'|}{|Q^*|} \leq \log_2 |Q'| \leq |Q'| \leq q$. Finally, it follows from the definition of $q$ that $m\,|\mD| \in O(q^2)$.
\end{proof}

\subsubsection{Solution Produced by Phase 3}
\label{sec:solution_produced_by_P3}
Altogether, phase P3, i.e.\ query expansion (Step~1) along with optimized query contraction (Step~2), using the QP returned by phase P1, achieves the following:
\begin{theorem}\label{theorem:P3_solves_problem_1}
Let Conjecture~\ref{conj:CQPs=QPs} hold and the QCM be $c_{\max}$. Then P1 and P3 (using the threshold $t_m := 0$)
solve Prob.~\ref{prob:query_optimization} with full search space $\mathbf{S} = \mQ_{\mD}^{\bcancel{0}}$ (of queries discriminating among all elements of $\mD$, cf.\ Def.~\ref{def:query_q-partition}). 
Moreover, for any predefined set of preferred (query) sentences, the returned solution will contain only preferred elements, if such a solution exists.	
\end{theorem}
\begin{remark}
We again stress that Conjecture~\ref{conj:CQPs=QPs} is by no means necessary for the proper functioning of our presented algorithms. Please see our more detailed discussion on this on page~\pageref{etc:conjecture_discussion}, after we stated Conjecture~\ref{conj:CQPs=QPs}.\qed
\end{remark} 

\subsubsection{Recapitulation of the Presented Query Selection Algorithm}
\label{sec:recap_of_presented_algo}
To wrap up Sec.~\ref{sec:contribution}, let us exemplify the entire query selection process executed by Alg.~\ref{algo:query_comp} using the MBD example stated by Fig.~\ref{fig:circuitreiter87} and Tab.~\ref{tab:circuitreiter87_KBD-DPI}:
\begin{example}\label{ex:MBD_overall}
Suppose we got the information from the manufacturer of the gates that and-, or- and xor-gates fail with a probability of $0.05$, $0.02$ and $0.01$, respectively. As we have already discussed (cf.\ Ex.~\ref{ex:reduction_MBD_to_KBD}), the set of minimal diagnoses for $\exdpiMK$ (see Tab.~\ref{tab:circuitreiter87_KBD-DPI}) is $\minD_{\exdpiMK} = \setof{\md_1,\md_2,\md_3} = \setof{\{\tax_1\},\{\tax_2,\tax_4\},\{\tax_2,\tax_5\}}$ corresponding to the (abnormality assumptions of the) sets of components $\{\{X_1\},$  $\{X_2,A_2\}, \{X_2,O_1\}\}$. Let the leading diagnoses be $\mD := \minD_{\exdpiMK}$. Exploiting the formula given in \citep[Sec.~4.4]{dekleer1987}, the diagnoses probabilities (normalized over $\mD$ and rounded) amount to $\tuple{p(\md_1),p(\md_2),p(\md_3)} = \tuple{0.93,0.05,0.02}$.
	
\emph{(Phase P1:)} Starting from the initial partition $\tuple{\emptyset,\mD,\emptyset}$, the generated successors are $\Pt_1 := \tuple{\setof{\md_1},\setof{\md_2,\md_3},\emptyset}$, $\Pt_2 :=\tuple{\setof{\md_2},\setof{\md_1,\md_3},\emptyset}$ and $\Pt_3 :=\tuple{\setof{\md_3},\setof{\md_1,\md_2},\emptyset}$. Note that all these successors are QPs (proven by Prop.~\ref{prop:--di,MD-di,0--_is_canonical_q-partition_for_all_di_in_mD}). Assuming the same QSM $m$, threshold $t_m$, heuristic $h$ and pruning function as used in Ex.~\ref{ex:canonical_q-partition_search_ENT}, the heuristic values $\tuple{h(\Pt_1),h(\Pt_2),h(\Pt_3)}$ of these QPs are $\tuple{0.43, 0.45, 0.48}$. 
Since $\Pt_1$ has the best (i.e.\ least) $h$-value, but is not a goal, P1 continues with the expansion of $\Pt_1$ after storing $\Pt_1$ as the currently best visited QP so far.
However, since $p(\dx{}(\Pt_1)) = 0.93 > 0.5$, the pruning criterion is met and no successors are generated. Instead, the next best sibling of $\Pt_1$, namely $\Pt_2$ is considered. Here, no pruning takes place and the successors generated based on the $\subseteq$-minimal traits (cf.\ Def.~\ref{def:trait}) $\mathsf{Tr}_{\min}(\Pt_2) = \setof{\md_1^{(2)},\md_3^{(2)}} = \setof{\setof{\tax_1},\setof{\tax_5}}$ are $\Pt_{21}:=\tuple{\setof{\md_2,\md_1},\setof{\md_3},\emptyset}$ and $\Pt_{22}:=\tuple{\setof{\md_2,\md_3},\setof{\md_1},\emptyset}$ with $h(\Pt_{21}) = 0.48$ and $h(\Pt_{22}) = 0.43$. Due to the facts that for $\Pt_{21}$ the pruning condition is satisfied, $\Pt_{22}$ has no successor QPs (cf.\ Cor.~\ref{cor:S_next_sound+complete}), and none of $\Pt_{21}$, $\Pt_{22}$ is a goal, P1 backtracks and proceeds with the QP $\Pt_3$. In an analogue way as shown for $\Pt_2$, the successor QP $\Pt_{31} = \tuple{\setof{\md_3,\md_1},\setof{\md_2},\emptyset}$ is generated. Note that P1, in that it stores diagnoses that must not be moved from $\dnx{}$ to $\dx{}$ to avoid duplicates (for details see the extended version of the paper \citep[p.~92 ~ff.]{DBLP:journals/corr/Rodler16a}), does not generate $\Pt_{32} = \tuple{\setof{\md_3,\md_2},\setof{\md_1},\emptyset}$ because it is equal to $\Pt_{22}$ which has already been explored.
	%
Again, no successors are generated for $\Pt_{31}$ (pruning). 
Hence, the complete (pruned) backtracking search tree has been constructed and the stored best (of all) CQP(s) for $\mD$, $\Pt_{1}$, is returned.
	
\emph{(Phase P2:)} Let us suppose that a globally optimal query w.r.t.\ the QCM $c_{\Sigma}$ (see page~\pageref{etc:QCM}) over the restricted search space considered by P2 (see Theorem~\ref{theorem:P1_sound_complete}) is desired by the user (accounted for by setting $\mathsf{enhance} := \false$, see Alg.~\ref{algo:query_comp}). Moreover, let the expected cost of testing an and-, or- and xor-gate, respectively, be $1$, $3$ and $2$.
Then $\mathsf{Tr}_{\min}(\Pt_1) = \setof{\md_2^{(1)},\md_3^{(1)}} = \setof{\setof{\tax_2,\tax_4},\setof{\tax_2,\tax_5}}$is used to extract the \text{$c_{\Sigma}$-optimal} query $Q^* = \setof{\tax_2} = \setof{beh(X_2)} = \setof{out(X_2) = xor(in1(X_2),in2(X_2))}$ as the minimal hitting set with least cost ($c_{\Sigma}(Q^*) = 2$) of all elements of $\mathsf{Tr}_{\min}(\Pt_1)$  (cf.~Prop.~\ref{prop:phase_P2_returns_EE-queries_in_best-first_order}). Note, the (only) other possible $\subseteq$-minimal explicit-entailments query for $\Pt_1$ is $Q := \setof{\tax_4,\tax_5}$ with a cost of $c_{\Sigma}(Q) = 1+3=4$. $Q^*$ is a direct component probe (cf.\ Ex.~\ref{ex:query_representation} and \citep{DBLP:conf/ijcai/BrodieRMO03}) and can be understood as the question ``Does gate $X_2$ work properly?''. 
	
\emph{(Phase P3:)} Given that a query optimized over the full search space is wanted, $\mathsf{enhance}$ must be set to $\true$ in Alg.~\ref{algo:query_comp}. This causes the execution of phase P3 (instead of phase P2). As an input $\mathit{Inf}$ to Alg.~\ref{algo:query_comp} we assume e.g.\ some constraint propagator, similar to the one described in \citep{dekleer1987}, which computes predictions of the values at the circuit's wires (cf.\ Fig.~\ref{fig:circuitreiter87}). Moreover, we suppose that the preferred entailment types $ET$ are exactly those stating values of wires, e.g.\ $out(A_1) = 1$.  
	
In P3, the CQ of $\Pt_1$, given by $Q := U_{\mD} \setminus U_{\dx{}(\Pt_1)} = \setof{\tax_1,\tax_2,\tax_4,\tax_5} \setminus \setof{\tax_1} = \setof{\tax_2,\tax_4,\tax_5}$ (cf.\ Lem.~\ref{lem:CQ_equal_to_U_D_setminus_U_D+}), is first needed for the query enhancement (Step~1). To this end, the query expansion, $Q_{\mathsf{exp}}$, is computed as per Eq.~\eqref{eq:Q_exp} as 
$[Ent_{\mathit{ET}}(\setof{\tax_3} \cup \setof{\tax_2,\tax_4,\tax_5} \cup \setof{\tax_6,\dots,\tax_{17}} \cup \emptyset) \setminus Ent_{\mathit{ET}}(\setof{\tax_3} \cup \setof{\tax_6,\dots,\tax_{17}} \cup \emptyset)] \setminus \setof{\tax_2,\tax_4,\tax_5} 
= 
[Ent_{\mathit{ET}}(\setof{beh(A_1)} \cup \{beh(X_2),beh(A_2)$, 
$beh(O_1)\} \cup \sd_{gen} \cup \obs) \setminus Ent_{\mathit{ET}}(\setof{beh(A_1)} \cup \sd_{gen} \cup \obs)] \setminus \{beh(X_2),beh(A_2)$, $beh(O_1)\} = [ \setof{out(X_1)=0,out(A_2)=0,out(A_1)=0} \setminus \setof{out(A_1)=0}] \setminus \{beh(X_2), beh(A_2)$, $beh(O_1)\} = \{out(X_1)=0,out(A_2)=0\}$. Next, the contraction of the expanded query $Q' = Q \cup Q_{\mathsf{exp}} = \{beh(X_2), beh(A_2), beh(O_1), out(X_1)=0,out(A_2)=0\}$ (see Eq.~\eqref{eq:Q'_enriched_query}) takes place (Step~2). Let us assume that no preference order over query sentences is given, except that a user wants to avoid direct component tests (input argument $\mathit{pref}$, see Alg.~\ref{algo:query_comp}). In other words, the query should not include any $beh(.)$ sentences. This is reflected by setting $Q'_{+} := Q_{\mathsf{exp}}$ and by specifying the input to \textsc{minQ} as the (ordered) list $Q'_{\mathsf{sort}} = Q'_{+} \| Q'_{-} = [out(X_1)=0,out(A_2)=0,beh(X_2),beh(A_2), beh(O_1)]$ (cf.\ Cor.~\ref{cor:P3_computes_query_containing_only_preferred_elements_if_existent}). In an analogous manner as illustrated in Ex.~\ref{ex:query_contraction}, \textsc{minQ} determines an optimized contracted query $Q^*$ as $\setof{out(X_1)=0}$. We note that this is the only $\subseteq$-minimal query satisfying Cor.~\ref{cor:P3_computes_query_containing_only_preferred_elements_if_existent} 
because the only other $\subseteq$-minimal query comprising only elements from $Q'_{+}$ is $Q_{alt}:=\setof{out(A_2)=0}$ which has not the QP $\Pt_1$, i.e.\ is not QP-preserving. The actual QP $\Pt_\mD(Q_{alt})$ of $Q_{alt}$ is $\tuple{\setof{\md_1,\md_2},\setof{\md_3},\emptyset}$.
Hence, Alg.~\ref{algo:query_comp} suggests to probe at the wire connecting gate $X_1$ with gates $X_2$ and $A_2$. Taking into account the query outcome probabilities estimated from the given component fault probabilities, we see that there is a strong bias (probability $0.93$, cf.\ Eq.~\eqref{eq:prob_of_pos_query_answer}) towards a measurement outcome of $out(X_1)=0$. In this case, as we have shown in Ex.~\ref{ex:circuit_MBD-DPI_diags_conflicts_measurements}, only a single measurement is needed to single out $\md_1$ as the actual diagnosis, i.e.\ to come to the conclusion that $X_1$ must be faulty.
\qed
	%
	%
\end{example}

\section{Evaluation}
\label{sec:eval}
\subsection{The Experiments}
\label{sec:experiments}
\paragraph{The Used Dataset.} 
To evaluate the presented algorithm, we used real-world inconsistent know-ledge-based (KB) systems, i.e.\ KBD-DPIs. The reasons for this are as follows: 
\begin{enumerate}[noitemsep]
	\item As shown in Sec.~\ref{sec:reduction_of_mbd_to_kbd}, any MBD problem (i.e.\ an MBD-DPI as per Def.~\ref{def:MBD-DPI}) can be reduced to and hence viewed as a KBD problem (i.e.\ a KBD-DPI).
	\item \label{enum:experiment_dataset:system_type_irrelevant} The \emph{type} of the system underlying a DPI is irrelevant to our methods, only 
	the \emph{DPI size} (number of logical sentences), 
	the \emph{DPI structure} (size, \# or probability of diagnoses), and -- for the optional phase P3 -- 
	the \emph{DPI (reasoning) complexity} (expressivity of the underlying logic $\mathcal{L}$) are critical.
	\item KB systems pose a hard challenge for query selection methods due to the implicit nature and the generally infinite number of the possible queries.\footnote{Note that (in most logics) each sentence, i.e.\ element of a query, can be rewritten in infinitely many ways, each time resulting in a semantically equivalent, but syntactically different sentence  (cf.\ the argumentation in Rem.~\ref{rem:infinitely_many_solutionKBs}).} That is, the possible queries are not explicitly given, but must be derived by inference.
	For instance, in a digital circuit, all probing locations are given by all of the circuit's wires, which are known from the beginning. In the case of, e.g., a medical KB, however, 
	the set of all possible sentences (common entailments of sub-KBs, cf.\ Prop.~\ref{prop:properties_of_q-partitions}.\ref{prop:properties_of_q-partitions:enum:query_is_set_of_common_ent}) that might occur in questions to a medical expert are not known in advance.
	As a consequence, in KBD-DPIs the considered query search space is not explicit and infinite.
\end{enumerate}

\begin{table}[t]
	\setlength{\tabcolsep}{7mm}
	\begin{threeparttable}[t]
		\renewcommand\arraystretch{1.2}
		\footnotesize
		\centering
		\begin{tabular}{llll} 
			\hline
			KB $\mo$ 					& $|\mo|$& Expressivity \tnote{a} 		& \textbf{\#D}/\textbf{min}/\textbf{max} \tnote{b} \\ \hline
			University (U) \tnote{c}   			& 49 		& $\mathcal{SOIN}^{(D)}$& 90/3/4      \\		
			MiniTambis (M) \tnote{c}			& 173 		& $\mathcal{ALCN}$ 		& 48/3/3	  \\
			CMT-Conftool (CC) \tnote{d}		& 458  		& $\mathcal{SIN}^{(D)}$ & 934/2/16		  \\
			Conftool-EKAW (CE) \tnote{d}		& 491 		& $\mathcal{SHIN}^{(D)}$& 953/3/10		  \\
			Transportation (T) \tnote{c}		& 1300 		& $\mathcal{ALCH}^{(D)}$& 1782/6/9	  \\
			Economy (E) \tnote{c}			& 1781 		& $\mathcal{ALCH}^{(D)}$& 864/4/8     \\
			Opengalen-no-propchains (O) \tnote{e} & 9664 	& $\mathcal{ALEHIF}^{(D)}$ & 110/2/6  \\
			Cton (C) \tnote{e}				& 33203 	& $\mathcal{SHF}$ 		& 15/1/5     \\
			\hline
		\end{tabular}
		\begin{tablenotes}
			\scriptsize
			\item[a] Description Logic expressivity, cf.\ \citep[p.~525 ff.]{DLHandbook}.
			\item[b] \textbf{\#D}, \textbf{min}, \textbf{max} denote the number, the minimal and the maximal size of minimal diagnoses (computable in $\leq 8$ h).
			\item[c] Sufficiently complex systems (\textbf{\#D} $\geq 40$) used in \citep{Shchekotykhin2012}.
			\item[d] Hardest diagnosis problems mentioned in \citep{Stuckenschmidt2008}.
			\item[e] Hardest diagnosis problems tested in \citep{Shchekotykhin2012}.
		\end{tablenotes}
		\vspace{-7pt}
		\caption{\small KBs used in the experiments. 
		}
		\label{tab:experiment_data_set}
	\end{threeparttable}
\end{table}

Tab.~\ref{tab:experiment_data_set} (column 1) shows the dataset of KBs $\mo$ used in our tests. Each $\mo$ constitutes an inconsistent and/or incoherent OWL ontology (i.e.\ a KB formulated over some Description Logic $\mathcal{L}$). From each KB $\mo$ we constructed a DPI as $\dpi:=\tuple{\mo,\emptyset,\emptyset,\emptyset}_{\RQ}$ where $\RQ := \{\text{consistency},\text{coherency}\}$, i.e.\ the entries corresponding to the sets $\mb$, $\Tp$ and $\Tn$ (cf.\ Def.~\ref{def:dpi}) were defined as empty sets. 
Accounting for the aspects given under bullet \ref{enum:experiment_dataset:system_type_irrelevant}.\ that (potentially) affect the performance of our method, the table also shows for each constructed $\dpi$ its size in terms of $|\mo|$ (column~2), 
its structure in terms of the (within 8 hours computable) number of minimal diagnoses $\textbf{\#D}$ (which is less or equal to $|\minD_{\dpi}|$) as well as the minimal (\textbf{min}) and maximal (\textbf{max}) size of a minimal diagnosis (computable within 8 hours) w.r.t.\ $\dpi$ (column~4), and the reasoning complexity in terms of the expressivity of the Description Logic underlying $\dpi$ (column~3).   
\paragraph{Experimental Settings (EXP1 -- Comprehensive Evaluation of the New Method).}
In our experiments, for each $\dpi$ and each $n \in \setof{10,20,\dots,80}$, we randomly generated 5 different $\mD\in\minD_{\dpi}$ with $|\mD| = n$ by using \textsc{Inv-HS-Tree} \citep{Shchekotykhin2014} with randomly shuffled input each of the $5$ times. Each $\md \in \mD$ was assigned a uniformly random probability and probabilities were normalized over $\mD$. 
%
For each of these 5 $\mD$-sets, we used (a)~entropy (ENT) \citep{dekleer1987} and (b)~split-in-half (SPL) \citep{Shchekotykhin2012} 
as QSM $m$ and $c_{|\cdot|}$ (cf.\ page~\pageref{etc:QCM}) as QCM $c$, and then ran phases P1, P2 and P3 to compute a query as per Theorem~\ref{theorem:P1+P2_solve_query_optimization_problem} (obtained from the execution of phases P1 and P2) and Theorem~\ref{theorem:P3_solves_problem_1} (obtained from the execution of phases P1 and P3), respectively. We specified the optimality threshold $t_m$ as $0.01$ (cf.\ Ex.~\ref{ex:canonical_q-partition_search_ENT}) for $m = \text{ENT}$ and as $0$ for $m = \text{SPL}$. The setting for ENT to a value higher than zero arises from the observation that there is practically never a QP for which $p(\dx{})$ and $p(\dnx{})$ both have a probability of exactly $0.5$. Note that the value of $0.01$ we used is one order of magnitude smaller than the one used in other experiments, e.g.\ \citep{Shchekotykhin2012}, and thence the QSM properties postulated for an optimal query are stricter and the search problem is harder. For SPL, on the other hand, it is reasonable to require the returned query to exhibit an optimal split, i.e.\ half of the leading diagnoses in $\dx{}$ and the other half in $\dnx{}$, because such QPs are usually frequent (in case $|\mD|$ is an even number).  

For the search in P1 we employed the simple heuristic $h$ which assigns $h(\Pt) = |p(\dx{})-0.5|$ to a QP $\Pt:=\tuple{\dx{}, \dnx{},\dz{}}$ (cf.\ Ex.~\ref{ex:canonical_q-partition_search_ENT}). Similarly, we used a function $h$ where $h(\Pt) = \left| |\dx{}|-\frac{1}{2}|\mD| \right|$ as a heuristic for SPL. As regards pruning in P1, we stopped the generation of successors at $\Pt$ if $p(\dx{}) \geq 0.5$ for ENT, and if $|\dx{}| \geq \frac{1}{2}|\mD|$ for SPL.

In P3 we defined the preferred entailment types $ET$ to be the results of running classification and realization reasoning services 
\citep{DLHandbook}. In First-Order Logic terms, this means that $ET$ restricted the computed entailments to simple definite clauses of the form $\forall X (a(X) \to b(X))$ and facts of the form $a(c)$ where $a,b$ are unary predicates and $c$ is a constant. As a Description Logic reasoner (input $\mathit{Inf}$ to Alg.~\ref{algo:query_comp}) we employed HermiT \citep{Shearer2008}. Finally, the preferences $\mathit{pref}$ exploited during the second step of phase P3 were set in a way that $Q'_{+} := Q_{\mathsf{exp}}$, i.e.\ all the (simple) sentences $Q_{\mathsf{exp}}$ output by the reasoner were considered cost-preferred (cf.\ Ex.~\ref{ex:MBD_overall}).  

\paragraph{Experimental Settings (EXP2 -- Scalability Tests).} In these experiments we used $n = 500$ as a very test of the new approaches' scalability. Since there are fewer than 500 minimal diagnoses for the DPIs U, M, O and C (see last column of Tab.~\ref{tab:experiment_data_set}), the dataset for these experiments consisted of the DPIs CC, CE, T and E. For each of these DPIs, we performed one run with randomly generated leading diagnoses $\mD$, as described above. All other settings were equal to those in EXP1 explained above.

\paragraph{Experimental Settings (EXP3 -- Comparison with a Method not Using the Proposed Theory).}
To quantify the impact of the new theoretical notions (CQs, CQPs, traits) exploited by Alg.~\ref{algo:query_comp} 
in order to 
drastically reduce reasoning activity
during query computation, we compared Alg.~\ref{algo:query_comp} with a method that is as generally applicable (in terms of \emph{logics and reasoner independence}, handling of \emph{implicit} query spaces, cf.\ Sec.~\ref{sec:intro}), but does not use these notions. A generic such algorithm is described, e.g., in \citep[Alg.~2]{Shchekotykhin2012}. To give this algorithm a name and to clearly distinguish it from the newly proposed method (NEW), we call it OLD in Sec.~\ref{sec:experimental_results}. 

In a nutshell, this algorithm, starting from the set of leading diagnoses $\mD$, enumerates all 
subsets $\dx{}$ of $\mD$ in the form of a recursive binary tree. At each leaf node, corresponding to one $\dx{} \subset \mD$, a query is created. The latter is accomplished by calling a reasoner to compute a set of common entailments $X$ of all diagnoses in $\dx{}$ (cf.\ Prop.~\ref{prop:properties_of_q-partitions}.\ref{prop:properties_of_q-partitions:enum:query_is_set_of_common_ent}). In case this set $X$ is non-empty, the diagnoses in $\mD\setminus \dx{}$ are assigned to their respective set (i.e.\ $\dx{}(X)$, $\dnx{}(X)$ or $\dz{}(X)$) according to Prop.~\ref{prop:partition} by means of a reasoner. We call the tests whether $X \neq \emptyset$ (query is non-empty) and whether $\dx{}(X) \neq \emptyset,\dnx{}(X) \neq \emptyset$ ($X$'s partition is a QP) \emph{query verification} (cf.\ Def.~\ref{def:query_q-partition}).
As the recursion unwinds, at each inner node of the tree, 
the better one among the best query found in the left and the best query found in the right subtree is returned, where query goodness is measured as per some QSM. The final query returned at the root node is minimized using \textsc{QuickXplain} (cf.\ Sec.~\ref{sec:P3}). 

As suggested by preliminary tests using OLD, which could not handle any more than 20 leading diagnoses for any of the DPIs in Tab.~\ref{tab:experiment_data_set}, we used $n \in \setof{5,10,15,20}$ in the comparison experiments between NEW and OLD. We ran 5 query computation iterations for both NEW and OLD per (DPI,$|\mD|$) combination (cf.\ EXP1). In each iteration, both NEW and OLD were applied to exactly the same leading diagnoses sets $\mD$ and diagnoses probabilities. Further, the reasoner (HermiT) and the computed entailment types $ET$ used by both methods were the same. For NEW, all other settings remained the same as in EXP1 (see above). Only phase P1 was modified in a way a brute force search over all QPs was performed (no early termination by means of any threshold, no heuristics, no pruning) to achieve best comparability with OLD as regards performance and search completeness. 

\paragraph{Experiment Conditions.} All tests were run on a Core i7 with 3.4 GHz, 16 GB RAM
and Windows 7 64-bit OS.

\subsection{Experimental Results}
\label{sec:experimental_results} 
We subdivide the presentation of the experimental results\footnote{The source code implementing the experiments as well as the obtained results can be accessed on http://isbi.aau.at/ontodebug/evaluation.} into discussions of various observed aspects of the algorithms, e.g.\ times, reasoner calls or search space sizes. Each such set of related aspects is illustrated by a distinct figure (named in the heading of the respective paragraph). Whenever we will refer to a figure, we will mean exactly the figure mentioned in the heading. If some figure includes a secondary y-axis, which means a y-axis on the right side, then all aspects plotted with respect to the secondary axis are given in \emph{italic font} (whereas those plotted based on the primary, i.e.\ left, y-axis are written in normal font) in the figure's key. On the x-axis, all the plots show the 8 (or fewer) different categories (M, U, T, E, C, O, CE, CC), one for each $\mo$ in Tab.~\ref{tab:experiment_data_set}. Every plotted point shows the respective aspect, as indicated by the figure's key, in terms of a 5-iteration average value. That is, for each plotted point, $\dpi$ (i.e.\ the DPI for $\mo$) and $n$ is fixed, whereas $\mD$ varies over the 5 iterations (see the description of EXP1 in Sec.~\ref{sec:experiments}). Note that the range of $n$ is smaller for the categories M and C because there are no 80 minimal diagnoses for these two $\mo$'s (48 for M, 15 for C, cf.\ Tab.~\ref{tab:experiment_data_set}). Further, for clarity and better visibility the plots only show the values for the (more costly to compute) QSM ENT. The values observed for the QSM SPL were always comparable or better than for ENT. The unit of times is seconds in all figures. Figures \ref{fig:diag_vs_query_computation} -- \ref{fig:new_algo_vs_brute_force} address EXP1, Fig.~\ref{fig:scalability} EXP2 and Fig.~\ref{fig:naive_vs_new} EXP3.

\begin{figure}[t]
	\centering
	\includegraphics[width=\textwidth]{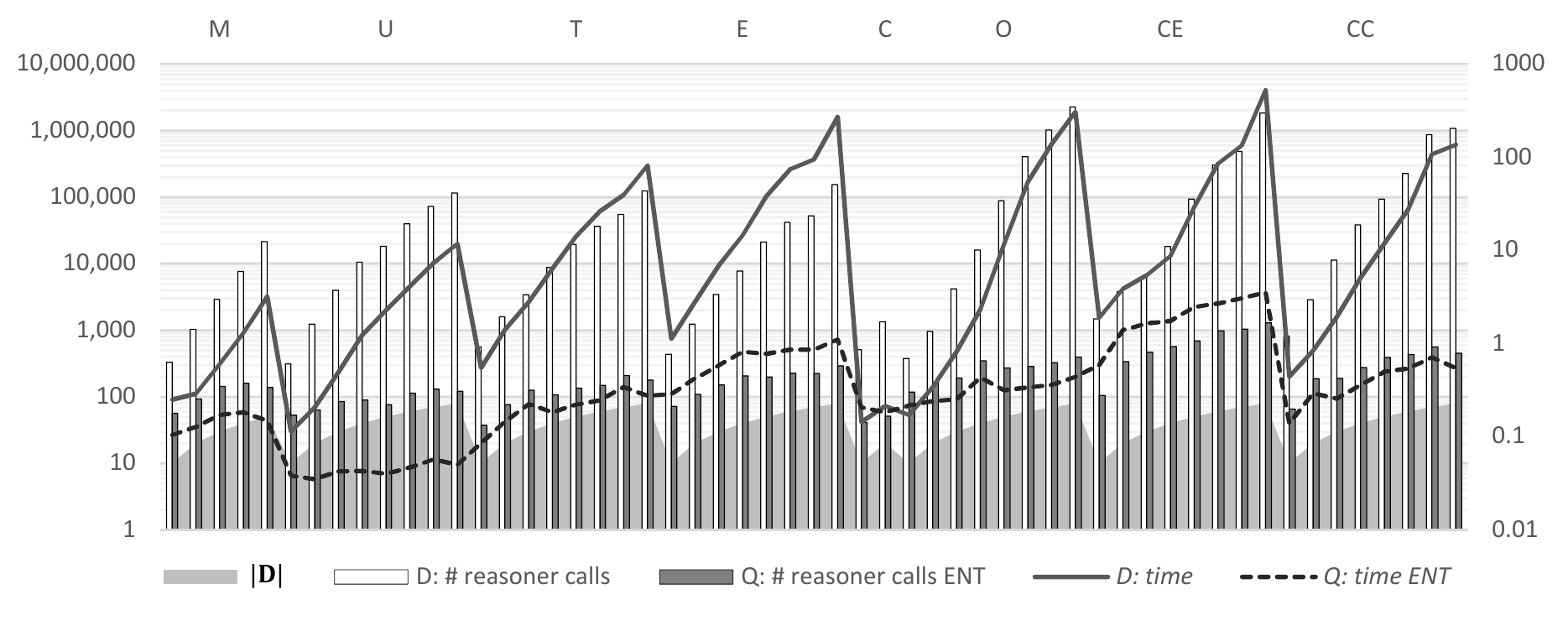}
	\caption{\small Diagnoses computation vs.\ query computation. D means diagnosis computation and Q means query computation. The QSM $m = \text{ENT}$ was used with the threshold $t_m = 0.01$.}
	\label{fig:diag_vs_query_computation}
\end{figure}

\paragraph{Diagnoses vs. Query Computation. (Fig.~\ref{fig:diag_vs_query_computation})}
For both diagnoses and queries, the decisive factor influencing the computation time is the number of required inference engine calls. 
This connection can be clearly observed in the figure where the light and dark bars show the number of the reasoner calls for diagnoses and query computation, respectively, and the continuous and dashed lines display the respective computation times. 
Note that the shown query computation time (dashed line) is the sum of the times for \emph{all} phases P1, P2 and P3, i.e.\ constitutes an upper bound of both default (i.e.\ phases P1+P2) and optional (i.e.\ phases P1+P3) mode of Alg.~\ref{algo:query_comp}. 

Algorithms 
that incorporate reasoners more strongly into query computation 
often have to limit the number of leading diagnoses to rather small numbers, e.g.\ $9$ \citep{Shchekotykhin2012}. This is necessary to keep query calculation practical because the worst case size of the QP search space is $2^{|\mD|}$, not to mention the size of the query search space which is generally again a multiple of it. For the new method, as the figure reveals, the growth of query computation time is very moderate
for increasing numbers of leading diagnoses. In fact, we can state it is \emph{at most linear}, as the growth of the dashed line is at most parallel to the growth of the shaded area (note the logarithmic y-axes).
Sometimes the time even sinks after raising $|\mD|$, e.g.\ for the cases $|\mD| \in \setof{70,80}$ for U, T and CC. In spite of its slight tendency to increase, the query computation time, by absolute numbers, is always below 3.6 sec and, except for the cases involving CE with $|\mD|\geq 30$, always lower than 1 sec. That is: 
	Even for high numbers of up to 80 leading diagnoses (QP search space size in $O(2^{80})$), optimized queries (as per Theorems~\ref{theorem:P1+P2_solve_query_optimization_problem} and \ref{theorem:P3_solves_problem_1}) are computed within almost negligible time.

This is due to the main merit of the new algorithm, which is the avoidance (in default mode)
or minimization (in optional mode)
of reasoner calls. 
Essentially, the slight tendency of the new algorithm's computation time to increase for larger diagnoses sets can be primarily attributed to increased costs of step 2 in phase P3, and secondarily to the substantially larger QP search space explored by phase P1 (see also Figures~\ref{fig:time_comparison_P1_P2_P3}, \ref{fig:summary_P3} and \ref{fig:new_algo_vs_brute_force} for an illustration of this fact). On the other hand, the reasoning costs of step 1 in phase P3 tend to fall as a response to increasing $|\mD|$ since the sizes of the arguments to the two $Ent_{\mathit{ET}}$ calls in Eq.~\eqref{eq:Q_exp} tend to decrease given a higher $|\mD|$. Second, despite an increased effort faced by phase P2 for growing $|\mD|$, the absolute times required by P2 are so small (between about $10^{-5}$ and $10^{-3}$ sec) that they hardly carry weight (see also Fig.~\ref{fig:new_algo_vs_brute_force}). 

The extension of the QP search space has not such a high impact due to the used heuristic functions that proved to guide the algorithm rather quickly towards a goal QP and due to the used tree pruning which avoided the exploration of hopeless subtrees. The higher costs of step 2 in phase P3 can be explained as follows: Whereas the number of \textsc{isQPartConst} calls is dictated by $|Q'|$ which does not (directly) depend on $|\mD|$, the number of reasoner calls within each call of \textsc{isQPartConst} does depend (linearly) on $|\mD|$ (see Prop.~\ref{prop:P3_step2_complexity}).

We further point out that the time axis (right y-axis) is logarithmic, i.e.\ an increase of 1 on the axis means an actual increase of one order of magnitude. One conclusion we can draw from this is that whenever diagnoses computation requires non-negligible time, let us say more than 10 sec, then query computation is always at least one order of magnitude and up to more than two orders of magnitude, i.e.\ a factor of 100 (case O, 80 diagnoses), faster than diagnoses computation. 
Note that the diagnoses computation time grows exponentially with $|\mD|$, i.e.\ our data shows quite constant time growing factors averaging to approximately 2 (visible by the mostly constant slope of the gray line in the figure) for all eight DPIs in Tab.~\ref{tab:experiment_data_set}. Hence, computing $10$ diagnoses more implies about the double computation time. 
Therefore, with the new method: 
	Whenever the computation of a set of diagnoses is feasible, the generation of an optimized query regarding the computed diagnoses is feasible and often significantly more efficient than the computation of diagnoses. Optimized query computation is thence a minor problem as compared to diagnoses computation.

\begin{figure}[t]
	\centering
	\includegraphics[width=\textwidth]{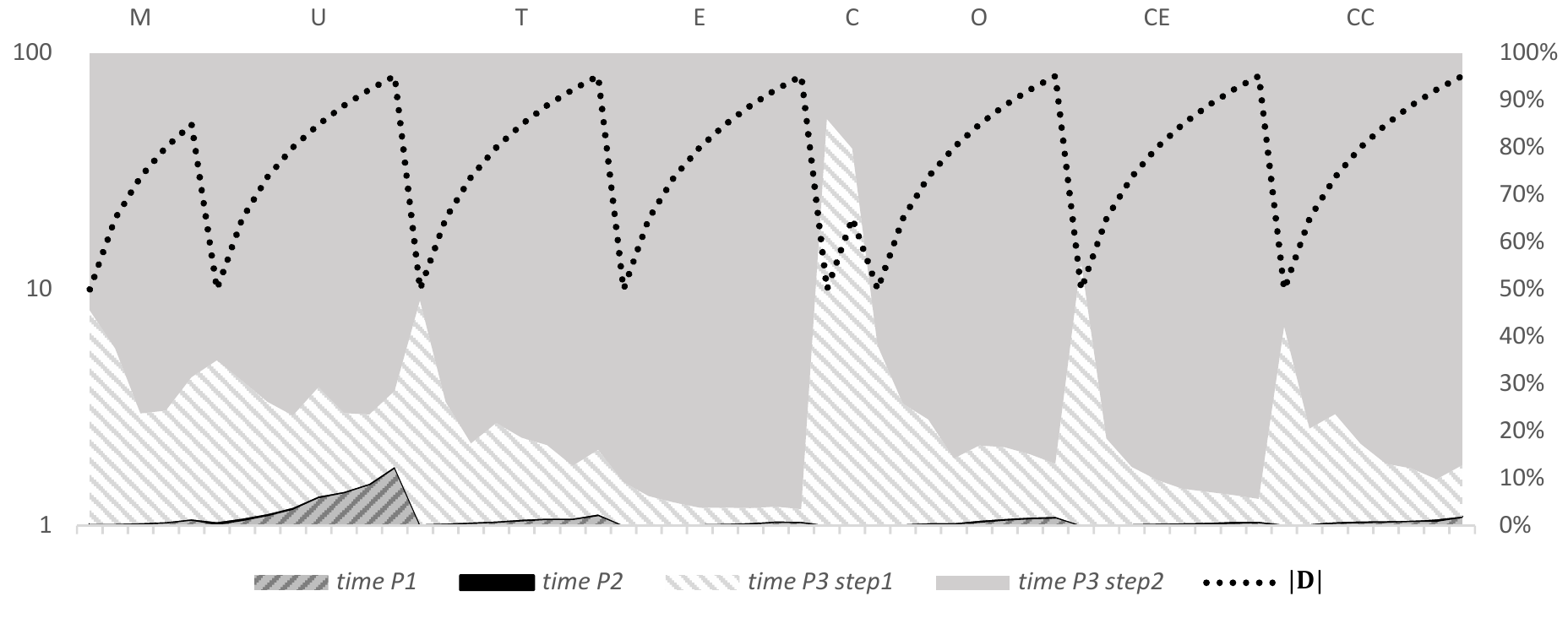}
	\caption{\small Comparison of times for P1, P2 and P3. All times were measured for the QSM $m = \text{ENT}$ with threshold $t_m = 0.01$.}
	\label{fig:time_comparison_P1_P2_P3}
\end{figure}

\paragraph{Comparison of Times for Phases P1, P2 and P3. (Fig.~\ref{fig:time_comparison_P1_P2_P3})}
The figure depicts the relative proportion of the overall query computation time consumed by the different phases of Alg.~\ref{algo:query_comp}. It is evident that phase P3 accounts for more than $\frac{7}{8}$th of the computation time in all test runs. If we exclude the case U -- for which the algorithm's computation time was the lowest amongst all DPIs in Tab.~\ref{tab:experiment_data_set}, i.e.\ below 0.1 sec for all runs, cf.\ Fig.~\ref{fig:diag_vs_query_computation} -- then P3 is even responsible for more than $97\%$ of the computation time in all runs. This reminds us again of the fact that reasoning (which is only performed in P3) has a substantially higher impact on the efficiency of query computation than the combinatorial problems solved in P1 and P2.
 
This suggests a variant of Alg.~\ref{algo:query_comp} which always runs the very fast P1+P2 first and shows the result to the user. Meanwhile in the background, or alternatively on demand, the algorithm executes P3 to further optimize the already computed query. In this manner the user can always get a first query suggestion \emph{instantaneously}.   

Moreover, we recognize that P2 (see the thin black area between the darker and lighter shaded areas in the figure), although it solves an NP-hard problem in general, makes up a negligible fraction of the method's computational load due to its fixed parameter tractability (cf.\ Prop.~\ref{prop:complexity_P2}). It is by far the fastest phase of the algorithm. Thus, even for large numbers of leading diagnoses, the solved hitting set problem remains easy. 

What we also point out is that the query expansion (P3, step 1) is sometimes (for C and CE) the most influencing factor regarding the computation time for small $|\mD|$ and successively loses importance against the query contraction (P3, step 2) as $|\mD|$ is increased. Reasons for this were discussed above.

\begin{figure}[t]
	\centering
	\includegraphics[width=\textwidth]{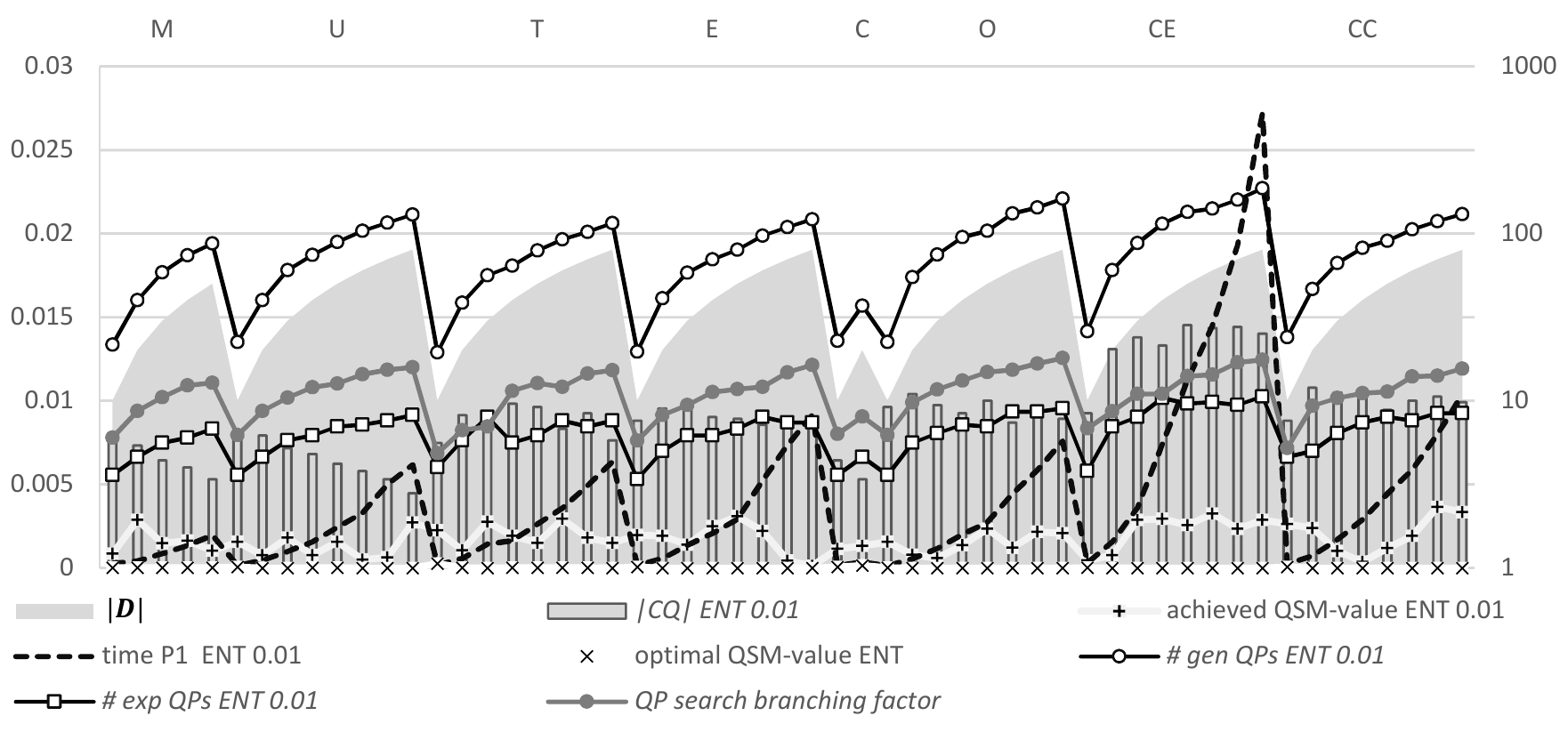}
	\caption{\small Summary of P1. ENT 0.01 means that the QSM $m = \text{ENT}$ was used with the threshold $t_m = 0.01$. $|\text{CQ}|$ denotes the size of the canonical query, \# gen QPs and \# exp QPs means the number of generated and expanded q-partitions, respectively. The branching factor is the average number of successors of nodes in the search tree.}
	\label{fig:summary_P1}
\end{figure}

\paragraph{Summary of Phase P1. (Fig.~\ref{fig:summary_P1})} 
By considering the generated and expanded QPs and the branching factor we get an impression of how the search tree looks like in P1. First, it is apparent that the number $g$ of generated QPs is approximately proportional to the number of leading diagnoses $|\mD|$, i.e.\ $g \approx c |\mD|$, where the factor $c$ averages to $\langle 1.94,1.86,1.67,1.75,2.07,2.27,2.63,\\ 1.99 \rangle$ for $\tuple{\text{M},\text{U},\text{T},\text{E},\text{C},\text{O},\text{CE},\text{CC}}$. Hence, we can state that, on average, for a very small threshold $t_m$ of $0.01$ (and 0), an optimal QP w.r.t.\ ENT (and SPL) can be found by generating no more than $3 |\mD|$ QPs. As a consequence, the effort arising in P1 -- notabene with heuristic and pruning -- grows \emph{linearly} with the number of leading diagnoses. By absolute numbers, $g$ was always below 200 (with a maximum of 187 for the case CE with $|\mD|=80$). 

Second, we notice that the branching factor as well as the number of expanded QPs are approximately proportional, but grow sublinearly with regard to $|\mD|$. 
For instance, for 10, 40 and 80 leading diagnoses, the branching factor and number of expanded QPs amounted on average (over all eight DPIs) to 6, 12 and 16 as well as 3.8, 7.0 and 8.5, respectively.   
That is, somewhat surprisingly, the branching factor is a rough (upper bound) estimate for the number of explored QPs until a goal is found. Moreover, continuously increasing the number of diagnoses, always by the same constant, leads to increases in the number of expanded QPs and in branching factor by continuously smaller factors. One reason for this is the tendency of diagnoses (and thus of their subset-minimal traits) to overlap more frequently if more diagnoses are computed. This overlap means that there are fewer equivalence classes as per Cor.~\ref{cor:S_next_sound+complete}, and thence affects the branching factor negatively.

Concerning the time required for P1 (black dashed line), we see that the maximum time over all cases was below $0.03$ sec and, excluding the DPI CE, below 0.01 sec. Therefore, an optimal QP (w.r.t.\ the threshold 0.01) can always be computed in less than $\frac{1}{20}$th of a second. 

Let us now draw our attention to the quality of the computed QP and imagine a thought horizontal line at 0.01 (left y-axis) denoting the specified threshold. It is easy to verify that the QSM-value of the computed QP (line labeled with the + signs) is always below this line, i.e.\ a QP with at least the required quality was determined in all cases. This analysis additionally shows that, although the threshold is at 0.01, the actually achieved QSM-value is quite close to the optimal QSM-value w.r.t.\ ENT, which is very close to zero (line labeled with the x signs). Note, w.r.t.\ SPL the optimal QSM-value of 0 was always hit. 
The optimal QSM-value was ascertained by performing a brute-force search over all QPs (cf.\ Fig.~\ref{fig:new_algo_vs_brute_force}) and storing the best found QSM-value.

Finally, the size of the canonical query, which constitutes an upper bound of the size of a query constructible in phase P2, attains values between $2.8$ (U, 80) and $28.4$ (CE, 50). The size of the CQ depends on the overlapping of the diagnoses in the $\dx{}$ set with those in the $\dnx{}$ set of the respective (C)QP. The higher it is, the lower the cardinality of the CQ (cf.\ Lem.~\ref{lem:CQ_equal_to_U_D_setminus_U_D+}).


\begin{figure}[t]
	\centering
	\includegraphics[width=\textwidth]{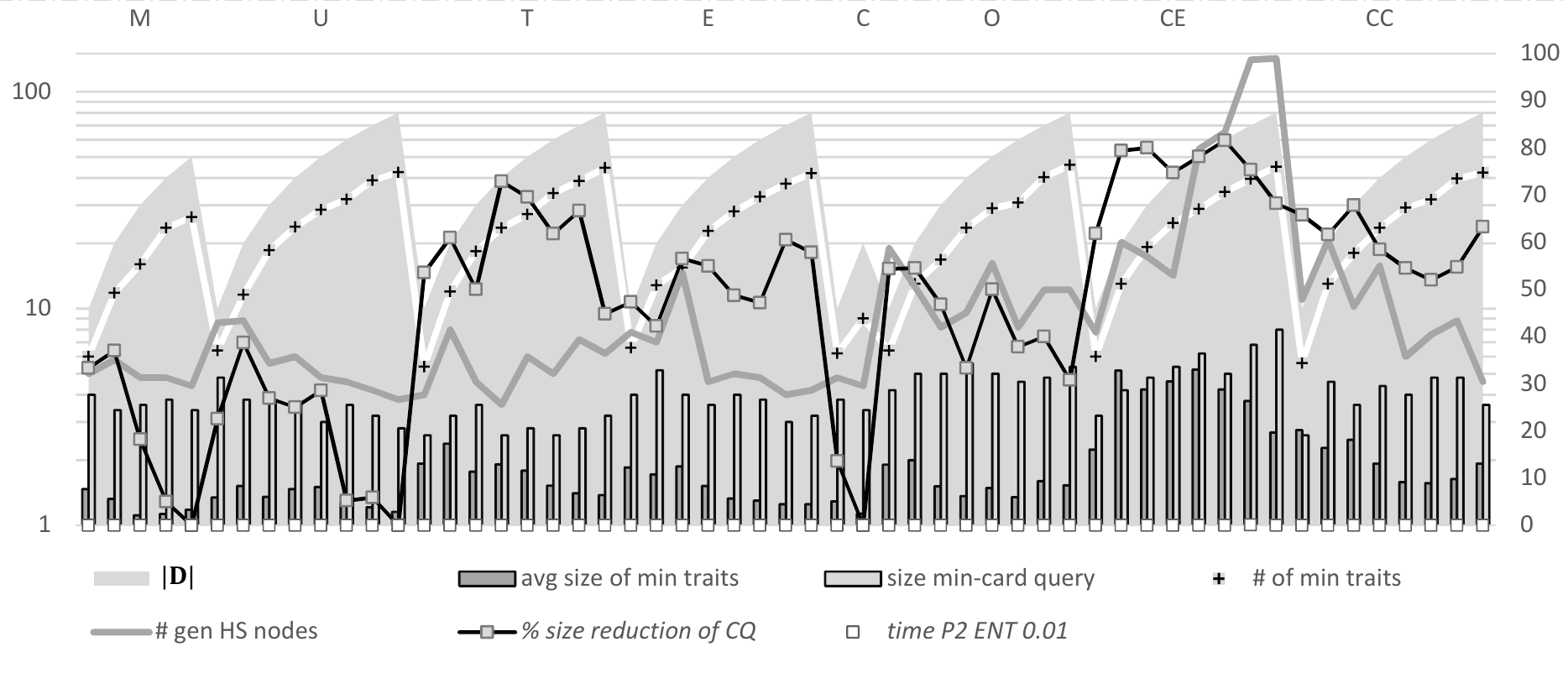}
	\caption{\small Summary of P2. Min traits means $\subseteq$-minimal traits w.r.t.\ the (fixed) QP returned by phase P1. Min-card means minimum-cardinality. Gen HS nodes refers to the generated nodes in the constructed hitting set tree. The size reduction of the CQ is computed as $(1-\frac{|Q^*|}{|Q|}) * 100 \%$ where $Q$ is the CQ and $Q^*$ the query output by P2.}
	\label{fig:summary_P2}
\end{figure}

\paragraph{Summary of Phase P2. (Fig.~\ref{fig:summary_P2})} 
In P2, the query with optimal QCM $c_{|.|}$ (see page~\pageref{etc:QCM}) is computed by performing a uniform cost hitting set search over the collection of all $\subseteq$-minimal traits of the optimal QP found in P1. The number of generated nodes measures the necessary effort for the hitting set tree construction and depends on the number of $\subseteq$-minimal traits (number of sets to be hit), their cardinality (branching factor of the hitting set tree) and their overlapping (the higher it is, the lower the depth of the tree and the minimum cardinality query tend to be). 
For $|\mD| \in \tuple{10, 40, 80}$, the average and maximal numbers of generated nodes are $\tuple{8.5, 8.5,28.9}$ and $\tuple{19, 16, 143}$, respectively. That is, the size of the generated tree is easily manageable, even for large sets of leading diagnoses. This fact is confirmed by the negligible time (in all runs between $\frac{1}{100\,000}$ and $\frac{6}{1\,000}$ sec) consumed by P2 (see the white squares in the figure).

The average size (where the average is taken over the traits of the optimal QP returned by phase P1) of the $\subseteq$-minimal traits is very small with an average / maximum of $1.59$ / $2.75$ over all cases, except for CE. For CE, we measure an average / maximum of $3.88$ / $5.24$. Hence, the branching factor of the hitting set tree is very low and the number of generated nodes is significantly higher for CE than for the other tested DPIs.
 
An explanation for the tendency of $\subseteq$-minimal traits to shrink for higher $|\mD|$ (which can be best observed for the cases T, E and CE, see the figure) is the tendency of diagnoses to more frequently overlap, if more diagnoses are computed (cf.\ Def.~\ref{def:trait}). The number of $\subseteq$-minimal traits, on the other hand, is proportional to $|\mD|$, which is quite intuitive as the number of diagnoses in $\dnx{}$ (i.e.\ the maximal possible number of $\subseteq$-minimal traits) tends to grow with increasing $|\mD|$, of course depending on (the QP properties favored by) the used QSM.

The median of the size of the query with optimal QCM computed by P2 is $3.8$ sentences (see the light gray bars). The achieved size reduction, starting from the CQ of the optimal QP returned by P1 and given as input to P2, ranges from zero percent (cases M, 5 and U, 8 and C, 15), where the CQ coincides with the QCM-optimal query, to more than $80\%$ (case CE, 60). In the latter case, CQs of average sizes of $27$ are reduced to an average size of $5$.

\begin{figure}[t]
	\centering
	\includegraphics[width=\textwidth]{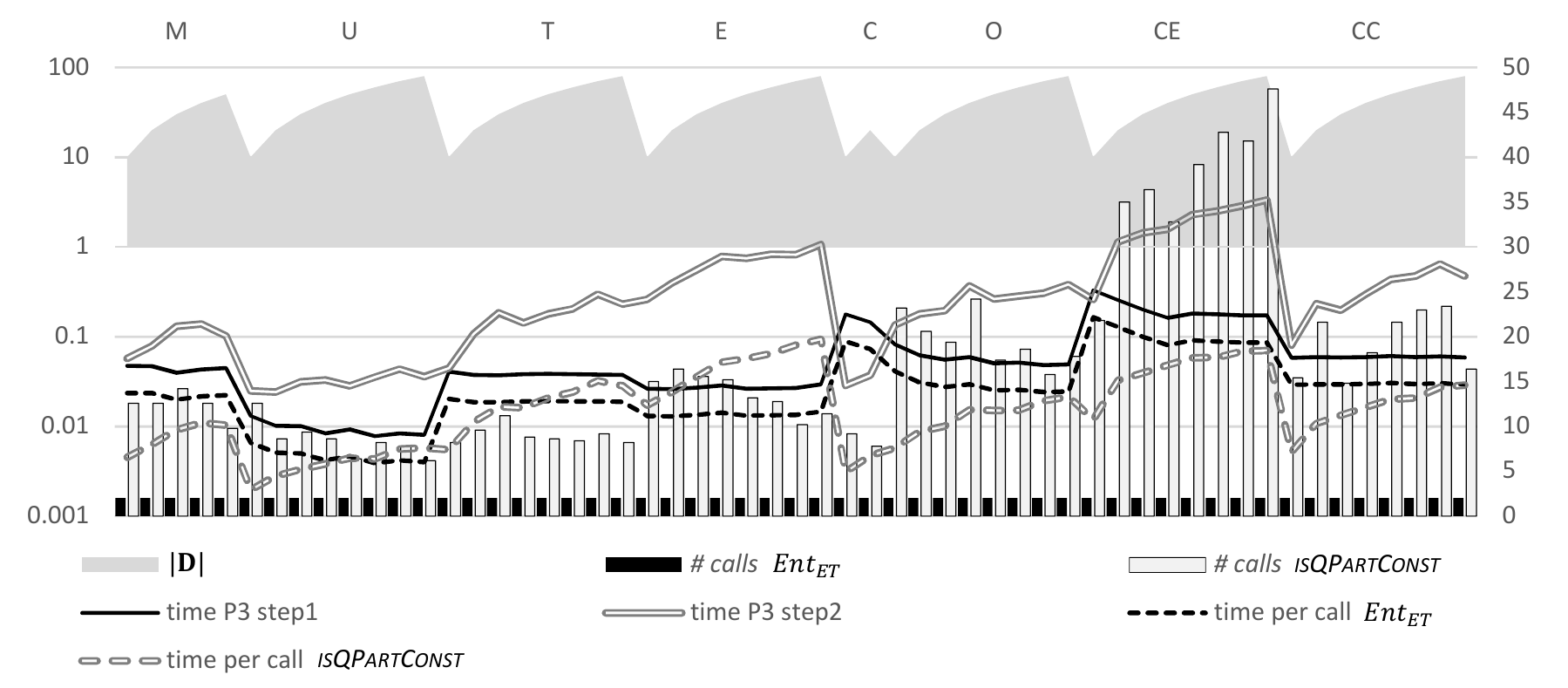}
	\caption{\small Summary of P3.}
	\label{fig:summary_P3}
\end{figure}

\paragraph{Summary of Phase P3. (Fig.~\ref{fig:summary_P3})} 
As the complexity analysis in Sec.~\ref{sec:P3} suggests, the crucial factors determining the efficiency of P3 are the number and the complexity of the required reasoner calls. For the first step of P3, these are the calls to $Ent_{ET}$ (cf.\ Eq.~\eqref{eq:Q_exp}). The figure (black bars) reminds us of the fact that their number is constant, i.e.\ $2$, independent from other parameters. Consequently, only the complexity of the $Ent_{ET}$ calls has an effect on the hardness of P3, step 1. As becomes clearly evident in the figure, this complexity is ruled by (the complexity, expressivity and number of implicit entailments of) the KB $\mo$ of the respective DPI, i.e.\ the black dashed line is more or less constant for each DPI. However, it tends to slightly decrease upon increasing $|\mD|$. This is exactly what one would expect (cf.\ the discussion of Fig.~\ref{fig:diag_vs_query_computation} above). Note, the time consumed by P3, step 1 (continuous black line) is exactly proportional to the time needed for an $Ent_{ET}$ call, which confirms that there are no other significant factors influencing the complexity of this computation step.
By absolute numbers, the time per $Ent_{ET}$ call never exceeded $0.2$ sec.

As regards the second step of P3, the number and complexity of the \textsc{isQPartConst} calls is decisive. The former is again influenced by the KB $\mo$ because its complexity and expressivity affects the number of computed entailments in P3, step 1. These in turn have an impact on the size of the expanded query, $|Q'|$, which rules the number of \textsc{isQPartConst} calls (cf.\ Prop.~\ref{prop:P3_step2_complexity}). 
In comparison to other DPIs, CE requires a relatively high number of \textsc{isQPartConst} calls (up to roundly $50$) on account of the large size of the computed query expansion in P3, step 1 (see Fig.~\ref{fig:query_evolution}). That is, the reasoner $\mathit{Inf}$ returned substantially more implicit entailments for CE than for other DPIs.  
%
The complexity of an average \textsc{isQPartConst} call is on the one hand determined by $|\mD|$ (as discussed above), thus slightly increasing for each DPI (see the figure), and on the other hand by the reasoning complexity of the respective KB $\mo$. 
For example, in case of O, although the average number of \textsc{isQPartConst} calls is clearly larger than for E, the latter requires more time one average for P3, step 2 due to the higher complexity per call (dashed transparent line).
Over all runs, no call of \textsc{isQPartConst} took longer than $0.01$ sec and the time for P3, step 2 was always below $3.5$ sec.

\begin{figure}[t]
	\centering
	\includegraphics[width=\textwidth]{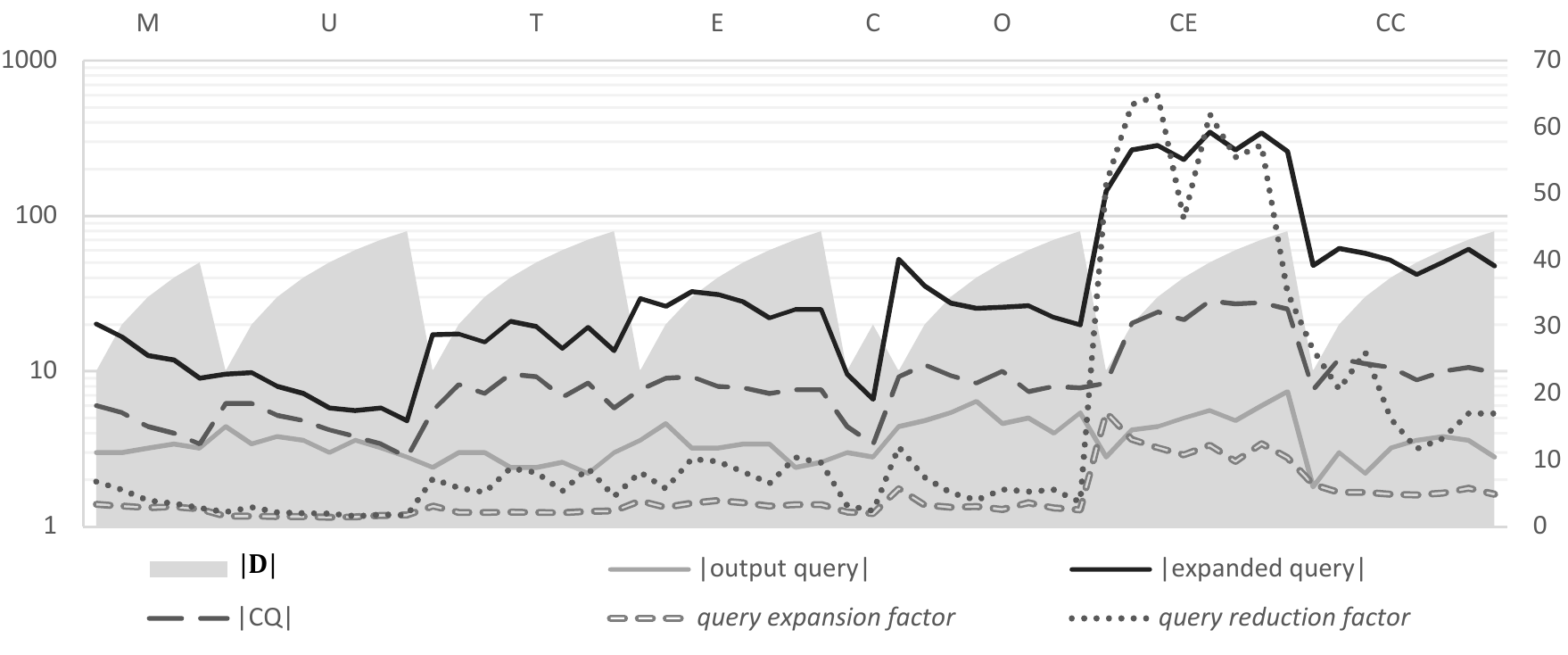}
	\caption{\small Query evolution over phases P1 and P3. Expanded query / output query refers to the query returned by P3 step 1 / P3 step 2. A query expansion factor of $k$ means that the expanded query is $k$ times as large in size as the CQ. A query reduction factor of $k$ means that the expanded query is $k$ times as large in size as the output query.
}
	\label{fig:query_evolution}
\end{figure}

\paragraph{Query Evolution. (Fig.~\ref{fig:query_evolution})} 
In this figure we see the comparison of the intermediate results in terms of the query size throughout phases P1 and P3 (optional mode of Alg.~\ref{algo:query_comp}). First, phase P1 returns a QP (from which the CQ $Q$ can be immediately computed, see Lem.~\ref{lem:CQ_equal_to_U_D_setminus_U_D+}). Then the CQ is enriched in phase P3, step 1 resulting in the expanded query $Q'$. This query is finally contracted again yielding the output query $Q^*$. 

We see that $Q$ is always larger than $Q^*$, i.e.\ altogether the enlargement and later reduction of the CQ $Q$   
produces a query smaller than $Q$. Note, $|Q|$ is a theoretical lower bound of $|Q'|$ (cf.\ Eq.~\ref{eq:Q'_enriched_query}) and hence always lower than $|Q'|$. As we already discussed above, the size of $Q'$ in relation to the size of $Q$ depends very much on the expressivity and (logical) complexity of the KB. Therefore, $|Q'|$ is larger for, e.g., CC than for, e.g., O, even though the size of $Q$ is approximately equal in both cases. In figures, $|Q|$ for O and CC averages to $8.9$ and $10.1$, whereas $|Q'|$ for O and CC amounts to $29.4$ and $52.5$. The most implicit entailments could be computed in case of CE, with average sizes of the expanded query $Q'$ of $268$. These differences in the number of entailments can be best seen by considering the query expansion factor (dashed transparent line) which ranges from $9.8$ to $17.2$ for, e.g., CE and from only $1.4$ to $1.7$ for, e.g., U.

The query reduction factor (dotted line), on the other hand, measures the degree of contraction effectuated by P3, step 2. A reduction factor of $k$ means that $|Q'| = k |Q^*|$, i.e.\ the size of the contracted and optimized query $Q^*$ is $\frac{1}{k}$th of the expanded one, $Q'$. The maximal values of $k$ are around $65$ for, e.g., CE and around $3$ for, e.g., U. That is, for CE CQs of average size $268$ are reduced to optimized queries of average size $5$ while for U CQs averaging to $7.0$ are minimized to queries averaging to $3.5$. 
Nevertheless, the size of the finally output query $Q^*$ does not fluctuate very strongly (gray continuous line) and has a median of $3.4$.

\begin{figure}[t]
	\centering
	\includegraphics[width=\textwidth]{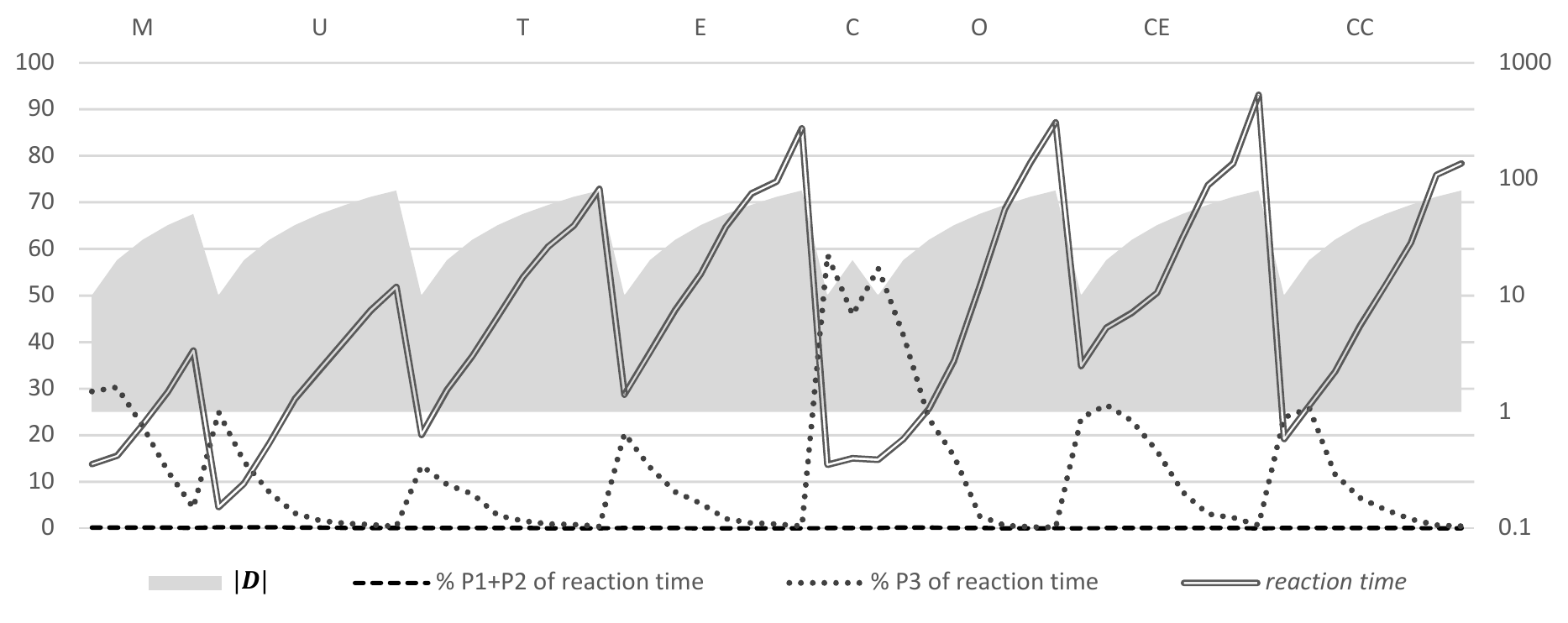}
	\caption{\small Query computation vs.\ debugger reaction time. Reaction time refers to the time passing between the submission of a query answer and the finalization of the next query's computation.}
	\label{fig:query_compuation_vs_debugger_reaction_time}
\end{figure}

\paragraph{Query Computation vs.\ Debugger Reaction Time. (Fig.~\ref{fig:query_compuation_vs_debugger_reaction_time})} 
On the one hand the figure shows the absolute reaction time (transparent line) of the debugger, i.e.\ the time passing between the submission of a query answer and the provision of the next query. In other words, the reaction time is the time required for leading diagnoses computation plus the time for query generation. On the other hand the figure gives insight into which proportion of the reaction time is due to query computation, where phases P1+P2 (dashed line) and P3 (dotted line) of the query computation are shown separately, and which proportion is due to diagnoses computation (difference between 100 on the left y-axis and the dotted line). The debugger's reaction time ranges from $0.15$ sec (U, 10) to $8$ min $50$ sec (CE, 80). Over all eight DPIs, the average reaction times for $|\mD| \in \tuple{10,20,30,40,50,60,70,80}$ are $\tuple{0.8,2.0,3.2,6.2,16.5,46.3,87.9,223.6}$. It is apparent from the figure that the reaction time grows superlinearly with increasing $|\mD|$. For all DPIs separately, the average factor by which the reaction time grows upon adding ten leading diagnoses is between $1.65$ and $2.56$. The average growing factor over the entire data is roundly $2$. That is, the reaction time is about doubly as high, if the number of leading diagnoses is raised by ten.

However, using the presented algorithm, the time spent for query computation accounts for only a minor fraction of the reaction time. In particular, whenever the reaction time is not very quick, i.e.\ it is, say, beyond 10 sec, the query computation is always responsible for less than $10$ percent of the reaction time when Alg.~\ref{algo:query_comp} is used in optional mode with query expansion and optimization, and for less than $3$ per mill of the reaction time when it is used in default mode. 
Hence,
with the new method, whenever the debugger fails to react within short time, this is due to diagnoses computation and not due to query computation. Moreover, the fraction of the reaction time needed for computing an explicit-entailments query optimizing both the QSM and the QCM 
is always negligible, independent of other parameters.

\begin{figure}[t]
	\centering
	\includegraphics[width=0.95\textwidth]{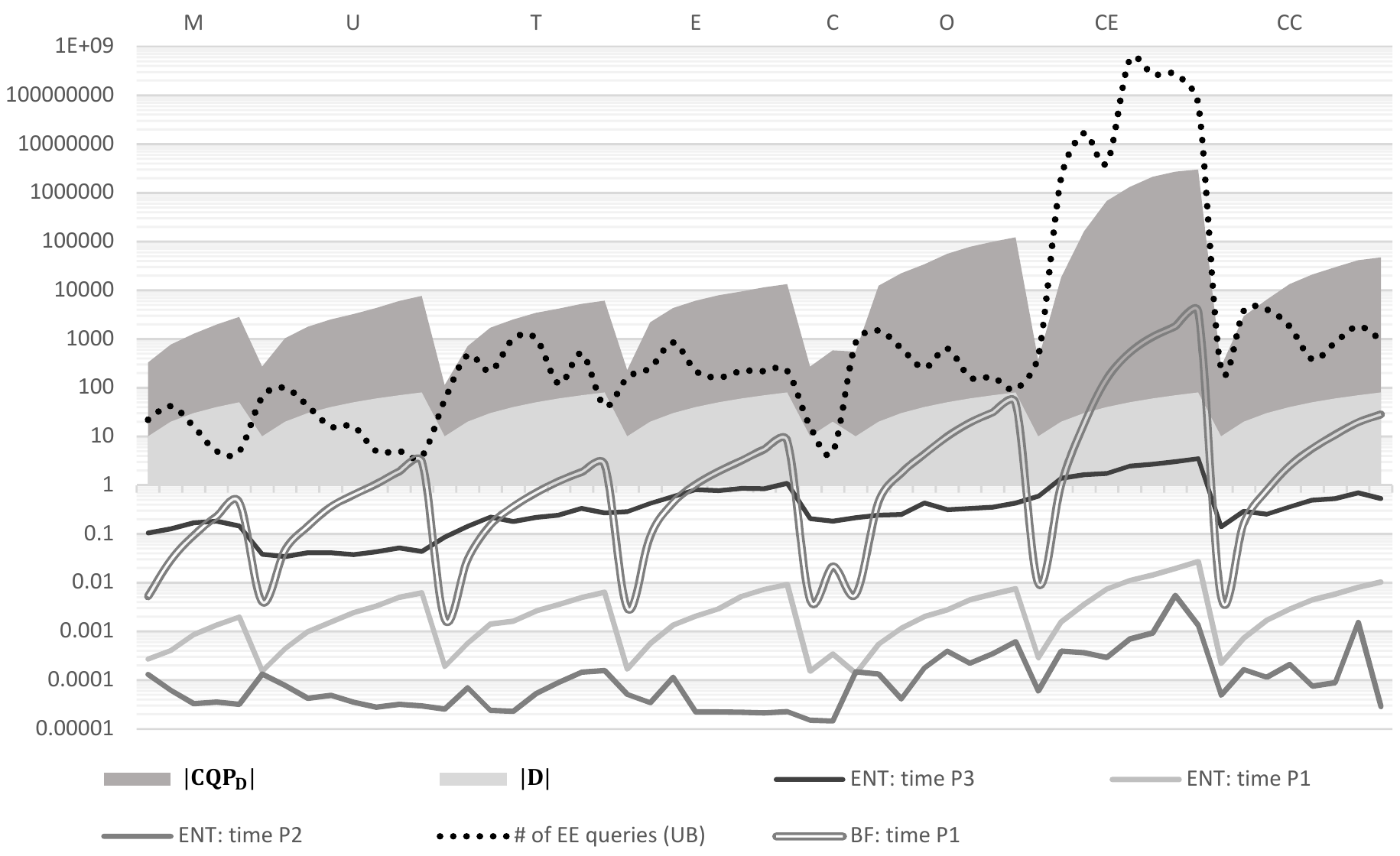}
	\caption{\small Search space sizes versus query computation times. $\mathbf{CQP}_\mD$ denotes the set of CQPs w.r.t.\ the leading diagnoses $\mD$. EE queries refers to explicit-entailments queries (cf.\ page~\pageref{etc:EE_query_def}), UB signalizes an upper bound. BF terms a brute force search over $\mathbf{CQP}_\mD$.}
	\label{fig:new_algo_vs_brute_force}
\end{figure}

\paragraph{Search Space Size vs.\ Computation Times. (Fig.~\ref{fig:new_algo_vs_brute_force})} 
%
Here, we see a comparison of the absolute computation times of the three phases of the new method (solid lines). Furthermore, we performed a brute force search over all CQPs, on the one hand to determine the size of the search space explored in P1, i.e.\ the (exact) number $|\mathbf{CQP}_\mD|$ of all existing CQPs for $|\mD|$ (see the dark shaded area in the figure, cf.\ Conjecture~\ref{conj:CQPs=QPs}), and on the other hand to get an idea of the efficiency of (C)QP generation with the new algorithm. The time required for the exploration of all CQPs is shown by the framed transparent line in the figure. Additionally, the figure displays an upper bound of the size of the search space tackled by P2 (dotted line), i.e.\ of the number of all explicit-entailments (EE) queries for the QP $\Pt = \tuple{\dx{},\dnx{},\emptyset}$ where $\Pt$ is the result returned by phase P1. 
We calculated this upper bound as $u := 2^{n}-\sum_{k=1}^{m} \binom{n}{k}$ where $n := |Q_{\mathsf{can}}(\dx{})|$ and $m := |Q^*|-1$, i.e.\ $2^n$ is the number of all subsets of the CQ $Q_{\mathsf{can}}(\dx{})$ of $\Pt$ and the subtracted sum counts all subsets of $Q_{\mathsf{can}}(\dx{})$ of size smaller than the minimum-cardinality query $Q^*$ computed by phase P2. Recall, each explicit-entailments query is a subset of $Q_{\mathsf{can}}(\dx{})$ and a superset of some minimal hitting set of all ($\subseteq$-minimal) traits in $\dnx{}$ (by Prop.~\ref{prop:explicit-ents_query_lower+upper_bound} and Cor.~\ref{cor:min_exp-ents_queries_are_minHS_of_all_traits_of_diags_in_Dnx}), and $Q^*$ is a hitting set of of all ($\subseteq$-minimal) traits in $\dnx{}$ of minimum-cardinality. Hence, $u$ is indeed an upper bound of the number of all explicit-entailments queries for $\Pt$.\footnote{Unfortunately, we cannot make any statement about the strictness of this bound, nor can we give a non-trivial general lower bound. We nevertheless included this upper bound in the figure with the intention to give an impression of the worst-case complexity (domain over which the QCM is optimized) of P2.}

It is evident from the figure that P1+P2 (default mode of Alg.~\ref{algo:query_comp}) always finish in less than $0.03$ sec outputting an optimized query w.r.t.\ the QSM $m$ and the QCM $c$. Importantly, this \emph{efficiency is independent of the type (e.g., knowledge base, physical system) of the diagnosis problem at hand}. Because P1+P2 only do combinatorial computations that depend solely on the diagnostic structure of the problem, i.e.\ the size, number, probabilities, overlapping, etc.\ of diagnoses.
Moreover, it takes P1 longer than P2 in all cases, and P1's execution time increases monotonically with $|\mD|$ whereas P2's does not.
Note that albeit P1+P2 solve Prob.~\ref{prob:query_optimization} for a restricted search space $\mathbf{S}$ (cf.\ Theorem~\ref{theorem:P1+P2_solve_query_optimization_problem}), the number $|\mathbf{CQP}_\mD|$ of CQPs w.r.t.\ $\mD$, which is just a fraction of $|\mathbf{S}|$, already averages to roundly $\langle 300, 5\,500, 28\,500, 105\,000, 200\,000, 370\,000, 475\,000, 530\,000 \rangle$ for $|\mD| \in \tuple{10,20,30,40,50,60,70,80}$.
That this restricted search space $\mathbf{S}$ is sufficiently large for all numbers of leading diagnoses $|\mD|$ is also substantiated by the fact that \emph{in each single test run} an optimal query w.r.t.\ the very small threshold $t_m = 0.01$ ($90\%$ smaller than the threshold used in \citep{Shchekotykhin2012}) was found in $\mathbf{S}$. 
The number of explicit-entailments queries per QP, 
i.e.\ the factor $c$ such that $|\mathbf{S}| = c |\mathbf{CQP}_\mD|$ might also be substantial, as hinted by the dotted line. Although this line describes only an upper bound, it gives at least a tendency of the size of the domain over which P2 seeks to optimize the given QCM. The meaningfulness of this trend line is corroborated by the fact that the time required for P2 (bottommost line in the plot) obviously correlates quite well with it.
%
The optional further query enhancement in P3 (omission of the search space restriction and switch to the full search space) terminates in all tests within less than 4 sec and returns the globally optimal query w.r.t.\ the QCM $c_{\max}$ (see Theorem~\ref{theorem:P3_solves_problem_1}). These practical times result from the moderate use of a reasoner in P3 (cf.\ Fig.~\ref{fig:summary_P3}).

In P1, also a brute force search computing all possible CQPs is feasible in most cases -- finishing within $50$ sec in all runs (up to search space sizes of more than $120\,000$) except for the $|\mD|\geq 30$ cases for the DPI CE 
(where up to 3 million CQPs were computed). 
This high computational speed is possible due to the \emph{complete avoidance of costly reasoner calls} by relying on the canonical notions, CQs and CQPs.

\begin{figure}[t]
	\centering
	\includegraphics[width=0.95\textwidth]{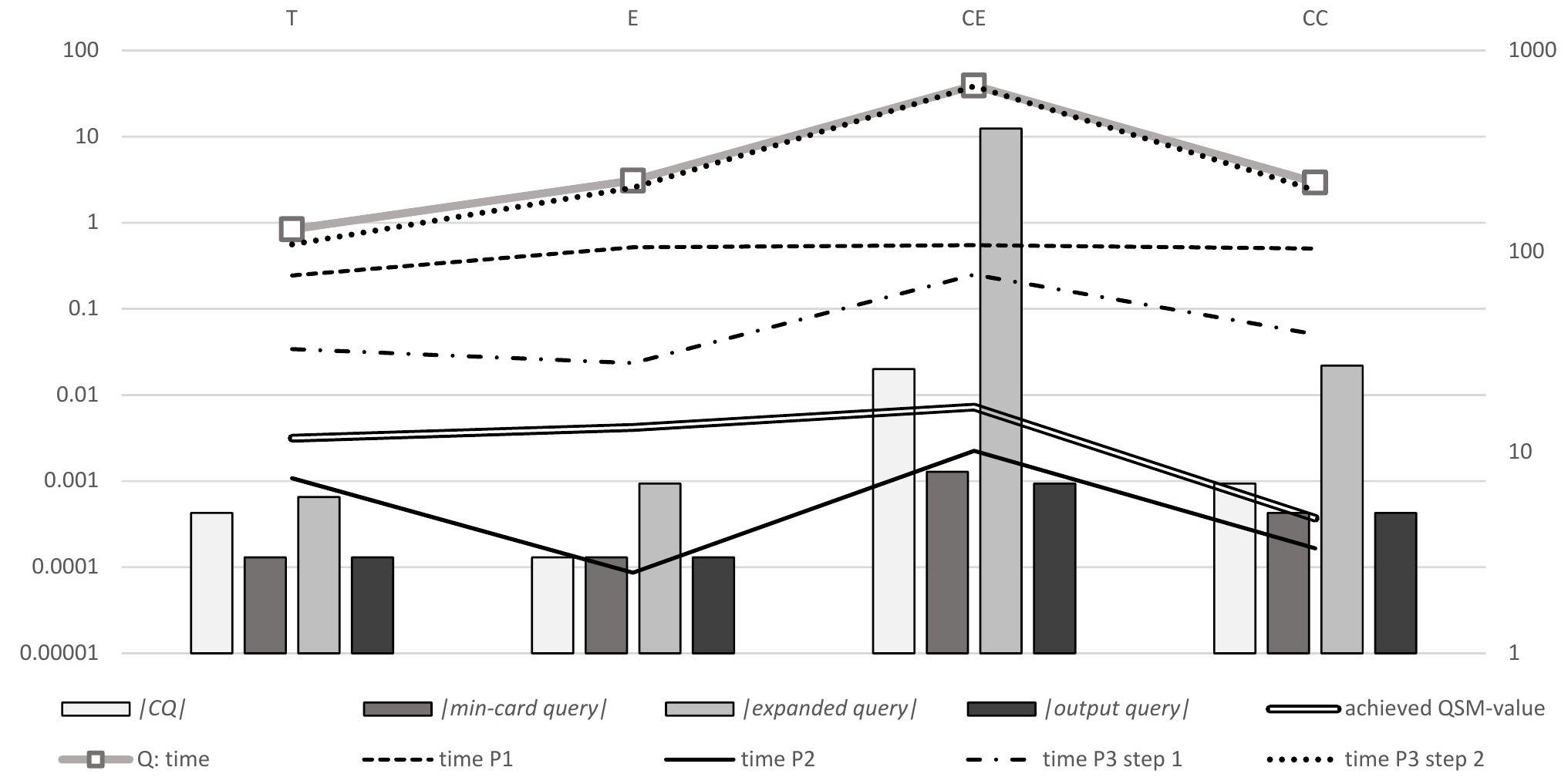}
	\caption{\small Scalability Tests. Min-card query / expanded query / output query denote the query returned by P2 / P3 step 1 / P3 step 2. The achieved QSM value refers to the QSM ENT. Q:\ time means overall query computation time (ignoring leading diagnoses computation time).}
	\label{fig:scalability}
\end{figure}

\paragraph{Scalability Tests. (Fig.~\ref{fig:scalability})}
The figure shows that even for an immense number of 500 leading diagnoses, Alg.~\ref{algo:query_comp} works in absolutely reasonable time, still producing a query optimized along the QSM and QCM dimension. The diagnoses computation times for $\tuple{\text{T},\text{E},\text{CE},\text{CC}}$ (not depicted in the figure) are $\tuple{8,11,1031,1405}$ sec. Concretely, we observe that the required overall query computation time amounts to roundly $\tuple{0.8, 3.1, 39.2, 3.0}$ sec for $\tuple{\text{T},\text{E},\text{CE},\text{CC}}$ (gray line with white squares) and is for the most part consumed by the reasoning activity in P2 (dotted line) which depends on $|\mD|$ (cf.\ Prop.~\ref{prop:P3_step2_complexity}). 

By way of comparison, the times achieved for $|\mD| = 80$ (cf.\ Fig.~\ref{fig:diag_vs_query_computation}) for $\tuple{\text{T},\text{E},\text{CE},\text{CC}}$ were $\tuple{0.28,1.11,3.55,0.54}$ sec, i.e.\ the time increase factors are roundly $\tuple{2.9,2.8,11.0,5.6}$ whereas $|\mD|$ was increased by the factor $\frac{500}{80}=6.25$. The most significant relative increase is due to P1 with factors around $\tuple{38,57,20,48}$ (i.e.\ each factor is computed as time for P1 with 500 diagnoses divided by time for P1 with 80 diagnoses). However, this substantial growth does not carry weight due to the very small absolute times required by P1. Similarly, for P2, absolute times were tiny such that the calculated factors of $\tuple{6.9,3.8,1.7,5.7}$ are of minor importance -- contrary to the growth factors $\tuple{2.4,2.4,11.5,5.0}$ measured for P3 step 2 where the highest absolute times are manifested. Note the strong correlation between the overall increase factors and the ones concerning step 2 of P3. Only for P3 step 1, times were basically decreasing, i.e.\ the factors in this case are $\tuple{0.9,0.8,1.4,0.9}$ (which is well explainable, see above). 
 
Whereas the computation times of phases P2 and P3 depend on the underlying DPI (see the parallel fluctuation of the respective lines in the figure, cf.\ discussions above), the time of P1 is mostly affected by $|\mD|$ (constancy of the dashed line). In fact, P1 finished in all cases within 0.6 sec, P2 within 0.003 sec and P3 step 1 within 0.25 sec. 

The achieved QSM-value for ENT was always below the postulated 0.01 (transparent framed line), confirming that a sufficiently good QP could be found in all cases. The number of QPs that had to be generated and expanded until a goal QP was identified is $\tuple{570,622,718,629}$ and $\tuple{10,12,18,15}$, respectively, for $\tuple{\text{T},\text{E},\text{CE},\text{CC}}$. Note the much higher search tree branching factors here, $\tuple{57,52,40,42}$, than observed in EXP1,  $\tuple{15,16,18,16}$ (cf.\ Fig.~\ref{fig:summary_P1}). 

Concerning the query evolution and size, we realize that the obtained values are reasonable and very similar to the ones seen in EXP1 (see, e.g., Fig.~\ref{fig:query_evolution}).

\begin{figure}[t]
	\centering
	\includegraphics[width=\textwidth]{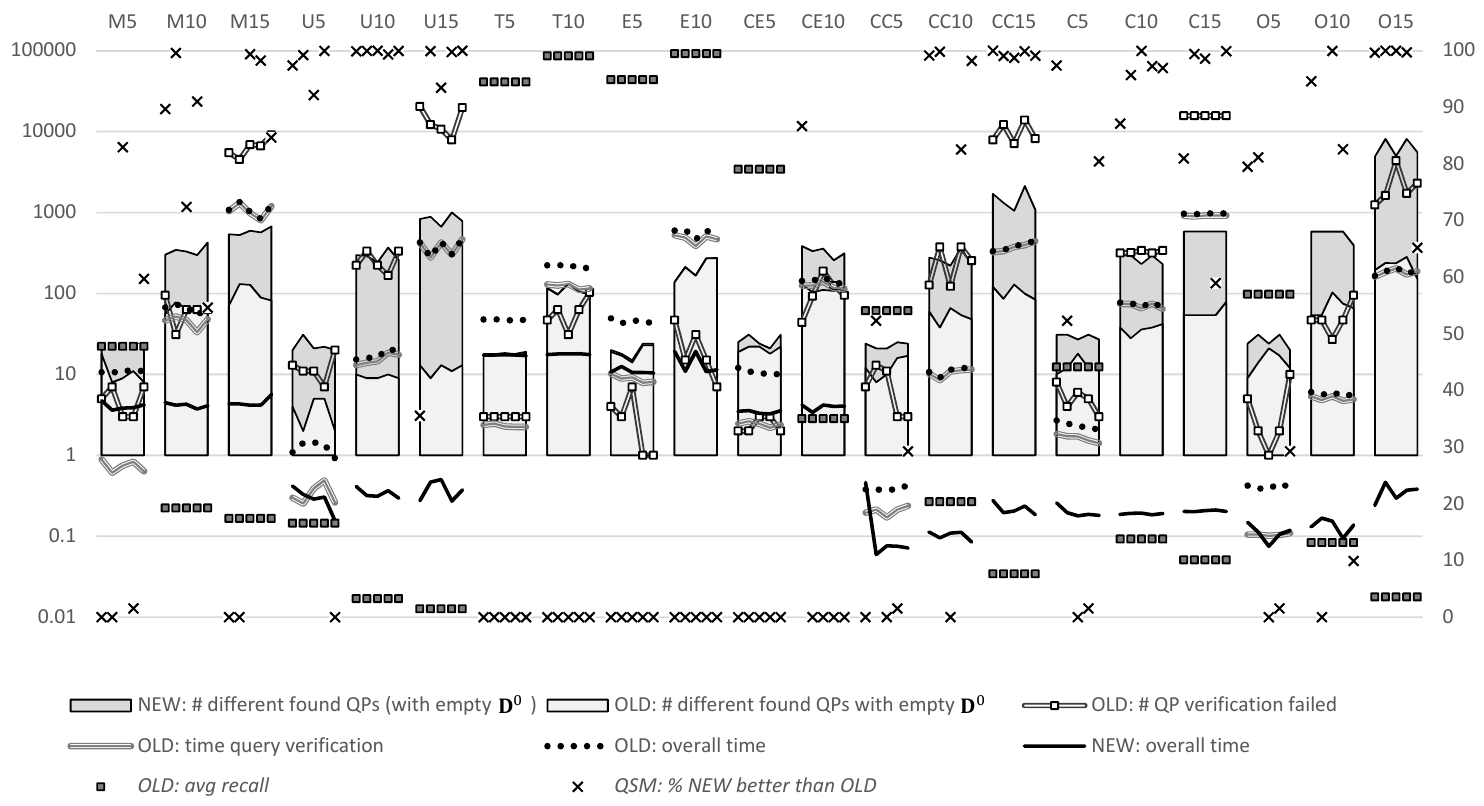}
	\caption{\small Comparison of Alg.~\ref{algo:query_comp} (NEW) with a method (OLD) not using the proposed theory. The \text{x-axis} (on top) shows the (DPI,$|\mD|$) combinations, e.g.\ U15 presents 5-iteration average values for the DPI U with $|\mD|=15$. Overall time refers to the time for the entire query computation (where diagnoses computation time is neglected). OLD: avg recall refers to the average (over 5 iterations) percentage of QPs (with empty $\dz{}$) generated by OLD among those generated by NEW.}
	\label{fig:naive_vs_new}
\end{figure}

\paragraph{Comparison of Alg.~\ref{algo:query_comp} (NEW) with a Method (OLD) not Using the Proposed Theory. (Fig.~\ref{fig:naive_vs_new})}
The figure shows the results for all (DPI,$|\mD|$) combinations where OLD (see the description of EXP3 in Sec.~\ref{sec:experiments}) terminated in all 5 iterations of query computation before a predefined timeout of 1 hour elapsed. That is, for all DPIs in Tab.~\ref{tab:experiment_data_set} it took OLD longer than 1 hour for $|\mD|=20$. Moreover, for T, E and CE, OLD 
exceeded the timeout even for $|\mD|=15$.
Diagnoses computation worked reasonably in all cases, requiring always less than 2 min.  

What might seem surprising at first sight is the fact that, in all cases, NEW generates at least as many (different) QPs with empty $\dz{}$ as OLD does, although both run a brute force search. Beyond that, in most of the cases the fraction $f$ of QPs OLD is able to determine among those generated by NEW is rather small (see the size relation between light and dark gray areas in the figure). In numbers, $f$ ranges from only $1\%$ (U15) to $99\%$ (T10). Note that the y-axis is logarithmic, i.e.\ albeit the light gray area for U15 seems about a third of the dark gray one, it amounts to only about one percent (two orders of magnitude difference). That is, by absolute numbers, OLD generates on average for U15 only $12$ QPs from $835$ QPs generated by NEW. To make this aspect better visible, the figure shows the average fraction $f$ (recall of OLD w.r.t.\ the QPs determined by NEW) over each 5 iterations in terms of the gray squares. 

The reason for this drastic degree of incompleteness of OLD regarding QPs is the strategy of a direct generation of a (potential) query before its (q-)partition is constructed. This is inverse to how NEW behaves. Since this query candidate is generated by a reasoner, it depends strongly on the reasoner's output, i.e.\ the computed entailments. Without proper care, 
this generally leads to a neglect of the discrimination sentences (see Def.~\ref{def:discax}), which are those responsible for the consideration of all and only (canonical) QPs, i.e.\ for soundness and completeness of (C)QP generation (cf.\ Sec.~\ref{sec:P1}). As a result (cf.\ the discussion on page~\pageref{etc:advantages_of_CQs+CQPs}), this can on the one hand effectuate unsoundness, i.e.\ the construction of unnecessary duplicate QPs and unnecessary partitions that are non-QPs (thus query candidates that turn out to be no queries). On the other hand, it can cause incompleteness, i.e.\ the disregard of QPs. 
Apart from that, the search space of queries for one (fixed) QP $\Pt$ explored by OLD is generally smaller than the one considered by NEW because NEW computes a query expansion based on the most informative seed for $\Pt$ (i.e.\ $\Pt$'s CQ, see Sec.~\ref{sec:P3}). As opposed to this, OLD computes a minimization of the (strongly reasoner-dependent) possibly much less informative initial query candidate computed for $\Pt$. 
Please note that these problems occur although both methods, NEW and OLD, use exactly the same reasoner and compute exactly the same (types of) entailments for one and the same KB. The crux is rather \emph{if}, \emph{when} and \emph{for what} the reasoner is used.\footnote{One way to overcome the incompleteness issue of OLD is the manual addition of sufficiently many discrimination sentences to the entailments returned by the reasoner. Nevertheless, the problems concerning duplicates and computation time persist.} 

The number of unnecessarily generated query candidates $X$ (those for which the query verification failed, cf.\ EXP3 in Sec.~\ref{sec:experiments}) by OLD is substantial in several cases (see light line with squares). For instance, for U5 / U 10 / U15 this number averaged to over $12$ / $250$ / $14\,000$ and for O5 / O10 / O15 to $4$ / $50$ / $2250$. Also, failed verifications outnumber successful ones (light shaded area) often by several orders of magnitude, e.g.\ for U15 or CC15. Consequently, OLD might spend most of the time unnecessarily.  
The reason for the failed query verification in all these cases is the emptiness of the $\dnx{}(X)$ set of $X$'s partition. This stems from a too small number of or too weak logical sentences in $X$ (cf.\ Prop.~\ref{prop:partition}). NEW avoids this by exploiting the canonical notions introduced in Sec.~\ref{sec:P1}. Thus, whereas OLD's query generation can be seen as a trial and error approach, NEW pursues a sound and complete systematic approach. 

The massive advantage brought about by this is proven by the query computation times measured. In fact, for all cases where $|\mD| \geq 10$, NEW required at least one order of magnitude less time than OLD (see black solid versus black dotted line). For example, the average execution times of NEW and OLD on $\tuple{\text{O5},\text{O10},\text{O15}}$ are $\tuple{0.11,0.14,0.35}$ sec and $\tuple{0.41,5.8,186}$ sec, respectively. 
Additionally, as OLD's query computation time increases exponentially (see the jumps of the dotted line by about an order of magnitude for one and the same DPI and different diagnoses numbers, e.g., CC5 vs.\ CC10 vs.\ CC15) and NEW's at most linearly (see Fig.~\ref{fig:diag_vs_query_computation}) with $|\mD|$, the time savings of NEW versus OLD grow (significantly) upon extending $|\mD|$. The gray line in the figure suggests that the performance of OLD suffers strongly from the enormous time amount spent for query verification. Note, the latter requires extensive usage of a reasoner. 
For instance, for U15 almost the entire computation time was dedicated to query verification. Another interesting fact is that, over all DPIs and $|\mD|=10$ (because for this configuration OLD always finished before the timeout), OLD could only explore a median of $1.8$ QPs within the \emph{overall} query computation execution time (P1+P2+P3) of NEW.

Finally, the quality (QSM-value) of NEW's returned query proves to be \emph{always} as good or better than for OLD. The percent improvement $[(m_{\mathit{old}} - m_{\mathit{new}}) / m_{\mathit{old}}] * 100$ regarding the QSM-value $m_{\mathit{new}}$ achieved by NEW as compared with the QSM-value $m_{\mathit{old}}$ achieved by OLD (crosses in the figure) was sometimes more and sometimes less substantial (note, smaller QSM-values are better). As optimality is directly related to search completeness, smaller or no improvements are given for cases where OLD evinces high or full completeness, and vice versa. Stated in terms of the notation used in the figure, this means that the gray squares and the crosses are negatively correlated. 

\section{Related Work}
\label{sec:related}

The thorough analyses of pool-based query generation and selection and the revealed shortcomings and improvement suggestions given in \citep{Rodler2015phd} provide a foundation of and a motivation for this work. Beyond that, we oriented ourselves by the structure and notation used in \citep{Rodler2015phd} 
in our preliminary sections. A more in-depth discussion of the theory and the algorithms described here and detailed, ready-to-be-implemented pseudocode plus additional remarks and considerations on search heuristics, search tree pruning and early termination rules for the discussed QSMs and further ones can be found in \citep{DBLP:journals/corr/Rodler16a}. 
The generic sequential diagnosis algorithm we describe in Sec.~\ref{sec:sequential_diagnosis_algo} is very similar to the ones given in \citep{Shchekotykhin2012,Rodler2015phd} and follows the principled approach outlined in \citep[Sec.~2]{dekleer1993}. The QSMs ENT and SPL we use in our evaluations have originally been used for decision tree learning \citep{quinlan1986induction} and optimization \citep{moret1982decision} (here, SPL is called ``separation heuristic''), but have later been adopted for diagnoses discrimination as well \citep{dekleer1987,Shchekotykhin2012,Rodler2013,rodler17dx_activelearning}. 

Whereas the works \citep{pattipati1990,DBLP:journals/tsmc/ShakeriRPP00,DBLP:journals/mam/ZuzekBN00,DBLP:conf/ijcai/BrodieRMO03,gonzalez2011spectrum} also focus on the query (probe/test) sequencing problem, they are not model-based, but rely on a test (coverage) matrix or test dictionary, respectively. Such approaches are often called spectrum-based. Pros (and cons) of model-based approaches against such methods are discussed in \citep{pietersma2005model} and \citep[Sec.~3]{davis1988}. 
%
Besides the frequent reliance upon a predefined (and possibly incomplete) set of possible faults (e.g.\ single faults\footnote{\citep{abreu2011simultaneous} extends spectrum-based approaches by considering multiple faults, but does not address the selection of optimal tests.}) in such systems, one important difference to our approach is the fixed set of tests from which the system might choose in each iteration.
Hence, the query \emph{computation} (as in our case) reduces to a query \emph{selection} and the space of queries is initially known and restricted (in a way the generation and size of the test matrix 
remain feasible) whereas it is unknown and, in principle, not restricted in the case of our work (queries need not be chosen or computed in advance). 
Another crucial aspect is the information about the influence of test outcomes on the diagnoses. 
While such information in terms of prespecified fault models (relating faults to test outcomes\footnote{Note that this information (e.g., given by the tests' \emph{traces} telling which components are involved in the each test's execution) is required \emph{in advance} of running the tests when dealing with the problem of optimal test selection.}) is required to be known (or at least reliably and efficiently estimable) in the mentioned approaches, this information needs to be logically inferred in model-based approaches like the one described here. As discussed, this logical reasoning is the major factor affecting the computational efficiency in the model-based case. 
Consequently, e.g., simulation-based non-myopic test selection strategies \citep{zamir2014} are not efficiently applicable to the problem tackled in this work. The issue is the very high reasoning cost required to update the diagnoses sets during these simulations. 
	
Model-based diagnosis methods for generating optimal \emph{tests}, i.e.\ new sets of system inputs facilitating new gainful observations of the system behavior, are presented by the works of \citep{pietersma2005model,DBLP:journals/jair/FeldmanPG10a}. Although test generation may be basically enforced with our approach by selecting appropriate entailment types $\mathit{ET}$ in phase P3 (cf.\ also Ex.~\ref{ex:query_representation}), the main intention of our work, similar to, e.g., \citep{Felfernig2004213,Shchekotykhin2012,Rodler2013,Rodler2015phd},
is the specification of an optimal next query (test case), defined generally as a (set of) sentence(s) formulated in some knowledge representation language. 
Translated to hardware or physical devices \citep{dekleer1987,dekleer1993,Siddiqi2011} a query usually corresponds 
to a \emph{probe}, i.e.\ the measurement of (a) system variable(s). 
Due to the generality of the query notion our approach addresses a more complex query (probe) search space than the methods of \citep{dekleer1987,dekleer1993}, thereby guaranteeing perfect diagnoses discriminability, i.e.\ the unambiguous localization of the actual fault (cf.\ Sec.~\ref{sec:applicability_diagnostic_accuracy}).

A general aspect that distinguishes our approach from probe selection algorithms is the not necessarily explicit availability of the possible probes. For instance, in a digital or combinatorial circuit \citep{Reiter87,dekleer1987,Siddiqi2011} all probe locations are given by the circuit's wires. In general, this is not necessarily the case, e.g.\ when considering any kind of knowledge-based system \citep{Felfernig2004213,Kalyanpur2006,Shchekotykhin2012,Rodler2013,Rodler2015phd}. In the latter scenario, (most of the) queries (e.g., possible questions to a domain expert) are normally implicit and must be (expensively) inferred.

\citep{DBLP:journals/jair/FeldmanPG10a} also describes a probing algorithm which, as opposed to ours, makes extensive use of an inference engine to compute the best query (probe). Moreover, it might consider non-discriminating probes as well as generate duplicates (in terms of the diagnoses eliminated for each outcome) in the search for an optimal probe. Both is impossible with our approach (cf.\ page~\pageref{etc:advantages_of_CQs+CQPs}).

The works of \citep{Shchekotykhin2012,Rodler2013} are similar to ours in that they deal with query selection for knowledge-based systems and also suggest a generic procedure for query generation and selection. This procedure can be seen to ``implement'' the definition of a query (cf.\ Def.~\ref{def:query_q-partition}) quite directly, i.e.\ computing query candidates $X$ for various seeds $\dx{} \subset \mD$ and verifying whether $X$ is indeed a query by checking $\dx{}(X) \neq \emptyset$ and $\dnx{}(X) \neq \emptyset$. However, the main focus of these works is on \emph{query selection strategies} (which we call QSMs, cf.\ Sec.~\ref{sec:measurement_selection}) and their pros and cons in various scenarios, especially with regard to the length of the query sequence until (high) diagnostic certainty is achieved. Our work, on the contrary, can be seen to complement these works
by concentrating and improving on the algorithms for \emph{query computation and optimization}. For this reason we have also used the QSMs discussed in \citep{Shchekotykhin2012} 
in our evaluations. The latter (cf.\ Sec.~\ref{sec:experimental_results},  Fig.~\ref{fig:naive_vs_new}) show that our novel theory enables significant enhancements of the above-noted generic procedure by a well-conceived (non-)employment of expensive reasoning. 

A difference to the works of \citep{dekleer1987,Shchekotykhin2012,Rodler2013} is the capability of our approach to reduce the query search space a-priori to only preferred queries, i.e.\ those that discriminate between \emph{all} (leading) diagnoses (cf.\ discussion on page~\pageref{etc:discussion_why_D0=emptyset}). The key idea for realizing this was brought up in \citep[Chap.~8]{Rodler2015phd}. 

\citep{Siddiqi2011} suggest a heuristic for query selection that does not require computing the entropy of diagnoses because the latter can generally be costly for state-of-the-art systems such as GDE \citep{dekleer1987} when the set of diagnoses is large. First, by the possibility to leverage heuristics and pruning techniques in the proposed systematic query search, our method needs to explore only minor parts of the search space (cf.\ Sec.~\ref{sec:conclusion}) and can save effort by computing the actual entropy of a query (q-partition) only if its heuristic value is sufficiently good. We have shown that in this vein queries deviating negligibly from the optimum could be computed in reasonable time for very large numbers up to 500 leading diagnoses. Hence, the scalability is not a problem with this strategy. Second, our approach can easily handle a variety of other QSMs (such as all those introduced in \citep{rodler17dx_activelearning}) apart from entropy in a simple plug-in fashion. Corresponding heuristics and pruning operations suited for these QSMs when incorporated into our approach are described and explained in \citep{DBLP:journals/corr/Rodler16a}. Third, as noted above, our approach can deal with implicit queries, i.e.\ those that need to be inferred and are not given in explicit form.
Fourth, \citep{Siddiqi2011} assume all measurements to have equal costs and thus do not consider the minimization of the latter. Instead, they focus on the minimization of the number of queries (measurements). Our approach incorporates both the number and the cost of queries in the optimization process.

\citep{heckerman1995decision}, e.g., does consider query (observation) costs, but, unlike our work, uses interleaved repair and observation actions. Also, it might test system components unnecessarily (given misleading probabilities) since no diagnostic evidence, i.e.\ the set of leading diagnoses pinpointing what we call \emph{discrimination sentences} (cf.\ Def.~\ref{def:discax}), is incorporated. 

Finally, and importantly, our approach uses the \emph{inference engine as a black-box} (i.e.\ as an oracle for consistency and/or entailment checks) and does not depend on any specific inference mechanism. This stands in contrast to glass-box approaches \citep{Parsia2005} that rely on modifications of the reasoner's internals, e.g., for storing justifications \citep{Horridge2008}, prime implicants \citep{quine1952} and environments \citep{dekleer1987}, respectively, for certain entailments.
This reasoner independence makes our approach very general in the sense that it can be applied to diagnosis problem instances formulated in any knowledge representation formalism for which a sound and complete reasoner is available. In-depth comparisons between black-box and glass-box approaches for monotonic KBs have been carried out by \citep{Kalyanpur2006a,Horridge2011a}, with the overall conclusion that, in terms of performance, black-box methods compare favorably with glass-box methods while offering a higher generality and being more easily and robustly implementable. 

\section{Conclusions}
\label{sec:conclusion}
We present a method that addresses the optimal measurement (\emph{query}) selection problem for sequential diagnosis and is applicable to any model-based diagnosis problem conforming to \citep{dekleer1987,Reiter87}. In particular, given a set of leading diagnoses, we allow a query to be optimized along two dimensions, i.e.\ the estimated number of queries and the cost per query. We show that the optimizations of these properties can be naturally decoupled and considered in sequence. 

That is, one can at first (\emph{phase P1}) tackle the optimization of the query's diagnoses discrimination properties (given by the query's \emph{q-partition}). Unlike existing methods do, this can be accomplished without any expensive reasoner calls and without yet knowing the actual query, as the theory evolved in this work proves. The key idea underlying this theory is the exploitation of useful information that is already implicitly present in the set of precomputed leading diagnoses. Contrary to most sequential methods which use a pool-based approach (selection of an optimal query within the collection of possible queries), we demonstrate a sound and complete \emph{systematic} search for an optimal q-partition that enables the powerful application of heuristics and sound pruning techniques. Our evaluations show that thereby, on average, only a negligible fraction (less than $1\%$) of the space of possible q-partitions needs to be explored to find a q-partition that deviates negligibly from the optimal q-partition. In all of the several hundreds of test runs we performed with search spaces encompassing up to millions of q-partitions, an optimal one is found in less than $0.03$ sec. 

Once the (optimized) q-partition $\Pt$ from phase P1 is known, the user has two choices, namely the \emph{instantaneous} provision of a globally cost-optimal query associated with $\Pt$ over a restricted (yet generally exponential) query search space (\emph{phase P2}), or the computation of a query for $\Pt$ including only ``cost-preferred'' sentences (e.g., low-cost measurements or observations from built-in sensors) over the full query search space (\emph{phase P3}).

Regarding phase P2, we prove that an optimal solution for the underlying problem can be found by solving a hitting set problem over an \emph{explicitly given} collection of sets to be hit and without employing an inference engine. Despite the general NP-hardness of the hitting set problem, P2 turns out to be the most efficient part of the presented algorithm as we prove that the problem can be viewed as fixed parameter tractable in our context. For up to tremendous numbers of $500$ leading diagnoses and query search spaces of size up to over $3$ million, P2 terminated always within less than $0.006$ sec, returning a globally optimal solution for any (monotonic set function) cost measure.

Phase P3 performs the computation of a query that is optimized in a more sophisticated way, i.e.\ over a substantially extended search space, than the one determined by P2. To this end, the employment of a logical reasoner is required. However, the premise is to minimize the reasoner calls for best efficiency. We prove that P3 gets along with a polynomial number of reasoner calls while it is guaranteed to provide a globally optimal query w.r.t.\ a cost measure that minimizes the maximal cost of a single measurement in the query. We show that the latter optimization is possible by using an existing divide-and-conquer algorithm \citep{junker04} for set-minimization under preservation of a monotonic property with a suitably modified input. 

Comprehensive experiments using real-world diagnosis problems of different size, diagnostic structure (size, number, probability of diagnoses) and (reasoning) complexity demonstrate the efficiency and practicality of the proposed algorithm (phases P1, P2, P3). For instance, for up to 80 leading diagnoses, two optimal queries (one from P2 and one from P3) are always established in no more than 4 sec. It further turns out that optimized query computation with the new approach is much faster than the computation of leading diagnoses. The time of the former grows at most linearly and the time of the latter exponentially with the number of leading diagnoses. Consequently, optimized query computation using our theory is a minor problem as compared to diagnoses computation. 
This was not the case with existing (black-box) methods where the leading diagnoses (and thence the information usable for query computation) needed to be restricted to single-digit numbers, cf., e.g., \citep{Shchekotykhin2012}. Regarding the reaction time $r$ (time between two queries) of the debugging system, the new algorithm accounts for less than $10$ percent (P3:\ full query search space) and less than $3$ per mill (P2: restricted query search space) of $r$ whenever $r$ amounts to at least a fifth of a minute. 

In comparative experiments we reveal that methods $M$ that are as generally applicable as the presented one and not using the proposed theory, but applying reasoners directly and improvidently during query generation, suffer from (1)~a drastic incompleteness w.r.t.\ the query space (sometimes their recall is as low as $1\%$ due to the strong dependence on the reasoner output), (2)~the unavoidable computation of duplicate queries (and/or q-partitions) and (3)~the generation of substantial numbers of unnecessary query candidates (up to tens of thousands) which turn out to be no queries after verification. Because query verification (without the presented theory) requires the reasoner, these methods manifest severe performance problems. The new method solves all these issues. In fact, it does so while always consuming orders of magnitude less time than $M$ for 10 or more leading diagnoses and outputting a query which is always at least as good as and up to more than $99\%$ better (regarding a query quality measure) than the one returned by $M$.    

Finally, tests involving query computations given $500$ leading diagnoses (search space sizes in $O(2^{500})$ in P1 alone) corroborate the scalability of the new approach. In concrete terms, optimized queries over the restricted search space (P2) and full search space (P3) could be computed in less than $0.7$ and $40$ sec, respectively.

\newpage

\acks{This work was supported by the Carinthian Science Fund (KWF) contract KWF-
	3520/26767/38701.
}

\appendix

\section*{Appendix A: Proofs}
\label{appendix:A:proofs}

\subsection*{Proof of Theorem~\ref{theorem:relation_between_max-sol-KB_and_min-diagnosis}}
\begin{proof}
By definition, each maximal canonical solution KB w.r.t.\ $\dpi$ has the form $(\mo \setminus \md)\cup U_\Tp$ for some $\md \subseteq \mo$. We show that any $\ot$ 
is not a maximal canonical solution KB w.r.t.\ $\dpi$ if (*) $\ot$ cannot be constructed as $(\mo \setminus \md)\cup U_\Tp$ for some minimal KBD-diagnosis $\md$ w.r.t.\ $\dpi$. 

Let us assume $\ot$ is a maximal solution KB w.r.t.\ $\dpi$ and $\ot$ cannot be constructed as $(\mo \setminus \md)\cup U_\Tp$ for any minimal KBD-diagnosis $\md$ w.r.t.\ $\dpi$. Then either (a)~$\ot$ cannot be constructed as $(\mo \setminus X)\cup U_\Tp$ for any set $X \subseteq \mo$, or (b)~$\ot$ can be constructed as $(\mo \setminus X)\cup U_\Tp$ for some $X \subseteq \mo$ where $X$ is not a KBD-diagnosis w.r.t.\ $\dpi$, or (c)~$\ot$ can be constructed as $(\mo \setminus X)\cup U_\Tp$ for some $X \subseteq \mo$ where $X$ is a non-minimal KBD-diagnosis w.r.t.\ $\dpi$.

Case (a) together with the definition of a canonical solution KB implies that $\ot$ is not a (maximal) canonical solution KB w.r.t.\ $\dpi$, a contradiction.

Case (b) together with Def.~\ref{def:diagnosis} implies that $\ot$ is not a (maximal canonical) solution KB w.r.t.\ $\dpi$, a contradiction.

Case (c): Here we have that $\ot = (\mo \setminus X)\cup U_\Tp$ is a (canonical) solution KB by Def.~\ref{def:diagnosis} (and the definition of a canonical solution KB). However, due to the non-minimality of the KBD-diagnosis $X$, there must be an $X_{\min} \subset X$ such that $X_{\min}$ is a minimal KBD-diagnosis w.r.t.\ $\dpi$. Hence, $\ot_{\min} := (\mo \setminus X_{\min})\cup U_\Tp$ is a (canonical) solution KB as well. Then either (c1)~$\ot = \ot_{\min}$ or (c2)~$\ot \neq \ot_{\min}$.

Case (c1): Here we derive that $\ot = (\mo \setminus X_{\min})\cup U_\Tp$ for the minimal KBD-diagnosis $X_{\min}$ w.r.t.\ $\dpi$. This contradicts the assumption that $\ot$ cannot be constructed as $(\mo \setminus X)\cup U_\Tp$ for any minimal KBD-diagnosis $X$ w.r.t.\ $\dpi$.

Case (c2): Since clearly $\mo \cap \ot = (\mo \setminus X) \cup (\mo \cap U_\Tp) \subseteq (\mo \setminus X_{\min}) \cup (\mo \cap U_\Tp) = \mo \cap \ot_{\min}$, we conclude that $\mo \cap \ot \subset \mo \cap \ot_{\min}$ which is why $\ot$ cannot be a maximal (canonical) solution KB by Def.~\ref{def:solution_KB}.
%
%
%
%
%
%
\end{proof} 

\subsection*{Proof of Proposition~\ref{prop:partition}}
\begin{proof}
By the definition of $\dz{}(X)$, we have that $\dx{}(X) \cup \dnx{}(X) \cup \dz{}(X) = \mD$, $\dx{}(X) \cap \dz{}(X) = \emptyset$ and $\dnx{}(X) \cap \dz{}(X) = \emptyset$. 
Let us assume that $\dx{}(X) \cap \dnx{}(X) \neq \emptyset$. Then, there must be some $\md_i \in \dx{}(X) \cap \dnx{}(X)$. For $\md_i$ it holds that $\mo^{*}_i \models X$ and $(\exists \tn \in \Tn: \mo^{*}_i \cup X \models \tn) \lor (\exists r \in \RQ: \mo^{*}_i \cup X \text{ violates } r)$. This implies by the idempotency of $\mathcal{L}$ that $(\exists \tn \in \Tn: \mo^{*}_i \models \tn) \lor (\exists r \in \RQ: \mo^{*}_i \text{ violates } r)$. By Def.~\ref{def:diagnosis}, we can conclude that $\md_i$ is not a diagnosis w.r.t. $\dpi$. But, due to $\md_i \in \dx{}(X) \cap \dnx{}(X)$, we have that $\md_i \in \mD$ and thus that $\md_i$ is a diagnosis w.r.t.\ $\dpi$ due to $\mD \subseteq \minD_{\dpi}$. This is a contradiction. Thus, each diagnosis in $\mD$ is an element of exactly one set of $\dx{}(X), \dnx{}(X), \dz{}(X)$.
\end{proof}

\subsection*{Proof of Proposition~\ref{prop:expl_ent_query_must_neednot_mustnot_include_ax}}
\begin{proof} 
The statement of Prop.~\ref{prop:expl_ent_query_must_neednot_mustnot_include_ax} is a direct consequence of Lemmata~\ref{lem:EE-query_no_intersection_with_I_D}, \ref{lem:EE-query_must_have_intersection_with_U_D} and \ref{lem:reduct_to_discax} that are proven below.
\end{proof}

\begin{lemma}\label{lem:EE-query_no_intersection_with_I_D}
Let $\dpi := \langle\mo,\mb,\Tp,\Tn\rangle_\RQ$ be a DPI, $\mD \subseteq \minD_{\dpi}$ and $Q$ be any query in $\mQ_\mD$. 
Then $Q \cap I_\mD = \emptyset$.
\end{lemma}
\begin{proof}
Assume that $Q \cap I_\mD \neq \emptyset$. So, let us assume that one element in $Q \cap I_\mD$ is $\tax$. Let $\md_k$ be an arbitrary diagnosis in $\mD$. Since $\tax \in I_\mD = \bigcap_{\md_i\in\mD} \md_i$ we can conclude that $\tax \in \md_k$. Now, because $\tax \in Q$, we observe that $\mo_{k}^* \cup Q = (\mo\setminus\md_k)\cup\mb\cup U_\Tp \cup Q = (\mo\setminus(\md_k\setminus\tax))\cup\mb\cup U_\Tp \cup Q$. However, as $\md_k\setminus\tax \subset \md_k$ due to $\tax \in \md_k$, and because of the $\subseteq$-minimality of the diagnosis $\md_k$, the KB $(\mo\setminus(\md_k\setminus\tax))\cup\mb\cup U_\Tp$ must violate $\RQ$ or $\Tn$ (cf.\ Def.~\ref{def:diagnosis}). In consequence of the monotonicity of $\mathcal{L}$, it must be the case that $\mo_{k}^* \cup Q$ violates $\RQ$ or $\Tn$ as well. 
Hence, by Prop.~\ref{prop:partition}, we have that $\md_k \in \dnx{}(Q)$. As $\md_k$ was chosen arbitrarily among elements of $\dnx{}(Q)$, we infer that $\dnx{}(Q) = \mD$. Thence, $\dx{}(Q) = \emptyset$ which is a contradiction to $Q$ being a query in $\mQ_\mD$ by Def.~\ref{def:query_q-partition}.
\end{proof}
Note, since Lem.~\ref{lem:EE-query_no_intersection_with_I_D} holds for arbitrary queries in $\mQ_\mD$, it must in particular hold for explicit-entailments queries in $\mQ_\mD$. 
\begin{lemma}\label{lem:EE-query_must_have_intersection_with_U_D}
Let $\dpi := \langle\mo,\mb,\Tp,\Tn\rangle_\RQ$ be a DPI, $\mD \subseteq \minD_{\dpi}$ and $Q$ be any query in $\mQ_\mD$ such that $Q \subseteq \mo$. Then $Q \cap U_\mD \neq \emptyset$.
\end{lemma}
\begin{proof}
Let $\md_k$ be an arbitrary diagnosis in $\mD$.
Suppose that $Q \cap U_\mD = \emptyset$. By assumption, $Q \subseteq \mo$. Hence, $Q \subseteq \mo \setminus U_\mD$.  For this reason, we have that $\mo^{*}_k \supseteq \mo\setminus\md_k \supseteq \mo\setminus U_\mD \supseteq Q$. By the extensiveness of $\mathcal{L}$, it holds that $\mo^{*}_k \models Q$. Since $\md_k$ was arbitrarily chosen among the elements of $\mD$, the same must hold for all $\md_i \in \mD$. Therefore, $\dx{}(Q) = \mD$ by Prop.~\ref{prop:partition} which entails that $\dnx{}(Q) = \emptyset$. The latter, due to Def.~\ref{def:query_q-partition}, contradicts the assumption that $Q$ is a query in $\mQ_\mD$. 
\end{proof}

\begin{lemma}\label{lem:reduct_to_discax}
Let $\dpi := \langle\mo,\mb,\Tp,\Tn\rangle_\RQ$ be a DPI, $\mD \subseteq \minD_{\dpi}$ and $Q$ be any query in $\mQ_\mD$ such that $Q \subseteq \mo$. Then $Q' := Q \setminus (\mo \setminus U_\mD) \subseteq Q$ is a query in $\mQ_\mD$ where $\Pt_\mD(Q') = \Pt_\mD(Q)$ (i.e.\ $Q$ and $Q'$ have the same q-partition).
\end{lemma}
\begin{proof}
If $Q' = Q$ then the validity of the statement is trivial. Otherwise, $Q' \subset Q$ and for each sentence $\tax$ such that $\tax \in Q\setminus Q'$ we have that $\tax \in \mo \setminus U_\mD$. By extensiveness of $\mathcal{L}$, we conclude that $\mo \setminus U_\mD \models \tax$. Let $\md_k$ be an arbitrary diagnosis in $\mD$. 
Then $\mo_k^* \supseteq (\mo\setminus U_\mD)\cup\mb\cup U_\Tp \supseteq \mo \setminus U_\mD$. So, by monotonicity of $\mathcal{L}$, it holds that $\mo_k^* \models \tax$ for all $\tax \in Q \setminus Q'$ and all $\md_k \in \mD$. 
	
Let us now assume that $\Pt_\mD(Q') \neq \Pt_\mD(Q)$. By Prop.~\ref{prop:explicit-entailments_queries_have_empty_dz} and $Q' \subseteq Q \subseteq \mo$, it follows that $\dz{}(Q)= \dz{}(Q') = \emptyset$. Therefore, either $\dx{}(Q') \subset \dx{}(Q)$ or $\dnx{}(Q') \subset \dnx{}(Q)$ must be true.
First, suppose some diagnosis $\md_k \in \dx{}(Q)$ such that $\md_k \notin \dx{}(Q')$. Then, $\mo_k^* \models Q$, but $\mo_k^* \not\models Q'$ which is a contradiction due to $Q' \subseteq Q$.
Second, suppose some diagnosis $\md_k \in \dnx{}(Q)$ such that $\md_k \notin \dnx{}(Q')$. From the latter fact we can conclude that $\mo_k^* \cup Q'$ satisfies $\RQ$ and $\Tn$. Now, due to the idempotence of $\mathcal{L}$ and because $\mo_k^* \models \tax$ for all $\tax\in Q\setminus Q'$ as shown above, we can deduce that $\mo_k^* \cup Q = \mo_k^* \cup Q' \cup (Q\setminus Q')$ is logically equivalent to $\mo_k^* \cup Q'$. Hence $\mo_k^* \cup Q$ satisfies $\RQ$ and $\Tn$. But this is a contradiction to the assumption that $\md_k$ is an element of $\dnx{}(Q)$.

Overall, we conclude that $Q$ and $Q'$ have equal q-partitions. From that and the fact that $Q$ is a query in $\mQ_\mD$, we derive by Def.~\ref{def:query_q-partition} that $Q'$ must be a query in $\mQ_\mD$ as well.
\end{proof}

\subsection*{Proof of Proposition~\ref{prop:canonical_query_is_query}}

\begin{proof}
	Let $Q$ be a canonical query w.r.t.\ some seed $\mathbf{S}$ satisfying $\emptyset \subset \mathbf{S} \subset \mD$. Then, $Q \neq \emptyset$ due to Def.~\ref{def:canonical_query}. 
	Hence, it suffices to demonstrate that $\dx{}(Q) \neq \emptyset$ as well as $\dnx{}(Q) \neq \emptyset$ in order to show that $Q$ is a query (cf.\ Def.~\ref{def:query_q-partition}). 
	
	First, since for all $\md_i \in \mathbf{S}$ it holds that $\md_i \subseteq U_{\mathbf{S}}$, we have that $\mo \setminus \md_i \supseteq \mo \setminus U_{\mathbf{S}}$. Due to the fact that the entailment relation in the used logic $\mathcal{L}$ is extensive, we can derive that $\mo \setminus \md_i \models \mo \setminus U_{\mathbf{S}}$. Hence, by the monotonicity of $\mathcal{L}$, also $\mo_i^* := (\mo \setminus \md_i) \cup \mb \cup U_\Tp \models \mo \setminus U_{\mathbf{S}}$ for all $\md_i \in \mathbf{S}$. 
	As $Q \subseteq \mo \setminus U_{\mathbf{S}}$ by Def.~\ref{def:canonical_query}, it becomes evident that $\mo_i^* \models Q$ for all $\md_i \in \mathbf{S}$. Therefore $\dx{}(Q) \supseteq \mathbf{S} \supset \emptyset$ (cf.\ Prop.~\ref{prop:partition}). 
	
	Second, (*): If $U_{\mathbf{S}} \subset U_\mD$, then $\dnx{}(Q) \neq \emptyset$. 
	To see why (*) holds, we point out that the former implies the existence of a diagnosis $\md_i \in \mD$ which is not in $\mathbf{S}$. 
	Also, there must be some sentence $\tax \in \md_i \subseteq \mo$ where $\tax \notin U_{\mathbf{S}}$. 
	This implies that $\tax \in U_\mD$ because it is in some diagnosis in $\mD$ and that $\tax \notin I_\mD$ because clearly $I_\mD \subseteq U_{\mathbf{S}}$ (due to $\mathbf{S} \supset \emptyset$). 
	Thus, $\tax \in (U_{\mD} \setminus I_{\mD}) = \Disc_\mD$. 
	Since $Q = (\mo \setminus U_{\mathbf{S}}) \cap \Disc_\mD$ by Def.~\ref{def:canonical_query},
	we find that $\tax \in Q$ must hold. 
	So, due to $Q \subseteq \mo$, we have $\mo_i^* \cup Q := (\mo \setminus \md_i) \cup \mb \cup U_\Tp \cup Q = (\mo \setminus \md'_i) \cup \mb \cup U_\Tp$ for some $\md'_i \subseteq \md_i \setminus \setof{\tax} \subset \md_i$. Now, the $\subseteq$-minimality of all diagnoses in $\mD$, and thence in particular of $\md_i$, lets us derive that $\mo_i^* \cup Q$ must violate $\Tn$ or $\RQ$. This means that $\md_i \in \dnx{}(Q)$ (cf.\ Prop.~\ref{prop:partition}) which is why (*) is proven.
	
	Finally, let us assume that $\dnx{}(Q) = \emptyset$. By application of the law of contraposition to (*), this implies that $U_{\mathbf{S}} = U_\mD$. 
	Hence, by Def.~\ref{def:canonical_query}, $Q = (\mo \setminus U_\mD) \cap (U_\mD \setminus I_\mD) \subseteq (\mo \setminus U_\mD) \cap (U_\mD) = \emptyset$, i.e.\ $Q = \emptyset$. From this we get by Def.~\ref{def:canonical_query} that $Q$ is undefined and hence not a canonical query. As a consequence, $\dnx{}(Q) \neq \emptyset$ must hold.
\end{proof}

\subsection*{Proof of Proposition~\ref{prop:1-to-1_relation_between_CQs_and_CQPs}}
\begin{proof}
	From Prop.~\ref{prop:properties_of_q-partitions}.\ref{prop:properties_of_q-partitions:enum:q-partition_is_unique_for_query} we already know that each query $Q$ has one and only one q-partition. Since each CQ $Q$ is in particular a query (Prop.~\ref{prop:canonical_query_is_query}), it must also have one and only one q-partition.
	
	On the other hand, assume that there are two CQs $Q_1 \neq Q_2$ for one and the same CQP $\Pt$. Since the only variable in the computation of a CQ is $U_{\dx{}}$ where $\dx{}$ is the used seed (see Def.~\ref{def:canonical_query}), there must be seeds $\dx{1}$ for $Q_1$ and $\dx{2}$ for $Q_2$ such that $U_{\dx{1}} \neq U_{\dx{2}}$. Hence, there must be at least one diagnosis $\md_k \in \mD$ which is in one, but not in the other set among $\dx{1}$ and $\dx{2}$. Suppose w.l.o.g.\ that $\md_k \in \dx{1}$ and $\md_k \notin \dx{2}$. Now, by extensiveness of $\mathcal{L}$ and Def.~\ref{def:canonical_query} we have that $\mo_k^* = (\mo \setminus \md_k) \cup \mb \cup U_\Tp \models (\mo \setminus \md_k) \supseteq (\mo \setminus U_{\dx{1}}) \supseteq (\mo \setminus U_{\dx{1}}) \cap \Disc_\mD = Q_1$. From this
	we see that $\mo_k^* \models Q_1$. That is, $\md_k \in \dx{}(Q_1)$. 
	Further, there must be a sentence $\tax \in \md_k$ such that $\tax \notin U_{\dx{2}}$. Clearly, $\tax \in U_\mD$ (since $\md_k \in \mD$) and $\tax \notin I_\mD$ (since otherwise $\tax$ would be an element of $U_{\dx{2}}$). Thence, $\tax \in \Disc_\mD$. Because $Q_2 = (\mo \setminus U_{\dx{2}}) \cap (\Disc_\mD)$, we obtain that $\tax \in Q_2$ must hold. 
	Consequently, 
	$\mo_k^* \cup Q_2 \supseteq [(\mo \setminus \md'_k) \cup \mb \cup U_\Tp]$ for $\md'_k \subseteq \md_k \setminus \setof{\tax} \subset \md_k$. Since $\mD$ includes only minimal diagnoses, we can derive that $\mo_k^* \cup Q_2$ violates $\RQ$ or $\Tn$. Thus, $\md_k \in \dnx{}(Q_2)$ (cf.\ Prop.~\ref{prop:partition}) and therefore $\md_k \notin \dx{}(Q_2)$. Overall, we have shown that $\dx{}(Q_1) \neq \dx{}(Q_2)$. This is a contradiction to $\Pt$ being a CQP since this implies that $\dx{}(Q_1) 
	= \dx{}(Q_2) = \dx{}(\Pt)$.
	%
	%
	
	That the unique CQ associated with the CQP $\Pt$ is $Q_{\mathsf{can}}(\dx{}(\Pt))$ follows from the CQ's uniqueness for a CQP demonstrated above and Def.~\ref{def:canonical_q-partition}. 
\end{proof}

\subsection*{Proof of Proposition~\ref{prop:--di,MD-di,0--_is_canonical_q-partition_for_all_di_in_mD}}
\begin{proof}
	That $\langle\{\md_i\},\mD \setminus \{\md_i\},\emptyset\rangle$ is a q-partition follows immediately from Prop.~\ref{prop:properties_of_q-partitions}.\ref{prop:properties_of_q-partitions:enum:D+=d_i_is_q-partition_and_lower_bound_of_queries}. We now show that it is canonical, i.e.\ a CQP. Also from Prop.~\ref{prop:properties_of_q-partitions}.\ref{prop:properties_of_q-partitions:enum:D+=d_i_is_q-partition_and_lower_bound_of_queries}, we obtain that $Q := U_\mD \setminus \md_i$ is an (explicit-entailments) query for the q-partition $\langle \dx{}(Q),\dnx{}(Q),\dz{}(Q)\rangle = \langle\{\md_i\},\mD \setminus \{\md_i\},\emptyset\rangle$. 
	As $U_\mD \subseteq \mo$ and $I_\mD \subseteq \md_i$, we can infer that $Q_{\mathsf{can}}(\dx{}(Q)) = (\mo\setminus U_{\dx{}(Q)}) \cap (\Disc_\mD) = (\mo \setminus \md_i) \cap (U_\mD \setminus I_\mD) = U_\mD \setminus \md_i$. Hence, $Q_{\mathsf{can}}(\dx{}(Q)) = Q$ which is why $\langle \dx{}(Q),\dnx{}(Q),\dz{}(Q)\rangle$ is a CQP.
\end{proof}

\subsection*{Proof of Proposition~\ref{prop:suff+nec_criteria_when_partition_is_q-partition}}
\begin{proof}
	``$\Rightarrow$'': We prove the ``only-if''-direction by contradiction. That is, we derive a contradiction by assuming that $\Pt$ is a canonical \text{q-partition} and that $\neg (1.)$ or $\neg (2.)$ is true.
	
	By the premise that $\Pt = \tuple{\dx{},\dnx{},\dz{}}$ is a canonical q-partition, the query $Q_{\mathsf{can}}(\dx{}) := (\mo \setminus U_{\dx{}}) \cap \Disc_\mD$ must have exactly $\Pt$ as its q-partition, i.e.\ $\dx{}(Q_{\mathsf{can}}(\dx{})) = \dx{}$ and $\dnx{}(Q_{\mathsf{can}}(\dx{})) = \dnx{}$.
	Moreover, $(*): Q_{\mathsf{can}}(\dx{}) \subseteq \mo \setminus U_{\dx{}} \subseteq (\mo \setminus U_{\dx{}}) \cup \mb \cup U_\Tp 
	\subseteq (\mo \setminus \md_i) \cup \mb \cup U_\Tp =: \mo_i^*$ for all $\md_i \subseteq U_{\dx{}}$. 
	
	Now, assuming that $(1.)$ is false, i.e.\ $U_{\dx{}} \not\subset U_{\mD}$, we observe that this is equivalent to $U_{\dx{}} = U_{\mD}$ since $U_{\dx{}} \subseteq U_{\mD}$ due to $\dx{} \subseteq \mD$. Due to $U_{\dx{}} = U_\mD$, $(*)$ is true for all $\md_i\in\mD$. It follows that $\mo_i^* \supseteq Q_{\mathsf{can}}(\dx{})$ and, due to the fact that $\mathcal{L}$ is extensive, that $\mo_i^* \models Q_{\mathsf{can}}(\dx{})$. Therefore, we can conclude that $\dx{} = \dx{}(Q_{\mathsf{can}}(\dx{})) = \mD$ (cf.\ Prop.~\ref{prop:partition}) and thus $\dnx{} = \emptyset$. The latter is a contradiction to $\dnx{} \neq \emptyset$.
	
	Assuming $\neg (2.)$, on the other hand, we obtain that there is some diagnosis $\md_j\in\dnx{}$ with $\md_j \subseteq U_{\dx{}}$. By $(*)$, however, we can derive that $\mo_j^* \models Q_{\mathsf{can}}(\dx{})$ and therefore $\md_j \in \dx{}(Q_{\mathsf{can}}(\dx{})) = \dx{}$ which contradicts $\md_j\in\dnx{}$ by the fact that $\Pt$ is a partition w.r.t.\ $\mD$ which implies $\dx{} \cap \dnx{} = \emptyset$.
	%
	%
	
	``$\Leftarrow$'': To show the ``if''-direction, we must prove that $\Pt$ is a canonical q-partition, i.e.\ that $\Pt$ is a q-partition and that 
	$\Pt$ is exactly the q-partition associated with $Q_{\mathsf{can}}(\dx{})$ given that $(1.)$ and $(2.)$ hold. 
	
	By $\dx{} \neq \emptyset$ and $(1.)$, it is true that $\emptyset \subset U_{\dx{}} \subset U_\mD$. So, there is some sentence $\tax \in U_\mD \subseteq \mo$ such that $(**): \tax \notin U_{\dx{}}$. Hence, $\tax \in \mo \setminus U_{\dx{}}$.
	Clearly, $\tax \notin I_\mD$ since otherwise $\tax$ would be an element of $U_{\dx{}}$. Therefore, $\tax \in (\mo \setminus U_{\dx{}}) \cap (U_\mD \setminus I_\mD) = Q_{\mathsf{can}}(\dx{})$ which is why (Q1): $Q_{\mathsf{can}}(\dx{}) \neq \emptyset$. 
	
	More precisely, since $\tax$ was an arbitrary axiom in $U_\mD$ with property $(**)$, we have that $U_\mD \setminus U_{\dx{}} \subseteq Q_{\mathsf{can}}(\dx{})$. By $(2.)$, for all $\md_j \in \dnx{}$ there is an axiom $\tax_j \in \mo$ such that $\tax_j \in \md_j \subset U_\mD$ and $\tax_j \notin U_{\dx{}}$ which implies $\tax_j \in U_\mD \setminus U_{\dx{}} \subseteq Q_{\mathsf{can}}(\dx{})$. Hence, $\mo_j^* \cup Q_{\mathsf{can}}(\dx{})$ must violate $\RQ$ or $\Tn$ due to the $\subseteq$-minimality of $\md_j \in \dnx{}$. Consequently, $\dnx{} \subseteq \dnx{}(Q_{\mathsf{can}}(\dx{}))$. As $\dnx{} \neq \emptyset$ by assumption, we have that (Q2): $\emptyset \subset \dnx{}(Q_{\mathsf{can}}(\dx{}))$.
	
	That $\mo_i^* \models Q_{\mathsf{can}}(\dx{})$ for $\md_i \in \dx{}$ follows by the same argumentation that was used above in $(*)$. Thus, we obtain $\dx{} \subseteq \dx{}(Q_{\mathsf{can}}(\dx{}))$. As $\dx{} \neq \emptyset$ by assumption, we have that (Q3): $\emptyset \subset \dx{}(Q_{\mathsf{can}}(\dx{}))$.
	We have shown that (Q1), (Q2) and (Q3) hold which altogether imply that $Q_{\mathsf{can}}(\dx{})$ is a query in $\mQ_\mD$ (cf.\ Def.~\ref{def:query_q-partition}).
	
	Now, let us assume that at least one of the two derived subset-relations is proper, i.e.\ (a)~$\dnx{} \subset \dnx{}(Q_{\mathsf{can}}(\dx{}))$ or (b)~$\dx{} \subset \dx{}(Q_{\mathsf{can}}(\dx{}))$. If (a) holds, then there exists some $\md \in \dnx{}(Q_{\mathsf{can}}(\dx{}))$ which is not in $\dnx{}$. Hence, $\md\in\dx{}$ or $\md\in\dz{}$. The former case is impossible since $\md\in\dx{}$ implies $\md\in\dx{}(Q_{\mathsf{can}}(\dx{}))$ by $\dx{} \subseteq \dx{}(Q_{\mathsf{can}}(\dx{}))$ (which was deduced above). From this we obtain that $\dnx{}(Q_{\mathsf{can}}(\dx{})) \cap \dx{}(Q_{\mathsf{can}}(\dx{})) \supseteq \setof{\md} \supset \emptyset$, a contradiction to the fact that $Q_{\mathsf{can}}(\dx{})$ is a query in $\mQ_\mD$ and Prop.~\ref{prop:properties_of_q-partitions}.\ref{prop:properties_of_q-partitions:enum:q-partition_is_partition}.
	The latter case cannot be true either as $\dz{} = \emptyset$ by assumption.
	In an analogue way we obtain a contradiction if we assume that case (b) holds. So, it must hold that $\Pt = \tuple{\dx{}(Q_{\mathsf{can}}(\dx{})),\dnx{}(Q_{\mathsf{can}}(\dx{})),\emptyset}$. This finishes the proof.
\end{proof}

\subsection*{Proof of Corollary~\ref{cor:not_q-partition_iff_md_i^(k)=emptyset_for_md_i_in_dnx_k}}

\begin{proof}
	Ad 1.: This statement follows directly from the definition of $\md_i^{(k)} := \md_i \setminus U_{\dx{k}}$ (see Eq.~\eqref{eq:md_i^(k)}) and the trivial fact that $\md_i \subseteq U_{\dx{k}}$ for all $\md_i \in \dx{k}$. Thence, this proposition can never be false.
	
	Ad 2.: We show the contrapositive of (2.), i.e.\ that $\Pt := \tuple{\dx{k},\dnx{k},\emptyset}$ is not a canonical q-partition iff $\md_i^{(k)} = \emptyset$ for some $\md_i \in \dnx{k}$:
	
	``$\Leftarrow$'': 
	By Prop.~\ref{prop:suff+nec_criteria_when_partition_is_q-partition}, a partition $\Pt_k = \langle \dx{k},\dnx{k},\emptyset\rangle$ with $\dx{k},\dnx{k} \neq \emptyset$ is a canonical q-partition iff (1)~$U_{\dx{k}} \subset U_{\mD}$ and (2)~there is no $\md_j \in \dnx{k}$ such that $\md_j \subseteq U_{\dx{k}}$. If $\md_j \in \dnx{k}$ and $\md_j^{(k)} := \md_j \setminus U_{\dx{k}} = \emptyset$, then $\md_j \subseteq U_{\dx{k}}$, which violates the necessary condition~(2). Therefore, $\Pt_k$ cannot be a canonical q-partition.
	
	``$\Rightarrow$'': By Prop.~\ref{prop:suff+nec_criteria_when_partition_is_q-partition}, a partition $\Pt_k = \langle \dx{k},\dnx{k},\emptyset\rangle$ with $\dx{k},\dnx{k} \neq \emptyset$ is not a canonical q-partition iff ($\neg 1$)~$U_{\dx{k}} \not\subset U_{\mD}$ or ($\neg 2$)~there is some $\md_i \in \dnx{k}$ such that $\md_i \subseteq U_{\dx{k}}$. Since condition ($\neg 1$) is assumed to be false, condition~($\neg 2$) must be true, which implies that $\md_i^{(k)} = \md_i \setminus U_{\dx{k}} = \emptyset$ for some $\md_i \in \dnx{k}$.
\end{proof}

\subsection*{Proof of Corollary~\ref{cor:S_next_sound+complete}}

\begin{proof}
The statement of the corollary is a direct consequence of Lem.~\ref{lem:minimal_transformation_for_D+_partitioning} and Def.~\ref{def:trait}.
\end{proof}

\begin{lemma}\label{lem:minimal_transformation_for_D+_partitioning}
	Let $\mD\subseteq\minD_{\langle\mo,\mb,\Tp,\Tn\rangle_\RQ}$ and $\Pt_k = \langle \dx{k}, \dnx{k}, \emptyset\rangle$ be a canonical q-partition of $\mD$. 
	Then 
	\begin{enumerate}
		\item (\emph{soundness}) $\Pt_k \mapsto \Pt_s$ for a partition $\Pt_s := \langle \dx{s}, \dnx{s}, \emptyset\rangle$ of $\mD$ with $\dx{s} \supseteq \dx{k}$ is a minimal $\dx{}$-transformation if
		\begin{enumerate}
			\item \label{prop:minimal_transformation_for_D+_partitioning:enum:dx_y} $\dx{y} := \dx{k} \cup \setof{\md}$ such that $U_{\dx{y}} \subset U_\mD$ for some $\md \in \dnx{k}$ and $U_{\dx{y}}$ is $\subseteq$-minimal among all $\md\in\dnx{k}$, and
			\item \label{prop:minimal_transformation_for_D+_partitioning:enum:dx_s} $\dx{s} := \{\md_i\,|\,\md_i \in \mD,\md_i^{(y)} = \emptyset\}$ and 
			\item \label{prop:minimal_transformation_for_D+_partitioning:enum:dnx_s} $\dnx{s} := \{\md_i\,|\,\md_i \in \mD,\md_i^{(y)} \neq \emptyset\}$. 
		\end{enumerate} 
		\item (\emph{completeness}) the construction of $\Pt_s$ as per (\ref{prop:minimal_transformation_for_D+_partitioning:enum:dx_y}), (\ref{prop:minimal_transformation_for_D+_partitioning:enum:dx_s}) and (\ref{prop:minimal_transformation_for_D+_partitioning:enum:dnx_s}) yields all possible minimal $\dx{}$-transformations $\Pt_k \mapsto \Pt_s$.
	\end{enumerate}
\end{lemma}
\begin{proof}
Ad 1.: 
By the definition of a minimal $\dx{}$-transformation (Def.~\ref{def:minimal_transformation}), we have to show that (i)~$\Pt_s$ is a canonical q-partition where $\dx{s} \supset \dx{k}$ and that (ii)~there is no canonical q-partition $\langle \dx{l},\dnx{l},\emptyset\rangle$ such that $\dx{k} \subset \dx{l} \subset \dx{s}$. 

Ad (i):
To verify that $\Pt_s$ is indeed a canonical q-partition, we check whether it satisfies the premises, $\dx{s} \neq \emptyset$ and $\dnx{s} \neq \emptyset$, and both conditions of Prop.~\ref{prop:suff+nec_criteria_when_partition_is_q-partition}.
The first condition, 
i.e. $U_{\dx{s}} \subset U_{\mD}$, is met due to the following argumentation. 
First, the inclusion of only diagnoses $\md_i$ with $\md_i^{(y)} = \emptyset$ (and thus $\md_i \subseteq U_{\dx{y}}$) in $\dx{s}$ implies $U_{\dx{s}} \not\supset U_{\dx{y}}$. Further, $\dx{s} \supseteq \dx{y}$ must hold since, trivially, for each $\md_i \in \dx{y}$ it must be true that $\md_i^{(y)} = \emptyset$ wherefore, by (\ref{prop:minimal_transformation_for_D+_partitioning:enum:dx_s}), $\md_i \in \dx{s}$. Hence, $U_{\dx{s}} \supseteq U_{\dx{y}}$ must be given. Combining these findings yields $U_{\dx{s}} = U_{\dx{y}}$. By the postulation of $U_{\dx{y}} \subset U_\mD$ in (\ref{prop:minimal_transformation_for_D+_partitioning:enum:dx_y}), we obtain $U_{\dx{s}} \subset U_\mD$. This finishes the proof of the validity of Prop.~\ref{prop:suff+nec_criteria_when_partition_is_q-partition},(1.).
	
Due to $U_{\dx{s}} \subset U_\mD$, we must have that $\dx{s} \subset \mD$ which implies that $\dnx{s} = \mD \setminus \dx{s} \neq \emptyset$. Moreover, we have seen above that $\dx{s} \supseteq \dx{y}$. By definition of $\dx{y}$ it holds that $\dx{y} \supset \dx{k}$ which is why $\dx{s} \neq \emptyset$. Thence, both premises of Prop.~\ref{prop:suff+nec_criteria_when_partition_is_q-partition} are given.
	
Prop.~\ref{prop:suff+nec_criteria_when_partition_is_q-partition},(2.), i.e. that there is no $\md_i \in \dnx{s}$ such that $\md_i \subseteq U_{\dx{s}}$, is shown next. Each $\md_i \in \mD$ with $\md_i \subseteq U_{\dx{s}}$ fulfills $\md_i^{(y)} = \emptyset$ by $U_{\dx{s}} = U_{\dx{y}}$ which we derived above. Thus, by the definition of $\dx{s}$ and $\dnx{s}$ in (\ref{prop:minimal_transformation_for_D+_partitioning:enum:dx_s}) and (\ref{prop:minimal_transformation_for_D+_partitioning:enum:dnx_s}), respectively, each $\md_i \in \mD$ with $\md_i \subseteq U_{\dx{s}}$ must be an element of $\dx{s}$ and cannot by an element of $\dnx{s}$. 
This finishes the proof of Prop.~\ref{prop:suff+nec_criteria_when_partition_is_q-partition},(2.).
Hence, $\Pt_s$ is a canonical q-partition.
	
Moreover, since $\dx{y} := \dx{k} \cup \setof{\md}$ for some diagnosis, we obtain that $\dx{y} \supset \dx{k}$. But, before we argued that $\dx{s} \supseteq \dx{y}$. All in all, this yields $\dx{s} \supset \dx{k}$.
This finishes the proof of (i).

Ad (ii): 
To show the minimality of the transformation $\Pt_k \mapsto \Pt_s$, let us assume that there is some canonical q-partition $\Pt_l := \langle \dx{l},\dnx{l}, \emptyset\rangle$ with $\dx{k} \subset \dx{l} \subset \dx{s}$. 
From this, we immediately obtain that $U_{\dx{l}} \subseteq U_{\dx{s}}$ must hold. Furthermore, we have shown above that $U_{\dx{s}} = U_{\dx{y}}$. Due to the fact that $\Pt_s$ is already uniquely defined as per (\ref{prop:minimal_transformation_for_D+_partitioning:enum:dx_s}) and (\ref{prop:minimal_transformation_for_D+_partitioning:enum:dnx_s}) given $U_{\dx{y}} = U_{\dx{s}}$ and since $\dx{l} \neq \dx{s}$, we conclude that $U_{\dx{l}} \subset U_{\dx{s}}$. Thence, $U_{\dx{l}} \subset U_{\dx{y}}$. Additionally, by (\ref{prop:minimal_transformation_for_D+_partitioning:enum:dx_y}), for all diagnoses $\md\in\dnx{k}$ it must hold that $U_{\dx{k} \cup \setof{\md}} \not\subset U_{\dx{y}}$. However, as $\dx{k} \subset \dx{l}$, there must be at least one diagnosis among those in $\dnx{k}$ which is an element of $\dx{l}$. If there is exactly one such diagnosis $\md^*$, then we obtain a contradiction immediately as $U_{\dx{l}} = U_{\dx{k} \cup \setof{\md^*}} \not\subset U_{\dx{y}}$. Otherwise, we observe that, if there is a set $\mD'\subseteq\dnx{k}$ of multiple such diagnoses, then there is a single diagnosis $\md' \in \mD' \subseteq \dnx{k}$ such that $U_{\dx{l}} = U_{\dx{k} \cup \mD'} \supseteq U_{\dx{k} \cup \setof{\md'}}$ wherefore we can infer that $U_{\dx{l}} \not\subset U_{\dx{y}}$ must hold. Consequently, the transformation $\Pt_k \mapsto \Pt_s$ is indeed minimal and (ii) is proven.

Ad 2.: 
Assume that $\Pt_k \mapsto \Pt_s$ is a minimal $\dx{}$-transformation and that $\Pt_s$ cannot be constructed as per (\ref{prop:minimal_transformation_for_D+_partitioning:enum:dx_y}), (\ref{prop:minimal_transformation_for_D+_partitioning:enum:dx_s}) and (\ref{prop:minimal_transformation_for_D+_partitioning:enum:dnx_s}). 

By Def.~\ref{def:minimal_transformation}, $\Pt_s$ is a canonical q-partition. Since it is a q-partition, we have that $\dx{s} \neq \emptyset$ and $\dnx{s} \neq \emptyset$. Thence, by Prop.~\ref{prop:suff+nec_criteria_when_partition_is_q-partition}, $U_{\dx{s}} \subset U_\mD$ which is why, by Cor.~\ref{cor:not_q-partition_iff_md_i^(k)=emptyset_for_md_i_in_dnx_k}, there must be some $\dx{y}$ (e.g.\ $\dx{s}$) such that $\dx{s} := \{\md_i\,|\,\md_i \in \mD,\md_i^{(y)} = \emptyset\}$ and $\dnx{s} := \{\md_i\,|\,\md_i \in \mD,\md_i^{(y)} \neq \emptyset\}$. Thence, for each minimal $\dx{}$-transformation there is some $\dx{y}$ such that (\ref{prop:minimal_transformation_for_D+_partitioning:enum:dx_s}) $\land$ (\ref{prop:minimal_transformation_for_D+_partitioning:enum:dnx_s}) is true wherefore we obtain that $\lnot$(\ref{prop:minimal_transformation_for_D+_partitioning:enum:dx_y}) must be given. That is, at least one of the following must be false: (i)~there is some $\md \in \dnx{k}$ such that $\dx{y} = \dx{k} \cup \setof{\md}$, (ii)~$U_{\dx{y}}\subset U_\mD$, (iii)~$U_{\dx{y}}$ is $\subseteq$-minimal among all $\md \in \dnx{k}$.

First, we can argue analogously as done in the proof of (1.)\ above that $U_{\dx{y}} = U_{\dx{s}}$ must hold. This along with Prop.~\ref{prop:suff+nec_criteria_when_partition_is_q-partition} entails that $U_{\dx{y}} \subset U_\mD$ cannot be false. So, (ii) cannot be false.

Second, assume that (i) is false. 
That is, no set $\dx{y}$ usable to construct $\dx{s}$ and $\dnx{s}$ as per (\ref{prop:minimal_transformation_for_D+_partitioning:enum:dx_s}) and (\ref{prop:minimal_transformation_for_D+_partitioning:enum:dnx_s}) is defined as $\dx{y} = \dx{k} \cup \setof{\md}$ for any $\md \in \dnx{k}$. But, clearly, $\dx{s}$ is one such set $\dx{y}$. And, $\dx{s} \supset \dx{k}$ as a consequence of $\Pt_k \mapsto \Pt_s$ being a minimal $\dx{}$-transformation. Hence, there is some set $\dx{y} \supset \dx{k}$ usable to construct $\dx{s}$ and $\dnx{s}$ as per (\ref{prop:minimal_transformation_for_D+_partitioning:enum:dx_s}) and (\ref{prop:minimal_transformation_for_D+_partitioning:enum:dnx_s}). Now, if $\dx{y} = \dx{k} \cup \setof{\md}$ for some diagnosis $\md \in \dnx{k}$, then we have a contradiction to $\lnot$(i). Thus, $\dx{y} = \dx{k} \cup \mathbf{S}$ where $\mathbf{S} \subseteq \dnx{k}$ with $|\mathbf{S}| \geq 2$ must hold. In this case, there is some $\md \in \dnx{y}$ such that $\dx{y} \supset \dx{k} \cup \setof{\md}$ and therefore $U_{\dx{y}} \supseteq U_{\dx{k}\cup\setof{\md}}$.  
Let $\Pt_{s'}$ be the partition induced by $\dx{y'} := \dx{k}\cup\setof{\md}$ as per (\ref{prop:minimal_transformation_for_D+_partitioning:enum:dx_s}) and (\ref{prop:minimal_transformation_for_D+_partitioning:enum:dnx_s}). 

This partition $\Pt_{s'}$ is a canonical q-partition due to Cor.~\ref{cor:not_q-partition_iff_md_i^(k)=emptyset_for_md_i_in_dnx_k}. The latter is applicable in this case, first, by reason of $\dnx{s}\neq \emptyset$ (which means that $\mD \supset \dx{s}$) and $\dx{s} \supseteq \dx{s'} \supseteq \dx{y'} \supset \dx{k} \supset \emptyset$ (where the first two superset-relations hold due to (\ref{prop:minimal_transformation_for_D+_partitioning:enum:dx_s}), (\ref{prop:minimal_transformation_for_D+_partitioning:enum:dnx_s}) and Eq.~\eqref{eq:md_i^(k)}, and the last one since $\Pt_k$ is a q-partition by assumption), which lets us derive $\dx{s'} \neq \emptyset$ and $\dnx{s'} \neq \emptyset$. Second, from the said superset-relations and Prop.~\ref{prop:suff+nec_criteria_when_partition_is_q-partition} along with $\Pt_s$ being a canonical q-partition, we get $U_\mD \supset U_{\dx{s}} \supseteq U_{\dx{s'}}$.

But, due to $\dx{s'} \subseteq \dx{s}$ we can conclude that either $\Pt_k \mapsto \Pt_s$ is not a minimal $\dx{}$-transformation (case $\dx{s'} \subset \dx{s}$) or $\Pt_s$ can be constructed by means of $\dx{k}\cup\setof{\md}$ for some $\md \in \dnx{k}$ (case $\dx{s'} = \dx{s}$). The former case is a contradiction to the assumption that $\Pt_k \mapsto \Pt_s$ is a minimal $\dx{}$-transformation. The latter case is a contradiction to 
$|\mathbf{S}| \geq 2$. Consequently, (i) cannot be false.

Third, as (i) and (ii) have been shown to be true, we conclude that (iii) must be false.
Due to the truth of (i), we can assume that $\dx{y}$ used to construct $\Pt_s$ can be written as $\dx{k}\cup\setof{\md}$ for some $\md \in \dnx{k}$. Now, if $U_{\dx{y}}$ is not $\subseteq$-minimal among all $\md \in \dnx{k}$, then there is some $\md' \in \dnx{k}$ such that $U_{\dx{k}\cup\setof{\md'}} \subset U_{\dx{y}}$. 
Further, due to $U_{\dx{y}} = U_{\dx{s}} \subset U_\mD$ (because of Prop.~\ref{prop:suff+nec_criteria_when_partition_is_q-partition} and the fact that $\Pt_s$ is a canonical q-partition), we get $U_{\dx{k}\cup\setof{\md'}} \subset U_\mD$.

Let $\Pt_{s'}$ be the partition induced by $\dx{y'} := \dx{k}\cup\setof{\md'}$ as per (\ref{prop:minimal_transformation_for_D+_partitioning:enum:dx_s}) and (\ref{prop:minimal_transformation_for_D+_partitioning:enum:dnx_s}). It is guaranteed that $\Pt_{s'}$ is a canonical q-partition due to Cor.~\ref{cor:not_q-partition_iff_md_i^(k)=emptyset_for_md_i_in_dnx_k}. The first reason why the latter is applicable here is $U_{\dx{k}\cup\setof{\md'}} \subset U_\mD$. The second one is $\md \notin \dx{s'}$ which implies $\dx{s'} \neq \mD$ and thus $\dnx{s'} \neq \emptyset$, and $\md' \in \dx{s'}$ (due to (\ref{prop:minimal_transformation_for_D+_partitioning:enum:dx_s})) which means that $\dx{s'} \neq \emptyset$. The fact $\md \notin \dx{s'}$ must hold due to $\md \not\subseteq U_{\dx{k}\cup\setof{\md'}}$. To realize that the latter holds, assume the opposite, i.e.\ $\md \subseteq U_{\dx{k}\cup\setof{\md'}}$. Then, since $U_{\dx{k}} \subseteq U_{\dx{k}\cup\setof{\md'}}$ and $U_{\dx{k}\cup\setof{\md}} = U_{\dx{k}} \cup \md$, we obtain that $U_{\dx{k}\cup\setof{\md'}} \supseteq U_{\dx{k}\cup\setof{\md}} = U_{\dx{y}}$, a contradiction. 
So, both $\Pt_{s'}$ and $\Pt_{s}$ are canonical q-partitions. However, since $\Pt_{s'}$ is constructed as per (\ref{prop:minimal_transformation_for_D+_partitioning:enum:dx_s}),(\ref{prop:minimal_transformation_for_D+_partitioning:enum:dnx_s}) by means of $U_{\dx{k}\cup\setof{\md'}}$ and $\Pt_{s}$ as per (\ref{prop:minimal_transformation_for_D+_partitioning:enum:dx_s}),(\ref{prop:minimal_transformation_for_D+_partitioning:enum:dnx_s}) by means of $U_{\dx{y}}$ and since $U_{\dx{k}\cup\setof{\md'}} \subset U_{\dx{y}}$, it must hold that $\dx{s} \supseteq \dx{s'}$. In addition, we observe that $\md \in \dx{s}$ (due to $\md \subseteq U_{\dx{k}\cup\setof{\md}} = U_{\dx{k}} \cup \md$), but $\md \notin \dx{s'}$ (as shown above). Thence, $\md \in \dx{s}\setminus\dx{s'}$ which is why $\dx{s} \supset \dx{s'}$. This, however, constitutes a contradiction to the assumption that $\Pt_k \mapsto \Pt_s$ is a minimal $\dx{}$-transformation. Consequently, (iii) must be true. 

Altogether, we have shown that neither (i) nor (ii) nor (iii) can be false. The conclusion is that (\ref{prop:minimal_transformation_for_D+_partitioning:enum:dx_y}) and (\ref{prop:minimal_transformation_for_D+_partitioning:enum:dx_s}) and (\ref{prop:minimal_transformation_for_D+_partitioning:enum:dnx_s}) must hold for $\Pt_k \mapsto \Pt_s$, a contradiction.
\end{proof}

\subsection*{Proof of Corollary~\ref{cor:upper_lower_bound_for_canonical_q-partitions}}

\begin{proof}
	The inequality holds due to Prop.~\ref{prop:--di,MD-di,0--_is_canonical_q-partition_for_all_di_in_mD}. 
	We know by Prop.~\ref{prop:canonical_q-partition_has_empty_dz} that a canonical q-partition has empty $\dz{}$. Thus, there can be only one canonical q-partition for one set $\dx{}$. Further, since a canonical q-partition is a q-partition, $\dx{} \neq \emptyset$ and $\dnx{} \neq \emptyset$ (Prop.~\ref{prop:properties_of_q-partitions}.\ref{prop:properties_of_q-partitions:enum:for_each_q-partition_dx_is_empty_and_dnx_is_empty}). Thus $\dx{}$ must neither be the empty set nor equal to $\mD$. By Prop.~\ref{prop:suff+nec_criteria_when_partition_is_q-partition}, $U_{\dx{}} \subset U_{\mD}$ must hold for each canonical q-partition. By Cor.~\ref{cor:not_q-partition_iff_md_i^(k)=emptyset_for_md_i_in_dnx_k}, there is a unique canonical q-partition for each set $U_{\dx{}}$, i.e.\ we must count each $U_{\dx{}}$ only once. Further, different sets $U_{\dx{i}} \neq U_{\dx{j}}$ clearly imply different sets $\dx{i}$ and $\dx{j}$ and thus different canonical q-partitions (Cor.~\ref{cor:not_q-partition_iff_md_i^(k)=emptyset_for_md_i_in_dnx_k}), i.e.\ we do not count any canonical q-partition twice. Hence, we must count exactly all different $U_{\dx{}}$ such that $U_{\dx{}} \subset U_{\mD}$. This is exactly what Eq.~\eqref{eq:number_of_canonical_q-partitions} specifies.
\end{proof}

\subsection*{Proof of Proposition~\ref{prop:explicit-ents_query_lower+upper_bound}}

\begin{proof}
	The left set inclusion follows directly from Lem.~\ref{lem:explicit-ents_query_lower_bound}. The right set inclusion is a consequence of Lem.~\ref{lem:explicit-ents_query_upper_bound}, which states that $Q \subseteq U_\mD \setminus U_{\dx{}}$, and Lem.~\ref{lem:CQ_equal_to_U_D_setminus_U_D+}, which testifies that $Q_{\mathsf{can}}(\dx{}) = U_\mD \setminus U_{\dx{}}$.
\end{proof}

\begin{lemma}\label{lem:explicit-ents_query_lower_bound} 
	Let $\mD \subseteq \minD_{\tuple{\mo,\mb,\Tp,\Tn}_\RQ}$, $\Pt = \langle \dx{}, \dnx{}, \emptyset\rangle$ be a q-partition w.r.t.\ $\mD$ and $Q \subseteq \Disc_\mD$ an (explicit-entailments) query for $\Pt$. 
	Then $Q' \subseteq Q$ is a query with q-partition $\Pt$, i.e.\ $\Pt_\mD(Q') = \Pt$, iff $Q' \cap \md_i \neq \emptyset$ for each $\md_i \in \dnx{}$.
\end{lemma}
\begin{proof}
	``$\Leftarrow$'': Proof by contraposition. Assume there is a $\md_i \in \dnx{}$ such that $Q' \cap \md_i = \emptyset$. Then $\mo_i^* = (\mo \setminus \md_i)\cup\mb\cup U_\Tp \supseteq Q'$ since $\mo \setminus \md_i \supseteq \Disc_\mD \setminus \md_i \supseteq Q'$. From this $\mo_i^* \models Q'$ follows by the fact that the entailment relation in $\mathcal{L}$ is extensive. As a result, we have that $\md_i \in \dx{}(Q')$. Consequently, as $\md_i \in \dnx{}$, the q-partition $\Pt_\mD(Q')$ of $Q'$ must differ from the q-partition $\Pt$ of $Q$.
	
	``$\Rightarrow$'': Proof by contradiction. Assume that $Q' \subseteq Q$ is a query with q-partition $\Pt$ and $Q' \cap \md_i = \emptyset$ for some $\md_i \in \dnx{}$. Then $(\mo \setminus \md_i) \cup Q' = (\mo \setminus \md_i)$ since $Q' \subseteq Q \subseteq \Disc_\mD \subseteq \mo$ and $Q'\cap\md_i = \emptyset$. Therefore $\mo_i^* \cup Q' = \mo_i^*$ which implies that $Q' \subseteq \mo_i^*$ and thus $\mo_i^* \models Q'$ due the extensive entailment relation in $\mathcal{L}$. Consequently, $\md_i \in \dx{}(Q')$ must hold. Since $\md_i \in \dnx{}$, we can derive that the q-partition $\Pt_\mD(Q')$ of $Q'$ is not equal to the q-partition $\Pt$ of $Q$, a contradiction.
\end{proof}

\begin{lemma}\label{lem:explicit-ents_query_upper_bound}
	Let $\mD \subseteq \minD_{\tuple{\mo,\mb,\Tp,\Tn}_\RQ}$, $\Pt = \langle \dx{}, \dnx{}, \emptyset\rangle$ be a q-partition w.r.t.\ $\mD$ and $Q \subseteq \Disc_\mD$ an (explicit-entailments) query associated with $\Pt$. 
	Then $Q'$ with $\Disc_\mD \supseteq Q' \supseteq Q$ is a query with q-partition $\Pt$, i.e.\ $\Pt_\mD(Q') = \Pt$, iff $Q' \subseteq U_\mD \setminus U_{\dx{}}$.
\end{lemma}
\begin{proof}
	``$\Rightarrow$'': Proof by contraposition. If $Q' \not\subseteq U_\mD \setminus U_{\dx{}}$ then there is an axiom $\tax \in Q'$ such that $\tax \notin U_\mD \setminus U_{\dx{}}$. This implies that $\tax \in U_{\dx{}}$ because $\tax \in Q' \subseteq \Disc_\mD = U_\mD \setminus I_\mD$ which means in particular that $\tax \in U_\mD$. Consequently, $\tax \in \md_j$ for some diagnosis $\md_j \in \dx{}$ must apply which is why $\mo_j^* \cup Q'$ must violate $\RQ$ or $\Tn$ due to the $\subseteq$-minimality of $\md_j$. As a result, $\md_j$ must belong to $\dnx{}(Q')$ and since $\md_j \not\in \dnx{}$, we obtain that the q-partition $\Pt_\mD(Q')$ of $Q'$ is different from $\Pt$.
	
	``$\Leftarrow$'': Direct proof. If $Q' \supseteq Q$ and $Q' \subseteq U_\mD \setminus U_{\dx{}}$, then for each $\md_i \in \dx{}$ it holds that $\mo_i^* \models Q'$ by the fact that the entailment relation in $\mathcal{L}$ is extensive and as $Q' \subseteq U_\mD \setminus U_{\dx{}} \subseteq \mo \setminus \md_i \subseteq \mo_i^*$. Hence, each $\md_i \in \dx{}$ is an element of $\dx{}(Q')$. 
	
	For each $\md_j \in \dnx{}$, $\mo_j^* \cup Q'$ must violate $\RQ$ or $\Tn$ by the monotonicity of the entailment relation in $\mathcal{L}$, since $\mo_j^* \cup Q$ violates $\RQ$ or $\Tn$, and because $Q' \supseteq Q$. Thus, each $\md_j \in \dnx{}$ is an element of $\dnx{}(Q')$. 
	
	So far, we have shown that $\dx{} \subseteq \dx{}(Q')$ as well as $\dnx{} \subseteq \dnx{}(Q')$. To complete the proof, assume that that some of these set-inclusions is proper, e.g.\ $\dx{} \subset \dx{}(Q')$. In this case, by $\dz{} = \emptyset$, we can deduce that there is some $\md \in \dnx{}$ such that $\md \in \dx{}(Q')$. This is clearly a contradiction to the fact that $\dnx{} \subseteq \dnx{}(Q')$ and the disjointness of the sets $\dx{}(Q')$ and $\dnx{}(Q')$ which must hold by Prop.~\ref{prop:properties_of_q-partitions}.\ref{prop:properties_of_q-partitions:enum:q-partition_is_partition}. The other case $\dnx{} \subset \dnx{}(Q')$ can be led to a contradiction in an analogue way. Hence, we conclude that $\Pt_\mD(Q') = \Pt$.
\end{proof}

\begin{lemma}\label{lem:CQ_equal_to_U_D_setminus_U_D+}
Let $\dpi = \tuple{\mo,\mb,\Tp,\Tn}_\RQ$ be a DPI, $\mD \subseteq \minD_{\dpi}$ and $\emptyset\subset\dx{}\subset\mD$. Then the canonical query $Q_{\mathsf{can}}(\dx{})$ w.r.t.\ the seed $\dx{}$ is equal to $U_\mD \setminus U_{\dx{}}$ which is in turn equal to the union of all traits of diagnoses in $\dnx{} = \mD \setminus \dx{}$.
\end{lemma}
\begin{proof}
$Q_{\mathsf{can}}(\dx{}) := \Disc_\mD \cap (\mo \setminus U_{\dx{}}) = (U_\mD \setminus I_\mD) \cap (\mo \setminus U_{\dx{}}) = (U_\mD \cap \mo) \setminus (I_{\mD} \cup U_{\dx{}}) = U_\mD \setminus U_{\dx{}}$ where the last equality holds due to $U_{\mD} \subseteq \mo$ ($U_{\mD}$ is a union of diagnoses and diagnoses are subsets of $\mo$, cf.\ Def.~\ref{def:diagnosis}) and $I_{\mD} \subseteq U_{\dx{}}$ ($I_{\mD}$ is the intersection of all diagnoses in $\mD$, hence a subset of all diagnoses in $\mD$ and in particular of the ones in $\dx{} \subset \mD$, hence a subset of the union $U_{\dx{}}$ of diagnoses in $\dx{}$). The equality of $U_\mD \setminus U_{\dx{}}$ to the union of all traits of diagnoses in $\dnx{} = \mD \setminus \dx{}$ follows directly from Def.~\ref{def:trait}.
\end{proof}

\subsection*{Proof of Proposition~\ref{prop:each_expl_ents_query_has_CQP_as_q-partition}}

\begin{proof}
Assume the opposite, i.e.\ that the q-partition $\Pt_\mD(Q) = \tuple{\dx{}(Q),\dnx{}(Q),\dz{}(Q)}$ is not canonical. Due to $Q \subseteq \mo$, Def.~\ref{def:query_q-partition}, Prop.~\ref{prop:properties_of_q-partitions}.\ref{prop:properties_of_q-partitions:enum:query_is_set_of_common_ent} and the $\subseteq$-minimality of all $\md \in \mD$, $Q$ must be a non-empty set of common explicit entailments of all $\mo \setminus \md_i$ for $\md_i\in\dx{}(Q)$. That is, $\emptyset \subset Q \subseteq \mo \setminus U_{\dx{}(Q)}$. Due to Prop.~\ref{prop:expl_ent_query_must_neednot_mustnot_include_ax}, $Q \cap I_\mD = \emptyset$. Hence, by Prop.~\ref{prop:expl_ent_query_must_neednot_mustnot_include_ax}, $Q' := Q \setminus (\mo \setminus U_\mD) = Q \cap U_\mD = (Q \cap U_\mD)\setminus I_\mD = Q \cap (U_\mD\setminus I_\mD) = Q \cap \Disc_\mD$ has the same q-partition as $Q$, i.e.\ $\Pt_\mD(Q') = \Pt_\mD(Q)$. Further, we observe from these equalities that $Q' \subseteq Q$ and $Q' \subseteq \Disc_\mD$ must hold.
So, the canonical query $Q_{\mathsf{can}}(\dx{}(Q)) = (\mo \setminus U_{\dx{}(Q)}) \cap \Disc_\mD \supseteq Q'$. Hence, $\dx{}(Q_{\mathsf{can}}(\dx{}(Q))) \subseteq \dx{}(Q') = \dx{}(Q)$ as each KB that entails $Q_{\mathsf{can}}(\dx{}(Q))$ must also entail its subset $Q'$. If $\dx{}(Q_{\mathsf{can}}(\dx{}(Q))) = \dx{}(Q)$, then both $Q_{\mathsf{can}}(\dx{}(Q))$ and $Q$ have the same q-partition due to Prop.~\ref{prop:explicit-entailments_queries_have_empty_dz} and since both are explicit-entailments queries. This means that $\Pt_\mD(Q)$ is canonical due to Def.~\ref{def:canonical_q-partition} -- contradiction. Otherwise, $\dx{}(Q_{\mathsf{can}}(\dx{}(Q))) \subset \dx{}(Q)$. That is, there must be some $\md_i \in \dx{}(Q)$ such that $\mo_i^* \not\models Q_{\mathsf{can}}(\dx{}(Q))$. But, this is a contradiction to $\mo_i^* = (\mo \setminus \md_i) \cup \mb \cup U_\Tp \supseteq (\mo \setminus \md_i) \supseteq \mo \setminus U_{\dx{}(Q)} \supseteq \mo \setminus U_{\dx{}(Q)} \cap \Disc_\mD = Q_{\mathsf{can}}(\dx{}(Q))$ due to the extensiveness of $\mathcal{L}$.
\end{proof}

\subsection*{Proof of Theorem~\ref{theorem:P1+P2_solve_query_optimization_problem}}

\begin{proof}
	Due to Prop.~\ref{prop:S_init_sound+complete} and Cor.~\ref{cor:S_next_sound+complete}, phase P1 (with threshold $t_m := 0$) finds the 
	optimal CQP $\Pt^*$ w.r.t.\ the given QSM $m$ for the leading diagnoses $\mD$. As a consequence of Prop.~\ref{prop:each_expl_ents_query_has_CQP_as_q-partition}, $\Pt^*$ is the optimal q-partition of \emph{all} explicit-entailments queries $Q \in \mathbf{S}$. Therefore, $\mathbf{OptQ}(m,\mathbf{S})$ (see Prob.~\ref{prob:query_optimization}) is given by $\setof{Q \mid Q\subseteq \mo, \Pt_\mD(Q) = \Pt^*}$, i.e.\ each explicit-entailments query with q-partition $\Pt^*$ optimizes the QSM $m$ over all queries in $\mathbf{S}$. Due to Prop.~\ref{prop:phase_P2_returns_EE-queries_in_best-first_order}, the query returned by phase P2 is $Q^* = \argmin_{Q \in \mathbf{OptQ}(m,\mathbf{S})} c(Q)$.
\end{proof}

\subsection*{Proof of Proposition~\ref{prop:enrichment_function}}
\begin{proof}
	Ad \ref{enum:EQ1}.: The function $Ent_{\mathit{ET}}$ either does or does not compute explicit entailments (amongst other entailments). In case the function $Ent_{\mathit{ET}}$ does not compute explicit entailments, $Q_{\mathsf{exp}}$ clearly cannot contain any elements of $\mo \cup \mb \cup U_\Tp$. Otherwise, we distinguish between explicit entailments in $(\mo\setminus U_{\mD}) \cup \mb \cup U_{\Tp}$ and those in $Q$ (clearly, there cannot be any other explicit entailments in $Q_{\mathsf{exp}}$). Note that $Q \subseteq U_{\mD} \setminus U_{\dx{}} \subseteq U_\mD$ since due to Prop.~\ref{prop:explicit-ents_query_lower+upper_bound} and Lem.~\ref{lem:CQ_equal_to_U_D_setminus_U_D+}. 
	Additionally, $Q \cap \mb = \emptyset$ due to $Q \subseteq \mo$ and Def.~\ref{def:dpi}. And, $Q \cap U_{\Tp} = \emptyset$ due to $Q \subseteq U_\mD$ and since no element of any minimal diagnosis $\md$ (in $\mD$), and hence no element in $U_\mD$, can occur in $U_{\Tp}$. The latter holds as in case $\md' \cap U_{\Tp} \neq \emptyset$ for $\md' \in \mD$ we would have that $\md'' := \md' \setminus U_{\Tp} \subset \md'$ is a diagnosis w.r.t.\ $\tuple{\mo,\mb,\Tp,\Tn}_\RQ$, a contradiction to the $\subseteq$-minimality of $\md'$. All in all, we have derived that $(\mo\setminus U_{\mD}) \cup \mb \cup U_{\Tp}$ and $Q$ are disjoint sets. 
	
	Now, $Q_{\mathsf{exp}}$ cannot include any elements of $(\mo\setminus U_{\mD}) \cup \mb \cup U_{\Tp}$. This must be satisfied since, first, $(\mo\setminus U_{\mD}) \cup \mb \cup U_{\Tp}$ is a subset of the left- as well as right-hand $Ent_{\mathit{ET}}()$ expression in the definition of $Q_{\mathsf{exp}}$ (Eq.~\eqref{eq:Q_exp}) and, second, both $Ent_{\mathit{ET}}()$ expressions must return the same set of entailments of 
	$(\mo\setminus U_{\mD}) \cup \mb \cup U_{\Tp}$ by property \ref{enum:Ent_ET:ents_generated_for_set_are_genereated_for_all_supersets} made about the operator $Ent_{\mathit{ET}}$. Therefore, the set defined by the squared brackets in Eq.~\eqref{eq:Q_exp} cannot include any (explicit) entailments of $(\mo\setminus U_{\mD}) \cup \mb \cup U_{\Tp}$. 
	
	Further on, $Q_{\mathsf{exp}}$ cannot contain any elements of $Q$. This is guaranteed by the elimination of all elements of $Q$ from the set defined by the squared brackets in Eq.~\eqref{eq:Q_exp}. Finally, we summarize that $Q_{\mathsf{exp}} \cap (\mo \cup \mb \cup U_{\Tp}) = \emptyset$.

	Ad \ref{enum:EQ2}.: Clearly, by the definition of a diagnosis (Def.~\ref{def:diagnosis}), $(\mo\setminus \md) \cup \mb \cup U_{\Tp}$ is a solution KB w.r.t.\ $\tuple{\mo,\mb,\Tp,\Tn}_\RQ$ for all $\md \in \mD$. In addition, since $\dx{} \neq \emptyset$ (cf.\ Prop.~\ref{prop:properties_of_q-partitions}.\ref{prop:properties_of_q-partitions:enum:for_each_q-partition_dx_is_empty_and_dnx_is_empty}), there must be some diagnosis $\md' \in \dx{} \subset \mD$ such that $(\mo\setminus \md') \cup \mb \cup U_{\Tp} \models Q$. This implies that $(\mo\setminus \md') \cup \mb \cup U_{\Tp} \cup Q$ is a solution KB w.r.t.\ $\tuple{\mo,\mb,\Tp,\Tn}_\RQ$ since $\mathcal{L}$ is idempotent. By $U_\mD \supseteq \md'$ and by the monotonicity of the logic $\mathcal{L}$ we conclude that $S := (\mo\setminus U_\mD) \cup \mb \cup U_{\Tp} \cup Q$ is a solution KB w.r.t.\ $\tuple{\mo,\mb,\Tp,\Tn}_\RQ$. 
	%
	
	Obviously, $S \supseteq Q$. Moreover, $S \subseteq \mo \cup \mb \cup U_\Tp$ because $Q \subseteq \Disc_\mD \subseteq \mo$. Finally, by the left-hand $Ent_{\mathit{ET}}()$ expression in Eq.~\eqref{eq:Q_exp}, we obtain that $S \models Q_{\mathsf{exp}}$. 
	%
	
	Ad \ref{enum:EQ3}.: Assume that $S$ is as defined in the proof of (\ref{enum:EQ2}.)\ above and that there is some $\alpha_i \in Q_{\mathsf{exp}}$ such that $S \setminus Q \models \alpha_i$. Then, $(\mo\setminus U_\mD) \cup \mb \cup U_{\Tp} \models \alpha_i$. However, in the proof of (\ref{enum:EQ1}.)\ above we have derived that $Q_{\mathsf{exp}}$ cannot comprise any entailments of $(\mo\setminus U_{\mD}) \cup \mb \cup U_{\Tp}$. Hence, $\alpha_i \notin Q_{\mathsf{exp}}$, contradiction.
	
	Ad \ref{enum:EQ4}.: This property must be met since $Ent_{\mathit{ET}}$ satisfies the type soundness condition~\ref{enum:Ent_ET:type_soundness}.
	
	We sum up that (\ref{enum:EQ1}.)-(\ref{enum:EQ4}.)\ holds for $Q_{\mathsf{exp}}$.
\end{proof}

\subsection*{Proof of Proposition~\ref{prop:entailment_extraction_is_q-partition_preserving}}
\begin{proof}
	Let $\md \in \dx{}(Q)$. Then, $(\mo\setminus \md) \cup \mb \cup U_{\Tp} \models Q$. Since the entailment relation in $\mathcal{L}$ is idempotent, we have that (*): $(\mo\setminus \md) \cup \mb \cup U_{\Tp} \cup Q \equiv (\mo\setminus \md) \cup \mb \cup U_{\Tp}$. Further, since $Q_{\mathsf{exp}}$ is a set of entailments of $(\mo\setminus U_\mD) \cup \mb \cup U_{\Tp} \cup Q$ (see left-hand $Ent_{\mathit{ET}}(.)$ expression in Eq.~\eqref{eq:Q_exp}), by the monotonicity of the entailment relation in $\mathcal{L}$ and because of $(\mo\setminus \md) \cup \mb \cup U_{\Tp} \cup Q \supseteq (\mo\setminus U_\mD) \cup \mb \cup U_{\Tp} \cup Q$ we deduce that $(\mo\setminus \md) \cup \mb \cup U_{\Tp} \cup Q \models Q_{\mathsf{exp}}$. 
	By (*), we have that $(\mo\setminus \md) \cup \mb \cup U_{\Tp} \models Q_{\mathsf{exp}}$. Due to $(\mo\setminus \md) \cup \mb \cup U_{\Tp} \models Q$ the KB $(\mo\setminus \md) \cup \mb \cup U_{\Tp}$ must entail $Q' = Q_{\mathsf{exp}}\cup Q$. Thus, $\md \in \dx{}(Q')$ holds.
	
	Let $\md \in \dnx{}(Q)$. Then, $(\mo\setminus \md) \cup \mb \cup U_{\Tp} \cup Q$ violates some $x \in \RQ \cup \Tn$. Due to the monotonicity of $\mathcal{L}$ and the fact that $Q' = Q_{\mathsf{exp}} \cup Q \supseteq Q$, we immediately obtain that $(\mo\setminus \md) \cup \mb \cup U_{\Tp} \cup Q'$ violates some $x \in \RQ \cup \Tn$. Thus, $\md \in \dnx{}(Q')$. 
	
	Up to this point, we have demonstrated that $\dx{}(Q) \subseteq \dx{}(Q')$ as well as $\dnx{}(Q) \subseteq \dnx{}(Q')$. Since $Q \subseteq \Disc_\mD$, Prop.~\ref{prop:explicit-entailments_queries_have_empty_dz} ensures that $\dz{}(Q) = \emptyset$. At this point, an analogue argumentation as we gave in the last paragraph of the proof of Lem.~\ref{lem:explicit-ents_query_upper_bound} can be used to realize that $\dx{}(Q) = \dx{}(Q')$, $\dnx{}(Q) = \dnx{}(Q')$ as well as $\dz{}(Q) = \dz{}(Q')$. Hence, $\Pt_\mD(Q) = \Pt_\mD(Q')$.
\end{proof}

\subsection*{Proof of Theorem~\ref{theorem:P3_solves_problem_1}}
\begin{proof}
	By Conjecture~\ref{conj:CQPs=QPs}, the q-partition of each query in $\mQ_{\mD}^{\bcancel{0}}$ is canonical. Along with Theorem~\ref{theorem:P1_sound_complete} and the premise that the optimality theshold $t_m$ is set to zero, this implies that phase P1 returns a best q-partition among the set of all q-partitions for queries in $\mQ_{\mD}^{\bcancel{0}}$. For if a goal q-partition (see definition on page~\pageref{etc:goal_state}) is found,
	then it features the best theoretically possible $m$-value and must be (one of) the best q-partition(s) in $\mQ_{\mD}^{\bcancel{0}}$. Otherwise, the entire q-partition search space is explored (Theorem~\ref{theorem:P1_sound_complete}) since no goal is found and the best among all visited q-partitions is returned. Since the QSM of a query depends only on its q-partition, we obtain that $\mathbf{OptQ}(m,\mathbf{S})$ (see Prob.~\ref{prob:query_optimization}) is optimized over $\mathbf{S} = \mQ_{\mD}^{\bcancel{0}}$.
	By Cor.~\ref{cor:P3_optimizes_c_max}, the QCM $c_{\max}$ is optimized over all queries from $\mathbf{OptQ}(m,\mathbf{S})$. This completes the proof of the first statement of the theorem. The second statement is a direct consequence of Cor.~\ref{cor:P3_computes_query_containing_only_preferred_elements_if_existent}.
\end{proof}

\section*{Appendix B: Symbols and Meanings}
Tab.~\ref{tab:abbreviations} provides an overview of the symbols used in this work and their meaning.
\label{appendix:B:symbols_meanings}
\renewcommand{\arraystretch}{1.4}
\begin{table*}[!h]
	\scriptsize
	\centering
	\rowcolors[]{2}{gray!8}{gray!16}
	\begin{tabular}{ll}
		\rowcolor{gray!40}
		\toprule\addlinespace[0pt] 
		Symbol & Meaning \\ \hline
		$2^X$ & the powerset of $X$ where $X$ is a set \\
		$U_X$ & the union of all elements in $X$ where $X$ is a collection of sets \\
		$I_X$ & the intersection of all elements in $X$ where $X$ is a collection of sets \\
		$\mathcal{L}$ & a (monotonic, idempotent, extensive) logical knowledge representation language \\
		%
		$\mo$ & a (usually faulty) KB over $\mathcal{L}$ \\
		$\tax_{(i)}$ & a sentence over $\mathcal{L}$ (optionally with an index)\\
		$\mb$ & a (correct) background KB over $\mathcal{L}$ \\
		$\Tp$ & the set of positive test cases (each test case is a set of sentences) \\
		$\tp_{(i)}$ & a positive test case (optionally with an index) \\
		$\Tn$ & the set of negative test cases (each test case is a set of sentences) \\
		$\tn_{(i)}$ & a negative test case (optionally with an index) \\
		$\RQ$ & a set of (logical) requirements to the correct KB including at least \emph{consistency} \\
		$\langle\mo,\mb,\Tp,\Tn\rangle_\RQ$ & a (KBD) diagnosis problem instance (DPI) \\
		$\allD_{X}$ & the set of all diagnoses w.r.t.\ the (KBD-)DPI $X$ \\
		$\minD_{X}$ & the set of minimal diagnoses w.r.t.\ the (KBD-)DPI $X$ \\
		$\mD$ & a set of leading diagnoses where $\mD \subseteq \minD_{X}$ for a given DPI $X$ \\ 
		$\md_{(i)}$ & a diagnosis (optionally with index $i$)\\
		$\mc_{(i)}$ & a conflict (optionally with index $i$)\\
		$\mo^{*}_i$ & $(\mo \setminus \md_i) \cup \mb \cup U_\Tp$ \\
		$\dx{}(Q)$ / $\dnx{}(Q)$ / $\dz{}(Q)$ & the diagnoses predicting the positive / the negative / no answer of $Q$ \\
		$\mQ_\mD$ & the set of all queries w.r.t.\ the leading diagnoses $\mD$ \\
		$\mQ_{\mD}^{\bcancel{0}}$ & $\setof{Q\mid Q \in \mQ_\mD, \dz{}(Q)=\emptyset}$ \\

		$\Pt_\mD(Q)$ & the q-partition $\tuple{\dx{}(Q), \dnx{}(Q), \dz{}(Q)}$ of query $Q \in \mQ_\mD$ \\
		$\dx{}(\Pt)$  / $\dnx{}(\Pt)$ / $\dz{}(\Pt)$ & the leftmost / middle / rightmost entry of the q-partition $\Pt$ \\
		$Q_{\mathsf{can}}(\mathbf{S})$ & the canonical query (CQ) w.r.t.\ seed $\emptyset\subset\mathbf{S}\subset\mD$ \\
		$\mathbf{CQP}_\mD$ & the set of canonical q-partitions (CQPs) w.r.t.\ the leading diagnoses $\mD$ \\
		$m$ & a query selection measure (QSM) estimating each query's diagnoses discrimination strength \\
		$c$ & a query cost measure (QSM) assigning (measurement / answering) costs to each query \\
		$\Disc_\mD$ & the discrimination sentences $U_\mD \setminus I_\mD$ w.r.t.\ the leading diagnoses $\mD$ \\
		$\md_i^{(k)}$ & the trait $\md_i \setminus U_{\dx{k}}$ of $\md_i$ w.r.t.\ $\Pt_k$ \\
		$\mathsf{MHS}(X)$ & the set of all minimal hitting sets of a collection of sets $X$ \\
		$p(\md)$ & the probability of a diagnosis $\md$ \\
		$p(\mathbf{X})$ & the sum of probabilities of diagnoses in $\mathbf{X}$ where $\mathbf{X}$ is a set of diagnoses \\
		\addlinespace[0pt]\bottomrule 	 
	\end{tabular}
	\caption{\small Symbols, abbreviations and their meaning.}
	\label{tab:abbreviations}
\end{table*}

\vskip 0.2in
\bibliography{library}
\bibliographystyle{plainnat}

\end{document}